\documentclass[12pt,openany,oneside]{cmuthesis}

\usepackage{amsmath,amsfonts,bm}

\def\eqref#1{equation~\ref{#1}}

\def\1{\bm{1}}

\def\eps{{\epsilon}}
\def\diag{{\textnormal{diag}}}

\def\sspan{{\textnormal{span}}}

\def\rt{{\textnormal{t}}}

\def\rw{{\textnormal{w}}}

\def\rvx{{\mathbf{x}}}
\def\rvy{{\mathbf{y}}}

\def\rmL{{\mathbf{L}}}

\def\vzero{{\bm{0}}}

\def\vtheta{{\bm{\theta}}}
\def\valpha{{\bm{\alpha}}}

\def\va{{\bm{a}}}
\def\vb{{\bm{b}}}

\def\ve{{\bm{e}}}

\def\vh{{\bm{h}}}

\def\vq{{\bm{q}}}

\def\vu{{\bm{u}}}
\def\vv{{\bm{v}}}
\def\vw{{\bm{w}}}
\def\vx{{\bm{x}}}
\def\vy{{\bm{y}}}

\def\mA{{\bm{A}}}
\def\mB{{\bm{B}}}
\def\mC{{\bm{C}}}
\def\mD{{\bm{D}}}
\def\mE{{\bm{E}}}
\def\mF{{\bm{F}}}
\def\mG{{\bm{G}}}
\def\mH{{\bm{H}}}
\def\mI{{\bm{I}}}

\def\mK{{\bm{K}}}

\def\mM{{\bm{M}}}

\def\mQ{{\bm{Q}}}
\def\mR{{\bm{R}}}
\def\mS{{\bm{S}}}
\def\mT{{\bm{T}}}
\def\mU{{\bm{U}}}
\def\mV{{\bm{V}}}
\def\mW{{\bm{W}}}
\def\mX{{\bm{X}}}
\def\mY{{\bm{Y}}}

\def\mPhi{{\bm{\Phi}}}

\DeclareMathAlphabet{\mathsfit}{\encodingdefault}{\sfdefault}{m}{sl}
\SetMathAlphabet{\mathsfit}{bold}{\encodingdefault}{\sfdefault}{bx}{n}

\def\gA{{\mathcal{A}}}

\def\gD{{\mathcal{D}}}
\def\gE{{\mathcal{E}}}
\def\gF{{\mathcal{F}}}
\def\gG{{\mathcal{G}}}
\def\gH{{\mathcal{H}}}

\def\gJ{{\mathcal{J}}}

\def\gL{{\mathcal{L}}}

\def\gN{{\mathcal{N}}}

\def\gR{{\mathcal{R}}}

\def\gV{{\mathcal{V}}}

\def\gX{{\mathcal{X}}}
\def\gY{{\mathcal{Y}}}
\def\gZ{{\mathcal{Z}}}

\def\sI{{\mathbb{I}}}

\def\sR{{\mathbb{R}}}

\newcommand{\ptrain}{P_{\rm{train}}}

\newcommand{\E}{\mathbb{E}}

\newcommand{\R}{\mathbb{R}}

\newcommand{\reg}{\lambda}

\newcommand{\Var}{\mathrm{Var}}

\newcommand{\Cov}{\mathrm{Cov}}

\DeclareMathOperator*{\argmax}{arg\,max}
\DeclareMathOperator*{\argmin}{arg\,min}

\DeclareMathOperator*{\arginf}{arg\,inf}

\DeclareMathOperator{\Tr}{Tr}

\usepackage{scalerel}

\newcommand{\paren}[1] {{\left ( #1 \right )}}
\newcommand{\brac}[1] {{\left [ #1 \right ]}}
\newcommand{\sbrac}[2]{{\left [ #1 \; \middle \vert \; #2 \right ]}}
\newcommand{\oset}[1] {{\left \{ #1 \right \}}}
\newcommand{\sset}[2]{{\left \{ #1 \; \middle \vert \; #2 \right \}}}
\newcommand{\dotp}[1]{{\left \langle #1 \right \rangle}}
\newcommand{\res}[2]{{$#1_{#2}$}}
\newcommand{\bres}[2]{{\boldmath\res{#1}{#2}}}

\newcommand{\norm}[1] {{\left \| #1 \right \|}}
\newcommand{\abs}[1] {{\left | #1 \right |}}

\def \ie {\textit{i.e.}}
\def \eg {\textit{e.g.}}
\def \aew {\textit{a.e.}}
\def \iid {\textit{i.i.d.}}
\def \wrt {\textit{w.r.t.}}
\def \psd {\textit{p.s.d.}}

\def \err {\textnormal{err}}

\def \px {{P_\gX}}

\def \pa {{P_\gA}}

\def \lxp {{L^2(P_\gX)}}

\def \lap {{L^2(P_\gA)}}

\def \da {d_{\gA}}

\def \tax {\Gamma}

\def \rad {\mathfrak{R}}

\def \hrad {\hat{\rad}}

\def \bt {\kappa}

\def \eph {{\hat{\gH}_{\tax}}}

\def \laph {{\gH_{\tax}}}
\def \ephd {{\hat{\gH}_d}}

\def \hPhi {{\hat{\Phi}}}

\def \ka {{k_{A}^+}}
\def \kx {{k_{X}^+}}

\def \hk {{\gH_k}}
\def \ks {{k_s}}
\def \hatks {{\hat{k}_s}}
\def \hatkp {{\hat{k}^p}}

\def \hks {{\gH_{\ks}}}

\def \kssquare {{k_{s^2}}}

\def \gk {\mG_k}

\def \hkp {\gH_{k^p}}

\def \ghatks {\mG_{\hatks}}
\def \gkn {\mG_{k,n}}
\def \vkn {\vv_{k,n}}
\def \gksn {\mG_{\ks,n}}
\def \vksn {\vv_{\ks,n}}
\def \ghatksn {\mG_{\hatks,n}}
\def \vhatksn {\vv_{\hatks,n}}

\def \regu {\textrm{reg}}

\def \bPsi {\bar{\Psi}}
\def \tPsi {\tilde{\Psi}} 
\def \bPhi {\bar{\Phi}}
\def \tPhi {\tilde{\Phi}}
\def \tphi {\tilde{\phi}}
\def \tpsi {\tilde{\psi}}

\def \dw {d_{\omega}}

\def \tp {T_{P^+}}
\def \tpstar {T_{P^+}^{*}}
\def \tkx {T_{\kx}}
\def \tka {T_{\ka}}
\def \fp {\gF(P^+)}
\def \fep {\gF_{\epsilon}(P^+)}
\def \pw {P_{\omega}}
\def \kp {k_{\Lambda}}
\def \py {{P_Y}}

\def \pyw {P_{Y|\omega}}
\def \tpyw {T_{\Omega}}
\def \tpywstar {T_{\Omega}^{*}}

\def \lwp {{L^2(\pw)}}
\def \kw {k_{\Omega}}
\def \tkw {T_{\kw}}
\def \dx {d_{\gX}}
\def \tkw {T_{k_{\Omega}}}

\def \rt {\textnormal{RT}}
\def \tg {\textnormal{TG}}
\def \Sum {\textnormal{sum}}

\def \fPhi {f_{\Phi}}
\def \hpp {\gH_{P^+}}
\def \fhpp {f_{\hpp}}
\def \hPhi {\gH_{\Phi}}

\def \fhkp {\gF_{\hk}(P^+)}
\def \Phiref {\Phi_{\textrm{ref}}}

\def \dl {\tilde{d}}
\def \erm {\textnormal{ERM}}
\def \lin {\textnormal{lin}}
\def \reg {\textnormal{reg}}
\def \mm {\textnormal{MM}}
\def \cvar {\textnormal{CVaR}}
\def \iw {\textnormal{IW}}
\def \gdro {\textnormal{GDRO}}

\def \gen {\textnormal{GEN}}

\newcommand{\tv}{\textnormal{TV}}

\usepackage{times}
\usepackage{fullpage}
\usepackage{graphicx}
\usepackage{amsmath}
\usepackage[numbers,sort]{natbib}
\usepackage[backref,pageanchor=true,plainpages=false, pdfpagelabels, bookmarks,bookmarksnumbered,
]{hyperref}
\usepackage{subfigure}
\usepackage[utf8]{inputenc}
\usepackage{CJK}
\usepackage[T1]{fontenc}
\usepackage{tgbonum}
\usepackage{tgpagella}
\usepackage{enumitem}
\usepackage{booktabs}
\usepackage{multirow}
\usepackage{amsthm}

\usepackage[letterpaper,vscale=.9,hscale=.75,nomarginpar]{geometry}

\usepackage[skip=1pt plus1pt, indent=18pt]{parskip}

\usepackage{tikz}
\usetikzlibrary{shapes,3d,calc,fit}

\usepackage{pgfplots}
\pgfplotsset{compat=1.11,
legend image code/.code={
\draw[mark repeat=2,mark phase=2]
plot coordinates {
(0cm,0cm)
(0.15cm,0cm)        
(0.3cm,0cm)         
};%
}
}
\usepgfplotslibrary{statistics}
\usepackage{tikzlings} 

\usepackage{cleveref}
\crefname{equation}{Eqn.}{Eqns.}

\newcommand{\BlackBox}{\rule{1.5ex}{1.5ex}}  
\ifdefined\proof
    \renewenvironment{proof}{\par\noindent{\bf Proof\ }}{\hfill\BlackBox\\[2mm]}
\else
    \newenvironment{proof}{\par\noindent{\bf Proof\ }}{\hfill\BlackBox\\[2mm]}
\fi

\newtheorem{theorem}{Theorem}[chapter]
\newtheorem{lemma}[theorem]{Lemma} 
\newtheorem{example}[theorem]{Example}
\newtheorem{proposition}[theorem]{Proposition} 
\newtheorem{remark}[theorem]{Remark}
\newtheorem{corollary}[theorem]{Corollary}
\newtheorem{definition}[theorem]{Definition}

\newtheorem{assumption}[theorem]{Assumption}

\newtheorem*{problem*}{Problem}

\makeatletter
\def\thm@space@setup{%
  \thm@preskip=.3cm
  \thm@postskip=.1cm
}
\makeatother

\usepackage{algorithm}
\usepackage[noend]{algpseudocode}

\algrenewcommand\algorithmicindent{0.5em}%

\begin{document} 

{\fontfamily{qpl}

\frontmatter

\pagestyle{empty}

\title{ 
{\bf Contextures: The Mechanism of Representation Learning}}
\author{\textbf{Runtian Zhai}}
\date{April 2025}
\Year{2025}
\trnumber{CMU-CS-25-104}

\committee{
Pradeep Ravikumar, Co-chair \\
Zico Kolter, Co-chair \\
Andrej Risteski \\
Yuandong Tian (Meta)
}

\support{This research was sponsored by the Air Force Research Laboratory under award number FA8750-17-2-0152 and FA8750-23-2-1015; Robert Bosch GMBH under award number 0087016732PCRPO0087023984; Robert Bosch LLC under award number OSP00009188; the Defense Advanced Research Projects Agency under award number HR00112020006; the Office of Naval Research via N00014-23-1-2368; and the National Science Foundation under award number IIS-2211907.}
\disclaimer{The views and conclusions contained in this document are those of the author and should not be interpreted as representing the official policies, either expressed or implied, of any sponsoring institution, the U.S. government or any other entity.}



\maketitle


\pagestyle{plain} 


\begin{abstract}
This dissertation establishes the contexture theory to mathematically characterize the mechanism of representation learning, also known as pretraining. Despite the remarkable empirical success of foundation models, it is not very clear what representations they learn, and why these representations are useful for various disparate downstream tasks.
A scientific understanding of representation learning is critical, especially at this point when scaling up the model size is producing diminishing returns, and designing new pretraining methods is imperative for further progress.

Prior work treated different representation learning methods quite differently, whereas the contexture theory provides a unified framework for delineating the representations these methods learn.
The central argument is that a representation is learned from the association between the input $X$ and a context variable $A$.
We prove that if an encoder captures the maximum information of this association, in which case we say that the encoder \textit{learns the contexture}, then it will be optimal on the class of tasks that are compatible with the context.
We also show that a context is the most useful when the association between $X$ and $A$ is neither too strong nor too weak.
The important implication of the contexture theory is that increasing the model size alone will achieve diminishing returns, and further advancements require better contexts.

We demonstrate that lots of existing pretraining objectives can learn the contexture, including supervised learning, self-supervised learning, generative models, etc.
Based on that, we introduce two general objectives---SVME and KISE, for learning the contexture.
We also show how to mix multiple contexts together, which is an effortless way to create better contexts from existing ones.
Then, we prove statistical learning bounds for representation learning, and extend the framework to spectrally transformed kernel regression for semi-supervised learning.
Finally, we discuss the effect of the data distribution shift from pretraining to the downstream task.

\end{abstract}

\tableofcontents

\mainmatter


%
%
%
%
%

\chapter{Introduction to the Contexture Theory}

Since around 2018, the field of machine learning (ML) has been shifting from mainly end-to-end deep learning to a new paradigm driven by \textit{foundation models} \cite{bommasani2021opportunities}, which are very large models trained on huge datasets.
Foundation models achieve great success on a variety of domains, including computer vision (CV) \cite{chen2020simple,he2022masked,oquab2023dinov2}, natural language processing (NLP) \cite{devlin2018bert,Liu2019RoBERTaAR,radford2019language}, and more recently tabular data \cite{hollmann2025accurate,pmlr-v235-van-breugel24a}.
It has already become common practice to apply foundation models to any new learning task by supervised fine-tuning (SFT) or alignment via reinforcement learning (RL) \cite{ouyang2022training}.
Another popular area in ML right now is generative modeling, thanks to the remarkable success of diffusion models \cite{JMLR:v6:hyvarinen05a,ho2020denoising,song2020denoising} and large language models (LLMs) \cite{Achiam2023GPT4TR,2023arXiv231211805G,Dubey2024TheL3}. 

Foundation models are trained by \textit{representation learning}, which aims to train an encoder for the inputs that ``makes it easier to extract useful information when building classifiers or other predictors'' \cite{bengio2013representation}.
In the context of foundation models, representation learning is also known as ``pretraining'', after which a predictor is fitted on top of the representation for a specific downstream task.
\Cref{fig:teaser} illustrates the modern ML paradigm using LLMs as an example.
An LLM is pretrained on a huge dataset such as Wikipedia, with a self-supervised learning (SSL) task such as masked token prediction or next token prediction.
Then, it is applied to a variety of downstream tasks such as sentiment analysis, summarizing, translation, question answering, etc.

Despite the large body of work on representation learning, we do not have a systematic characterization of the mechanism of representation learning.
A critical question that has remained unanswered to a satisfactory extent is the following:

~

\begin{center}
\textit{What representations do foundation models learn, and why are these representations useful for a variety of downstream tasks?}
\end{center}

~

In classical statistical learning theory, there is no mystery regarding what is being learned---a mapping from input $X$ to target $Y$ is being learned.
However, in representation learning, the very target itself is unclear.
For example, what representations does masked token prediction learn, and why are they useful in understanding the sentiment of user reviews on Netflix?
For a long time, this has been attributed to the \textit{transferability} of deep learning, but (a) the essence of such transferability is vague, and (b) the assumption that transferability is a property of ``deep learning'' is questionable.

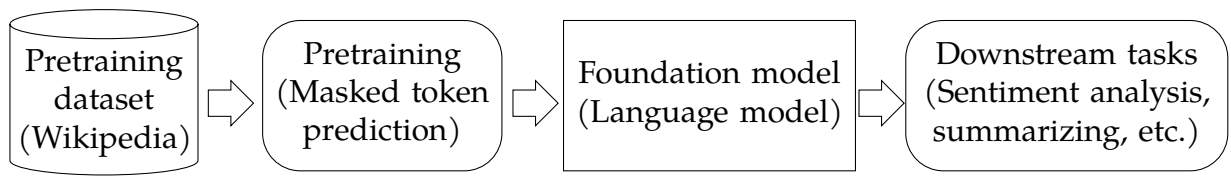
\begin{figure}[t]
    \centering
    \begin{tikzpicture}

\node(A) [cylinder, shape border rotate=90, draw,minimum height=2.2cm,minimum width=2.5cm] at (0,0)
{ ~ };

\node[text width=2.3cm, align=center] at (0,0) {Pretraining dataset (Wikipedia)};

\node(B) [draw=black, minimum width=3cm, minimum height=2cm, rounded corners=.5cm, text width=2.9cm, align=center] at (3.7,0.1) {Pretraining (Masked token prediction)};

\node[single arrow, draw=black,
      minimum width = 6mm, single arrow head extend=3pt,
      minimum height=6mm] at (1.6,0) {}; 

\node(C) [draw=black, minimum width=3.7cm, minimum height=2cm, text width=3.6cm, align=center] at (8.05,0.1) {Foundation model (Language model)};

\node[single arrow, draw=black,
      minimum width = 6mm, single arrow head extend=3pt,
      minimum height=6mm] at (5.65,0) {}; 

\node(D) [draw=black, minimum width=4.2cm, minimum height=2cm, rounded corners=.5cm, text width=4cm, align=center] at (12.8,0.1) {Downstream tasks (Sentiment analysis, summarizing, etc.)};

\node[single arrow, draw=black,
      minimum width = 6mm, single arrow head extend=3pt,
      minimum height=6mm] at (10.25,0) {}; 

\end{tikzpicture}
    \vskip -.2cm
    \caption{Illustration of the modern ML paradigm driven by foundation models and representation learning, using language models as an example in the parentheses.}
    \label{fig:teaser}
\end{figure}

The lack of understanding in the mechanism of representation learning also leads to other mysteries.
For example, supervised learning has been widely used for learning representations---neural networks trained on ImageNet \cite{ILSVRC15} were the most popular representations in the early days of the deep learning boom \cite{huh2016makes}.
One uses the output of an intermediate layer, typically the layer before the last linear layer, as the representation of the input.
However, \cite{papyan2020prevalence} found that these representations tend to collapse to a few clusters, a phenomenon known as \textit{neural collapse}. In this case, why are these representations still useful?
Another mystery is \textit{representational convergence}---\cite{pmlr-v235-huh24a} empirically showed that large neural networks of different architectures trained by optimizing different objectives all align with a common representation independent of the architecture and the objective, under the measurement of \textit{representational alignment} \cite{pmlr-v97-kornblith19a}.
Is this universally true? And how to characterize this common representation?

The above questions are naturally interesting to learning theorists, but why should the broader ML community care about understanding the mechanism of representation learning, if empirical success seems to be always achievable with existing approaches by scaling up the model size, an observation known as \textit{scaling laws} \cite{kaplan2020scaling}?
This is because sustainable success or progress is not always guaranteed.
Although some argue that scaling up the size of the model can allow some abilities to ``emerge'' \cite{wei2022emergent}, substantial evidence suggests that many abilities cannot be obtained solely from scaling, which is why additional training signals such as alignment \cite{ouyang2022training} are necessary.
Meanwhile, it is widely observed that the current pretraining paradigm is producing diminishing returns, which is why Ilya Sutskever, the scientist behind AlexNet \cite{krizhevsky2012imagenet} and GPT \cite{radford2019language}, remarked that ``pretraining as we know it will end'' recently at NeurIPS 2024 \cite{ilya2024talk}.
To make further progress, we need a better understanding of the mechanism of pretraining, which is crucial for designing future generations of pretraining methods, and this is how this field can make scientific progress.

Another important reason why understanding the mechanism of representation learning and foundation models is imperative is the safety concerns of AI.
There has long been a debate on whether AI poses an existential threat to human beings, and neither camp can convince the opposing camp.
Such a debate has become more and more heated since the advent of ChatGPT, and reached its peak at the recent Paris AI summit held on February 12, 2025.
We saw world leaders arguing about how to regulate AI, but reaching an agreement is extremely difficult, though most leaders agreed that some extent of regulation on AI is necessary, and one reason is that these leaders received quite different opinions from their scientists.
Advancing learning theory and the science of foundation models is necessary for us to understand the potential risk of AI, so that we can develop a universal and scientific protocol for AI regulation.

\section{Central Theme: Representations from Association}

The purpose of this dissertation is to establish a new theoretical framework dubbed \textbf{the contexture theory} in order to characterize the mechanism of representation learning.
The central argument of this theory is that \textbf{\boldmath representations are learned from the association between the input $X$ and a context variable $A$.}
We refer to such an association as a \textbf{contexture}.
This thesis will prove this argument mathematically and rigorously.
In addition, this idea is related to a key concept in psychology---the two systems of thinking.

Psychologist Daniel Kahneman categorizes human thinking into two systems \cite{kahneman2011thinking}.
System 1 thinking refers to fast, automatic, and associative thinking, such as associating a photo of a cat with the animal cat that can meow, associating $2+2$ with $4$, associating an English word with its Chinese equivalent, etc. 
System 2 thinking refers to slow, effortful, and logical thinking, such as looking for a golden retriever in an image of 30 dogs, calculating $177 \times 284$, and following the proof in a math paper.
Decades of research in psychology has shown that the human brain works differently for these two systems of thinking.
For example, one piece of evidence is that the pupils dilate when a person is doing system 2 thinking, but not system 1 thinking.

The contexture theory implies that representation learning is capable of doing any type of system 1 thinking, such as image recognition, sentence completion, simple translation, etc.
As long as one can specify $X$ and $A$, their association can be learned by a large model with a sufficient amount of data.
This result substantiates the famous \textbf{deep learning hypothesis} by Ilya Sutskever \cite{ilya2024talk}, stating that ``a large neural network can do anything a human can do in a fraction of a second''.
Hence, system 1 thinking is generally easy, but system 2 thinking is still very hard.
For example, teaching an LLM how to reason usually requires complicated methods such as chain of thought \cite{wei2022chain} and test-time scaling \cite{jaech2024openai,guo2025deepseek}, which are beyond the scope of this thesis.

The contexture theory resolves lots of mysteries about deep learning and foundation models, and can lead to better pretraining algorithms.
Specifically, in this thesis we will address the following questions:
\begin{itemize}
    \item What representations do foundation models learn, and why are they useful for a wide range of downstream tasks?
    \item What variational objectives can be used to learn such representations?
    \item What does the mechanism of representation learning imply about scaling laws?
    \item How can we further improve foundation models beyond scaling?
    \item Are there statistical guarantees for representation learning in the finite data regime?
\end{itemize}

\paragraph{Takeaways.} The key takeaways from this thesis are summarized as follows:
\begin{itemize}
    \item Representation learning can be understood as recovering the space spanned by the \textbf{top singular functions} of the expectation operator jointly induced by the input $X$ and a \textbf{context variable} $A$. We call this process \textbf{learning the contexture}.
    \item These top singular functions can be learned by training a very expressive model to optimize certain variational objectives.
    \item Scaling up the model size alone inevitably leads to a diminishing return. Further improvement requires better contexts.
    \item A context is the most useful when the association between $X$ and $A$ is neither too strong nor too weak, in which case the singular values of the expectation operator decay neither too fast nor too slowly.
    \item If we have multiple contexts whose associations are either too strong or too weak, then we can obtain a better context by mixing them together.
    \item The representation dimension controls the trade-off between approximation error and estimation error, both of which are influenced by the \textbf{context complexity}.
    \item Achieving good generalization under data distribution shift is extremely hard, and heuristic methods such as reweighting samples usually do not work as expected.
\end{itemize}

This introductory chapter is primarily devoted to establishing the foundations of the contexture theory.
After this chapter, the rest of this thesis is organized as follows.

\paragraph{Chapter 2.}
This chapter demonstrates that the contexture can be learned using a variety of variational objectives, because these objectives are optimized if and only if the encoder learns the contexture. These objectives include supervised learning, self-supervised learning, generative models, knowledge distillation, etc.
Moreover, two general objectives for learning the contexture are introduced: SVME and KISE.
The key implication is that scaling brings the representation more aligned to the top singular functions, and when the alignment is high enough, further scaling will achieve a diminishing return.

\paragraph{Chapter 3.}
This chapter studies how to evaluate an encoder or a context.
For encoders, we focus on intrinsic evaluation, which does not depend on any task.
Intrinsic evaluation is carried out on a class of tasks that are \textbf{compatible} with the context, and we prove that if the task is known to be compatible with the context a priori, then learning the contexture is the optimal thing to do.
For contexts, we evaluate them only with their spectra.
The key result is that a good context should have a moderate association between $X$ and $A$, so that the decay rate of its singular values is neither too fast nor too slow.
Then, we propose a quantitative evaluation metric, and show that it correlates with the actual downstream performance on real datasets.

\paragraph{Chapter 4.}
This chapter studies how to learn representations from a \textit{mixture} of multiple contexts.
The general approach consists of three base operations: convolution, convex combination and concatenation.
Mixing multiple contexts allows us to obtain contexts with moderate associations from strong or weak ones.
Detailed algorithms for learning the contexture of these mixtures are provided.
We test these algorithms on real tabular datasets, and find that they can achieve higher performance than state-of-the-art methods such as XGBoost \cite{chen2016xgboost}.

\paragraph{Chapter 5.}
This theory-intense chapter establishes statistical guarantees for contexture learning.
A key object is the context complexity, which characterizes the smoothness of the top singular functions.
Then, we extend these results to the more general spectrally transformed kernel regression (STKR) for semi-supervised learning.

\paragraph{Chapter 6.}
The theory developed so far has assumed that the data distribution is fixed, but in practice there is always a distribution shift from the pretrain to the downstream data.
This chapter discusses some challenges in studying such distribution shifts, including the sensitivity to outliers, and the hardness of distributionally robust generalization.

\section{Contexts: Definition and Examples}
\label{sec:def-exp}
This thesis studies the following learning setting: the number of unlabeled samples is much larger than the number of labeled samples.
Learning methods in this situation can be categorized as either semi-supervised learning or representation learning.
Semi-supervised learning directly learns a predictor on both labeled and unlabeled samples.
Representation learning first learns an encoder with the unlabeled samples, and then fits a predictor on the encoder with the labeled samples, as illustrated in \Cref{fig:teaser}.
This thesis mainly studies representation learning, but the theory can also be generalized to semi-supervised learning, which will be discussed in \Cref{sec:stkr} when introducing spectrally transformed kernel regression.
For now, let us focus on representation learning.

Let $\gX$ be the \textbf{input space}, and let $\px$ be the data distribution.
$\px$ is always assumed to be fixed until Chapter 6.
The $L^2$ functional space \wrt{} $\px$ is a Hilbert space denoted by $\lxp$, whose inner product is given by $\dotp{f_1, f_2}_\px = \E_{X \sim \px}[ f_1(X) f_2(X) ]$, and norm is given by $\norm{f}_\px = \sqrt{\dotp{f,f}_\px}$.

Representation learning aims to learn an encoder $\Phi : \gX \rightarrow \R^d$.
$\Phi(x)$ is called the \textbf{embedding} of $x$, and $d$ is the embedding dimension.
We denote $\Phi = [\phi_1,\cdots,\phi_d]$, and assume that $\phi_i \in \lxp$ for all $i$.
The encoder $\Phi$ can be either deterministic or randomized.
A randomized $\Phi$ is a random variable that takes value in $\gE_d$, which is the space of all deterministic $d$-dimensional encoders.
Recall that a random variable is formally a measurable function $\Phi: \Omega \rightarrow \gE_d$ for a sample space $\Omega$.

There are various ways to use a pretrained encoder $\Phi$ in a downstream task.
This thesis exclusively uses the simplest yet a very common way called a \textbf{linear probe}, which fits a linear predictor on top of $\Phi$ such that the final predictor is $\mW \Phi(x) + \vb$.
If $\Phi$ is randomized, then we first draw a deterministic encoder from the distribution of $\Phi$ (that is, fix a sample in $\Omega$), and then fit a linear probe on top of this encoder.

A context is provided by a \textbf{context variable} $A \in \gA$, 
and $\gA$ is called the \textbf{context space}.
The contexture is the relationship between $X$ and $A$, given by their joint distribution $P^+(x,a)$.
Let $\px$ and $\pa$ be the marginal distributions of $P^+$.
Let $\lap$ be the $L^2$ functional space \wrt{} $\pa$, with inner product $\dotp{\cdot, \cdot}_{\pa}$ and norm $\norm{\cdot}_{\pa}$.
For simplicity, we assume that the probability spaces of both $\px$ and $\pa$ are compact Hausdorff spaces.

The definition of contexts covers a wide range of machine learning methods. \Cref{tab:context-examples} lists some examples of $X$ and $A$.
Here are some concrete examples.

\begin{table}[t]
    \centering
    \begin{tabular}{lll}
    \toprule
    \textbf{Method} &  \textbf{Input $X$}  & \textbf{Context Variable $A$}  \\
     \midrule
      Supervised learning on ImageNet  & Image & Label of the object in the image \\  
      Rotation prediction \cite{gidaris2018unsupervised} & Image & Rotated image \\ 
      BERT \cite{devlin2018bert}: masked token prediction & Text & Masked text \\ 
      Vision-language model CLIP \cite{CLIP} & Image & Text caption describing the image \\ 
      K-nearest neighbors (KNN) & Sample & A nearest neighbor of $X$ \\ 
      Diffusion models for images & Image & Image plus additive noise \\ 
      GPT \cite{radford2019language}: next token prediction & Text & First $k$ tokens of the text \\ 
    \bottomrule
    \end{tabular}
    \caption{Examples of inputs and context variables.}
    \label{tab:context-examples}
\end{table}

\paragraph{Labels.}
$A$ is the label of $X$.
Labels can take different forms, such as discrete categories in classification, continuous values in regression, or structured outputs like text captions of images in vision-language models.
Labels may be obtained from human annotators, or in pseudo-forms such as clusters.
Typically, labels are provided as compatible pairs of $(x,a)$ sampled from the joint distribution $P^+$.

\paragraph{Random transformations.}
These are perturbations (augmentations) on the inputs that presumably do not change the semantics of the inputs by too much.
In this case, $A$ is the corrupted version of $X$.
For example, transformations for images include translation, rotation, flipping, masking, Gaussian noise, Cutout \cite{devries2017improved} and Mixup \cite{zhang2018mixup}.
$P^+$ is provided by the transformation such that one can sample $A \sim P^+(\cdot|x)$ for arbitrarily many times.

\paragraph{Graphs.}
Graph data is very common in industry, such as social networks, drug discovery, cybersecurity, etc.
Graphs also appear in domains that do not involve with graphs explicitly.
For example, K-nearest neighbors (KNN) gives a graph where every sample is connected to its nearest neighbors.
Graphs can also be continuous: for example, manifold learning \cite{belkin2003laplacian,coifman2006diffusion} approximates a differential operator such as the Laplace-Beltrami operator on a manifold with the continuous limit of a family of graphs.
For graphs, we have $\gA = \gX$, and $P^+(a|x)$ is proportional to the weight of the edge between $x$ and $a$.

\paragraph{Feature maps and teacher models.}
Both are functions that map $x \in \gX$ to a feature encoding $z \in \gZ$, where $\gZ$ is typically a Hilbert space.
Feature maps are usually defined by humans, while teacher models are learned from data.
For example, feature maps for images include PNG, JPEG, etc.
Feature maps can also be implicitly defined by kernels \cite{scholkopf2002learning}.
A popular example of teacher models is pretrained language models released by tech companies.
It is important to note that $\gZ$ is not $\gA$, and it is possible that $\gA$ is unknown.
For example, for these LLMs, if they are close-sourced, then we do not know how they are exactly pretrained.

\section{Spectral Properties of a Context}
\label{sec:spectral-properties}
The joint distribution $P^+$ of $X$ and $A$ induces an operator from $\lap$ to $\lxp$, which we call the expectation operator. Basically, it computes the conditional expectation of a function $g \in \lap$ given $x \in \gX$.
This operator is very intuitive: Suppose we want to predict for an input $x$, but we only have a predictor $g$ on space $\gA$; the most reasonable prediction in this scenario is $\E[g(A)|x]$.
The adjoint operator of the expectation operator is also an expectation operator, but in the reverse direction.

\begin{definition}
  The \textbf{expectation operator} $\tp: \lap \rightarrow \lxp$ is defined as
\begin{equation*}
        \paren{\tp g} (x) = \int g(a) P^+(a|x) da = \E[g(A) | x] \qquad \text{for all } g \in \lap .
\end{equation*}
Its \textbf{adjoint operator} $\tpstar : \lxp \rightarrow \lap$, which satisfies $\langle f, \tp g \rangle_{\px} = \langle \tpstar f, g \rangle_{\pa}$ for all $f \in \lxp$ and $g \in \lap$, is given by $\paren{\tpstar f}(a) = \int f(x) \frac{P^+(a|x) \px(x)}{\pa(a)} dx = \E[f(X) | a]$.
\end{definition}

\begin{remark}
In general, the operator $\tp$ is independent of the data distribution $\px$,
because the stochastic mapping $x \mapsto A$ does not depend on the distribution of $X$.
For example, the mapping from an image to its label is independent of the data distribution on the image space.
On the other hand, the adjoint operator $\tpstar$ depends on $\px$ due to Bayes' rule.
For example, given the same label ``dog'', a class-conditional generative model trained on CIFAR-10 \cite{krizhevsky2009learning} and another one trained on ImageNet \cite{ILSVRC15} will generate very different images.
\end{remark}

An easier way to understand these operators is to use the following shorthand notations.
Since computing $\tp$ requires drawing $A$ from $P^+(\cdot|X)$ given $X$, we write $\tp: X \rightarrow A$.
Similarly, we can write $\tpstar: A \rightarrow X$.
We can also compose the two operators as $\tpstar \tp: X \rightarrow A \rightarrow X'$, meaning that we first sample $A \sim P^+(\cdot|X)$, and then sample $X' \sim P^+(\cdot | A)$.
Similarly, we have $\tp \tpstar: A \rightarrow X \rightarrow A'$.
Both $\tpstar \tp: \lap \rightarrow \lap$ and $\tp \tpstar: \lxp \rightarrow \lxp$ are self-adjoint operators, and they are the integral operators of the following \psd{} kernels.
\begin{definition}
\label{def:two-kernels}
    The \textbf{positive-pair kernel} \cite{johnson2022contrastive} $\ka$ and its \textbf{dual kernel} $\kx$ are defined as
    \begin{align*}
        \ka(a,a') & = \frac{P^+(a,a')}{\pa(a) \pa(a')} 
        = \frac{\int P^+(a|x) P^+(a'|x) d \px (x)}{\pa(a) \pa(a')}; \\ 
     \kx(x,x') & = \frac{P^+(x,x')}{\px(x) \px(x')}  = \frac{\int P^+(x|a) P^+(x'|a) d \pa(a)}{\px(x) \px(x')} = \int \frac{P^+(a|x) P^+(a|x')}{\pa(a)} da . 
    \end{align*}
Their integral operators are given by 
\begin{align*}
    \tka = \tpstar \tp : \lap \rightarrow \lap \qquad & \paren{\tka g} (a) = \int g(a') \ka(a,a') d \pa (a') ; \\ 
    \tkx = \tp \tpstar: \lxp \rightarrow \lxp \qquad & \paren{\tkx f} (x) = \int f(x') \kx(x,x') d \px(x') . 
\end{align*}
\end{definition}

We use $T_k: f \mapsto \int f(x') k(\cdot,x') d \px(x')$ to denote the integral operator of any kernel $k$.
Throughout this work, we assume that $\tka$ and $\tkx$ are both Hilbert-Schmidt operators.
We say that $\lambda \in \R$ is an \textbf{eigenvalue} of $\tka$ with \textbf{eigenfunction} $\nu \in \lap$, if $\tka \nu = \lambda \nu$. 
Under the above assumption, by Hilbert-Schmidt theorem, we can order the eigenvalues by $1 = \lambda_0 \ge \lambda_1 \ge \cdots \ge 0$, and the corresponding eigenfunctions $\nu_0,\nu_1,\cdots$ form an orthonormal basis (ONB) of $\lap$.
Here $\lambda_i \le 1$ because of Jensen's inequality,
and it is easy to see that $\nu_0 \equiv 1$ is always an eigenfunction of $\tka$ with $\lambda_0 = 1$.

Similarly, denote the eigenfunctions and eigenvalues of $\tkx$ by $\mu_i$ and $\kappa_i$. Then, we can order the eigenvalues by $1 = \kappa_0 \ge \kappa_1 \ge \cdots \ge 0$, where $\mu_0 \equiv 1$, and $\mu_0,\mu_1,\cdots$ form an ONB of $\lxp$.
The two sets of eigenfunctions have the following connection.

\begin{lemma}[Duality property]
\label{lem:duality}
    For all $i$, we have $\lambda_i = \kappa_i \in [0,1]$. And if $\lambda_i > 0$, then we have $\mu_i = \lambda_i^{-\frac{1}{2}} \tp \nu_i$, and $\nu_i = \lambda_i^{-\frac{1}{2}} \tpstar \mu_i$.
\end{lemma}
\begin{proof}
    For any $i$ such that $\lambda_i > 0$, we have $\tpstar \tp \nu_i = \lambda_i \nu_i$.
    Thus, $\tp \tpstar \tp \nu_i = \lambda_i \tp \nu_i$, which shows that $\tp \nu_i$ is an eigenfunction of $\tp \tpstar$ with eigenvalue $\lambda_i$. For any $i, j$ such that $\lambda_i > 0$ and $\lambda_j > 0$, we have
    \begin{align*}
        \dotp{\lambda_i^{-\frac{1}{2}} \tp \nu_i, \lambda_j^{-\frac{1}{2}} \tp \nu_j }_{\px} & = \lambda_i^{-\frac{1}{2}} \lambda_j^{-\frac{1}{2}} \dotp{\tp \nu_i, \tp \nu_j}_{\px} = \lambda_i^{-\frac{1}{2}} \lambda_j^{-\frac{1}{2}} \dotp{\tpstar \tp \nu_i, \nu_j}_{\pa} \\ 
        & = \lambda_i^{-\frac{1}{2}} \lambda_j^{-\frac{1}{2}} \dotp{\lambda_i \nu_i, \nu_j}_{\pa} = \sI[i = j],
    \end{align*}
    which implies that $\oset{\lambda_i^{-\frac{1}{2}} \tp \nu_i }_{i:\lambda_i > 0}$ is orthonormal.
    Similarly, all $\kappa_i > 0$ are eigenvalues of $\tpstar \tp$, and $\oset{\kappa_i^{-\frac{1}{2}} \tpstar \mu_i }_{i: \kappa_i > 0}$ is orthonormal. This implies the result.
\end{proof}

This result leads to the singular value decomposition (SVD) of $\tp$.
We say that $s_i = \lambda_i^{\frac{1}{2}}$ is a \textbf{singular value} of $\tp$, associated with the left \textbf{singular function} $\mu_i \in \lxp$ and the right singular function $\nu_i \in \lap$.
Since we choose $\mu_0 \equiv 1$ and $\nu_0 \equiv 1$. all other $\mu_i$ (and $\nu_i$) must have zero mean because they are orthogonal to $\mu_0$ (and $\nu_0$).
Moreover, we have the following spectral decomposition of $P^+$.
\begin{lemma}[Spectral decomposition]
\label{lem:spectral-decomposition}
    We have $P^+(x,a) = \sum_i s_i \mu_i(x) \nu_i(a) \px(x) \pa (a)$.
\end{lemma}
\begin{proof}
   $\forall i, \;   
   \big\langle\frac{P^+(x,a)}{\px(x) \pa(a)}, \nu_i \big\rangle_{\pa} = \int P^+(a|x) \nu_i(a) da = (\tp \nu_i) (x) = 
   \big(\lambda_i^{\frac{1}{2}} \mu_i \big) (x) = s_i \mu_i(x)$.
    Since $\oset{\nu_i}_{i \ge 0}$ is an ONB, we have $\frac{P^+(x,a)}{\px(x) \pa(a)} = \sum_{i=0}^{\infty} s_i \mu_i(x) \nu_i(a)$.
\end{proof}

This result immediately leads to the following spectral decomposition of the two kernels.
The proof is left as an exercise to the reader.
\begin{corollary}
\label{cor:spec-decompose}
$\kx(x,x') = \sum s_i^2 \mu_i(x) \mu_i(x')$, and $\ka(a,a') = \sum s_i^2 \nu_i(a) \nu_i(a')$.
\end{corollary}

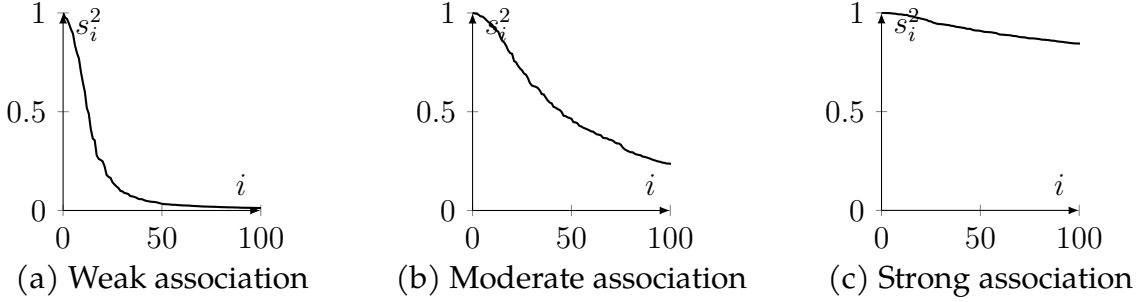
\begin{figure}[t]
    \centering
    \begin{tikzpicture}
\begin{axis}[
name=plot1,
 height = .26\linewidth,
    width = .26\linewidth,
    axis lines=left,  
    axis line style={-latex},  
    xlabel={$i$},
    ylabel={$s_i^2$},
    ylabel style={
        at={(0.13,0.8)}, 
        anchor=south,  
        rotate=270  
    },
    xlabel style = {
        at={(0.9,0.02)},
        anchor=south,
    },
title style={at={(0.5,-0.3)}, anchor=north},
title={(a) Weak association},
    xmin=0, xmax=100,
    ymin=0, ymax=1,
    xtick distance = 50,
    ytick distance = 0.5,
]
    \addplot[thick] table [
        x expr=\coordindex,  
        y=value,            
        col sep=comma       
    ] {data/ag1.csv};
\end{axis}

\begin{axis}[
name=plot2,
at={($(plot1.south east)+(80,0)$)},
 height = .26\linewidth,
    width = .26\linewidth,
    axis lines=left,  
    axis line style={-latex},  
    xlabel={$i$},
    ylabel={$s_i^2$},
    ylabel style={
        at={(0.13,0.8)}, 
        anchor=south,  
        rotate=270  
    },
    xlabel style = {
        at={(0.9,0.02)},
        anchor=south,
    },
title style={at={(0.5,-0.3)}, anchor=north},
title={(b) Moderate association},
    xmin=0, xmax=100,
    ymin=0, ymax=1,
    xtick distance = 50,
    ytick distance = 0.5,
]
    \addplot[thick] table [
        x expr=\coordindex,  
        y=value,            
        col sep=comma       
    ] {data/ag3.csv};
\end{axis}

\begin{axis}[
name=plot3,
at={($(plot2.south east)+(80,0)$)},
 height = .26\linewidth,
    width = .26\linewidth,
    axis lines=left,  
    axis line style={-latex},  
    xlabel={$i$},
    ylabel={$s_i^2$},
    ylabel style={
        at={(0.13,0.8)}, 
        anchor=south,  
        rotate=270  
    },
    xlabel style = {
        at={(0.9,0.02)},
        anchor=south,
    },
title style={at={(0.5,-0.3)}, anchor=north},
title={(c) Strong association},
    xmin=0, xmax=100,
    ymin=0, ymax=1,
    xtick distance = 50,
    ytick distance = 0.5,
]
    \addplot[thick] table [
        x expr=\coordindex,  
        y=value,            
        col sep=comma       
    ] {data/ag5.csv};
\end{axis}
\end{tikzpicture}
    \vskip -.1cm
    \caption{The association between $X$ and $A$ determines the shape of the spectrum.}
    \label{fig:spectrum-demo}
\end{figure}

The set of eigenvalues $\oset{s_0^2,s_1^2,\cdots}$ is called the \textbf{spectrum} of the context.
The shape of the spectrum, or more precisely the decay rate of the eigenvalues, is determined by the strength of association between $X$ and $A$, which we also call the \textbf{association of the context}.
In general, the stronger the association, the slower the decay.
Consider two extreme cases: (i) When $A$ and $X$ are independent, the association is the weakest; (ii) When $A = X$, the association is the strongest.
In case (i), only $s_0^2=1$ is positive, and all other eigenvalues are $0$, so the eigenvalues decay the fastest.
In case (ii), all eigenvalues are $1$, so there is no decay at all (in fact, in this case $\tkx$ is not Hilbert-Schmidt if $\gX$ and $\gA$ are infinite sets).
\Cref{fig:spectrum-demo} illustrates the spectrum on different association levels.

There are two key results in the contexture theory.
First, a useful context should have a moderate association, and its eigenvalues should decay neither too fast nor too slowly.
Obviously, the context in either extreme case above is useless, since $A$ provide no additional information.
Second, given a context, among all $d$-dimensional encoders, the ``optimal'' one (``optimal'' to be formally defined later) should recover the linear space spanned by $\mu_1,\cdots,\mu_d$, for which we say that the encoder learns the contexture of $P^+$.

\begin{definition}
\label{def:learn-contexture}
    A deterministic $d$-dimensional encoder $\Phi = [\phi_1,\cdots,\phi_d]$ \textbf{learns the contexture} of $P^+$, if there exists a set of top-$d$ singular functions $\oset{\mu_1,\cdots,\mu_d}$ of $\tp$ (excluding $\mu_0 \equiv 1$), such that $\sspan \oset{\phi_1,\cdots,\phi_d} =  \sspan \oset{\mu_1,\cdots,\mu_d}$.
    If the multiplicity of $s_d > 1$, then any set of top-$d$ singular functions suffices. We also say that such a $\Phi$ \textbf{\boldmath extracts the top-$d$ eigenspace of $\tkx$}.
\end{definition}
\begin{definition}
\label{def:learn-contexture-randomized}
    A randomized $d$-dimensional encoder $\Phi = [\phi_1,\cdots,\phi_d]$ \textbf{learns the contexture} of $P^+$ (or extracts the top-$d$ eigenspace of $\tkx$), if it learns the contexture almost surely.
\end{definition}
\begin{remark}
In this definition, $\mu_0 \equiv 1$ is excluded, because the bias term $\vb$ in the downstream linear probe implicitly includes $\mu_0$, so there is no reason to waste one dimension to encode $\mu_0$.
Note that extracting the top-$d$ eigenspace only requires recovering the linear span, so any invertible linear transformation on $\Phi$ makes no difference.
A harder task is extracting the \textit{exact} top-$d$ eigenfunctions, which requires estimating every function $\mu_i$ for $i \in [d]$.
Although $\tp$ is independent of $\px$, the contexture of $P^+$ depends on $\px$, since $\mu_1,\cdots,\mu_d$ are defined \wrt{} $\px$. Thus, the same $\tp$ leads to different contextures when there is a distribution shift in $\px$, in which case we say that \textbf{the contexture is skewed}.
\end{remark}

The intuition why learning the contexture is ideal is that such a representation keeps the most information (variance) of the context, which is analogous to principal component analysis (PCA) in the finite-dimensional case.
Consider the case where $\gX$ and $\gA$ are both finite sets.
Let $N = |\gX|$ and $M = |\gA|$.
Then, a function $f \in \lxp$ is a vector in $\R^N$, $g \in \lap$ is a vector in $\R^M$, and $\tp$ is essentially a matrix $\mT \in \R^{N \times M}$.
Suppose we want to learn a $d$-dimensional embedding $\mE \in \R^{N \times d}$ for the $N$ samples in $\gX$, and it should preserve the information of $\mT$ as much as possible, then what should we do?
PCA states that we should use the top-$d$ left singular vectors of $\mT$ as $\mE$, which are equivalent to the top-$d$ eigenvectors of $\mT \mT^{\top}$, because they maximize the explained variance.
Similarly, functional spaces are essentially infinite-dimensional vector spaces, so the $d$-dimensional embedding of $X$ that preserves the most information of $\tp$ consists of the top-$d$ left singular functions of $\tp$, or equivalently the top-$d$ eigenfunctions of $\tp \tpstar$.

\section{Three Types of Access and Example Contexts}
\label{sec:three-access}

In reality, contexts can be provided in a variety of ways. 
Let us analyze the examples in \Cref{sec:def-exp}.
Labels are usually provided one alongside each sample, that is the training set is $\oset{(x_i,a_i)}_{i=1}^m$.
Random transformations are provided as subroutines, which can be called infinitely many times for the same input.
A graph is provided as either an adjacency list or an adjacency matrix; the former allows one to sample a neighbor of $x$, and the latter can be viewed as a kernel.
Finally, when given a teacher model, we might not even know what space $\gA$ it was trained on.

Generally speaking, there are three types of access we can have to a context.
For the context of labels, we have \textbf{pair access}.
For random transformations, we have \textbf{transformation access}.
For a graph, if it is given by an adjacency matrix, then we have \textbf{kernel access}; if it is given by an adjacency list, then we have transformation access.
As for the teacher model, we will show that we have kernel access to its context in \Cref{sec:teacher-models}.

\begin{definition}
\label{def:three-access}
\begin{enumerate}[label=(\roman*)]
    \item We say that a context has \textbf{pair access}, if we have access to a dataset of $\oset{(x_i, a_i)}_{i=1}^m$ that is \iid{} sampled from $P^+$.
    \item We say that a context has \textbf{kernel access} (\textbf{\boldmath $k$-access}), if we have access to a kernel $k: \gX \times \gX \rightarrow \R$ that approximates the dual kernel of the context. We do not need to know the space $\gA$ to have $k$-access.
    \item We say that a context has \textbf{transformation access} (\textbf{\boldmath $T$-access}), if for any $x \in \gX$, one can sample $a \sim P^+(\cdot|x)$ for arbitrarily many times.
\end{enumerate}
\end{definition}

In practice, for a context with $T$-access, after randomly transforming an input $X$ to $A$, we usually want to map it back to the input space $\gX$.
For example, after we apply random cropping to an image, we usually map it back to the original dimension by either stretching the crop or padding it with white pixels;
after masking a sentence, we fill in the masked position with a special token \texttt{[MASK]}.
Such a mapping is called the heuristic inverse of $P^+$, which we denote by $Q^+$.
Its expectation operator $T_{Q^+}$ is normally different from $\tpstar$, because $\tpstar$ depends on $\px$ while $T_{Q^+}$ in most cases does not depend on $\px$.

\begin{definition}
\label{def:heuristic-inverse}
A user-defined conditional distribution $Q^+(x|a)$ that maps $a \in \gA$ back to the input space $\gX$ is called the \textbf{heuristic inverse} of $P^+$.
\end{definition}

$T$-access is stronger than pair access. For example, in supervised learning, the context has pair access but not $T$-access, because the label of any $x$ outside the training set is not given, so we cannot sample $A \sim P^+(\cdot|x)$ for such $x$.
As an exercise, the reader can think about what type of access each context in \Cref{tab:context-examples} has.

Contexts with different types of access require different representation learning methods.
For example, contrastive learning \cite{chen2020simple} needs to sample two views $A,A^+ \sim P^+(\cdot|x)$ for each training sample $x$, and this requires $T$-access. If we only have pair access, then we might only have one $a$ for each $x$.
In the next chapter, we will show how to learn the contexture for each type of access.

In what follows, we analyze the three example contexts above in greater detail, and calculate their dual kernels.

\begin{example}[Classification tasks]
\label{exp:classification-task}
    Let $\gA = \oset{1,2,\cdots,C}$ be a finite set of labels. The label is deterministic, meaning that each $x \in \gX$ is mapped to one label with probability $1$ by $P^+$. Denote the label of $x$ by $a_x$.
    The training set is $\oset{(x_i,a_{x_i})}_{i=1}^m$, so this context has pair access.
    For this context, $\kx(x,x') = \sI[a_x = a_{x'}] \pa(A = a_x)^{-1}$, $s_0 = \cdots = s_{C-1} = 1$, and all other singular values are $0$.
    The span of $\mu_0,\cdots,\mu_{C-1}$ is the same as $\sspan \oset{f_1,\cdots,f_C}$, where $f_i(x) = \sI[a_x = i]$.
\end{example}

\begin{example}[Graphs]
\label{exp:graphs}
Let $\gG = (\gV, \gE)$ be an undirected graph, where each edge $(u,v)$ has a non-negative weight $w(u,v)$ such that $w(u,v) = w(v,u)$.
Let $w(u,v) = 0$ if $u$ and $v$ are not connected.
Let the degree of node $u$ be $D(u) = \sum_{v \in \gV} w(u,v)$, and let $D_{\Sum} = \sum_{u \in \gV} D(u)$.
Let $\px(u) = \frac{D(u)}{D_{\Sum}}$, and $P^+(v|u) = \frac{w(u,v)}{D(u)}$, where $\gA = \gX = \gV$.
Then, it can be shown that $\pa = \px$, and $\kx(x,x') = \frac{D_{\Sum}}{D(x) D(x')} \sum_{u \in \gV} \frac{w(x,u) w(x',u)}{D(u)} $. 
This graph is given by an adjacency list, and this context has $T$-access.
\end{example}

\begin{example}[Random transformations on a finite input space]
    Let $\gX$ be a finite set, and let the context be given by a random transformation.
    For example, if $\gX$ is the set of all text of up to $512$ tokens on a vocabulary of size $30,000$, then $\gX$ is a finite set.
    Let $N$ be the size of $\gX$.
    Without loss of generality, assume that $\px(x) > 0$ for every $x \in \gX$.
    This context has $T$-access.
    
    The \textbf{transformation graph} $\gG = (\gV, \gE)$ is defined as $\gV = \gX$, and $(x_1,x_2) \in \gE$ if they can be transformed to the same $a$.
    \Cref{fig:aug-graph} gives an illustration of a transformation graph, which was originally introduced as the augmentation graph by \cite{haochen2021provable} in the context of self-supervised learning based on data augmentation.

    Define matrix $\mW \in \R^{N \times N}$ as $\mW[i,j] = P^+(x_i, x_j) = \int P^+(x_i|a) P^+(x_j|a) d \pa(a)$, which is the weight of edge $(x_i,x_j)$.
    Then, the degree of node $x_i$ is $\sum_{x} P^+(x_i, x) = \px(x_i)$.
    Define matrix $\mD \in \R^{N \times N}$ as $\mD = \diag \oset{\px (x_1),\cdots,\px(x_N)}$. The singular function $\mu_i$ is now an $N$-dimensional vector.
    It is easy to see that $\mW \mu_i = s_i^2 \mD \mu_i$, that is $(s_i^2,\mu_i)$ is a pair of generalized eigenvalue and eigenfunction.
    We can also rewrite it as $(\mD - \mW) \mu_i = (1-s_i^2) \mD \mu_i$, where $\mD - \mW$ is the (unnormalized) Laplacian matrix of the transformation graph \cite{chung1997spectral}.
    Thus, $\mu_1,\cdots,\mu_d$ are the bottom eigenvectors of the graph Laplacian.
\end{example}

\begin{figure}[t]
    \centering
    \begin{tikzpicture}

\node(x) [diamond, draw, thick] at (14,0) {$a$};
\node(a01) [circle, draw, thick] at (12,0.7) {$x_1$};
\node(a02) [circle, draw, thick] at (12,-0.7) {$x_2$};

\draw[->, very thick] (a01.east) -- (x.west);
\draw[->, very thick] (a02.east) -- (x.west);

\node(a1) [circle, draw, thick] at (5,0.8) {$x_1$};
\node(a2) [circle, draw, thick] at (5,-0.8) {$x_2$};
\node(a3) [circle, draw, thick] at (7,0.8) {$x_3$};
\node(a4) [circle, draw, thick] at (7,-0.8) {$x_4$};
\node(a5) [circle, draw, thick] at (9,0.8) {$x_5$};
\node(a6) [circle, draw, thick] at (9,-0.8) {$x_6$};

\draw[very thick] (a1) -- (a2);
\draw[very thick] (a1) -- (a3);
\draw[very thick] (a1) -- (a4);
\draw[very thick] (a1) -- (a6);
\draw[very thick] (a2) -- (a4);
\draw[very thick] (a3) -- (a6);
\draw[very thick] (a5) -- (a6);
\draw[very thick] (a3) -- (a4);
\draw[very thick] (a2) -- (a3);

\end{tikzpicture}
    \caption{Illustration of a transformation graph.}
    \label{fig:aug-graph}
\end{figure}
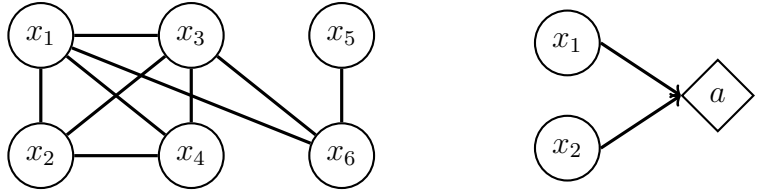

\section{Prior Work}

In this century, machine learning has gone through three paradigms: kernel methods, end-to-end deep learning, and foundation models.
Representation learning is a key part in all three paradigms.
This section provides a short review of their history.

Kernel methods were quite popular when machine learning was dominated by support vector machines (SVMs) \cite{SVM}.
An SVM is a linear predictor, but when the relationship between the input $X$ and the target $Y$ is not linear, we need a non-linear feature map $\Phi$ such that there is a linear relationship between $\Phi(X)$ and $Y$.
The \textbf{kernel trick} says that we do not need to explicitly specify $\Phi$, but only need to specify a kernel $k$ such that $k(x,x') = \dotp{\Phi(x), \Phi(x')}$, and then we can train an SVM on top of this kernel.
The reason why a kernel is more ideal in some applications is that $k$ is easier to define than $\Phi$.
For example, if we assume that samples close in the Euclidean space are similar, then $k$ can be defined as a kernel that is larger when the distance is smaller, such as the RBF kernel or the KNN kernel. However, to compute $\Phi$, one needs to learn the eigenmap that consists of the top eigenfunctions of $T_k$ \cite{belkin2003laplacian}, and this is hard in practice.
For this reason, at that time kernel methods were more popular than representation learning in semi-supervised learning \cite{scholkopf2002learning,bengio2004learning}.

However, kernel methods and SVMs began to fall short when people started to apply machine learning to harder tasks, such as ImageNet classification \cite{krizhevsky2012imagenet}.
The relationship between $Y$ and $X$ in these tasks are too complex, and no human-designed kernel can work well on these tasks.
Deep learning became extremely popular when people found out that a deep neural network trained to approximate the target function can automatically learn the complex relationship between $Y$ and $X$, so there is no need to design a kernel.
In other words, kernel methods define $\Phi$ by specifying a kernel, whereas deep learning obtains $\Phi$ by optimizing a variational objective using a large function approximator.
End-to-end deep learning has been very successful in supervised learning.

For semi-supervised learning, a variety of deep learning methods have been proposed \cite{van2020survey}.
To use the unlabeled samples, many methods require the model to ``behave well'' on the unlabeled samples, which is commonly known as \textit{consistency regularization}:
For each unlabeled sample $x$, we augment it into $x'$ and $x''$ in two different ways, and ask the model to give similar outputs to $x'$ and $x''$.
Some methods define $x'$ and $x''$ as the outputs of models at different epochs, such as temporal ensembling \cite{laine2017temporal} and mean teachers \cite{tarvainen2017mean}.
Some methods use a strong augmentation to obtain $x'$ and a weak one to obtain $x''$, such as FixMatch \cite{sohn2020fixmatch} and noisy student \cite{xie2020self}.
Some methods use adversarial attack, such as virtual adversarial training \cite{miyato2018virtual}.
Some methods use the interpolation between two input/target pairs, such as MixUp \cite{zhang2018mixup} and MixMatch \cite{berthelot2019mixmatch}.
These methods use context variables $x'$ and $x''$ to learn a predictor for semi-supervised learning.

In the modern ML paradigm, a variety of downstream tasks share the same gigantic unlabeled dataset.
In this scenario, end-to-end semi-supervised learning is not ideal, because for every different task we need to learn a predictor on all labeled and unlabeled samples, which is very inefficient.
To solve this problem, people proposed to use transfer learning, which is the basis of foundation models.
The assumption is that a large model trained on task 1 can be transferred to tasks 2, 3 and so on, which is a widely observed phenomenon in deep learning \cite{huh2016makes}, though why such transfer learning works has not been fully explained.
Under this assumption, one can pretrain a large encoder on the huge unlabeled dataset using a general task, so that when learning a predictor for another task, one does not need to use the gigantic unlabeled dataset again.

Foundation models are usually trained by self-supervised learning, and there is a large body of work on the theoretical analysis of self-supervised learning.
One line of research studies the effectiveness of contrastive learning by showing its features are optimal when used to fit a linear predictor on certain downstream tasks~\cite{saunshi2019theoretical,tosh2021contrastive,tosh2021multiview}, 
robust to class imbalance~\cite{liu2021selfsupervised}, 
and suitable for unsupervised domain adaptation~\cite{shen2022connect,haochen2022beyond}.
Masked prediction tasks have been shown to be useful for reducing the downstream sample complexity~\cite{lee2021predicting} and for parameter identifiability~\citep{liu2022masked}.
In terms of language applications, \cite{saunshi2019theoretical} explained why next-word prediction can benefit sentiment classification, 
and \cite{wei2021pretrained} studied the effect of prompt tuning through the lens of implicit Bayesian inference.
Regarding the optimization in representation learning,
there have been prior works on the training dynamics and loss landscapes of contrastive learning~\cite{WenLi21,jing2022understanding,tian2022deep}, non-contrastive learning~\cite{TianCG21,PokleTLR22,wen2022mechanism}, and masked prediction \cite{xiong2020layer,pmlr-v119-huang20f}.
There is also a line of theoretical work that connects self-supervised learning to information theory \cite{achille2018emergence,balestriero2022contrastive,shwartz2023information}.

More related to this thesis is a line of work that formulates contrastive learning as a Laplacian operator over the augmentation graph.
The idea of studying data augmentation from a kernel perspective was first explored in \cite{mroueh2015learning,raj2017local,dao2019kernel}.
\cite{haochen2021provable} defined the augmentation graph and then proved a generalization bound for the spectral contrastive loss.
Then, \cite{saunshi2022understanding} pointed out that this model-class-free bound could be vacuous with a hypercube construction.
As a response to this argument, \cite{haochen2022theoretical} included the effect of the encoder's inductive bias into their new generalization bounds.
Then, \cite{johnson2022contrastive} defined the positive-pair kernel for the augmentation graph, \cite{wang2023a} connected contrastive learning to message passing on the augmentation graph, and \cite{cabannes2023ssl} proved generalization bounds that do not depend on the function class of $\Phi$ (but they still need to assume that the target function belongs to the RKHS of a known kernel).
My own work \cite{zhai2023understanding,zhai2024spectrally} extended these results to any augmentation-based self-supervised learning (not only contrastive learning), and then the more general spectrally transformed kernel regression, which builds the link between representation learning and semi-supervised learning.

The common weakness of the papers mentioned above is that they treated different representation learning methods quite differently.
For example, contrastive learning, non-contrastive learning and masked autoencoders have been regarded as inherently distinct methods, and for each method there is a line of theoretical work.
Consequently, despite the large body of theoretical work, our understanding of representation learning is still quite muddled.
The contexture theory established in this thesis provides a universal and lucid characterization of the mechanism of a wide range of representation learning methods.
The key takeaway is that the various methods are all learning the contexture of a context, so using which method is less important than the context itself.

Also closely related to this thesis is a line of work on representation alignment \citep{pmlr-v97-kornblith19a,pmlr-v235-huh24a,fumero2024latent,insulla2025towards}.
Representation similarity has also been studied in neuroscience \citep{kriegeskorte2008representational}.
These papers aims to compare between two representations, while this thesis mainly focuses on evaluating a single representation, or the context on which it is trained.

\chapter{Learning the Contexture with Variational Objectives}
\label{chap:objectives}

Learning the contexture requires extracting the top-$d$ eigenspace of a kernel integral operator $\tkx$.
Conventionally, this is done by \textbf{kernel PCA} \cite[Chapter~14]{scholkopf2002learning}.
Let $k$ be a \psd{} kernel, and let $T_k: \lxp \rightarrow \lxp$ be its integral operator.
Suppose we can compute $k(x,x')$ for all $x,x' \in \gX$.
Then, given a training set $\oset{x_i}_{i=1}^m$, kernel PCA estimates the top-$d$ eigenfunctions of $T_k$ as follows:
\begin{enumerate}
    \item Build the Gram matrix $\mG \in \R^{m \times m}$ of $k$, such that $\mG[i,j] = k(x_i, x_j)$.
    \item Compute the eigenvalues and eigenvectors $\oset{(\lambda_i, \vv_i)}_{i=1}^m$ of $\mG$, where $\lambda_1 \ge \cdots \ge \lambda_m$, and $\vv_1,\cdots,\vv_d$ form an ONB of $\R^m$.
    Assume that $\lambda_d > 0$.
    \item $\hat{\mu}_i(x) = \lambda_i^{-1} \sum_{j=1}^m k(x,x_j) \vv_i[j]$ is an estimation of the $i$-th eigenfunction.
\end{enumerate}

Kernel PCA has two issues.
First, it requires that $k(x,x')$ can be efficiently estimated for all $x,x'$ (which is $k$-access), but this is not always possible; for example, $k(x,x')$ is hard to estimate when it is the dual kernel of a random transformation context, since estimating $\pa(a)^{-1}$ with high precision requires lots of samples from $P^+$.
Second, it is not scalable for huge datasets.
In general, the time complexity of eigen-decomposition can be regarded as $O(m^3)$.
As of today, the fastest algorithm for eigen-decomposition in theory has $O(m^\omega)$ complexity with $\omega \approx 2.38$ \cite{demmel2007fast}, which is still not very scalable.

However, it is possible to have a more efficient algorithm than kernel PCA, because kernel PCA can extract the exact top-$d$ eigenfunctions, while our goal is to only extract the top-$d$ eigenspace of $\tkx$, that is we do not need to estimate the exact function $\mu_i$.
Moreover, we can also make the algorithm faster if we are willing to sacrifice precision.

This chapter shows how to learn the contexture with a variational objective $\gR(\Phi)$, meaning that this objective is optimized if and only if $\Phi$ extracts the top-$d$ eigenspace of $\tkx = \tp \tpstar$.
Recall that this excludes $\mu_0 \equiv 1$.
Provided with such an objective, one can learn the contexture by optimizing an expressive deep neural network.
We will also show that some existing objectives, such as generative models and RLHF, can learn the contexture of a \textit{mixture} of multiple contexts.

Some of the objectives we are going to discuss extract the top-$d$ eigenspace of $\tp \Lambda \tpstar$ instead, where $\Lambda: \lap \rightarrow \lap$ is the integral operator of a kernel $\kp$ called the \textbf{loss kernel}.
The loss kernel depends on the loss function.
In this case, since the constant function is not necessarily the top-$1$ eigenfunction of $\tp \Lambda \tpstar$, we do not exclude any eigenfunction.

This chapter does not discuss the numerical aspect of obtaining the optima of these objectives.
Doing so in a generalizable way requires an expressive model architecture and a good optimizer, which this thesis will not discuss.
In fact, \cite{cohen2021gradient} showed that neural networks trained with popular gradient methods such as Adam \cite{DBLP:journals/corr/KingmaB14} will not converge to any point, but will oscillate around what they termed the \textit{edge of stability}.
How to extend the contexture theory to that situation is an open problem.
Moreover, all objectives to be discussed are spectral (that is $L^2$) rather than information theoretic, because the contexture theory is based on the spectral properties of $\tp$.
For example, the mean squared error (MSE) is used for classification tasks instead of the cross entropy loss, and the spectral contrastive loss \cite{haochen2021provable} is used for contrastive learning instead of the NT-Xent loss \cite{chen2016xgboost}.
How to extend the contexture theory to information theoretic loss is posed as an open problem.

The important implication of the analysis in this chapter is that one key role of scaling up the model size is to bring the learned representation space more aligned to the one spanned by the top-$d$ eigenfunctions of $\tkx$.
This will be empirically demonstrated in \Cref{sec:scaling-laws}.
Consequently, when the two spaces have already become close enough, additional scaling will be less helpful.
This is a major reason why scaling has been achieving a diminishing return recently.
Further improvement requires the creation of new contexts.
We envision that the next major breakthrough in pretraining will be a result of \textit{context scaling}, where very powerful and complicated contexts are obtained from an enormous amount of data, rather than human heuristics.

\textbf{Notation:} For any $f \in \lxp$, denote its mean by $\bar{f} = \E_{\px}[f(X)]$, and its centered version by $\tilde{f} = f - \bar{f}$. The same notation is used for multi-dimensional functions and random variables, as long as the distribution is clear from context.

\begin{definition}
The \textbf{covariance matrix} of any $\Phi: \gX \rightarrow \R^d$, denoted by $\Cov_{\px}[\Phi]$, is a $d \times d$ matrix $\mC$ where $\mC[i,j] = \dotp{\tphi_i, \tphi_j}_{\px}$.
\end{definition}

\section{Three Illustrative Examples}
\label{sec:illustrative-examples}
Let us revisit the three examples in \Cref{sec:three-access}, namely supervised learning, learning with a graph and learning with a random transformation.
Through these examples, the reader can get a sense of how to learn the contexture via variational objectives.

\paragraph{Classification tasks.}
Let there be $C$ classes, and let $A$ be a $C$-dimensional one-hot vector.
Let the predictor be a linear predictor on top of $\Phi$, defined as $\mW \Phi(x) + \vb$. For a neural network, $\Phi(x)$ is the output of the layer before the last linear layer.
If $\vb$ is an arbitrary vector, then the linear predictor is biased;
if $\vb = \vzero$ is fixed, then it is unbiased.
Consider training an unbiased linear predictor with the mean squared error:
\begin{equation}
\label{eqn:obj-classification}
    \gR(\Phi) = \min_{\mW \in \R^{C \times d}} \;  \underset{(X,A) \sim P^+}{\E}  \brac{ \norm{ \mW \Phi(X) - A }_2^2 } .
\end{equation}
\begin{theorem}[Proof in \Cref{app:proof-thm-obj-classification}]
\label{thm:obj-classification}
Suppose $A$ is a one-hot vector.
Then, $\Phi^*$ minimizes \Cref{eqn:obj-classification} if and only if $\Phi^*$ extracts the top-$d$ eigenspace of $\tp \Lambda \tpstar$, where $\Lambda$ is the integral operator of $\kp(a,a') = \sI[a=a']$, and $(\Lambda g)(a) = g(a) \pa(a)$.
If all classes have the same size, then $\tp \Lambda \tpstar$ and $\tp \tpstar$ share the same top-$d$ eigenfunctions.
\end{theorem}
\begin{remark}
Note that the constant function is not necessarily the top eigenfunction of $\tp \Lambda \tpstar$, so in this result no eigenfunction is excluded, which is different from \Cref{def:learn-contexture} where $\mu_0 \equiv 1$ is excluded. 
\end{remark}

It turns out that $\Lambda$ is a consequence of \textbf{class imbalance}.
In \Cref{sec:three-access} we showed that the top eigenfunctions of $\tp \tpstar$ are indicator functions for the $C$ classes, and these functions are independent of the class sizes.
However, in practice we know that when there is a class imbalance, the smaller classes are harder to learn.
The operator $\Lambda$ gives more weights to larger classes, and as a result the larger classes have more impact on the top-$d$ eigenfunctions.
With this insight, in order to get rid of $\Lambda$, we can use the following balanced loss, also known as \textbf{importance weighting} \cite{shimodaira2000improving}:
\begin{equation}
\label{eqn:obj-classification-balanced}
    \gR(\Phi) = \min_{\mW \in \R^{C \times d}} \; \underset{(X,A) \sim P^+}{\E} \brac{  \frac{1}{\sqrt{\pa(A)}} \norm{ \mW \Phi(X) + \vb - A }_2^2 } .
\end{equation}
\begin{theorem}[Proof in \Cref{app:proof-thm-obj-classification-balanced}]
\label{thm:obj-classification-balanced}
$\Phi^*$ minimizes \Cref{eqn:obj-classification-balanced} if and only if $\Phi^*$ learns the contexture of $P^+$.
\end{theorem}
\begin{remark}
    Compared to the original objective in \cite{shimodaira2000improving}, the denominator in \Cref{eqn:obj-classification-balanced} is $\sqrt{\pa(A)}$ instead of the original $\pa(A)$.
\end{remark}

The above results can partially explain the phenomenon of \textbf{neural collapse} \cite{papyan2020prevalence}:
When the label $A$ is deterministic, and there are $d$ classes of the same size, neural collapse is the phenomenon that a sufficiently trained deep representation collapses to an \textbf{equiangular tight frame (ETF)} $\phi_1,\cdots,\phi_d$, where $\phi_i(x) = c \paren{ \sI[x \text{ belongs to class }i] - d^{-1} }$ for some non-zero constant $c$.
Note that the span of $\phi_1,\cdots,\phi_d$ is the same as the span of $\mu_0,\cdots,\mu_{d-1}$, computed in \Cref{sec:three-access}.
However, the above results cannot explain why the representation exactly converges to these $d$ functions.
To explain this, one needs to analyze the training dynamics, which depends on the specific optimizer such as gradient methods, whereas all results proved in this chapter are independent of the optimizer.

When the classes have different sizes, it is easy to see that the dual kernel of $\tp \Lambda \tpstar$ is $\kx(x,x') = \sI[x \text{ and } x' \text{ have the same label}]$. This is equivalent to the simplex-encoded labels interpolation (SELI) defined by \cite{thrampoulidis2022imbalance}, which generalizes neural collapse.
When $\gX$ is a finite set, the SEL matrix defined in their Definition 2 is the centered kernel of $\kx$.

\paragraph{Regression tasks.}
Let $A \in \R^{\da}$ be a real-valued random variable, and consider training $\Phi$ using the mean squared error:
\begin{equation}
\label{eqn:obj-regression}
    \gR(\Phi) = \min_{\mW \in \R^{\da \times d}, \vb \in \R^{\da}} \;  \underset{(X,A) \sim P^+}{\E}  \brac{ \norm{ \mW \Phi(X) + \vb - A }_2^2 } .
\end{equation}

\begin{theorem}[Proof in \Cref{app:proof-thm-obj-regression}]
\label{thm:obj-regression}
$\Phi^*$ minimizes \Cref{eqn:obj-regression} if and only if $\Phi^*$ extracts the top-$d$ eigenspace of $\tp \Lambda \tpstar$.
If the linear predictor is unbiased ($\vb = \vzero$), then $\kp(a,a') = \dotp{a, a'}$;
if it is biased ($\vb$ can be arbitrary), then $\kp(a,a') = \dotp{\tilde{a}, \tilde{a'}}$.
\end{theorem}
\begin{remark}
    Kernel $k(a,a') = \dotp{a,a'}$ is called the \textbf{linear kernel}, and $k(a,a') = \dotp{\tilde{a}, \tilde{a'}}$ is called the \textbf{centered linear kernel} \wrt{} distribution $\px$. \Cref{thm:obj-classification} is a special case of \Cref{thm:obj-regression}.
\end{remark}

\paragraph{Graphs.}
Let $\gG = (\gV, \gE)$ be an undirected graph.
Let the weight of each edge $w(u,v)$ be non-negative.
Let $D(u) = \sum_{v \in \gV} w(u,v)$ and $D_{\Sum} = \sum_{u \in \gV} D(u)$.
Define a node distribution $\px(u) = \frac{D(u)}{D_{\Sum}}$.
Define the context as $\gA = \gV$, and $P^+(u,v) = \frac{w(u,v)}{D_{\Sum}}$.
Then, the following constrained optimization problem learns the contexture:
\begin{equation}
\label{eqn:node-repre-obj}
    \underset{\Phi: \gX \rightarrow \R^d}{\text{minimize}} \quad  \frac{1}{2} \E_{(u,v) \sim P^+} \brac{ \norm{\Phi(u) - \Phi(v)}_2^2 } \qquad \text{s.t.} \qquad  \Cov_{\px}[\Phi] = \mI  .
\end{equation}
The constraint $\Cov_{\px}[\Phi] = \mI$ is called the \textbf{orthonormality constraint}.
This constraint is necessary because without it, a degenerate solution could be $\Phi$ being a constant function, which is called the \textbf{feature collapse} problem.
Implementing this constraint in practice is not easy.
The most straightforward approach is to minimize the Lagrangian of this optimization problem, in which the constraint is implemented as a penalty term,
such as in VICReg \cite{bardes2021vicreg}. More details will be discussed later, and for now let us assume that this constraint can be enforced.

\begin{theorem}[Proof in \Cref{app:proof-thm-node-repre}]
\label{thm:node-repre}
    Let $\Phi^*$ be any solution to \Cref{eqn:node-repre-obj} (so that for any constant $c$, $\Phi^* + c$ is also a solution).
    Then, $\tPhi^*$ learns the contexture of $P^+$.
\end{theorem}

\paragraph{Self-supervised learning (SSL) with random transformations.}

SSL is usually based on data augmentation, a random transformation that does not alter the semantics of the input by too much.
There are two popular methods: multi-view learning and reconstruction, as illustrated in \Cref{fig:aug-tasks}.
In multi-view learning, one draws two \textbf{views} that are transformed from the same input, and enforces the encoder to give similar embeddings to these two views.
In reconstruction, one trains a neural network with the goal of mapping each view $A$ to its original input $X$.
Then, the penultimate layer of this neural networks is taken as the representation.

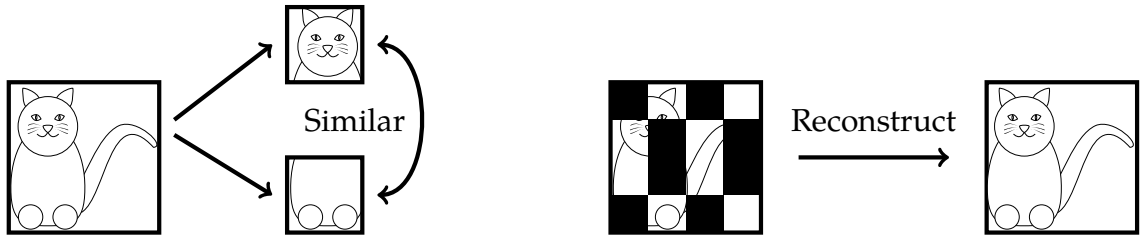
\begin{figure}[t]
    \centering
    \begin{tikzpicture}

\draw[black, ultra thick] (-2,-1) rectangle (0,1);

\cat[contour=black, scale=0.9, xshift=-3.9em, yshift=-2.8em]
\cat[contour=black, scale=0.9, xshift=5.9em, yshift=0em]
\cat[contour=black, scale=0.9, xshift=5.9em, yshift=-2.8em]

\fill [white] (1.5,0) rectangle (4,1);
\fill [white] (2.7,0.9) rectangle (4,2);
\fill [white] (2.7,-1) rectangle (3.5,0);

\draw[black, ultra thick] (1.7,1) rectangle (2.7,2);
\draw[black, ultra thick] (1.7,-1) rectangle (2.7,0);

\draw[->, black, ultra thick] (0.2,0.5) -- (1.5, 1.5);
\draw[->, black, ultra thick] (0.2,0.3) -- (1.5, -0.5);

\draw [<->,ultra thick] (2.9,-0.5) to [out=0,in=0] (2.9,1.5);
\node[] at (2.6,0.5) {Similar};

\draw[black, ultra thick] (6,-1) rectangle (8,1);
\cat[contour=black, scale=0.9, xshift=17.2em, yshift=-2.8em]

\fill [black] (6,-1) rectangle (6.5,-0.5);
\fill [black] (7,-1) rectangle (7.5,-0.5);
\fill [black] (6.5,-0.5) rectangle (7,0.5);
\fill [black] (7.5,-0.5) rectangle (8,0.5);
\fill [black] (6,0.5) rectangle (6.5,1);
\fill [black] (7,0.5) rectangle (7.5,1);

\draw[black, ultra thick] (11,-1) rectangle (13,1);
\cat[contour=black, scale=0.9, xshift=30.3em, yshift=-2.8em]

\draw[->, black, ultra thick] (8.5, 0) -- (10.5, 0);
\node[] at (9.5, 0.5) {Reconstruct};

\end{tikzpicture}
    \caption{Two widely used self-supervised learning algorithms with random transformations. \textbf{Left:} Multi-view learning. \textbf{Right:} Reconstruction.}
    \label{fig:aug-tasks}
\end{figure}

Let us start with multi-view learning.
If $A, A^+$ are two views independently drawn from $P^+(\cdot|X)$ for the same $X$, then $(A,A^+)$ is called a \textbf{positive pair}.
If $A,A^-$ are independently drawn from $\pa$, meaning that they can be views of different inputs, then $(A,A^-)$ is called a \textbf{negative pair}.
Multi-view learning trains an encoder $\Psi: \gA \rightarrow \R^d$ by enforcing $\Psi(A) \approx \Psi(A^+)$.
Note that $\Psi$ is an encoder on $\gA$, while the required $\Phi$ should be an encoder on $\gX$.
Although it has been common practice to directly apply $\Psi$ to $\gX$ at downstream whenever possible, such a practice is not theoretically correct.
For example, a BERT \cite{devlin2018bert} is trained on masked sentences and it never sees a complete sentence at pretrain time; as such, one cannot assume that the embeddings it gives to complete sentences retain 100\% quality, even though it works well on many real tasks.
The theoretically correct way of using $\Psi$ is converting it to $\Phi$ by means of the \textbf{average encoder}:
\begin{equation*}
    \Phi = \tp \Psi; \qquad \Phi(x) = \E_{A \sim P^+(\cdot|x)} \brac{\Psi(A)} .
\end{equation*}
For any input $x$, $\Phi(x)$ can be estimated via Monte Carlo: first sample a number of $A$ from $P^+(\cdot|x)$, and then take the mean of $\Psi(A)$.
This Monte Carlo estimation requires $T$-access to the context, as defined in \Cref{def:three-access}.

Multi-view learning also has the feature collapse problem.
If only $\Psi(A) \approx \Psi(A^+)$ is enforced, then one degenerate solution is to give all $A$ the same embedding.
There are two popular solutions to this problem: contrastive learning and non-contrastive learning.
Contrastive learning brings the embeddings of negative pairs far apart, that is increasing the distance between $\Psi(A)$ and $\Psi(A^-)$.
Non-contrastive learning has two types: asymmetry-based and constraint-based.
The asymmetry-based approach applies two encoders with slight (training or architectural) differences to the same $x$.
For example, BYOL \cite{NEURIPS2020_f3ada80d} updates one encoder with exponential moving average (EMA) and the other in the standard way to create asymmetry, and SimSiam \cite{Chen_2021_CVPR} puts stop gradient on one of the encoders but not the other.
The constraint-based approach, such as Barlow Twins \cite{zbontar2021barlow} and VICReg \cite{bardes2021vicreg}, uses the same orthonormality constraint as \Cref{eqn:node-repre-obj}.

Let us demonstrate that multi-view learning can learn the contexture with two example objectives.
The spectral contrastive loss \cite{haochen2021provable} is given by
\begin{equation}
\label{eqn:obj-spectral-contrastive}
\gR(\Psi) = \underset{X \sim \px}{\E} \; \underset{A,A^+ \sim P^+(\cdot|X)}{\E} \; \underset{A^- \sim \pa}{\E} \brac{ - \dotp{ \tPsi(A), \tPsi(A^+)} + \frac{1}{2} \dotp{\tPsi(A), \tPsi(A^-) } ^2 } ;
\end{equation}
and constraint-based non-contrastive learning solves the following problem:
\begin{equation}
\label{eqn:obj-non-contrastive}
\underset{\Psi: \gA \rightarrow \R^d}{\textrm{minimize}} \; \underset{X \sim \px}{\E} \; \underset{A,A^+ \sim P^+(\cdot|X)}{\E} \brac{\norm{\Psi(A) - \Psi(A^+)}_2^2} \qquad \text{s.t.} \qquad \Cov_{\pa}\brac{\Psi} = \mI .
\end{equation}
\begin{theorem}[Proof in \Cref{app:proof-thm-obj-multi-view}]
\label{thm:obj-multi-view}
Let $\Psi^*$ be any minimizer of $\Cref{eqn:obj-spectral-contrastive}$, or any optimal solution to \Cref{eqn:obj-non-contrastive}. Then, $\tPsi^*$ extracts the top-$d$ eigenspace of $\tpstar \tp$, and $\tPhi^* = \tp \tPsi^*$ learns the contexture of $P^+$.
\end{theorem}

For reconstruction, suppose $\gX \subseteq \R^{d_X}$, and let the predictor be $\mW \Psi(a) + \vb$, where $\vb$ can be $\vzero$ if the predictor is unbiased.
The pretraining objective is given by
\begin{equation}
\label{eqn:obj-reconstruction}
\gR(\Psi) = \min_{\mW \in \R^{d_X \times d} , \; \vb \in \R^{d_X}} \; \underset{(X,A) \sim P^+}{\E} \; \brac{ \norm{ \mW \Psi(A) + \vb - X }_2^2 }  .
\end{equation}

\begin{theorem}
\label{thm:obj-reconstruction}
Let $\Psi^*$ be any minimizer of \Cref{eqn:obj-reconstruction}. Then, $\tPsi^*$ extracts the top-$d$ eigenspace of $\tp \Lambda \tpstar$, where $\Lambda$ is the integral operator of $\kp(x,x') = \dotp{\tilde{x}, \tilde{x}'}$ if the predictor is biased, or $\kp(x,x') = \dotp{x,x'}$ if the predictor is unbiased.
\end{theorem}
\begin{proof}
    The proof is the same as \Cref{thm:obj-regression}, which is left as an exercise.
\end{proof}

The three illustrate examples demonstrate that many existing popular variational objectives can already learn the contexture.
The next step is to make them more general so that they can be adapted to a wider variety of contexts.

\section{General Objectives: SVME and KISE}
\label{sec:svme-spin}

\Cref{sec:three-access} introduced three types of access to a context: pair access, kernel access ($k$-access), and transformation access ($T$-access).
This section presents two general objectives: SVME for pair access, and KISE for $k$-access.
SVME can also be used for $T$-access since it is stronger than pair access.

To motivate SVME, think about the weaknesses of contrastive and non-contrastive learning objectives in \Cref{eqn:obj-spectral-contrastive,eqn:obj-non-contrastive}.
First, they require the stronger $T$-access, since for pair access, it is not always possible to draw two views of the same $x$.
Second, using the average encoder is not ideal since it slows down inference.

\textbf{Single-view multi-encoder (SVME)} learning addresses both weaknesses.
It produces the encoder $\Phi$ directly, and needs only one view for each $x$; thus, it only requires pair access.
It does so at the cost of training more than one encoders.
SVME with two encoders is formulated as the following optimization problem.
\begin{equation}
\label{eqn:svme}
\textbf{SVME:} \quad \underset{\Phi : \gX \rightarrow \R^d; \; \Psi: \gA \rightarrow \R^d}{\textrm{minimize}} \; \;  \underset{(X,A) \sim P^+}{\E} \brac{\norm{\Phi(X) - \Psi(A)}_2^2} \quad \textrm{s.t.} \quad \Cov_{\px}\brac{\Phi} = \mI  .
\end{equation}
SVME can be viewed as a combination of asymmetry-based and constraint-based non-contrastive learning.
It uses a similar double-encoder architecture as asymmetry-based non-contrastive learning, and the asymmetry here is that the constraint is only imposed on $\Phi$ but not $\Psi$.
Meanwhile, it uses the same orthonormality constraint as mentioned earlier.
If $\gA = \gX$, then similar to BYOL and SimSiam, one can implement $\Phi$ and $\Psi$ as two heads on top of the same neural network backbone.
SVME can have more than two encoders, which we shall see later in this chapter.

SVME is inspired by multi-modal learning.
For example, in vision-language models like CLIP \cite{CLIP}, $\Phi$ can be the image encoder, and $\Psi$ can be the text encoder.

For $k$-access, usually $k$ is an approximation of the dual kernel.
For simplicity, let us assume that we have access to $\kx(x,x')$ for all $x,x'$.
Then, for any $f \in \lxp$, we can estimate $\tkx f$ with Monte Carlo given a set of inputs.
The objective of \textbf{kernel-integral single-encoder (KISE)} learning is formulated as
\begin{equation}
\label{eqn:spin}
    \textbf{KISE:} \quad \underset{\Phi: \gX \rightarrow \R^d}{\text{minimize}} \; \; \underset{X \sim \px}{\E} \; \brac{ \norm{\tPhi(X)}_2^2 - \dotp{\tPhi(X), \tkx \tPhi(X)}  } \quad \text{s.t.} \quad \Cov_{\px}[\Phi] = \mI .
\end{equation}
Different from SVME, KISE only trains one encoder $\Phi$.
KISE is similar to the spectral inference network (SpIN) method proposed by \cite{pfau2018spectral}.
SpIN maximizes $\dotp{\tPhi(X), \tkx \tPhi(X)}$, while KISE is more numerically stable when using gradient methods because the objective is lower bounded by zero, provided that all eigenvalues of $\tkx$ are in $[0,1]$.
Maximizing SpIN might cause the model weights to explode, but minimizing KISE will not.

The following result shows that both SVME and KISE can learn the contexture of $P^+$, and they are in fact equivalent.
\begin{theorem}[Proof in \Cref{app:proof-thm-svme-spin}]
\label{thm:svme-spin}
In \Cref{eqn:svme}, if $\Psi$ is substituted with the optimal $\Psi$ when $\Phi$ is fixed, then this problem becomes equivalent to \Cref{eqn:spin}.
Let $\Phi^*$ be any optimal solution to \Cref{eqn:spin}, then $\tPhi^*$ learns the contexture of $P^+$.
\end{theorem}

\paragraph{Converting $k$-access to $T$-access.}
There is an additional way of using $k$-access, which requires the kernel $k$ to be always non-negative.
The idea is similar to the objective for graphs that we saw earlier.
Define the degree $D$ as $D(x) = \int k(x,x') d \px(x')$, and define a new context $P^+$ as $P^+(a|x) = k(x,a) \px(a) / D(x)$, where $\gA = \gX$.
One has $T$-access to this context if one can access $k(x,x')$ for all $x,x'$.

If $k$ is the dual kernel of the original context, then we have $D(x) \equiv 1$.
For any eigenfunction $\mu_i$ of $T_k$ with eigenvalue $s_i^2$, we have $\tp \mu_i = s_i^2 \mu_i$.
Since obviously $\tpstar = \tp$, $\mu_i$ is a singular function of $\tp$ with singular value $s_i^2$.
Hence, the context of $P^+$ has the same singular functions as the original context, with all singular values squared.

\paragraph{Implementing the orthonormality constraint.}
The orthonormality constraint in SVME and KISE can be implemented by VICReg \cite{bardes2021vicreg}.
Let $\oset{(x_1,a_1),\cdots,(x_m,a_m)}$ be a batch of training samples.
The VICReg objective is the sum of an invariance loss, a variance loss and the covariance loss, defined as
\begin{equation*}
    \gL(\Phi, \Psi) = \frac{1}{m} \sum_{k=1}^m \norm{\Phi(x_k) - \Psi(a_k)}_2^2 + \frac{\alpha}{d} \sum_{i=1}^d \paren{ 1 - \sqrt{\mC[i,i] + \epsilon}}_+ + \frac{\beta}{d(d-1)} \sum_{i \neq j} \mC[i,j]^2 ,
\end{equation*}
where $(x)_+ = \max \oset{0,x}$ is the ReLU function, $\epsilon$ is a small positive constant for numerical stability, and $\mC$ is the empirical covariance matrix: $\mC[i,j] = \frac{1}{m-1} \sum_{k=1}^m \tphi_i(x_k) \tphi_j(x_k)$.
$\alpha,\beta$ are two positive hyperparameters.
The first term is the invariance loss that aligns the two encoders.
The second term is the variance loss that pulls $\norm{\tphi}_{\px}$ closer to $1$.
The third term is the covariance loss that makes the $d$ dimensions orthogonal to each other.

Although VICReg is used a lot in our experiments, we observe that the two regularization terms cannot enforce the orthonormality constraint perfectly.
In fact, both the variance loss and the covariance loss cannot converge to zero.
Alternatively, we tried projecting $\Phi$ to become orthonormal after each epoch, but doing so does not improve the performance of the encoder.
Finding better ways than VICReg to enforce the orthonormality constraint is posed as an open problem.

\section{Distilling Knowledge from Teacher Models}
\label{sec:teacher-models}

This part discusses how to learn from the context given by a teacher model.
Not only does this have lots of applications in practice, but also it gives us a way to convert a context with any access to one with $k$-access, which will be very useful in later chapters.
Teacher models are very common in practice nowadays. Many big tech companies release their large language models or generative models, all of which can be viewed as teacher models.
However, using these models usually incurs a cost, and these models are often too large to be deployed locally (especially in academic labs).
As a result, knowledge distillation \cite{hinton2015distilling} becomes very attractive---people would like to distill their knowledge to smaller and more affordable models of their own.

Let $\Phi_t: \gX \rightarrow \R^{d_t}$ be a teacher model, which is presumably learned from some context variable $A$. If the model is close-sourced, then we cannot see $A$ at all, and we may not even know what space $\gA$ is.
Even if the model is open-sourced and we know the space $\gA$, very few companies release $A$ as a part of their pretraining data.
Indeed, it appears that the quality of $\Phi_t$ largely depends on the quality of $A$ (the quality of the context), so most companies keep it as their business secret.

Even though we might not know $A$, it is still possible to distill the knowledge of $\Phi_t$ if we can query $\Phi_t$ for a sufficient number of times.
We can construct its centered linear kernel $k_t(x,x') = \dotp{\tPhi_t(x), \tPhi_t(x')}$, and use KISE to extract its top eigenspace.
It is easy to see that $\mu_0 \equiv 1$ is an eigenfunction of $k_t$ with eigenvalue $0$.

Two remarks on this method.
First, we still cannot access $\px$ if the data is not released.
If we use our own data from another distribution during distillation, then the contexture will be skewed.
Second, the linear kernel $k_t$ of the teacher model is not equal to $\kx$. In fact, it contains the information of at most the top-$d_t$ eigenfunctions of $\kx$, which nonetheless is already sufficient for knowledge distillation purpose.

In addition to KISE, we can use the following objective for distillation:
\begin{equation}
\label{eqn:distillation}
\gR(\Phi) = \min_{\mW \in \R^{d_t \times d}, \; \vb \in \R^{d_t}} \; \E_{X \sim \px} \brac{ \norm{ \mW \Phi(X) + \vb - \Phi_t(X) }_2^2 } .
\end{equation}
This objective extracts the top-$d$ eigenspace of $T_{k_t}$, with a (centered) linear kernel as the loss kernel. This can be proved in the same way as \Cref{thm:obj-regression}.
We leave this proof as an exercise.
Let $\tkx$ be the dual kernel of the original context $\Phi_t$ was trained on.
If $\Phi_t$ extracts the top-$d_t$ eigenspace of $\tkx$ and preserves their order as well, and $d \le d_t$, then apparently the $\Phi$ we learn will extract the top-$d$ eigenspace of $\tkx$.

However, one issue is that $\Phi_t$ might not preserve the original order of the eigenfunctions.
Among all the objectives we have discussed so far, only the spectral contrastive loss \Cref{eqn:obj-spectral-contrastive} can preserve the original eigenvalues; all the other objectives can only recover the space spanned by the top eigenfunctions.
Therefore, given a teacher model, it is unsafe to assume that it preserves the original order of the eigenfunctions.
One way to completely erase the eigenfunctions of the teacher model is to whiten $\Phi_t$ before constructing $k_t$, that is multiplying $\tPhi_t$ by $[\Cov(\Phi_t)]^{-1/2}$.
The resulting $k_t$ will have the same eigenfunctions as the original $k_t$, but its eigenvalues become all either $0$ or $1$.
This is called a spectrally transformed kernel (STK), which will be discussed in more detail in \Cref{chap:bounds}.
Note that whitening is not useful for knowledge distillation from a single teacher model, but will be quite useful later when there are multiple contexts.

\paragraph{Converting any access to $k$-access.}
Any pretrained encoder can be a teacher model.
Thus, for any context, we can first pretrain an encoder with it, view the encoder as a teacher model, and then obtain a context with $k$-access using the above approach.
This also makes storing contexts very simple.
Suppose a context is only available for a limited amount of time, for example due to copyright limitations.
All we need to do is to pretrain an encoder using the context and store it.
Then, we can use the context whenever we want later on.

\paragraph{Social impact.}
Whether knowledge distillation constitutes a copyright infringement is an important problem.
Recently, DeepSeek \cite{guo2025deepseek} showed that it is possible to use a fairly small budget to distill the knowledge of an OpenAI model that cost billions of dollars to pretrain.
Moreover, the analysis above shows that such knowledge distillation cannot be prevented by making the model close-sourced.
As such, big tech companies might be less and less incentivized to grant public access to their models at a low price, fearing that it would be too easy for other companies to copy their work.
This might slow down the development of AI.

\section{Learning from a Mixture of Contexts}
\label{sec:learn-mixture}

The objectives discussed so far can only learn from one context.
In practice, it is often the case that there are multiple training signals that we can leverage.
For example, there might be multiple random transformations, such as translation, flipping, cropping and color distortion for images.
There could also be multiple labels for each sample, or labels and a graph.
Here we demonstrate that one can learn representations from multiple training signals by mixing multiple contexts,
and in fact some existing learning algorithms are implicitly doing this.
Specifically, we discuss two base operations: convolution and convex combination.
\Cref{chap:mixture} will provide a more general framework for mixing multiple contexts.

Suppose there are $r$ contexts given by $P^+_1,\cdots,P^+_r$, and $\px$ is the marginal distribution of every $P^+_j$.
Let $\kx_1,\cdots,\kx_r$ be their dual kernels.
Then, the \textbf{convolution} of these contexts is defined to have the top eigenfunctions of $T_{\kx_1} T_{\kx_2} \cdots T_{\kx_r} \cdots T_{\kx_2} T_{\kx_1}$ as its contexture;
and the \textbf{convex combination} of these contexts is defined to have the top eigenfunctions of $\sum_{j=1}^r w_j T_{\kx_j}$ as its contexture, for some fixed non-negative $w_1,\cdots,w_r$.
The convolution usually appears when one composes multiple random transformations; for example, when one applies translation, flipping and cropping to the same image.
The convex combination appears when the objective is the weighted sum of multiple individual objectives, each of which is designed for one context, as illustrated in \Cref{fig:mixture-demo}.

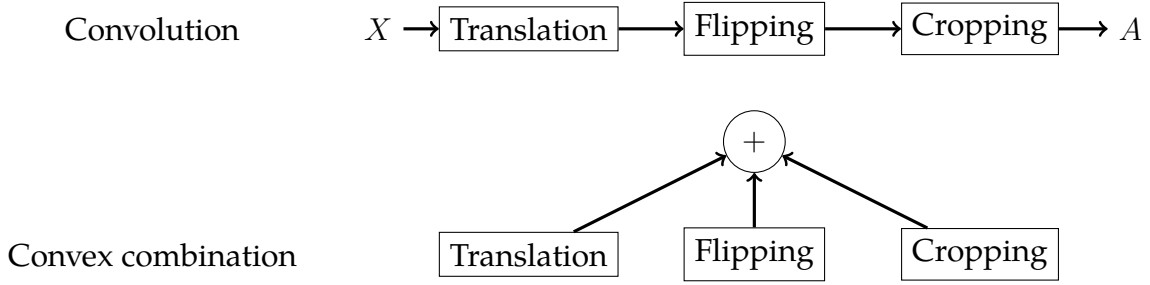
\begin{figure}[t]
    \centering
    \begin{tikzpicture}

\node[rectangle, draw] (b1) at (0,0) {Translation};
\node[rectangle, draw] (b2) at (3,0) {Flipping};
\node[rectangle, draw] (b3) at (6,0) {Cropping};
\node[circle, draw] (p) at (3,1.5) {$+$};
\node[] at (-5,0) {Convex combination};
\draw[->, very thick] (b1) -- (p);
\draw[->, very thick] (b2) -- (p);
\draw[->, very thick] (b3) -- (p);

\node[] (x) at (-2,3) {$X$};
\node[rectangle, draw] (a1) at (0,3) {Translation};
\node[rectangle, draw] (a2) at (3,3) {Flipping};
\node[rectangle, draw] (a3) at (6,3) {Cropping};
\node[] (a) at (8,3) {$A$};
\draw[->,very thick] (x) -- (a1);
\draw[->,very thick] (a1) -- (a2);
\draw[->,very thick] (a2) -- (a3);
\draw[->,very thick] (a3) -- (a);

\node[] at (-5,3) {Convolution};

\end{tikzpicture}
    \caption{Convolution and convex combination of multiple transformations on images.}
    \label{fig:mixture-demo}
\end{figure}

Now let us look at three concrete examples: supervised learning with a feature map, multi-step generative models, and RLHF.

\paragraph{Supervised learning with a feature map.}
When $\gX$ is a space of real-world objects, the inputs need to be mapped to numerical vectors via a feature map $\Omega : \gX \rightarrow \R^{\dw}$ so that they can be stored in a computer.
For example, if $\gX$ is the space of images, then $\Omega$ can be PNG or JPEG.
The quality of $\Omega$ affects the quality of the context.
Generally speaking, PNG is better than JPEG because it is lossless, while JPEG loses information.

Assume that $\Omega$ is a deterministic mapping. Define operator $\tpyw: \lwp \rightarrow \lxp$ and its adjoint operator $\tpywstar: \lxp \rightarrow \lwp$ as
\begin{equation*}
    (\tpyw h)(x) = h(\Omega(x)); \qquad (\tpywstar f) (\omega) = \int f(x) dP(x | \omega) .
\end{equation*}

In this scenario, an encoder $\Phi$ can be trained as follows.
First, train $\Gamma: \R^{\dw} \rightarrow \R^d$ via supervised learning.
Then, define $\Phi = \tpyw \Gamma$.
For classification tasks, similar to \Cref{sec:illustrative-examples},
$\Gamma$ is trained using an unbiased predictor and the mean squared error:
\begin{equation}
\label{eqn:obj-classification-omega}
    \gR(\Gamma) = \min_{\mW \in \R^{d_{\gA} \times d}, \vb \in \R^{d_{\gA}}} \; \underset{(X,A) \sim P^+, \; \omega = \Omega(X)}{\E} \; \brac{ \norm{A - \mW \Gamma(\omega) - b}_2^2 }  .
\end{equation}

\begin{theorem}[Proof in \Cref{app:proof-thm-obj-classification-omega}]
\label{thm:obj-classification-omega}
$\Phi^*$ minimizes \Cref{eqn:obj-classification-omega} if and only if $\Gamma^*$ extracts the top-$d$ eigenspace of $\tpywstar \tp \Lambda \tpstar \tpyw$ (including $\mu_0 \equiv 1$), where $\Lambda$ is the integral operator of $\kp(a,a') = \sI[a=a']$.
In this case, $\Phi^* = \tpyw \Gamma^*$ extracts the top-$d$ eigenspace of $\tkw \tp \Lambda \tpstar \tkw$, where $\kw$ is the dual kernel associated with $\Omega$ such that $\tkw = \tpyw \tpywstar$.
\end{theorem}

Similar to \Cref{thm:obj-classification-balanced}, one can get rid of $\Lambda$ using the balanced loss.
We can see that this $\Phi^*$ learns the convolution of two contexts, one given by $\Omega$ and the other given by $P^+$.
One important application of supervised learning with a feature map is node representation learning on graphs.
Two contexts are available in this application: the graph, and each node has a node feature.
Popular methods such as graph neural networks (GNNs) \cite{hamilton2020graph} train an encoder $\Gamma(\omega)$, where $\omega$ is the node feature.

\paragraph{Multi-step generative models.}

There are two extremely popular generative models nowadays---large language models and denoising diffusion models \cite{ho2020denoising,song2020denoising}.
Both models can generate input $X$ from a starting point $A$ in a number of steps.
The starting point $A$ may or may not contain information about $X$.
In large language models, $A$ is a prompt, which contains partial information about $X$, and generation is done token by token;
in diffusion models, $A$ is white noise that has zero information about $X$, and generation is done by a number of denoising steps.

Although generative models have achieved remarkable success in many applications, what representations these models are learning is quite unclear, and whether these representations can be applied to tasks other than generation is an active research topic.
For example, recently \cite{xiang2023denoising,chen2025deconstructing} studied whether the representations of diffusion models can be used for recognition tasks, and they found that these models ``have strong recognition power for understanding the visual content''.
The common perception is that generative modeling and representation learning are two completely different paradigms in machine learning, but in fact they can be connected via the contexture theory.

A multi-step generative model, such as a diffusion model or a GPT, generates an input $x$ by $a_1 \rightarrow a_2 \rightarrow \cdots \rightarrow a_r \rightarrow x$, where $a_{j+1}$ contains more information about $x$ than $a_j$.
The starting point $a_1$ may or may not contain information about $x$.
For diffusion models, $a_j$ is $x$ plus Gaussian noise; for language models, $a_{j+1}$ is $a_j$ plus one more token at the end.
The following is a general formulation of multi-step generative modeling, which does not need to assume the specific form of $x$ or $a_j$.

Let $\gA_j$ be the space of $a_j$. Let $P_j^+$ be the joint distribution of $X$ and $A_j$. 
We use SVME to train $(r+1)$ encoders.
Specifically, we train an encoder $\Psi_j : \gA_j \rightarrow \R^d$ for every $j$, along with $\Phi: \gX \rightarrow \R^d$.
Each $\Psi_j$ is trained with the goal of generating $x$ in one shot: Given $a_j \in \gA_j$, we find the $\hat{x}$ such that $\Phi(\hat{X})$ is the closest to $\Psi_j(a_j)$, and this $\hat{x}$ should be close to the original $x$.
The training objective for this goal is given by
\begin{equation*}
    \gR_j(\Phi; \Psi_j) = \E_{(X,A_j) \sim P^+_j} \; \brac{\norm{ \Psi_j(A_j) - \Phi(X) }_2^2} .
\end{equation*}
Let $\vw = [w_1,\cdots,w_L]$ be a weight vector where $w_j \ge 0$. The overall objective is
\begin{equation}
\label{eqn:generative-models}
 \underset{\Phi, \Psi_1,\cdots,\Psi_r}{\textrm{minimize}} \; \;   \gR_{\gen}(\Phi;\Psi_1,\cdots,\Psi_r) = \sum_{j=1}^r w_j \gR_j (\Phi; \Psi_j) \quad \text{s.t.} \quad \Cov_{\px}[\Phi] = \mI .
\end{equation}
Although each $\Psi_j$ is pretrained with the goal of generating $x$ in one shot, actual generation is still performed in multiple steps.
Specifically, given $a_j \in \gA_j$, we find $a_{j+1} \in \gA_{j+1}$ such that $\Psi_{j+1}(a_{j+1})$ is the closest to $\Psi_j(a_j)$.
Such a procedure resembles a denoising diffusion model, which trains a noise approximator $\epsilon_\theta(\cdot, j)$ for $j \in [r]$ by closing the gap between $x$ and $a_j - \epsilon_\theta(a_j, j)$ (this objective aims to denoise $a_j$ in one shot),
while the actual generation is done in multiple steps.

One great thing about SVME is that even though there are lots of encoders, the constraint is only imposed on $\Phi$.
The following result shows that this objective learns the contexture of a convex combination of the $r$ contexts.
It is a corollary of the general result in \Cref{sec:convex-combination}, so the proof will be deferred until then.
\begin{corollary}
    Let $\kx_j$ be the dual kernel of $P_j^+$.
    Let $\Phi^*$ be any optimal solution to \Cref{eqn:generative-models}.
    Then, $\tPhi^*$ extracts the top-$d$ eigenspace of $\sum_j w_j \kx_j$.
\end{corollary}

It should be clarified that both diffusion models and large language models use a single-model architecture, commonly known as a decoder-only architecture, but can also be called an encoder-only architecture if the model is used as a representation.
The above analysis does not apply to encoder-decoder architectures such as VAE \cite{VAE} and GAN \cite{goodfellow2014generative}, because the representations of these encoders are meaningless without the corresponding decoders.
The output of VAE and GAN is a Gaussian random vector, whose association with $X$ is indecipherable without the decoder.
Therefore, VAE and GAN cannot be analyzed in the same way as above.

\paragraph{Alignment in RLHF.}
Reinforcement learning with human feedback (RLHF) is a common technique in fine-tuning LLMs.
Let $\Phiref$ be a reference model, such as a model trained by supervised fine-tuning (SFT).
Let $A_1, A_2$ be two random variables on $\gA$ with possibly different distributions, both of which are associated with $X$.
For instance, $X$ can be an English sentence, and $A_1,A_2$ can be two Chinese translations of the same sentence.
There is also a critic function $C: \gX \rightarrow \oset{1, 2}$, such that for any input $x$, if $C(x) = 1$, then $A_1$ is preferred; otherwise, $A_2$ is preferred.
The goal of RLHF is to fine-tune the model to learn such preference (which is also known as alignment), while still keeping the model close to $\Phiref$.
In practice, it is observed that if $\Phi$ is too far away from $\Phiref$, then its performance will be very poor, a phenomenon known as over-optimization \citep{ouyang2022training,gao2023scaling}.
Thus, we need to make sure that $\Phi$ does not go too far away from $\Phiref$.

Let $\gR_{\textrm{align}}(\Phi, \Psi) = \sum_{i = 1,2} \; \underset{X \sim \px}{\Pr} \brac{C(X) = i} \underset{X \sim \px}{\E} \sbrac{\norm{\Phi(X) - \Psi(A_i)}_2^2}{C(X) = i}$ be the alignment loss based on SVME.
To implement this loss, for each sample $(x, a_1, a_2, c)$ where $c$ is the critic output, one simply updates the model to minimize $\norm{\Phi(x) - \Psi(a_1)}_2^2$ if $c = 1$, and $\norm{\Phi(x) - \Psi(a_2)}_2^2$ if $c = 2$.
Then, define a loss that reflects the gap between $\Phi$ and $\Phiref$ as $\gR_{\textrm{ref}}(\Phi) = \underset{\mW, \vb}{\min} \; \E \brac{ \norm{\mW \Phi(X) + \vb - \Phiref(X)}_2^2 }$.
This loss function is invariant under invertible linear transformations on $\Phi$, because such transformations have no impact on the downstream performance when $\Phi$ is used with a linear probe.
The overall objective is given by
\begin{equation*}
     \underset{\Phi, \Psi}{\textrm{minimize}} \quad \gR_{\textrm{align}}(\Phi, \Psi) + \beta \gR_{\textrm{ref}}(\Phi)  \quad \text{s.t.} \qquad \Cov_{\px}[\Phi] = \mI 
\end{equation*}
for some $\beta > 0$.
Similar to generative models, this objective also learns the contexture of a convex combination of two contexts.
The first context is given by a the random variable $A$ on $\gA$, such that $P^+(x,a) = \Pr[C(x) = 1] P_{A_1}^+(x,a) + \Pr[C(x)=2] P_{A_2}^+(x,a)$.
The second context is provided by the teacher model $\Phiref$, similar to \Cref{eqn:distillation}.

\section{Extracting Exact Eigenfunctions and Eigenvalues}
\label{sec:extract-exact}
So far we have seen that a variety of variational objectives can extract the top-$d$ eigenspace of $\tkx$, meaning that they can recover the linear space spanned by $\mu_1,\cdots,\mu_d$.
One might ask if it is possible to extract the exact top-$d$ eigenfunctions using a variational objective instead of kernel PCA.
This is indeed possible.
\cite{pmlr-v162-deng22b} proposed \textbf{neural eigenfunctions} to extract the exact top-$d$ eigenfunctions of $\tkx$. It solves the following problem:
\begin{equation}
\label{eqn:nef}
\begin{aligned}
 \underset{\Phi: \gX \rightarrow \R^d}{\textrm{minimize}} \qquad & \E \brac{ \norm{\tPhi(X)}_2^2 - \dotp{\tPhi(X), \tkx \tPhi(X)}  }  \\ 
 \text{ s.t.} \qquad  & \norm{\tphi_i}_\px^2 = 1, \quad \forall i \neq j: \;  \dotp{\tphi_i, \tkx \tphi_j}_\px = 0.
\end{aligned}
\end{equation}
Its difference from KISE is that it changes the constraint $\dotp{\tphi_i, \tphi_j}_\px = 0$, which enforces the different dimensions to be orthogonal, to $\dotp{\tphi_i, \tkx \tphi_j}_\px = 0$. Interestingly, this small change allows the extraction of the exact eigenfunctions.
It should be emphasized that the $d$ eigenfunctions are obtained simultaneously, not sequentially.
\begin{theorem}[Proof in \Cref{app:proof-thm-nef}]
\label{thm:nef}
    Let the optimal $\Phi^*$ of \Cref{eqn:nef} be $[\phi_1^*,\cdots,\phi_d^*]$, then we can choose the eigenfunctions $\mu_1,\mu_2,\cdots$ of $\tkx$ that have non-increasing eigenvalues and form an ONB of $\lxp$, such that $\tphi_1^*,\cdots,\tphi_d^*$ is a permutation of $\mu_1,\cdots,\mu_d$.
\end{theorem}

With $T$-access, we can use multi-view learning to extract the exact eigenfunctions.
\begin{equation*}
 \underset{\Psi: \gA \rightarrow \R^d}{\textrm{minimize}} \; \E \brac{ \norm{\tPsi(A) - \tPsi(A^+)}_2^2   } \quad \text{ s.t. } \; \norm{\tpsi_i}_\pa^2 = 1, \; \;  \forall i \neq j: \; \E[\tpsi_i(A) \tpsi_j(A^+)] = 0 ,
\end{equation*}
where $A,A^+$ are positive samples drawn from $P^+(\cdot|X)$ of the same $X$.
We can prove that $\tPsi$ extracts the exact top-$d$ eigenfunctions of $\tka$ in the same way as \Cref{thm:nef}. And by \Cref{lem:duality}, the average encoder $\tPhi = \tp \tPsi$ also extracts the exact top-$d$ eigenfunctions of $\tkx$ multiplied by some constants.

\paragraph{Post-hoc approach.}
Suppose we have a pretrained $\Phi$ that learns the contexture, can we obtain the exact eigenvalues and eigenfunctions of $\tkx$?
The answer is yes.
Since $\Phi$ does not necessarily contain the information of the exact eigenvalues and eigenfunctions, we still need to use the context, via either kernel access or pair access.

First, suppose we have kernel access to the context.
Since $\tPhi$ spans the same space as $\mu_1,\cdots,\mu_d$, it suffices to learn a matrix $\mQ \in \R^{d \times d}$ such that $\tPhi \mQ = [\alpha_1 \mu_1,\cdots,\alpha_d \mu_d]$ for some $\alpha_1,\cdots,\alpha_d \neq 0$.
Then, these $\alpha_i$ can be eliminated by normalizing $\tPhi \mQ$ to have unit variance in each dimension.
The matrix $\mQ$ and the eigenvalues can be estimated as follows:
\begin{enumerate}
    \item Estimate the covariance matrix $\mC_\Phi \in \R^{d \times d} = \Cov_{\px}[\Phi]$ with Monte Carlo.
    \item Estimate $\mB_{\Phi} \in \R^{d \times d}$, where $\mB_{\Phi}[i,j] = \dotp{\tphi_i, \tkx \tphi_j}_{\px}$, with Monte Carlo.
    \item Solve the generalized eigenvalue problem $\mB_\Phi \vv = \lambda \mC_\Phi \vv$. Let the eigenvalues be $\lambda_1 \ge \cdots \ge \lambda_d \ge 0$, and the orthonormal eigenvectors be $\vv_1,\cdots,\vv_d$. Then, $\mQ = [\vv_1,\cdots,\vv_d]$, and $\lambda_i$ is an estimation of the $i$-th eigenvalue of $\tkx$, which is $s_i^2$. Since $d$ is not very large, this eigen-decomposition is efficient.
\end{enumerate}

Let us elaborate on why this method works.
For simplicity, assume that the top-$d$ eigenvalues of $\tkx$ are distinct; without this assumption, the result can still be proved with a more verbose proof.
Let $U = [\mu_1,\cdots,\mu_d]$, and suppose $\tPhi = U \mR$ for some invertible $\mR \in \R^{d \times d}$.
Since $\Cov[U] = \mI$ and $\dotp{U, \tkx U}_{\px} = \diag \oset{s_1^2,\cdots, s_d^2}$, we have $\mC_\Phi = \mR^{\top} \mR$ and $\mB_\Phi = \mR^{\top} \diag \oset{s_1^2,\cdots, s_d^2} \mR$.
Thus, the generalized eigenvalues are equal to $s_1^2,\cdots,s_d^2$, and $\mR \vv_i = \alpha_i \ve_i$ for some $\alpha_i \neq 0$, where $\ve_i = [0,\cdots,0,1,0,\cdots,0]$.
This implies that $\mR \mQ = \diag \oset{\alpha_1,\cdots,\alpha_d}$.
Hence, $\tPhi \mQ = [\alpha_1 \mu_1,\cdots,\alpha_d \mu_d]$.

Second, if we have pair access instead of $k$-access, then the eigenvalues and $\mQ$ can be obtained as follows:
Let $\Phi$ and $\Psi$ be trained via SVME.
We can estimate $\mC_\Phi = \Cov_{\px}[\Phi]$ and $\mB_{\Psi} = \Cov_{\pa}[\Psi]$ via Monte Carlo.
Then similarly, we solve the generalized eigenvalue problem $\mB_\Phi \vv = \lambda \mC_\Phi \vv$ to obtain the eigenvalues and $\mQ$.

Finally, if our goal is only to estimate the eigenvalues but not the eigenfunctions, then $\Phi$ need not to be trained with the entire dataset.
In fact, \cite{shawe2005eigenspectrum} showed that for any fixed $d$, the sum $s_1^2+\cdots+s_d^2$ can be estimated with low error using $\Theta(d)$ \iid{} samples.
By union bound, all $s_1^2,\cdots,s_d^2$ can be estimated with low error using $\Theta (d \log d)$ \iid{} samples.

\paragraph{Estimating eigenvalues with the post-hoc approach.}
Let us demonstrate the post-hoc method on 3 real datasets from OpenML \cite{OpenML2013}: \texttt{abalone}, \texttt{fifa}, and \texttt{kings\_county}.
We only extract the eigenvalues here, and later in \Cref{sec:scaling-laws} we will investigate the eigenfunctions.
We use KNN with $K = 60$ as context, where $\gA = \gX$, and $P^+(x'|x) = K^{-1}$ if $x'$ is a $K$-nearest neighbor of $x$ and $0$ otherwise.
For this context, we can exactly compute $\kx$, and thus we can obtain the exact eigenvalues (ground truth) using kernel PCA.
Meanwhile, we pretrain $\Phi$ with one of the variational objectives using a random subset of $m$ samples, and estimate the eigenvalues using the post-hoc approach.
Then, we compare the estimation with the ground truth.

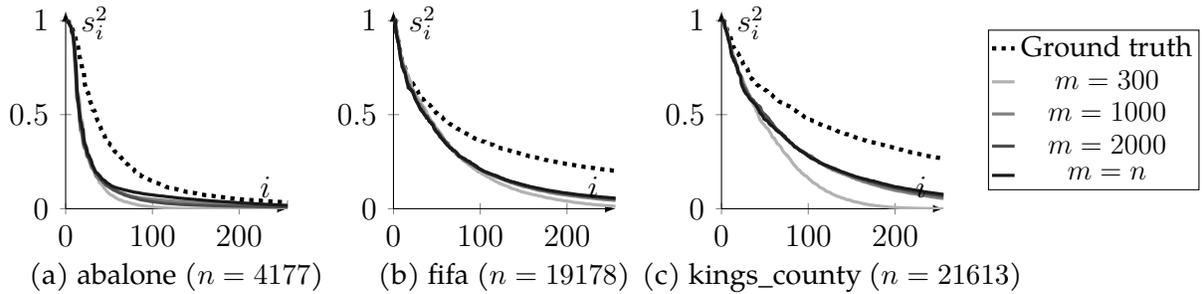
\begin{figure}[t]
    \centering
    \begin{tikzpicture}

\begin{axis}[
name=plot1,
 height = .26\linewidth,
    width = .28\linewidth,
    axis lines=left,  
    axis line style={-latex},  
    xlabel={$i$},
    ylabel={$s_i^2$},
    ylabel style={
        at={(0.13,0.8)}, 
        anchor=south,  
        rotate=270  
    },
    xlabel style = {
        at={(0.9,0)},
        anchor=south,
    },
title style={at={(0.5,-0.3)}, anchor=north,font=\small},
title={(a) abalone ($n=4177$)},
    xmin=0, xmax=256,
    ymin=0, ymax=1.05,
    xtick distance = 100,
    ytick distance = 0.5,
]
    \addplot[ultra thick,black,dotted] table [
        x=id,  
        y=exact,            
        col sep=comma       
    ] {data/spectrum_abalone.csv};
    \addplot[very thick,black!30] table [
        x=id,  
        y=m300,            
        col sep=comma       
    ] {data/spectrum_abalone.csv};
    \addplot[very thick,black!50] table [
        x=id,  
        y=m1000,            
        col sep=comma       
    ] {data/spectrum_abalone.csv};
    \addplot[very thick,black!70] table [
        x=id,  
        y=m2000,            
        col sep=comma       
    ] {data/spectrum_abalone.csv};
    \addplot[very thick,black!90] table [
        x=id,  
        y=mNone,            
        col sep=comma       
    ] {data/spectrum_abalone.csv};
\end{axis}

\begin{axis}[
name=plot2,
at={($(plot1.south east)+(40,0)$)},
 height = .26\linewidth,
    width = .28\linewidth,
    axis lines=left,  
    axis line style={-latex},  
    xlabel={$i$},
    ylabel={$s_i^2$},
    ylabel style={
        at={(0.13,0.8)}, 
        anchor=south,  
        rotate=270  
    },
    xlabel style = {
        at={(0.9,0)},
        anchor=south,
    },
title style={at={(0.5,-0.3)}, anchor=north,font=\small},
title={(b) fifa ($n=19178$)},
    xmin=0, xmax=256,
    ymin=0, ymax=1.05,
    xtick distance = 100,
    ytick distance = 0.5,
legend style = {nodes={scale=0.9, transform shape},at={(1.2,0.5)},anchor=west,inner sep=.5pt},
]
    \addplot[ultra thick,black,dotted] table [
        x=id,  
        y=exact,            
        col sep=comma       
    ] {data/spectrum_fifa.csv};
    \addplot[very thick,black!30] table [
        x=id,  
        y=m300,            
        col sep=comma       
    ] {data/spectrum_fifa.csv};
    \addplot[very thick,black!50] table [
        x=id,  
        y=m1000,            
        col sep=comma       
    ] {data/spectrum_fifa.csv};
    \addplot[very thick,black!70] table [
        x=id,  
        y=m2000,            
        col sep=comma       
    ] {data/spectrum_fifa.csv};
    \addplot[very thick,black!90] table [
        x=id,  
        y=mNone,            
        col sep=comma       
    ] {data/spectrum_fifa.csv};
\end{axis}

\begin{axis}[
name=plot3,
at={($(plot2.south east)+(40,0)$)},
 height = .26\linewidth,
    width = .28\linewidth,
    axis lines=left,  
    axis line style={-latex},  
    xlabel={$i$},
    ylabel={$s_i^2$},
    ylabel style={
        at={(0.13,0.8)}, 
        anchor=south,  
        rotate=270  
    },
    xlabel style = {
        at={(0.9,0)},
        anchor=south,
    },
title style={at={(0.5,-0.3)}, anchor=north,font=\small},
title={(c) kings\_county ($n=21613$)},
    xmin=0, xmax=256,
    ymin=0, ymax=1.05,
    xtick distance = 100,
    ytick distance = 0.5,
legend style = {nodes={scale=0.9, transform shape},at={(1.2,0.5)},anchor=west,inner sep=.5pt},
]
    \addplot[ultra thick,black,dotted] table [
        x=id,  
        y=exact,            
        col sep=comma       
    ] {data/spectrum_kings_county.csv};
    \addplot[very thick,black!30] table [
        x=id,  
        y=m300,            
        col sep=comma       
    ] {data/spectrum_kings_county.csv};
    \addplot[very thick,black!50] table [
        x=id,  
        y=m1000,            
        col sep=comma       
    ] {data/spectrum_kings_county.csv};
    \addplot[very thick,black!70] table [
        x=id,  
        y=m2000,            
        col sep=comma       
    ] {data/spectrum_kings_county.csv};
    \addplot[very thick,black!90] table [
        x=id,  
        y=mNone,            
        col sep=comma       
    ] {data/spectrum_kings_county.csv};
    \addlegendentry{Ground truth}
    \addlegendentry{$m=300$}
    \addlegendentry{$m=1000$}
    \addlegendentry{$m=2000$}
    \addlegendentry{$m=n$}
\end{axis}
\end{tikzpicture}
    \caption{Estimating the eigenvalues using the post-hoc approach with $m$ samples.}
    \label{fig:spectrum-estimation}
\end{figure}

We use a 2-layer wide Tanh-activated neural network with embedding dimension $d=512$ and hidden dimension 20,000 as $\Phi$.
We train the model through non-contrastive learning \Cref{eqn:obj-non-contrastive}, with the orthonormality constraint implemented by VICReg, and AdamW \cite{DBLP:journals/corr/KingmaB14,loshchilov2018decoupled} as the optimizer.
We vary $m$ and compare the estimated top-$d_0$ eigenvalues with the ground truth, where $d_0 = 256$.
The estimated eigenvalues and the ground truth are plotted in \Cref{fig:spectrum-estimation}.
From the plots, we observe that the eigenvalues estimated by our estimation method decay faster than the ground truth, even if the full dataset is used.
We hypothesize that the main reason is that even though we use a very wide neural network, its function class is still a subset of $\lxp$.
Consequently, the inductive bias of the model architecture has an impact on the encoder, and therefore the learned contexture can be viewed as a mixture of the inductive bias and the original KNN context.
This mixture causes the eigenvalues to decay faster, which explains the observation in \Cref{fig:spectrum-estimation}.
Another reason is related to optimization. Since the model is non-convex, gradient methods cannot find the minima of the objective.

The average estimation error of the top-$256$ eigenvalues is reported in \Cref{tab:spectrum-est}.
The error is defined as $\frac{1}{d_0} \sum_{i=1}^{d_0} |\hat{s}_i^2 - s_i^2|$, where $\hat{s}_i^2$ is the estimated eigenvalue.
The table shows that when $m \in [600,1000] \approx [0.5 d_0 \log d_0, 0.7 d_0 \log d_0]$, the performance is comparable to using the full dataset, which verifies the theoretical result of \cite{shawe2005eigenspectrum}.
The estimation error is not zero even if the full dataset is used due to the aforementioned reasons.
In summary, the post-hoc method can estimate the eigenvalues using a small subset of samples, but the estimated eigenvalues decay faster than the ground truth.

\begin{table}[t]
    \centering
    \begin{tabular}{l|llllll}
    \toprule
    \textbf{Dataset} & $m=100$ & $m=300$ & $m=600$ & $m=1000$ & $m=2000$ & Full dataset \\ 
    \midrule 
    abalone & 0.157 & 0.124 & 0.088 & 0.104 & 0.110 & 0.088 \\ 
    fifa & 0.218 & 0.151 & 0.137 & 0.134 & 0.133 & 0.131 \\ 
    kings\_county & 0.278 & 0.264 & 0.190 & 0.183 & 0.177 & 0.177 \\ 
    \bottomrule
    \end{tabular}
    \caption{Average estimation error of the top-$256$ eigenvalues.}
    \label{tab:spectrum-est}
\end{table}

\section{Implications on the Scaling Law}
\label{sec:scaling-laws}

It has been widely observed that the performance of deep neural networks on many real tasks increases with the model size, which is known as the \textbf{scaling law} \cite{kaplan2020scaling}.
Furthermore, it has been observed that models of different architectures, such as ResNets \cite{He_2016_CVPR} and ViTs \cite{dosovitskiy2021an}, learn highly aligned representations under the metrics in \cite{pmlr-v97-kornblith19a} when the models are sufficiently large, even if they are trained with different objectives.

To explain this phenomenon, \cite{pmlr-v235-huh24a} proposed the \textbf{platonic representation hypothesis}, which states that ``neural networks, trained with different objectives on different data and modalities, are converging to a shared statistical model of reality in the representation spaces''.
The assumptions of this hypothesis are that the neural networks are large enough, and there is a sufficient amount of data.

The contexture theory provides a new perspective on the role of scaling.
It implies that the so-called ``reality'' is in fact the top eigenfunctions of $\tkx$.
When using a neural network as $\Phi$, the function class of $\Phi$ is a subset of $\lxp$;
and when scaling up the model size, this subset gets closer to the entire space $\lxp$, and thus the learned representation becomes closer to the top-$d$ eigenfunctions of $\tkx$, which are independent of the model architecture (so this is not a special property of neural networks).
This also explains why recently it has been observed that increasing the model size is producing a diminishing return.
When the model is large enough so that the learned representation is highly aligned with the top-$d$ eigenfunctions, then further increasing the model size will be less useful.
\cite{pmlr-v235-huh24a} also observed that models trained in different modalities all align with the same shared representation.
This suggests that commonly used contexts in different modalities have similar top eigenfunctions.
Note that this cannot be true for all contexts.
For a very weird context such as randomly shuffling the pixels of an image, its top eigenfunctions are surely not aligned with this shared representation.

In this section, we use an experiment to show that the representation learned by a neural network is indeed aligned with the top-$d$ eigenfunctions,
which provides empirical evidence to support the above arguments.

\paragraph{Experiment overview.}
The purpose of this experiment is to examine whether a large neural network can learn the contexture well, and whether scaling up the model size makes the learned representation more aligned to the top-$d$ eigenfunctions.
We compare between two encoders.
The first encoder is obtained via kernel PCA on the dual kernel, so it consists of the exact top-$d$ eigenfunctions.
The second encoder is obtained via training a large neural network to optimize an objective that can learn the contexture.
Then, we compute the representational alignment of these two encoders.
The most classical metric is the canonical-correlation analysis (CCA) metric $R_{\textrm{CCA}}^2$, which is invariant under invertible linear transformations to the encoders.
\cite{pmlr-v97-kornblith19a} proposed a variant called linear CKA, which is only invariant under orthogonal transformations.
In our setting, since we only care about the span of $\phi_1,\cdots,\phi_d$, we would like the metric to be invariant under all invertible transformations, which is why we use CCA.
In addition, we also use the mutual KNN metric with 10 neighbors proposed by \cite{pmlr-v235-huh24a}, which measures the intersection over union (IoU) of nearest neighbors between the two representations.
This metric is not invariant under invertible linear transformations,
so we whiten the two representations such that their covariance matrices are both identities.

\paragraph{Setup.}
We use the \texttt{abalone} dataset from OpenML, and split the dataset into a pretrain set, a downstream train set and a downstream test set by 70\%-15\%-15\%.
Like what we used earlier, we choose K-nearest neighbors (KNN) with $K = 30$ to be the context.
The embedding dimension is set to be $d = 128$.
For the second encoder, we train a fully-connected neural network with Tanh activation and skip connections for a sufficient number of steps with full-batch AdamW, and vary the depth and width of the network so that we can study their effect on the alignment.
Here, ``depth'' refers to the number of hidden layers---for example, a 2-layer neural network has depth 1.
For each width and depth, we run the experiments 15 times with different random initializations, and report the average alignment.

In our experiments, we observe the \textbf{dimension collapse} problem \cite{jing2022understanding}---if we set the output dimension of the neural network to be $d$, then the rank of the learned representation will usually be less than $d$, meaning that it can only extract the top-$d'$ eigenspace for some $d' < d$.
\cite{jing2022understanding} proved that this problem can be caused by the training dynamics of self-supervised learning, that is a large neural network trained with a gradient method cannot find the minima, but will find a low-rank solution instead.

To fix this issue, we set the output dimension of the neural network to be $d_1 = 512 > d$.
After we obtain the $d_1$-dimensional encoder, similar to \Cref{sec:extract-exact} we estimate the matrices $\mC_\Phi$ and $\mB_\Phi$, and solve the generalized eigenvalue problem $\mB_\Phi \vv = \lambda \mC_\Phi \vv$.
Let $\mV = [\vv_1,\cdots,\vv_d] \in \R^{d_1 \times d}$ be the top-$d$ eigenvectors;
then, we use $\tPhi \mV$ as the $d$-dimensional representation.
In other words, we use the 128 principal components of the 512-dimensional embedding.

\begin{figure}[t]
    \centering
    \begin{tikzpicture}
\begin{axis}[
name=plot1,
 height = .35\linewidth,
    width = .7\linewidth,
    xlabel={Width},
    ylabel={Alignment},
    xlabel style = {
        at={(1.08,-0.05)},
        anchor=south,
    },
        legend style={
        at={(1.1,0.8)},
        anchor=north west,
        cells={anchor=west},
    },
title style={at={(0.5,-0.3)}, anchor=north},
    xmin=0, xmax=12,
    ymin=0.6, ymax=0.92,
    xtick={0,1,2,3,4,5,6,7,8,9,10,11,12},
    xticklabels={64,128,256,512,1024,2048,4000,8000,12000,16000,20000,25000,30000},
    x tick label style={rotate=45,font=\small},
    ytick distance = 0.1,
]
    \addplot[very thick,black] table [
        x expr=\coordindex, 
        y=CCA,            
        col sep=comma       
    ] {data/topd.csv};
    \addplot[very thick,black!50] table [
        x expr=\coordindex, 
        y=CCA,            
        col sep=comma       
    ] {data/topd2.csv};
    \addplot[very thick,black!20] table [
        x expr=\coordindex, 
        y=CCA,            
        col sep=comma       
    ] {data/topd3.csv};
    \addplot[very thick,black,dashed] table [
        x expr=\coordindex, 
        y=MKNN,            
        col sep=comma       
    ] {data/topd.csv};
    \addplot[very thick,black!50,dashed] table [
        x expr=\coordindex, 
        y=MKNN,            
        col sep=comma       
    ] {data/topd2.csv};
    \addplot[very thick,black!20,dashed] table [
        x expr=\coordindex, 
        y=MKNN,            
        col sep=comma       
    ] {data/topd3.csv};
    \legend{1,2,3};
\end{axis}
\node[] at (11.2,3.7) {Depth};

\begin{axis}[
name=plot2,
at={($(plot1.south west)+(0,-160)$)},
 height = .35\linewidth,
    width = .7\linewidth,
    xlabel={Depth},
    ylabel={Alignment},
    xlabel style = {
        at={(1.08,-0.05)},
        anchor=south,
    },
    x tick label style={rotate=45,font=\small},
title style={at={(0.5,-0.3)}, anchor=north},
    xmin=0, xmax=9,
    ymin=0.55, ymax=0.9,
    xtick={0,1,2,3,4,5,6,7,8,9},
    xticklabels={1,2,4,8,16,32,64,100,150,200},
    ytick distance = 0.1,
        legend style={
        at={(1.1,0.8)},
        anchor=north west,
        cells={anchor=west},
    },
]
    \addplot[very thick,black] table [
        x expr=\coordindex,  
        y=CCA,            
        col sep=comma       
    ] {data/topd-512.csv};
    \addplot[very thick,black!50] table [
        x expr=\coordindex,  
        y=CCA,            
        col sep=comma       
    ] {data/topd-1024.csv};
    \addplot[very thick,black!20] table [
        x expr=\coordindex,  
        y=CCA,            
        col sep=comma       
    ] {data/topd-2048.csv};
    \legend{512,1024,2048};
    \addplot[very thick,black,dashed] table [
        x expr=\coordindex,  
        y=MKNN,            
        col sep=comma       
    ] {data/topd-512.csv};
    \addplot[very thick,black!50,dashed] table [
        x expr=\coordindex,  
        y=MKNN,            
        col sep=comma       
    ] {data/topd-1024.csv};
    \addplot[very thick,black!20,dashed] table [
        x expr=\coordindex,  
        y=MKNN,            
        col sep=comma       
    ] {data/topd-2048.csv};
\end{axis}
\node[] at (11.5,-2) {Width};
\end{tikzpicture}
    \caption{Alignment between the learned representation and the top-$d$ eigenfunctions of $\tkx$ on the \texttt{abalone} dataset. Solid curves: CCA. Dashed curves: mutual KNN. Depth here means the number of hidden layers.}
    \label{fig:scaling-law}
\end{figure}
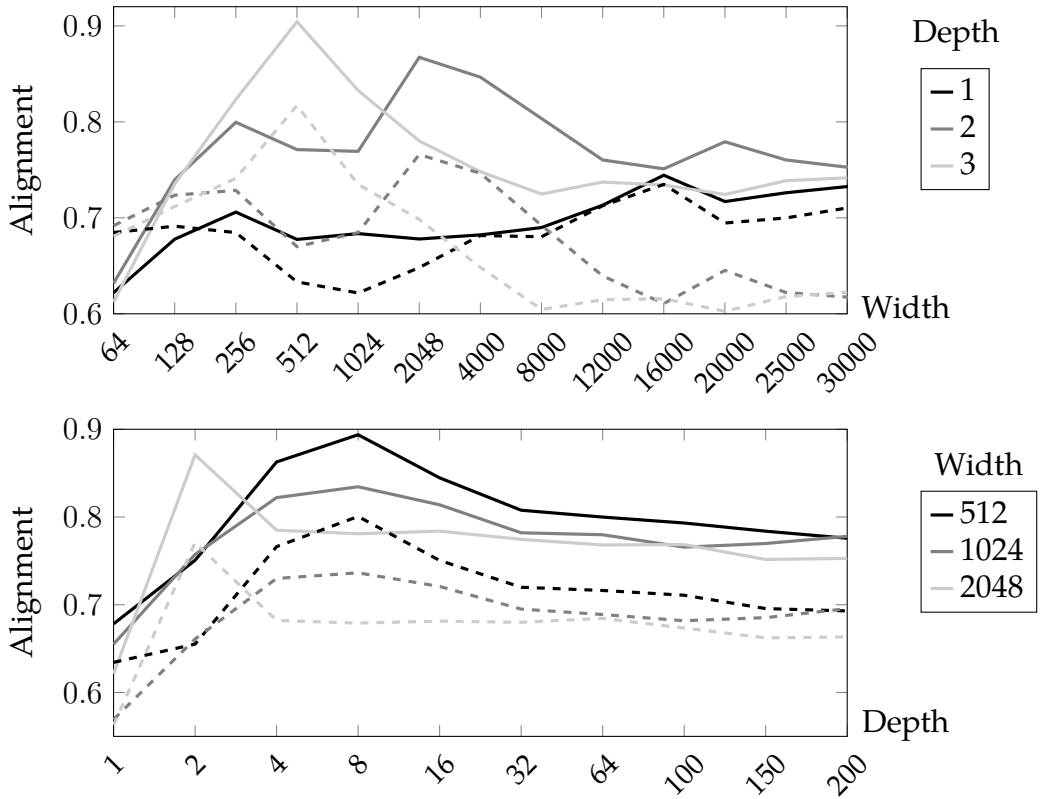

\paragraph{Results.}

\Cref{fig:scaling-law} plots the alignment between the two encoders while varying the depth and width of the neural network.
we can see that when the depth and width are chosen correctly, the CCA can be as high as 0.9, and the mutual KNN can be over 0.8.
Note that these alignment metric values are very high.
For example, in \cite{pmlr-v235-huh24a}, the mutual KNN metric value is usually below 0.2.
Hence, the representation learned by the neural network is highly aligned with the top-$d$ eigenfunctions.

The top plot studies neural networks with increasing widths.
We can see that when the neural network is not so wide, increasing the width will make the alignment higher.
However, once the neural network is wide enough, further increasing the width might have a negative effect.
For example, when the depth is 3, the alignment is the highest when the width is 512, and the alignment becomes lower when the network is wider than 512.
Since increasing the width can only make the function class of $\Phi$ larger, this phenomenon is not due to the expressivity of the neural network.
We hypothesize that this is because optimizing a larger model is harder.
Consequently, with the same number of pretraining steps, a larger model will be farther away from the minima, and the alignment decreases.

The bottom plot studies neural networks with increasing depths, and the observation is similar.
When the network is shallow, increasing the depth makes the alignment higher.
However, once the network is deep enough, further increasing the depth might have a negative effect.
We also observe from the bottom plot that a width-512 network has higher alignment than widths 1024 and 2048.
In addition, the alignment cannot reach 1.
This is because the model is non-convex, so the real optima (the precise top-$d$ eigenspace) cannot be found by gradient methods.

In summary, we draw two conclusions from this experiment:
(i) the representation learned by a large neural network is highly aligned with the top-$d$ eigenfunctions;
(ii) once the neural network is wide and deep enough, further increasing its size will not make the alignment higher, and might even have a negative effect.
Hence, we put forward the following argument about the scaling law:
Once the model is large enough such that $\Phi$ is already highly aligned with the top-$d$ eigenfunctions, further increasing the model size inevitably yields diminishing returns.

When the model is already large enough, a better context is necessary for further improvement.
The next part of this thesis studies how to obtain a better context.
There are two questions that need to be addressed.
First, what context is good and what context is ``better''?
In particular, if we want the representation to be transferable to a wide variety of downstream tasks, then how should we evaluate the representation without testing it on a specific task?
Second, how to obtain better contexts?
Creating new contexts from scratch is obviously difficult, so are there easier options?
These questions will be addressed in the following two chapters.

\paragraph{Code and data availability.}
The code for the experiment in this section can be found at {\scriptsize \url{https://1drv.ms/u/c/ea9fe908498c8b82/EWb_t4e-27VKsoSTWV6J_yQBUFsuWSAXbRHJK4GoUEynjw?e=r4ROdU}}.
The data can be downloaded from OpenML.

\chapter{Intrinsic Evaluation: The Optimality of Learning the Contexture}
\label{chap:optimality}

This chapter focuses on the evaluation of encoders and contexts.
How to tell if an encoder is good or not?
There are two methods as summarized in \cite[Sec.~4]{bommasani2021opportunities}.
The first method is extrinsic evaluation, where we choose a specific downstream task, fit a predictor on top of the encoder, and use the performance of that predictor as the evaluation of the encoder.
In practice, this performance is ultimately what we care about.
The second method is intrinsic evaluation, which is independent of any downstream task.
Intrinsic evaluation is more useful in the analysis of pretraining for two reasons:
(i) We would like the encoder to be transferable to a wide variety of downstream tasks, so only evaluating it on one task is usually insufficient;
(ii) Foundation models are often used on tasks they are not designed for, so we cannot assume that we know all the tasks at pretrain time.

Obviously, no encoder is good for all tasks.
For any encoder, one can adversarially construct a task for which the encoder is poor.
However, such tasks are not what we care about.
For example, in NLP, downstream tasks that we care about include sentiment analysis, sentence completion, translation, question answer, etc.
Counting how many words in a paragraph contain the letter "e" is a valid task, and large language models are not really good at it (see \Cref{fig:llm-bad-demo}), but this task is not one that we would usually care about.
Therefore, we need to define a class of relevant tasks, and the intrinsic evaluation of the encoder is conducted on this class.

\begin{figure}[t]
\begin{verbatim}
User: Count how many words in the following paragraph 
contains the letter "e":
Overall, although there does not ... the proposed metric 
is useful for choosing hyperparameters and comparing 
contexts in practice.

Claude 3.7 Sonnet: I'll count the words containing the 
letter "e" in the paragraph you provided. Let me go 
through it word by word:
1. "Overall" - contains "e"
...
58. "and" - no "e"
59. "comparing" - contains "e"
60. "contexts" - contains "e"
61. "in" - no "e"
62. "practice" - contains "e"
Counting the words with the letter "e": 42 words out 
of 62 total words contain the letter "e".
\end{verbatim}
\caption{An example where the LLM Claude 3.7 Sonnet (as of April 14, 2025) makes a mistake on a task that is not quite relevant to NLP. Line 59 is wrong because ``comparing'' does not contain the letter ``e''. The final answer is also wrong.}
\label{fig:llm-bad-demo}
\end{figure}

How to define such a class of relevant tasks at pretrain time?
The only information we have at pretrain time is the context, and the prior knowledge that the context should be useful for learning an encoder for any task of interest.
Note that a downstream task need not to ``look similar'' to the context; for example, a BERT trained with masked token prediction can be used for sentiment analysis, though these two tasks look quite different.
Thus, we need a quantitative description of the usefulness of a context for a task, for which we define the \textbf{compatibility} between a context and a task.

Once we define the class of compatible tasks, we can prove that the encoder that learns the contexture is optimal for this class of tasks, in the sense that it minimizes the worst-case approximation error.
On the other hand, if the task is not compatible, then it is likely that the encoder will have a bad performance.
This explains why LLMs are so powerful that they can perform well on various tasks, but they struggle with easy tasks such as answering factual questions, where they often produce seemingly correct but completely wrong answers, a phenomenon known as hallucination \cite{huang2025survey}.
In short, answering factual questions is not compatible with the context of next token prediction.

The significance of this result is that it reduces the problem of pretraining to the problem of finding a good context.
Once we have the context, we can use any variational objective in the last chapter, and the resulting foundation model is guaranteed to have the optimal performance if it can learn the contexture well.
We will also discuss how to intrinsically evaluate an arbitrary encoder, which depends on two key concepts---the induced RKHS and the ratio trace.
Thus, even if the foundation model does not learn the contexture, we are able to predict its performance on compatible tasks.

Finally, we study what contexts are good contexts.
The key finding is that a good context should have a moderate association between $X$ and $A$, so that the decay rate of the singular values is neither too fast nor too slow.
Qualitatively, if the association is too weak, then there will be very few tasks that are compatible with the context, so the resulting encoder will not be transferable to a wide range of tasks.
On the other hand, if the association is too strong, then the sample complexity of learning the contexture and the downstream predictor will be very high.
Quantitatively, we propose a metric that measures the usefulness of a context.
This metric only depends on the singular values, so it can be efficiently estimated using the post-hoc approach described in \Cref{sec:extract-exact}.
Experiments show that the metric correlates well with the actual performance of the encoder on real datasets.

\section{Compatibility, Optimality of Contexture}

A downstream task is represented by a \textbf{target function}.
Most downstream tasks, such as prediction, clustering, and segmentation, can be associated with a target function $f^* \in \lxp$.
For example, multi-class classification can be associated with multiple one-vs-all labeling functions.
After $\Phi$ is pretrained, and the training samples of the downstream task are revealed, there are a number of ways to use $\Phi$, such as fitting a small neural network on top, using a kernel method, supervised fine-tuning (SFT), etc.
This thesis focuses on the simplest way called a \textbf{linear probe}, where one fits the downstream data with a linear predictor on top of $\Phi$.
Specifically, the downstream predictor is given by $\hat{f}(x) = \mW \Phi(x) + \vb$.
Since $\mW$ and $\vb$ can be arbitrarily chosen, the mean and variance of $f^*$ have no impact on the difficulty of learning $f^*$. Hence, our definition of compatibility should be independent of them.

Given a context $P^+$, we say that a task is compatible with it, if the information the context provides makes it easier to learn a predictor for the task.
Formally, consider a training set for supervised learning $\oset{(a_i, y_i)}_{i=1}^n$, where $y_i = f^*(x_i) + \xi_i \in \R$ for some small random noise $\xi_i$.
In this dataset, the original inputs $x_i$ are unknown, and only one $a_i \sim P^+(\cdot | x_i)$ is provided for each $i$.
The most straightforward method of learning a predictor for this task is as follows:
we first fit a predictor $\hat{g} : \gA \rightarrow \R$ on this dataset, and then convert it to $\hat{f}: \gX \rightarrow \R$ with $\hat{f} = \tp \hat{g}$, assuming that we have $T$-access to the context.
$f^*$ is said to be compatible with the context if this method works well, for which two conditions are necessary:
\begin{enumerate}[label=(\roman*)]
    \item There exists a $g^* \in \lap$ such that $f^* = \tp g^*$.
    \item The variance of $g^*$ conditioned on $x$, denoted by $\Var[g^*(A) | x]$, is low on average.
\end{enumerate}

Condition (i) says that the labels $y_i$ can be approximated by a function in the range of $\tp$, up to the small noise $\xi_i$.
Condition (ii) ensures that $g^*(a_i) \approx f^*(x_i) \approx y_i$; without this condition, the $\hat{g}$ fit on $\oset{(a_i,y_i)}$ cannot approximate $g^*$.

The compatibility defined below is based on the above insights.
\begin{definition}
\label{def:compatibility}
    The \textbf{compatibility} with $P^+$ of $f \in \lxp$ such that $\tilde{f} \neq \vzero$ is defined as
    \begin{equation}
    \label{eqn:def-compatibility}
        \rho(f, P^+) = \max_{g \in \lap, g \neq \vzero} \frac{\dotp{\tilde{f}, \tp g}_{\px}}{\norm{\tilde{f}}_{\px} \norm{g}_{\pa}} \in [0,1]  .
    \end{equation}
\end{definition}
The compatibility is defined this way so that it is independent of the mean and variance of $f$.
Here is a formula for the compatibility of any $f^*$ with $P^+$.
For simplicity, let $f^*$ has mean zero.
Let $f^* = \underset{i \ge 1}{\sum} u_i \mu_i$ and $g^* = \underset{i \ge 1}{\sum} v_i \nu_i$. 
Then, $\rho(f^*,P^+) = \underset{v_i}{\max} \frac{\sum_i s_i u_i v_i}{\sqrt{\sum_i u_i^2} \sqrt{\sum_i v_i^2}} = \sqrt{\frac{\sum s_i^2 u_i^2}{\sum u_i^2}}$ by Cauchy-Schwarz inequality (the maximum is attained when $v_i \propto s_i u_i$). 
To estimate $\rho(f^*,P^+)$ for any $f^*$, the most straightforward way is to find the $g^*$ ths maximizes the fraction on the right-hand side. 
This is much more efficient than training a $d$-dimensional encoder.

Under this definition, the class of $(1-\epsilon)$-compatible tasks is defined as
\begin{equation}
\label{eqn:def-fep}
    \fep = \oset{f \in \lxp: \rho(f, P^+) \ge 1 - \epsilon},
\end{equation}
for any $\epsilon > 0$.
When $\epsilon$ is small, this is a class of compatible downstream tasks, and it can be shown that this class satisfies the two conditions aforementioned:
\begin{theorem}[Proof in \Cref{app:proof-two-conditions-fep}]
\label{thm:two-conditions-fep}
For any $f^* \in \fep$, there exists a $g^* \in \lap$ such that $f^*(x) = \E[g^*(A)|x]$, and $g^*$ satisfies
\begin{equation}
\label{eqn:thm-two-conditions-fep}
\underset{X \sim \px}{\E} \; \underset{A,A' \sim P^+(\cdot|X)}{\E} \brac{ \paren{g^*(A) - g^*(A')}^2 } \le 4\epsilon \norm{g^*}_{\pa}^2  .
\end{equation}
\end{theorem}

Next, we show why an encoder that learns the contexture is optimal.
First, we need to define ``optimal'', which involves the evaluation of an encoder.
The ultimate evaluation of $\Phi$ is the performance of the downstream predictor.
Let $f(x) = \mW \Phi(x) + \vb$ be the downstream linear probe.
The performance of $f$ can be measured by the mean squared error (MSE):
\begin{equation*}
    \err(f, f^*) = \norm{f - f^*}_{\px}^2 = \E_{X \sim \px}[(f(X) - f^*(X))^2] .
\end{equation*}
Let $\fPhi$ be the projection of $f^*$ onto the linear space spanned by $\phi_1,\cdots,\phi_d$. Then, when using a linear probe, the above error can be decomposed as
\begin{equation*}
    \err(f, f^*) = \underbrace{\norm{\fPhi - f^*}_{\px}^2}_{\textrm{Approximation error}} + \underbrace{\norm{f - \fPhi}_{\px}^2}_{\textrm{Estimation error}}  .
\end{equation*}

The embedding dimension $d$ controls the trade-off between the two errors.
If $d$ is larger, then the span of $\Phi$ will become larger, so the approximation error will be lower;
meanwhile, the downstream sample complexity will be higher, so the estimation error will increase.
There are two ways in which $\Phi$ affects the final prediction error.
First, the distance from $f^*$ to the span of $\Phi$  decides the approximation error.
Second, the smoothness of $\phi_1,\cdots,\phi_d$ affects the sample complexity of both pretraining and downstream.
The second part will be studied in \Cref{chap:bounds},
and this chapter evaluates $\Phi$ by the approximation error.
More specifically, $\Phi$ is evaluated on the class of compatible tasks $\fep$ by its worst-case approximation error.

\begin{definition}
\label{def:worst-case-app}
Let $\gF \subset \lxp$ be a function class where $f \in \gF \Rightarrow \alpha f \in \gF$ for all $\alpha \in \R$.
    The \textbf{worst-case approximation error} of $\Phi: \gX \rightarrow \R^d$ on $\gF$ is defined as
    \begin{align*}
        \err \paren{\Phi; \gF} & = \max_{f \in \fp, \; \norm{f}_{\px} = 1} \; \err \paren{\Phi, f} ; \\ 
        \text{where} \quad \err \paren{ \Phi, f} & = \E_{\Phi} \brac{\min_{\vw \in \R^d, \; b \in \R} \; \norm{\vw^{\top} \Phi + b - f}_\px^2}   .
    \end{align*}
\end{definition}

Here, $\E_{\Phi}$ is taken over the randomness of $\Phi$. When $\Phi$ is randomized, one first samples a deterministic $\Phi$ from the distribution, and then fits $\vw$ and $b$ accordingly.

The following result is one of the main results of the contexture theory.
It says two things.
First, if we know a priori the downstream task is compatible, then learning the contexture is the optimal thing to do, because it minimizes the worst-case approximation error.
Second, what if the task is incompatible?
We cannot argue that no encoder works for an incompatible task.
In a hypothetical scenario, if an oracle tells us the target function $f^*$ in advance, then we can set $\phi_1 = f^*$ to achieve perfect performance.
What we can argue is that for any low compatibility level and any encoder $\Phi$, there exists an $f^*$ on that level such that $\Phi$ is poor for $f^*$.

\begin{theorem}[Proof in \Cref{app:proof-top-d-optimal}]
\label{thm:top-d-optimal}
Suppose $1 - \epsilon \leq s_1$. 
For any $d$, among all $\Phi = [\phi_1,\cdots,\phi_d]$ where $\phi_i \in \lxp$ , $\Phi$ minimizes $\err(\Phi; \fep)$ if and only if it learns the contexture of $\tp$. The error is given by
\[
\min_{\Phi: \gX \rightarrow \R^d, \; \phi_i \in \lxp} \;  \err \paren{ \Phi; \gF_\epsilon(P^+) } = \frac{s_1^2 - (1 - \eps)^2}{s_1^2 - s_{d+1}^2} .
\]
Conversely, for any $d$-dimensional encoder $\Phi$ and any $\epsilon>0$, there exists $f \in \lxp$ such that $\rho(f,P^+) = 1-\epsilon$, and $\err(\Phi, f) \ge \frac{s_1^2 - (1 - \eps)^2}{s_1^2 - s_{d+1}^2}$.
\end{theorem}

\section{Intrinsic Evaluation of an Arbitrary Encoder}
\label{sec:induced-rkhs}

Given a context that is compatible with the task, the encoder that learns the contexture is optimal. 
Now what about an arbitrary encoder $\Phi$? Is it possible to bound its worst-case approximation error on the class of compatible tasks?
To derive such a bound, two key objects are necessary: the induced RKHS and the ratio trace.

Denote the range of $\tpstar$ by $R(\tpstar) = \sset{\tpstar f}{f \in \lxp}$.

\begin{definition}
\label{def:induced-rkhs}
The \textbf{induced RKHS} of $P^+$, denoted by $\hpp$, is the Hilbert space $R(\tpstar)$ with the inner product given by $\dotp{\tpstar f_1, \tpstar f_2}_{\hpp} = \dotp{f_1, f_2}_{\px}$.
\end{definition}

An alternative formula is that for any $h_1,h_2 \in \hpp$ where $h_1 = \sum u_i \nu_i$ and $h_2 = \sum v_i \nu_i$, there is $\dotp{h_1, h_2}_{\hpp} = \sum \frac{u_i v_i}{s_i^2}$.

\begin{proposition}
\label{prop:induced-rkhs}
The induced RKHS $\hpp$ has the following properties:
\begin{enumerate}[label=(\roman*)]
    \item $\ka$ is the reproducing kernel, such that $h(a) = \dotp{h, \ka(a, \cdot)}_{\hpp}$ for all $h \in \hpp$.
    \item $\hpp$ is isometric to $\sspan \oset{\mu_i: s_i > 0}$, which is a subspace of $\lxp$.
    \item $f^* \in \fep$ is equivalent to $h^* = \tpstar f^*$ satisfying the following \textbf{isometry property}:
    \begin{equation}
    \label{eqn:isometry-property}
    (1-\epsilon) \norm{\tilde{h}^*}_{\hpp} \le \norm{\tilde{h}^*}_{\pa} \le \norm{\tilde{h}^*}_{\hpp}  .
    \end{equation}
\end{enumerate}
\end{proposition}

\begin{proof}
For any $h \in \hpp$ where $h = \tpstar f$ and $f = \sum u_i \mu_i$, by \Cref{cor:spec-decompose} we have
\begin{equation*}
    \dotp{h, \ka(a,\cdot)}_{\hpp} = \dotp{\sum s_i u_i \nu_i, \sum s_i^2 \nu_i(a) \nu_i}_{\hpp} = \sum s_i u_i \nu_i(a) = h(a),
\end{equation*}
which proves (i). (ii) is obvious. Regarding (iii), recall that $f^* = \sum u_i \mu_i \in \fep$ is equivalent to $\sum_{i \ge 1} s_i^2 u_i^2 \ge (1-\epsilon)^2 \sum_{i \ge 1} u_i^2$, and this is $\norm{\tilde{h}^*}_{\pa} \ge (1-\epsilon) \norm{\tilde{h}^*}_{\hpp}$.
Meanwhile, it is obvious that $\norm{\tilde{h}^*}_{\pa} \le \norm{\tilde{h}^*}_{\hpp}$ always holds.
\end{proof}

\begin{definition}
\label{def:ratio-trace}
    Define covariance matrices $\mC_{\Phi} = \Cov_{\px}[\Phi]$, and $\mB_{\Phi} = \Cov_{\pa}[\tpstar \Phi]$. If $\mC_{\Phi}$ is invertible, then the \textbf{ratio trace} of $\Phi$ \wrt{} $P^+$ is defined as $\rt(\Phi; P^+) = \rt(\phi_1,\cdots,\phi_d; P^+) := \Tr(\mC_{\Phi}^{-1} \mB_{\Phi})$;
    otherwise, let $\Phi' = \brac{\phi_{i_1},\cdots,\phi_{i_t}}$ be the maximal linearly independent subset of $\brac{\phi_1,\cdots,\phi_d}$, and define the ratio trace of $\Phi$ the same as the ratio trace of $\Phi'$.
\end{definition}

The ratio trace of any $\Phi$ essentially measures how well $\Phi$ is aligned with the contexture of $P^+$.
Multiplying $\Phi$ by any invertible matrix does not change its ratio trace.
The matrices $\mC_\Phi$ and $\mB_\Phi$ here are the same as in \Cref{sec:extract-exact}.
If $\Phi$ learns the contexture, then its ratio trace is $s_1^2 + \cdots + s_d^2$, which can be easily shown by setting $\phi_i = \mu_i$.
In fact, this is the maximum ratio trace of any $d$-dimensional encoder.

\begin{lemma}
\label{lem:ratio-trace}
Suppose $\phi_1,\cdots,\phi_d$ are orthonormal and all have zero mean. Then, we have
\begin{equation*}
    \norm{\tpstar \phi_1}_{\pa}^2 + \cdots + \norm{\tpstar \phi_d}_{\pa}^2 \le s_1^2 + \cdots + s_d^2 .
\end{equation*}
\end{lemma}
\begin{proof}
    Let $\phi_i = \sum_{j \ge 1} q_{ij} \mu_j$ for $i \in [d]$.
    Then, $\mQ = (q_{ij})$ is a matrix with $d$ orthonormal rows and infinitely many columns.
    It is easy to see that the left-hand side is equal to $\Tr(\mQ \mD \mQ^{\top})$, where $\mD = \diag \oset{s_1^2, s_2^2, \cdots}$.
    Let $\vq_j$ be the $j$-th column of $\mQ$. For all $j \in [d]$, there is $\sum_{i=1}^j \vq_i^{\top} \vq_i \le j$; and for any $j > d$, there is $\sum_{i=1}^j \vq_i^{\top} \vq_i \le d$.
    Thus, using Abel transformation, we have
    \begin{equation*}
        \Tr(\mQ \mD \mQ^{\top}) = \Tr(\mD \mQ^{\top} \mQ) = \sum_{j=1}^{\infty} s_j^2 \vq_j^{\top} \vq_j = \sum_{j=1}^{\infty} \paren{\sum_{i=1}^j \vq_i^{\top} \vq_i } \paren{s_j^2 - s_{j+1}^2} \le \sum_{j=1}^d s_j^2 ,
    \end{equation*}
    as desired.
\end{proof}

The ratio trace induces a key quantity in the approximation error bound called the trace gap, which reflects the gap between $\Phi$ and the top-$d$ singular functions.
The larger the trace gap is, the larger the approximation error will be.
A simple definition is $s_1^2 + \cdots + s_{d+1}^2 - \rt(\Phi; P^+)$, whose lower bound $s_{d+1}^2$ can be achieved by the top-$d$ singular functions, the optimal encoder.
However, there is an issue with this definition.
For example, consider an encoder with $d = 1000$.
It learns the top-$10$ singular functions, but the other $990$ dimensions are complete noise that has zero contribution to $\rt(\Phi; P^+)$.
The approximation error of this encoder should be no higher than that of the top-$10$ singular functions, because adding more dimensions will never make the approximation error higher.
However, if $d$ becomes larger and $\rt(\Phi; P^+)$ stays the same, then $s_1^2+\cdots+s_{d+1}^2 - \rt(\Phi; P^+)$ will become larger, so this quantity does not correlate with the approximation error in this scenario.
The following definition fixes this issue.
\begin{definition}
\label{def:trace-gap}
For any linearly independent $f_1,\cdots,f_{d'} \in \lxp$, denote $F = [f_1,\cdots,f_{d'}]$, $\mC_F = \Cov_{\px}[F]$, and $\mB_F = \Cov_{\pa}[F]$.
The \textbf{trace gap} of $\Phi$ \wrt{} $P^+$ is defined as
\begin{equation*}
    \tg(\Phi; P^+) := \inf_{d' \le d} \; \inf_{f_1,\cdots,f_{d'}} \; \oset{s_1^2 + \cdots + s_{d'+1}^2 - \Tr(\mC_F^{-1} \mB_F)} .
\end{equation*}
\end{definition}

Obviously, this definition of trace gap is upper bounded by $s_1^2 + \cdots + s_{d+1}^2 - \rt(\Phi; P^+)$.
It solves the issue in the previous example, because having completely noisy dimensions does not affect the trace gap.
The following result bounds the approximation error.

\begin{theorem}
\label{thm:approx-error-bound}
Suppose $\tg(\Phi; P^+) < s_1^2$, and $\epsilon > 1 - s_1$. Then,
\begin{equation*}
    \err(\Phi; \fep) \le \frac{s_1^2 - (1-\epsilon)^2 + s_1 \tg(\Phi; P^+)}{s_1^2 - \tg(\Phi; P^+)^2}  .
\end{equation*}
\end{theorem}
\begin{remark}
    This bound is fairly tight. If $\Phi$ learns the contexture, then by \Cref{thm:top-d-optimal} we have $\err(\Phi; \fep) = \frac{s_1^2 - (1-\epsilon)^2}{s_1^2 - s_{d+1}^2}$, and $\tg(\Phi; P^+) = s_{d+1}$.
    Compared to this exact formula, the above upper bound only has an extra $s_1 \tg(\Phi; P^+)$ term in the numerator.
\end{remark}
\begin{proof}
Let $f_1,\cdots,f_{d'}$ be the functions that minimize $s_1^2 + \cdots + s_{d'+1}^2 - \Tr(\mC_F^{-1} \mB_F)$. Without loss of generality, assume that $f_1,\cdots,f_{d'}$ have zero mean and are orthonormal. Let $\gF = \sspan \oset{f_1,\cdots,f_{d'}}$, and $\gH = \sspan \oset{\tpstar f_1,\cdots,\tpstar f_{d'}}$.
For any $f \in \fep$ with $\norm{f}_{\px} = 1$, let $h = \tpstar f \in \hpp$, and let $f_{\gF}$ be the projection of $f$ onto $\gF$.
Since $\err(\Phi; \fep)$ is upper bounded by $\norm{f - f_{\gF}}_{\px}^2$, it suffices to show that $\norm{f - f_{\gF}}_{\px}^2$ is upper bounded by the right-hand side.

Let $\alpha^2 = \norm{f_{\gF}}_{\px}^2$, and $\beta^2 = \norm{f - f_{\gF}}_{\px}^2$, where $\alpha$ and $\beta$ are non-negative.
Then, $\alpha^2 + \beta^2 = \norm{f}_{\px}^2 = 1 = \norm{h}_{\hpp}^2$.
The isometry property says that $(1-\epsilon)^2(\alpha^2+\beta^2) \le \norm{h}_{\pa}^2$.
Let $f - f_{\gF} = \beta f_0$ where $\norm{f_0}_{\px} = 1$.
Let $h_{\gF} = \tpstar h_{\gF}$ and $h_0 = \tpstar f_0$.
Then, we have $\norm{h_{\gF}}_{\pa}^2 \le s_1^2 \norm{f_{\gF}}_{\px}^2 = s_1^2 \alpha^2$.
Meanwhile, since $f_0$ is orthogonal to $f_1,\cdots,f_{d'}$, by \Cref{lem:ratio-trace} we have $\norm{\tpstar f_0}_{\pa}^2 + \norm{\tpstar f_1}_{\pa}^2 + \cdots + \norm{\tpstar f_{d'}}_{\pa}^2 \le s_1^2 + \cdots + s_{d'+1}^2$, which implies that $\norm{\tpstar f_0}_{\pa}^2 \le s_1^2 + \cdots + s_{d'+1}^2 - \Tr(\mC_{F}^{-1} \mB_{F}^{-1})$.
Let $\tau = \tg(\Phi; P^+)$. Then, we have
\begin{equation*}
    \norm{h}_{\pa}^2 = \norm{h_{\gF} + \beta h_0}_{\pa}^2 \le \norm{h_{\gF}}_{\pa}^2 + \beta^2 \norm{h_0}_{\pa}^2 + 2 \beta \norm{h_{\gF}}_{\pa} \norm{h_0}_{\pa} \le s_1^2 \alpha^2 + \tau^2 \beta^2 + 2 s_1 \tau \alpha \beta .
\end{equation*}
Thus, we have $(1-\epsilon)^2(\alpha^2+\beta^2) \le s_1^2\alpha^2 +\tau^2 \beta^2 + 2 s_1 \tau \alpha \beta$, which implies that $(s_1^2 - \tau^2) \beta ^2 \le [s_1^2 - (1-\epsilon)^2] (\alpha^2 + \beta^2) + 2 s_1 \tau \alpha \beta \le [s_1^2 - (1-\epsilon)^2 + s_1 \tau] (\alpha^2 + \beta^2) $, as desired.
\end{proof}

\paragraph{Connection to Fisher discriminant analysis.}
Fisher discriminant analysis \cite{mika1999fisher,baudat2000generalized,liu2004improving}, or more generally linear discriminant analysis (LDA), is a classical method of learning linear classifiers in statistics.
Here we show that Fisher discriminant analysis has a strong connection to the contexture theory.
Suppose $\gX \subseteq \R^{\dx}$.
Fisher discriminant analysis defines the following \textbf{between-class covariance matrix} $\mS_B \in \R^{\dx \times \dx}$ and \textbf{within-class covariance matrix} $\mS_W \in \R^{\dx \times \dx}$:
\begin{align*}
    \mS_B & = \iint \oset{ \paren{ \E \sbrac{X}{A=a_1} - \E \sbrac{X}{A=a_2} } \paren{  \E \sbrac{X}{A=a_1} - \E \sbrac{X}{A=a_2}}^{\top}} ; \\ 
    \mS_W & = \int \E_{P^+} \sbrac{\paren{X - \E \sbrac{X}{A=a}} \paren{X - \E \sbrac{X}{A=a}}^{\top}}{A=a} d\pa (a).
\end{align*}
In the original formulation of Fisher discriminant analysis, $A$ is the label of $X$.
Here we extend it to a general context variable.
Consider a linear encoder $\Phi(x) = \mW x$, where $\mW \in \R^{d \times \dx}$.
Then, one solves the following optimization problem to find $\mW$:
\begin{equation*}
    \underset{\mW \in \R^{d \times \dx}}{\text{maximize}} \; J(\mW) = \Tr \brac{ \paren{\mW \mS_B \mW^{\top}} \paren{\mW \mS_W \mW^{\top}}^{-1} } \quad \text{s.t.} \quad \mW \mS_W \mW^{\top} \text{ is invertible} .
\end{equation*}
Here, $J(\mW)$ is called the \textbf{Fisher discriminant}.
Define $\Psi(a) = \E_{P^+}[\mW X |A=a]$. Then, we can see that
\begin{align*}
    \mW \mS_B \mW^{\top} & = \iint \paren{ \Psi(a_1) - \Psi(a_2) } \paren{ \Psi(a_1) - \Psi(a_2) }^{\top} d \pa (a_1) d \pa (a_2) ; \\ 
    \mW \mS_W \mW^{\top} & = \int \E_{P^+} \sbrac{\paren{\Phi(X) - \Psi(a)}\paren{\Phi(X) - \Psi(a)}^{\top}}{A = a} d \pa (a) .
\end{align*}
Let $\mC_{\Phi} = \E[\tPhi(X)\tPhi(X)^{\top}]$ and $\mB_{\Phi} = \E[\tPsi(A) \tPsi(A)^{\top}]$.
Then, we have
\begin{align*}
    \mW \mS_B \mW^{\top} & = 2 \oset{ \E \brac{\Psi(A) \Psi(A)^{\top}} - \bPsi \bPsi^{\top} } = 2 \E \brac{\tPsi(A) \tPsi(A)^{\top}} = 2 \mB_{\Phi} ; \\ 
    \mW \mS_W \mW^{\top} & = \int \E_{P^+} \sbrac{\Phi(X) \Phi(X)^{\top} - \Psi(a) \Psi(a)^{\top}}{A = a} d \pa (a)  \\ 
    & = \E\brac{ \Phi(X) \Phi(X)^{\top} } - \E \brac{\Psi(A) \Psi(A)^{\top}} \\ 
    & = \E\brac{ \tPhi(X) \tPhi(X)^{\top} } - \E \brac{\tPsi(A) \tPsi(A)^{\top}} = \mC_{\Phi} - \mB_{\Phi}  .
\end{align*}
Therefore, $J(\mW) = 2 \Tr[(\mC_{\Phi} - \mB_{\Phi})^{-1} \mB_{\Phi}]$, which is very similar to the ratio trace defined in \Cref{def:ratio-trace}.
Recall that an encoder that learns the contexture maximizes the ratio trace.
A well-known result is that $J(\mW)$ is maximized when $\mW$ consists of the top-$d$ eigenvectors of $\mS_W^{-1} \mS_B$.
Hence, Fisher discriminant analysis is almost equivalent to contexture learning under the constraint that the encoder must be linear.

\section{Evaluating Context Usefulness}
\label{sec:context-usefulness}
The previous chapter argued that better contexts are essential to further improve foundation models, but how to create better contexts is a challenging open problem.
This section studies an easier problem---how to evaluate a context before pretraining.
We say that a context is \textbf{useful} for a downstream task, if it can lead to an encoder with good performance on this task.
However, since we might not know the task at pretrain time, we can only predict if the context is useful \textit{in general}.
Solving this problem is a prerequisite, because if we cannot even decide whether a context is good or not, then there will be no way for us to create better contexts.

Evaluating a context is more difficult than evaluating an encoder.
When evaluating an encoder, we can assume that the task is known to be compatible with the context.
However, when evaluating a context, we cannot make this assumption.
Instead, a context is better if it is compatible with more tasks, because such a context can lead to more transferable encoders, and our evaluation should take this into consideration.

Given a context and a dataset, the only things we can use are the singular values and singular functions of the context.
However, estimating the singular functions is as hard as pretraining an encoder, whereas the singular values can be efficiently estimated using the post-hoc approach in \Cref{sec:extract-exact} with a small subset of samples.
Therefore, a metric will be more ideal if it only uses the singular values (the spectrum).

In this section, we show that it is possible to evaluate the usefulness of a context with only its singular values.
This might seem counter-intuitive, because suppose the encoder learns the contexture, then only the singular functions will affect the performance of the encoder.
When $\mu_1,\cdots,\mu_d$ are fixed, the space spanned by $\phi_1,\cdots,\phi_d$ is also fixed and is independent of $s_1,\cdots,s_d$.
The reason why this is possible is that the singular values and the singular functions are intrinsically connected.
Recall that $\kx(x,x') = \sum s_i^2 \mu_i(x) \mu_i(x')$, and that it must satisfy $\kx(x,x') \ge 0$ for all $x,x'$, which is quite a strict constraint.
Therefore, when $\mu_1,\mu_2,\cdots$ are fixed, one cannot choose $s_1^2,\cdots,s_d^2$ arbitrarily.
Conversely, the singular values limit the possible choices of the singular functions.

In this section, we propose a metric for evaluating context usefulness, and the metric only depends on the singular values.
Note that there does not exist a universal metric, because whether the context is useful on a task or not depends on their compatibility, which cannot be estimated without the knowledge of the task.
However, our experiments show that the proposed metric generally works well on real datasets.

\paragraph{Qualitative analysis.}
In \Cref{sec:spectral-properties} we showed that the shape of the spectrum depends on the association strength between $X$ and $A$, also called the association of the context.
The singular values decay slower if the association is stronger.
Here we show that the decay rate of the singular values has a great impact on the usefulness of the context.

The central argument is: \textbf{A useful context should have a moderate association}.
To get an intuition, consider the two extreme cases in \Cref{sec:spectral-properties}:
(i) $A$ is independent of $X$, then there is only one positive singular value;
(ii) $A = X$, then all singular values are $1$. 
Both contexts are clearly useless because they provide no additional information.
In what follows, we qualitatively explain why a context is not very useful if its association is too strong or too weak.

Recall that any compatible task $f^* = \sum u_i \mu_i \in \fep$ as defined in \Cref{eqn:def-fep} needs to satisfy $\sum_{i \ge 1} s_i^2 u_i^2 \ge (1-\epsilon)^2 \sum_{i \ge 1} u_i^2$.
This is easier to satisfy if $s_i$ are large.
Thus, if the association is too weak and the singular values decay too fast, then $\fep$ will be a very small set.
Consequently, very few tasks will be compatible with the context, so an encoder trained on the context will not be transferable to various tasks.

On the other hand, when the association is too strong, there are two consequences.
First, more singular functions have large singular values, and to learn all of them, one needs to use a larger $d$, which makes the sample complexity of learning the downstream linear predictor higher.
Second, the singular functions with large singular values become less smooth, so learning them requires more pretraining samples.
In \Cref{chap:bounds} we will define the context complexity that quantitatively measures the smoothness of the top singular functions, and we will show that the context complexity is higher when the context has a stronger association.

\paragraph{Quantitative metric.}

We now propose a metric that quantitatively measures the general usefulness of a context.
The metric only depends on the singular values of the context.
The metric assumes that the pretraining method aims to learn the contexture of the context, and it does not work for an arbitrary encoder.

Our metric is defined as
\begin{equation}
\label{eqn:usefulness-metric}
    \tau_d = \frac{1}{1 - s_{d+1}^2} + \beta \frac{\sum_{i=1}^{d} s_i^2}{\sum_{i=1}^{d_0} s_i^2} ,
\end{equation}
where $\beta > 0$ is a parameter, and $d_0$ is the maximum $d$ we consider. Typically $d_0$ ranges from $512$ to $8192$.
We choose $\beta = 1$ and $d_0 = 512$ in our experiments.
$\tau_d$ is a proxy of the prediction error when the embedding dimension is $d$.
Thus, the $d$ that minimizes $\tau_d$ can be viewed as the optimal embedding dimension predicted by the metric, and $\tau$ evaluates the context when $d$ is chosen optimally.
Since this metric only depends on the singular values, it can be efficiently estimated using the post-hoc approach in \Cref{sec:extract-exact}, with which we can estimate the spectrum using a subset of $\Theta(d_0 \log d_0)$ samples.

This metric is derived in the following way. Let the target function be $f^* = f_0 + f_1$, where $\dotp{f_0,f_1}_{\px} = 0$, $f_0$ is not compatible with the context, and $f_1$ is compatible with the context.
Then, the prediction error can be decomposed into three components:
\begin{enumerate}[label=(\roman*)]
    \item The approximation error of $f_1$
    \item The approximation error of $f_0$
    \item The estimation error
\end{enumerate}

By \Cref{thm:top-d-optimal}, component (i) can be bounded by $\frac{s_1^2 - (1-\epsilon)^2}{s_1^2 - s_{d+1}^2}$, and we simplify this to the first term because $s_1$ is very close to 1 in most real cases, and the numerator is a constant.
Component (ii) is smaller if the context has a stronger association; thus, it should be negatively correlated with $\sum_{i=1}^{d_0} s_i^2$.
Component (iii) is larger if the context has a stronger association, or if $d$ is large.
Based on the result in \Cref{chap:bounds}, it is positively correlated with $\sum_{i=1}^{d} s_i^2$.
The second term of the metric combines these two components, and it is designed to be bounded by $1$.

\paragraph{Compared to previous metrics.}
Some previously proposed metrics are also based on the decay rate of the spectrum. \cite{NEURIPS2022_70596d70} proposed a metric based on the eigenvalues of $\Phi \Phi^{\top}$ for a particular pretrained encoder $\Phi$; that is, they use the $\lambda$ that satisfies $\langle \Phi, \Phi\rangle_{\px} f = \lambda f$ for some $f \neq \vzero$.
Here, $\dotp{\Phi, \Phi}_{\px}$ is the covariance matrix, assuming that $\Phi$ is centered.
In contrast, our metric is based on the general eigenvalues satisfying $\dotp{\Phi, \tp \Phi}_{\px} f = s^2 \dotp{\Phi, \Phi}_{\px} f$ for some $f \neq \vzero$, as discussed in \Cref{sec:extract-exact}.
These two sets of eigenvalues are fundamentally different.
The eigenvalues $s_i^2$ we use are invariant under invertible linear transformations on $\Phi$, while $\lambda_i$ are not.
This suggests that our metric is more desirable since invertible linear transformations on $\Phi$ do not affect the performance of the downstream linear probe.

\begin{figure}[t]
    \centering
    \begin{tikzpicture}
\begin{axis}[
name=plot1,
 height = .26\linewidth,
    width = .26\linewidth,
    axis lines=left,  
    axis line style={-latex},  
    xlabel={$i$},
    ylabel={$s_i^2$},
    ylabel style={
        at={(0.13,0.8)}, 
        anchor=south,  
        rotate=270  
    },
    xlabel style = {
        at={(0.9,0.02)},
        anchor=south,
    },
title style={at={(0.5,-0.3)}, anchor=north},
    xmin=0, xmax=250,
    ymin=0, ymax=1,
    xtick distance = 100,
    ytick distance = 0.5,
]
    \addplot[thick] table [
        x expr=\coordindex,  
        y=value,            
        col sep=comma       
    ] {data/ag1.csv};
\end{axis}

\begin{axis}[
name=plot2,
at={($(plot1.south east)+(70,0)$)},
 height = .26\linewidth,
    width = .26\linewidth,
    axis lines=left,  
    axis line style={-latex},  
    xlabel={$i$},
    ylabel={$s_i^2$},
    ylabel style={
        at={(0.13,0.8)}, 
        anchor=south,  
        rotate=270  
    },
    xlabel style = {
        at={(0.9,0.02)},
        anchor=south,
    },
title style={at={(0.5,-0.3)}, anchor=north},
    xmin=0, xmax=250,
    ymin=0, ymax=1,
    xtick distance = 100,
    ytick distance = 0.5,
]
    \addplot[thick] table [
        x expr=\coordindex,  
        y=value,            
        col sep=comma       
    ] {data/ag3.csv};
\end{axis}

\begin{axis}[
name=plot3,
at={($(plot2.south east)+(70,0)$)},
 height = .26\linewidth,
    width = .26\linewidth,
    axis lines=left,  
    axis line style={-latex},  
    xlabel={$i$},
    ylabel={$s_i^2$},
    ylabel style={
        at={(0.13,0.8)}, 
        anchor=south,  
        rotate=270  
    },
    xlabel style = {
        at={(0.9,0.02)},
        anchor=south,
    },
title style={at={(0.5,-0.3)}, anchor=north},
    xmin=0, xmax=250,
    ymin=0, ymax=1,
    xtick distance = 100,
    ytick distance = 0.5,
]
    \addplot[thick] table [
        x expr=\coordindex,  
        y=value,            
        col sep=comma       
    ] {data/ag5.csv};
\end{axis}

\begin{axis}[
name=plot4,
at={($(plot1.south west)+(0,-100)$)},
 height = .26\linewidth,
    width = .26\linewidth,
    axis lines=left,  
    axis line style={-latex},  
    xlabel={$d$},
    xlabel style = {
        at={(0.95,0)},
        anchor=south,
    },
title style={at={(0.5,-0.3)}, anchor=north},
title={(a) Weak association},
    xmin=0, xmax=250,
    ymin=0.4, ymax=0.8,
    xtick distance = 100,
    ytick distance = 0.1,
]
    \addplot[thick] table [
        x expr=\coordindex,  
        y expr={min(\thisrow{tau} / 4, 1.4)},            
        col sep=comma       
    ] {data/ag2.csv};
    \addplot[thick, dashed] table [
        x expr=\coordindex,  
        y=mse,            
        col sep=comma       
    ] {data/ag2.csv};
\end{axis}

\begin{axis}[
name=plot5,
at={($(plot4.south east)+(70,0)$)},
 height = .26\linewidth,
    width = .26\linewidth,
    axis lines=left,  
    axis line style={-latex},  
    xlabel={$d$},
    xlabel style = {
        at={(0.95,0)},
        anchor=south,
    },
title style={at={(0.5,-0.3)}, anchor=north},
title={(b) Moderate association},
    xmin=0, xmax=250,
    ymin=0.4, ymax=0.8,
    xtick distance = 100,
    ytick distance = 0.1,
]
    \addplot[thick] table [
        x expr=\coordindex,  
        y expr={min(\thisrow{tau} / 4, 1.4)},     
        col sep=comma       
    ] {data/ag4.csv};
    \addplot[thick, dashed] table [
        x expr=\coordindex,  
        y=mse,    
        col sep=comma       
    ] {data/ag4.csv};
\end{axis}

\begin{axis}[
name=plot6,
at={($(plot5.south east)+(70,0)$)},
 height = .26\linewidth,
    width = .26\linewidth,
    axis lines=left,  
    axis line style={-latex},  
    xlabel={$d$},
    xlabel style = {
        at={(0.95,0)},
        anchor=south,
    },
title style={at={(0.5,-0.3)}, anchor=north},
title={(c) Strong association},
    xmin=0, xmax=250,
    ymin=0.4, ymax=0.8,
    xtick distance = 100,
    ytick distance = 0.1,
]
    \addplot[thick] table [
        x expr=\coordindex,  
        y expr={min(\thisrow{tau} / 6, 1.4)},            
        col sep=comma       
    ] {data/ag6.csv};
    \addplot[thick, dashed] table [
        x expr=\coordindex,  
        y=mse,            
        col sep=comma       
    ] {data/ag6.csv};
\end{axis}
\end{tikzpicture}
    \caption{Metric illustration on \texttt{abalone}. \textbf{Top row:} context spectra. \textbf{Bottom row:} solid curves are $\tau_d$ divided by $6$; dashed curves are the actual downstream prediction error. We divide $\tau_d$ by $6$ to fit it in the same plot.}
    \label{fig:taud}
\end{figure}
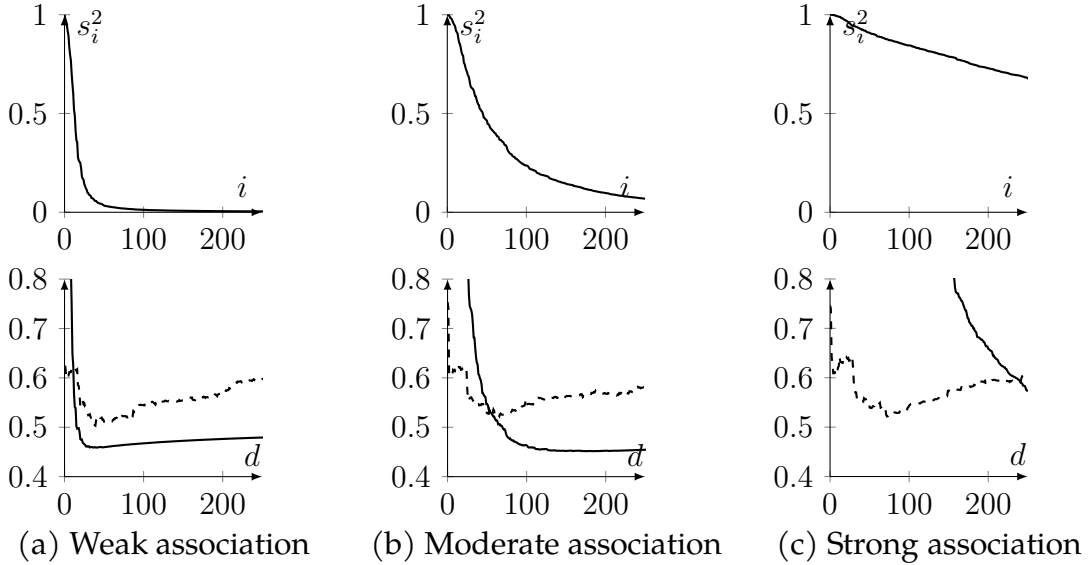

\begin{figure}[t]
    \centering
    \begin{tikzpicture}
\begin{axis}[
name=plot1,
 height = .26\linewidth,
    width = .26\linewidth,
    axis lines=left,  
    axis line style={-latex},  
    xlabel={$i$},
    ylabel={$s_i^2$},
    ylabel style={
        at={(0.13,0.8)}, 
        anchor=south,  
        rotate=270  
    },
    xlabel style = {
        at={(0.9,0.02)},
        anchor=south,
    },
title style={at={(0.5,-0.3)}, anchor=north},
    xmin=0, xmax=250,
    ymin=0, ymax=1,
    xtick distance = 100,
    ytick distance = 0.5,
]
    \addplot[thick] table [
        x expr=\coordindex,  
        y=eig,            
        col sep=comma       
    ] {data/mnist-05.csv};
\end{axis}

\begin{axis}[
name=plot2,
at={($(plot1.south east)+(70,0)$)},
 height = .26\linewidth,
    width = .26\linewidth,
    axis lines=left,  
    axis line style={-latex},  
    xlabel={$i$},
    ylabel={$s_i^2$},
    ylabel style={
        at={(0.13,0.8)}, 
        anchor=south,  
        rotate=270  
    },
    xlabel style = {
        at={(0.9,0.02)},
        anchor=south,
    },
title style={at={(0.5,-0.3)}, anchor=north},
    xmin=0, xmax=250,
    ymin=0, ymax=1,
    xtick distance = 100,
    ytick distance = 0.5,
]
    \addplot[thick] table [
        x expr=\coordindex,  
        y=eig,            
        col sep=comma       
    ] {data/mnist-02.csv};
\end{axis}

\begin{axis}[
name=plot3,
at={($(plot2.south east)+(70,0)$)},
 height = .26\linewidth,
    width = .26\linewidth,
    axis lines=left,  
    axis line style={-latex},  
    xlabel={$i$},
    ylabel={$s_i^2$},
    ylabel style={
        at={(0.13,0.8)}, 
        anchor=south,  
        rotate=270  
    },
    xlabel style = {
        at={(0.9,0.02)},
        anchor=south,
    },
title style={at={(0.5,-0.3)}, anchor=north},
    xmin=0, xmax=250,
    ymin=0, ymax=1,
    xtick distance = 100,
    ytick distance = 0.5,
]
    \addplot[thick] table [
        x expr=\coordindex,  
        y=eig,            
        col sep=comma       
    ] {data/mnist-005.csv};
\end{axis}

\begin{axis}[
name=plot4,
at={($(plot1.south west)+(0,-100)$)},
 height = .26\linewidth,
    width = .26\linewidth,
    axis lines=left,  
    axis line style={-latex},  
    xlabel={$d$},
    xlabel style = {
        at={(0.95,0)},
        anchor=south,
    },
title style={at={(0.5,-0.3)}, anchor=north},
title={(a) Weak association},
    xmin=0, xmax=250,
    ymin=0.2, ymax=0.8,
    xtick distance = 100,
    ytick distance = 0.1,
]
    \addplot[thick] table [
        x expr=\coordindex,  
        y expr={min(\thisrow{taud} / 6, 1.4)},            
        col sep=comma       
    ] {data/mnist-05.csv};
    \addplot[thick, dashed] table [
        x expr=\coordindex,  
        y=mse,            
        col sep=comma       
    ] {data/mnist-05.csv};
\end{axis}

\begin{axis}[
name=plot5,
at={($(plot4.south east)+(70,0)$)},
 height = .26\linewidth,
    width = .26\linewidth,
    axis lines=left,  
    axis line style={-latex},  
    xlabel={$d$},
    xlabel style = {
        at={(0.95,0)},
        anchor=south,
    },
title style={at={(0.5,-0.3)}, anchor=north},
title={(b) Moderate association},
    xmin=0, xmax=250,
    ymin=0.2, ymax=0.8,
    xtick distance = 100,
    ytick distance = 0.1,
]
    \addplot[thick] table [
        x expr=\coordindex,  
        y expr={min(\thisrow{taud} / 6, 1.4)},     
        col sep=comma       
    ] {data/mnist-02.csv};
    \addplot[thick, dashed] table [
        x expr=\coordindex,  
        y=mse,    
        col sep=comma       
    ] {data/mnist-02.csv};
\end{axis}

\begin{axis}[
name=plot6,
at={($(plot5.south east)+(70,0)$)},
 height = .26\linewidth,
    width = .26\linewidth,
    axis lines=left,  
    axis line style={-latex},  
    xlabel={$d$},
    xlabel style = {
        at={(0.95,0)},
        anchor=south,
    },
title style={at={(0.5,-0.3)}, anchor=north},
title={(c) Strong association},
    xmin=0, xmax=250,
    ymin=0.2, ymax=0.8,
    xtick distance = 100,
    ytick distance = 0.1,
]
    \addplot[thick] table [
        x expr=\coordindex,  
        y expr={min(\thisrow{taud} / 6, 1.4)},            
        col sep=comma       
    ] {data/mnist-005.csv};
    \addplot[thick, dashed] table [
        x expr=\coordindex,  
        y=mse,            
        col sep=comma       
    ] {data/mnist-005.csv};
\end{axis}
\end{tikzpicture}
\caption{Metric illustration on \texttt{MNIST}, similar to \Cref{fig:taud}.}
    \label{fig:taud-mnist}
\end{figure}

Now we empirically examine $\tau_d$ on the two datasets.
First, we apply the metric to the \texttt{abalone} dataset and use KNN as the context, similar to \Cref{sec:scaling-laws}.
We adjust the association of the context by changing $K$.
In particular, we choose $K=150$ (weak), $K = 30$ (moderate) and $K = 5$ (strong).
We obtain the exact eigenvalues and eigenfunctions of $\tkx$ using kernel PCA.
In \Cref{fig:taud}, we plot the spectra of the three contexts in the top row.
Then, in the bottom row, we compare $\tau_d$ against the prediction error of the linear probe under different $d$.
We can see that when the association is weak or moderate, $\tau_d$ first decreases and then increases, which tracks the actual error.
However, when the association is too strong, $\tau_d$ monotonically decreases with $d$, and it cannot track the actual error.

Second, we apply the metric to the \texttt{MNIST} dataset.
The context is random cropping with crop ratio $\alpha$.
We adjust the association of the context by changing $\alpha$.
In particular, we choose $\alpha=0.5$ (weak), $\alpha=0.2$ (moderate) and $\alpha=0.05$ (strong).
Since kernel PCA is not scalable to datasets as large as MNIST, we instead train a neural network.
Specifically, we train a LeNet \cite{lecun1998gradient} using the non-contrastive learning objective formulated earlier, and the AdamW optimizer.
Then, we estimate the top eigenvalues using the post-hoc approach in \Cref{sec:extract-exact}.
The downstream task is a binary classification task---whether the digit is greater than 4.
After pretraining, a linear probe is fit on top of $\Phi$ using ridge regression.
The result is plotted in \Cref{fig:taud-mnist}.

\begin{figure}[!t]
    \centering
\begin{tikzpicture}
\begin{axis}[
name=plot1,
title={(a) \texttt{abalone}},
 height = .3\linewidth,
    width = .38\linewidth,
    axis lines=left,  
    axis line style={-latex},  
    xlabel={$i$},
    ylabel={Cosine similarity},
    xlabel style = {
        at={(0.95,0.08)},
        anchor=south,
    },
title style={at={(0.5,-0.3)}, anchor=north},
    xmin=-5, xmax=150,
    ymin=0, ymax=0.5,
    xtick distance = 50,
    ytick distance = 0.1,
]
    \addplot[thick] table [
        x expr=\coordindex,  
        y=ui,            
        col sep=comma       
    ] {data/abalone-ui.csv};
\end{axis}

\begin{axis}[
name=plot2,
title={(b) \texttt{MNIST}},
at={($(plot1.south east)+(80,0)$)},
 height = .3\linewidth,
    width = .38\linewidth,
    axis lines=left,  
    axis line style={-latex},  
    xlabel={$i$},
    xlabel style = {
        at={(0.95,0.08)},
        anchor=south,
    },
title style={at={(0.5,-0.3)}, anchor=north},
    xmin=-5, xmax=150,
    ymin=0, ymax=0.5,
    xtick distance = 50,
    ytick distance = 0.1,
]
    \addplot[thick] table [
        x expr=\coordindex,  
        y=ui,            
        col sep=comma       
    ] {data/mnist-05.csv};
\end{axis}

\end{tikzpicture}
    \caption{Comparison of the downstream task between \texttt{abalone} and \texttt{MNIST}.}
    \label{fig:mnist-compare}
\end{figure}
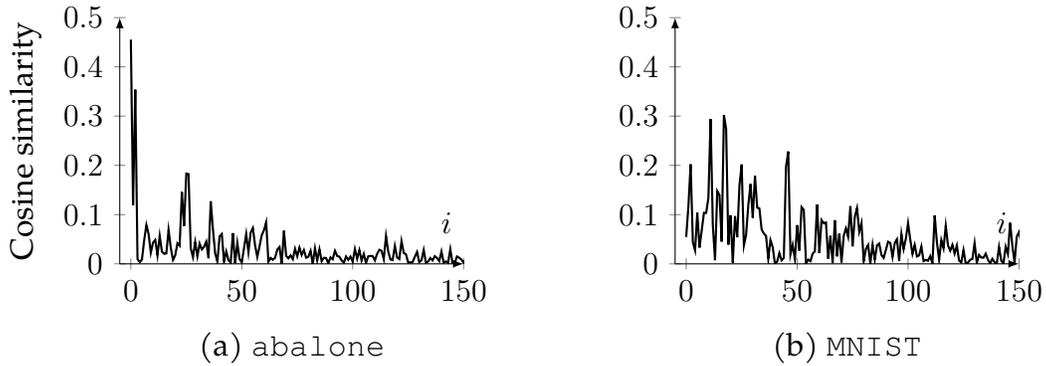

From \Cref{fig:taud-mnist}, we can see that when the association is not too strong, $\tau_d$ first decreases and then increases, similar to \Cref{fig:taud}.
However, on \texttt{MNIST}, the downstream error monotonically decreases with $d$, unlike \texttt{abalone}.
This disparity is due to the difference between the two downstream tasks.
To demonstrate this, in \Cref{fig:mnist-compare} we plot the cosine similarity between the target function $f^*$ and the estimated $i$-th eigenfunction on the two datasets.
We can see that the variance of $f^*$ on \texttt{abalone} is mostly concentrated on the top-$5$ eigenfunctions, with the first cosine similarity being almost $0.5$.
In contrast, the variance of $f^*$ on \texttt{MNIST} is more scattered, and the cosine similarity is still close to $0.1$ for the $150$-th eigenfunction.
Consequently, having a large $d$ on \texttt{abalone} will have a little impact on the approximation error but will increase the estimation error significantly.
On the other hand, having a larger $d$ on \texttt{MNIST} will decrease the approximation error more than it increases the estimation error, which is why the total error monotonically decreases with $d$.

The takeaway from this experiment is that although in general a context with a moderate association is good, in reality it still depends on the actual downstream task.
For example, on \texttt{abalone} the weakest context actually leads to the lowest error, because the variance of $f^*$ is concentrated on the top-$5$ eigenfunctions.
On the other hand, on \texttt{MNIST} the strongest context leads to the lowest error, because the variance of $f^*$ is scattered among a lot of features, and a stronger association allows more features to be discovered.
Hence, no evaluation metric would universally work for all contexts and downstream tasks, but a metric would still be useful if it correlates well with the actual error in most scenarios, and thus can provide insights into choosing the right context and the right hyperparameters, such as the mask or crop ratio.

\paragraph{Empirical verification of the proposed metric.}
Although our metric is derived from the decomposition of the prediction error, it is still heuristic and is not an upper bound of the actual error. Here we show that the metric correlates well with the actual error on many real datasets. 
Therefore, our metric is useful as it can help practitioners to select among various pretraining methods or choose the hyperparameters efficiently.

We use 28 real classification and regression datasets from OpenML that are widely used in machine learning research.
Each dataset is randomly split into a pretrain set, a labeled downstream training set, and a downstream test set by 70\%-15\%-15\%.
For each $d$, we obtain the top-$d$ eigenfunctions of $\tkx$ via kernel PCA, and then fit a linear probe on top of it using ridge regression.
Then, we select the best $d^*$ that achieves the lowest test mean squared error, which is denoted by $\err_{d^*}$.
The correlation between $\tau$ and $\err_{d^*}$ is reported.
The following four types of contexts are used in the experiment.
\begin{itemize}
    \item RBF kernels: $k(x,a) = \exp(-\gamma \norm{x-a}^2)$. Define $P^+(a|x) \propto k(x,a)$ for each $x$.
    \item KNN: $P^+(a|x) = K^{-1}$ if $a$ is a KNN of $x$; otherwise, it is $0$.
    \item RBF $\star$ Masking: First randomly mask 20\% of the features, and then apply RBF kernels to the other features. The $\tkx$ of this context can be estimated as follows: first randomly draw $50$ masks, and then compute their average $\tkx$.
    \item KNN $\star$ Masking: 20\% random masking and then apply KNN.
\end{itemize}

For each of these contexts, $\gA = \gX$.
For each type, 35 contexts are obtained by adjusting the $\gamma$ for RBF kernels, and $K$ for KNN. The association between $X$ and $A$ for these 35 contexts are different, and the experiment makes sure that the contexts in every type range from very weak to very strong association.
``$\star$ Masking'' here means the convolution with masking, as mentioned in \Cref{sec:learn-mixture}.
We do not use masking alone because the dual kernel of masking is hard to estimate.

\begin{table}[t]
    \centering
\begin{tabular}{llll|ll}
\toprule
Dataset & Size ($\uparrow$) & \#Feature & Type & Pearson  & Distribution  \\
\midrule
credit-approval    & 690   & 15   & Cls & 0.583  & 0.683  \\ 
breast-w           & 699   & 9    & Cls & 0.072  & 0.255  \\ 
diabetes           & 768   & 8    & Cls & 0.737  & 0.740  \\ 
solar\_flare       & 1066  & 10   & Reg & 0.019  & 0.262  \\ 
Moneyball          & 1232  & 14   & Reg & 0.680  & 0.650  \\ 
yeast              & 1269  & 8    & Cls & 0.221  & 0.256  \\ 
cmc                & 1473  & 9    & Cls & 0.867  & 0.860  \\ 
Wine               & 1599  & 11   & Reg & -0.084 & 0.212  \\ 
scene              & 2407  & 299  & Cls & 0.608  & 0.685  \\ 
dna                & 3186  & 180  & Cls & 0.881  & 0.843  \\ 
splice             & 3190  & 60   & Cls & 0.831  & 0.801  \\ 
kr-vs-kp           & 3196  & 36   & Cls & 0.543  & 0.512  \\ 
abalone            & 4177  & 8    & Reg & 0.028  & 0.470  \\  
spambase           & 4601  & 57   & Cls & 0.775  & 0.858  \\  
colleges           & 7603  & 44   & Reg & 0.155  & 0.387  \\  
mushroom           & 8124  & 22   & Cls & 0.185  & 0.340  \\  
kin8nm             & 8192  & 8    & Reg & 0.805  & 0.760  \\ 
pumadyn32nh        & 8192  & 32   & Reg & 0.938  & 0.961  \\ 
cpu\_activity      & 8192  & 21   & Reg & 0.709  & 0.825  \\ 
SpeedDating        & 8378  & 120  & Cls & 0.590  & 0.656  \\ 
grid\_stability    & 10000 & 12   & Reg & 0.925  & 0.911  \\ 
sulfur             & 10081 & 6    & Reg & -0.180 & 0.487  \\ 
brazilian\_houses  & 10692 & 9    & Reg & -0.290 & 0.563  \\ 
fifa               & 19178 & 28   & Reg & -0.349 & 0.663  \\ 
superconductivity  & 21263 & 81   & Reg & 0.141  & 0.367  \\ 
kings\_county      & 21613 & 21   & Reg & 0.842  & 0.882  \\ 
health\_insurance  & 22272 & 11   & Reg & 0.601  & 0.749  \\  
cps88wages         & 28155 & 6    & Reg & 0.250  & 0.479  \\ 
\midrule
\multicolumn{4}{r|}{\textbf{Mean}} & 0.431 &  0.611 \\ 
\multicolumn{4}{r|}{\textbf{Median}} & 0.587 & 0.659 \\ 
\bottomrule
\end{tabular}
    \caption{Correlation between $\tau$ and the actual error $\err_{d^*}$ on all 4 types of contexts.}
    \label{tab:metric-verification}
\end{table}

\Cref{tab:metric-verification} reports the correlation between $\tau$ and $\err_{d^*}$ over all the 140 contexts.
The most common metric is the Pearson correlation, but it can only detect linear correlations, while the correlation between $\tau$ and $\err_{d^*}$ is not necessarily linear.
Thus, we also report the distance correlation \cite{648ae489-f5d9-3200-b1e9-ab0b264416e5}, which is another common metric that can detect non-linear correlations but cannot tell if the correlation is positive or negative because this metric is always non-negative.

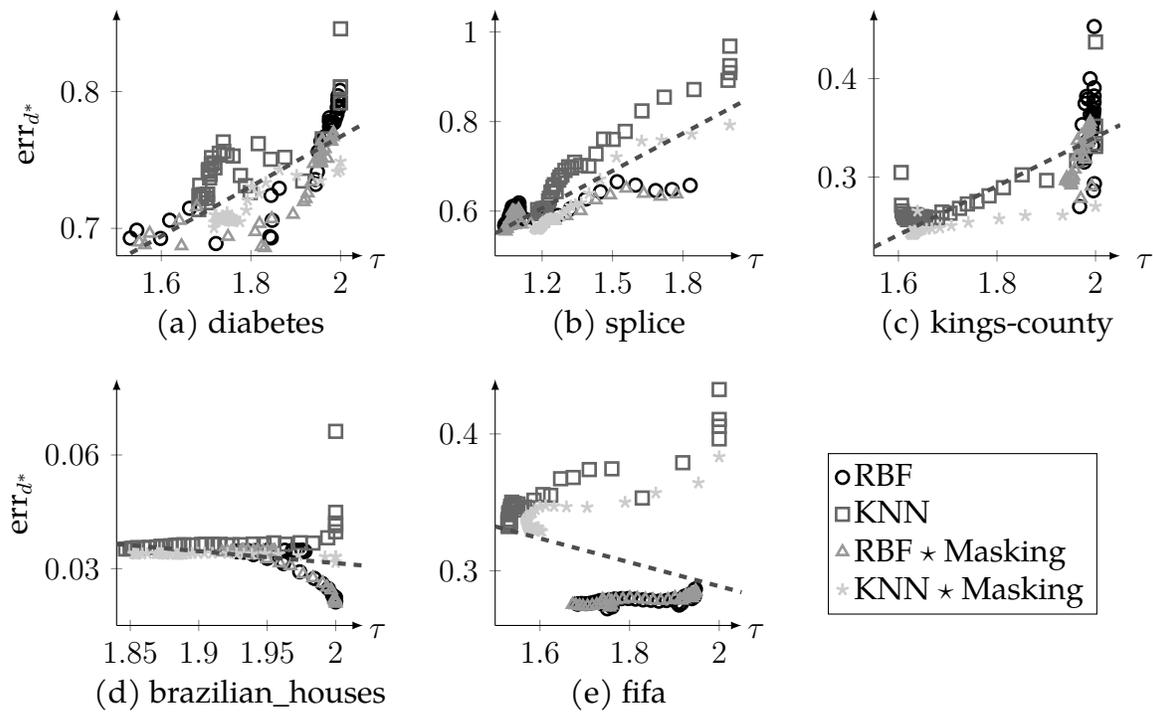
\begin{figure}[t]
    \centering
    \begin{tikzpicture}
\begin{axis}[
name=plot1,
clip marker paths=true,
 height = .30\linewidth,
    width = .30\linewidth,
    axis lines=left,  
    axis line style={-latex},  
    xlabel style = {
        at={(1.06,0.05)},
        anchor=north,
    },
title style={at={(0.5,-0.24)}, anchor=north},
title={(a) diabetes},
    xmin=1.5, xmax=2.05,
    ymin=0.68, ymax=0.86,
    xtick distance = 0.2,
    ytick distance = 0.1,
    xlabel={$\tau$},
    ylabel={$\err_{d^*}$},
    legend pos=north west,
    scatter/classes={
        1={mark=o,black,line width=1pt,mark size=2.4pt},   
        2={mark=square,black!60,line width=1pt,mark size=2.4pt}, 
        3={mark=triangle,black!40,line width=1pt,mark size=2.4pt}, 
        4={mark=star,black!20,line width=1pt,mark size=2.4pt}
    }
]

\addplot[scatter, only marks, scatter src=explicit symbolic] 
    table[col sep=comma, x=tau, y=mse, meta=type] {data/diabetes.csv};

    \addplot[black!70, ultra thick, dashed, domain=1.5:2.05] {0.18192012*x + 0.40324192};
\end{axis}

\begin{axis}[
name=plot2,
at={($(plot1.south east)+(50,0)$)},
clip marker paths=true,
 height = .30\linewidth,
    width = .30\linewidth,
    axis lines=left,  
    axis line style={-latex},  
    xlabel style = {
        at={(1.06,0.05)},
        anchor=north,
    },
title style={at={(0.5,-0.24)}, anchor=north},
title={(b) splice},
    xmin=1.0, xmax=2.05,
    ymin=0.5, ymax=1.05,
    xtick distance = 0.3,
    ytick distance = 0.2,
    xlabel={$\tau$},
    legend pos=north west,
    scatter/classes={
        1={mark=o,black,line width=1pt,mark size=2.4pt},   
        2={mark=square,black!60,line width=1pt,mark size=2.4pt}, 
        3={mark=triangle,black!40,line width=1pt,mark size=2.4pt}, 
        4={mark=star,black!20,line width=1pt,mark size=2.4pt}
    }
]

\addplot[scatter, only marks, scatter src=explicit symbolic] 
    table[col sep=comma, x=tau, y=mse, meta=type] {data/splice.csv};

    \addplot[black!70, ultra thick, dashed, domain=1.0:2.05] {0.27877287*x + 0.27161478};
\end{axis}

\begin{axis}[
name=plot3,
at={($(plot2.south east)+(50,0)$)},
clip marker paths=true,
 height = .30\linewidth,
    width = .30\linewidth,
    axis lines=left,  
    axis line style={-latex},  
    xlabel style = {
        at={(1.1,0.05)},
        anchor=north,
    },
title style={at={(0.5,-0.24)}, anchor=north},
title={(c) kings-county},
    xmin=1.55, xmax=2.05,
    ymin=0.22, ymax=0.47,
    xtick distance = 0.2,
    ytick distance = 0.1,
    xlabel={$\tau$},
    legend pos=north west,
    scatter/classes={
        1={mark=o,black,line width=1pt,mark size=2.4pt},   
        2={mark=square,black!60,line width=1pt,mark size=2.4pt}, 
        3={mark=triangle,black!40,line width=1pt,mark size=2.4pt}, 
        4={mark=star,black!20,line width=1pt,mark size=2.4pt}
    }
]

\addplot[scatter, only marks, scatter src=explicit symbolic] 
    table[col sep=comma, x=tau, y=mse, meta=type] {data/kings-county.csv};

    \addplot[black!70, ultra thick, dashed, domain=1.55:2.05] {0.24798084*x -0.15551522};
\end{axis}

\begin{axis}[
name=plot4,
at={($(plot1.south west)+(0,-140)$)},
clip marker paths=true,
 height = .30\linewidth,
    width = .30\linewidth,
    axis lines=left,  
    axis line style={-latex},  
    ylabel style={
        at={(-0.3,0.5)}
    },
    xlabel style = {
        at={(1.06,0.05)},
        anchor=north,
    },
scaled ticks=false, 
/pgf/number format/fixed   ,
title style={at={(0.5,-0.24)}, anchor=north},
title={(d) brazilian\_houses},
    xmin=1.84, xmax=2.02,
    ymin=0.015, ymax=0.08,
    xtick distance = 0.05,
    ytick distance = 0.03,
    xlabel={$\tau$},
    ylabel={$\err_{d^*}$},
    legend pos=north west,
    scatter/classes={
        1={mark=o,black,line width=1pt,mark size=2.4pt},   
        2={mark=square,black!60,line width=1pt,mark size=2.4pt}, 
        3={mark=triangle,black!40,line width=1pt,mark size=2.4pt}, 
        4={mark=star,black!20,line width=1pt,mark size=2.4pt}
    }
]

\addplot[scatter, only marks, scatter src=explicit symbolic] 
    table[col sep=comma, x=tau, y=mse, meta=type] {data/brazilian.csv};

    \addplot[black!70, ultra thick, dashed, domain=1.84:2.02] {-0.031479845*x + 0.09441593};
\end{axis}

\begin{axis}[
name=plot5,
at={($(plot4.south east)+(50,0)$)},
clip marker paths=true,
 height = .30\linewidth,
    width = .30\linewidth,
    axis lines=left,  
    axis line style={-latex},  
    xlabel style = {
        at={(1.06,0.05)},
        anchor=north,
    },
title style={at={(0.5,-0.24)}, anchor=north},
title={(e) fifa},
    xmin=1.5, xmax=2.05,
    ymin=0.26, ymax=0.44,
    xtick distance = 0.2,
    ytick distance = 0.1,
    xlabel={$\tau$},
    legend style={
        at={(1.35,0.7)},
        anchor=north west,
        cells={anchor=west},
    },
    scatter/classes={
        1={mark=o,black,line width=1pt,mark size=2.4pt},   
        2={mark=square,black!60,line width=1pt,mark size=2.4pt}, 
        3={mark=triangle,black!40,line width=1pt,mark size=2.4pt}, 
        4={mark=star,black!20,line width=1pt,mark size=2.4pt}
    }
]

\addplot[scatter, only marks, scatter src=explicit symbolic] 
    table[col sep=comma, x=tau, y=mse, meta=type] {data/fifa.csv};
\legend{RBF,KNN,RBF $\star$ Masking,KNN $\star$ Masking};
    \addplot[black!70, ultra thick, dashed, domain=1.5:2.05] {-0.08744848*x + 0.46356109};
\end{axis}

\end{tikzpicture}
    \caption{Scatter plots of $\tau$ versus $\err_{d^*}$. Dashed line: Linear fit.}
    \label{fig:usefulness-metric}
\end{figure}

The median reported in the table shows that on more than half of the datasets, the Pearson correlation is over $0.5$, 
which is generally considered a strong correlation.
The distance correlation is even higher.
As expected, the metric does not work on all datasets.
For example, the Pearson correlation is very negative on \texttt{brazilian\_houses} and \texttt{fifa}.

To understand when our metric might fail, we further visualize the results by plotting $\tau$ against $\err_{d^*}$ on five of the datasets in \Cref{fig:usefulness-metric}.
In this figure, plots (a), (b) and (c) are three success cases where a clear positive correlation can be observed, and plots (d) and (e) display two failure cases.
Plot (d) shows a common failure case: if $\tau$ is very close to $2 = \beta + 1$, meaning that the metric believes that the association is extremely weak or extremely strong, then the metric will predict that the context is bad.
However, a generally bad context can still be good on some tasks.
For example, a very weak context still works well on a task that only uses the top-$3$ singular functions of the context.
Therefore, it is advisable to abstain from using the metric when it is too close to $\beta + 1$.

Plot (e) shows a case where the metric is generally good for every single context type but has poor cross-type behavior. Specifically, it fails to predict that KNN is worse than RBF on this dataset.
This suggests that our metric might not be able to compare different types of contexts. For example, if two contexts of completely different types have similar spectra, then our metric will indicate that they are similarly useful. This is because our metric only depends on the spectrum.
However, it could be possible that for a particular task, one context is good and the other is bad, and our metric cannot reflect this disparity.

Overall, although there does not exist a universal metric that works for all contexts and tasks, and our metric does have failure cases, the experiment results here provide empirical evidence that more often than not, the proposed metric correlates well with the actual prediction error of the downstream linear probe.
Hence, the proposed metric is useful for choosing hyperparameters and comparing contexts in practice.

In summary, this chapter first defined the compatibility between a task and a context, and then discussed the intrinsic evaluation of encoders and contexts.
The key takeaway is that when we know a priori that the downstream task is compatible with the context we are given, then the optimal thing to do is learning the contexture.
The intrinsic evaluation of an arbitrary encoder depends on two key concepts---the induced RKHS and the ratio trace.
For the intrinsic evaluation of contexts, the key takeaway is that a good context should have a moderate association.
Moreover, we proposed a metric that only depends on the spectrum of the context.
The metric correlates with the actual prediction error on many real datasets.

\paragraph{Code and data availability.}
The code for \Cref{fig:taud} can be found at {\scriptsize \url{https://1drv.ms/u/c/ea9fe908498c8b82/EV11dVAUVCdCnkw31CHrZC0BOEmsudZ7swpTKvzcfod5uA?e=axm6XG}}.
The code for \Cref{tab:metric-verification} is at {\scriptsize \url{https://1drv.ms/u/c/ea9fe908498c8b82/EcqvS70ynvdCsma6MoSDRcwBZPRw5uThHUXQcz9P7vNQZQ?e=7ULQey}}.
All datasets can be downloaded from OpenML.

\chapter{Mixing Multiple Contexts}
\label{chap:mixture}

We have shown that creating better contexts is imperative for further improving the performance of foundation models.
However, creating new contexts from scratch is extremely challenging.
Normally, it requires new domain knowledge obtained from scientific research, or collecting new data through experiments or human labeling such as in RLHF, both of which require considerable effort.
As a result, the contexts used for pretraining rarely change.
For example, in computer vision, although various objectives have been proposed such as contrastive and non-contrastive learning and masked autoencoders, the context is always based on image corruption, such as random cropping, masking and color distortion.
In NLP, base language models have always been pretrained on the context of masking tokens like BERT \cite{devlin2018bert}.
There are many variants of BERT, such as determining whether a sentence completion is correct instead of actually completing a masked sentence like Electra \cite{Clark2020ELECTRA},
predicting the next token rather than tokens in the middle of a sentence like GPT \cite{radford2019language}, and predicting the next group of tokens like Medusa \cite{cai2024medusa}.
In all these variants, the context variable $A$ is a masked version of $X$.

This chapter introduces a much easier way to obtain better contexts.
We know that a good context should have a moderate association.
Consider the following scenario: we have a number of contexts, but none of them is useful enough because the association is either too strong or too weak.
Then, how can we get a better context?
The idea is to mix these contexts together, so that we can obtain a context with a moderate association.

In fact, mixing multiple contexts is quite common in practice, though it has never been fully formalized as a unified framework like this chapter.
One very widely used method is composing different data augmentation techniques together.
For example, common data augmentations for images include translation, rotation, random cropping, color distortion, etc.
Each augmentation defines one context, and sequentially applying them to the same image leads to the \textbf{convolution} of their contexts.

Another common practice when there are multiple contexts is optimizing a weighted sum of different objectives.
Suppose there are $r$ desiderata we want our model to satisfy, and desideratum $j$ can be achieved by minimizing objective $\gR_j$.
A natural idea is to minimize $\sum w_j \gR_j$ for some $w_1,\cdots,w_r > 0$, so that we minimize all $\gR_j$ simultaneously.
In machine learning, if $\gR_1$ is the main objective, then other $\gR_j$ are also called penalty terms.
Common penalty terms include the $L^2$ or RKHS norm of the model weight, the local smoothness of the model \wrt{} a certain manifold, the distance to a reference model, etc.
The weighted sum of the objectives learns the contexture of what we call the \textbf{convex combination} of the $r$ contexts.

To get an intuition of convolution and convex combination, consider the example in \Cref{fig:mix-graph}.
In this example, $|\gX| = 4$.
We know that $P^+$ induces a joint distribution on $\gX \times \gX: P^+(x,x') = \int P^+(x|a) P^+(x'|a) d \pa(a)$.
We can then obtain $P^+(x'|x)$, which are labeled as edge weights in \Cref{fig:mix-graph}.
In other words, each context induces a random walk on the graph.
There are two contexts in \Cref{fig:mix-graph}. $P^+_1$ corresponds to the solid edges, and $P^+_2$ corresponds to the dashed edges.

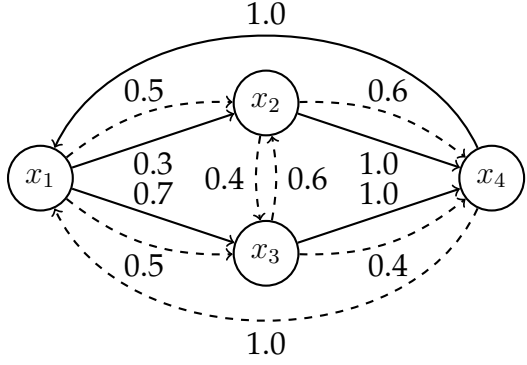
\begin{figure}[t]
    \centering
    \begin{tikzpicture}

\node(x1) [circle, draw, thick] at (0,0) {$x_1$};
\node(x2) [circle, draw, thick] at (3,1) {$x_2$};
\node(x3) [circle, draw, thick] at (3,-1) {$x_3$};
\node(x4) [circle, draw, thick] at (6,0) {$x_4$};

\draw[thick,->] (x1) to node[below] {0.3} (x2);
\draw[thick,->, dashed] (x1) to[bend left=20] node[above] {0.5} (x2);
\draw[thick,->] (x1) to node[above] {0.7} (x3);
\draw[thick,->, dashed] (x1) to[bend right=20] node[below] {0.5} (x3);

\draw[thick,->] (x2) to node[below] {1.0} (x4);
\draw[thick,->, dashed] (x2) to[bend left=20] node[above] {0.6} (x4);
\draw[thick,->, dashed] (x2) to[bend right=10] node[left] {0.4} (x3);

\draw[thick,->] (x3) to node[above] {1.0} (x4);
\draw[thick,->,dashed] (x3) to[bend right=10] node[right] {0.6} (x2);
\draw[thick,->, dashed] (x3) to[bend right=20] node[below] {0.4} (x4);

\draw[thick,->] (x4) to[bend right=65] node[above] {1.0} (x1);
\draw[thick, ->, dashed] (x4) to[bend left=65] node[below] {1.0} (x1);

\end{tikzpicture}
    \caption{An example of two contexts (solid and dashed edges) where $|\gX| = 4$.}
    \label{fig:mix-graph}
\end{figure}

Convolution, denoted by $\star$, is equivalent to a multi-step random walk.
$P^+_2 \star P^+_1$ means that we first walk one step on the solid edges ($P^+_1$), and then walk one step on the dashed edges ($P^+_2$).
For example, $(P^+_2 \star P^+_1) (x_4|x_1) = \sum_x P^+_2(x_4|x) P^+_1(x|x_1) = 0.6 \times 0.3 + 0.4 \times 0.7 = 0.46$.
Similarly, $P^+_1 \star P^+_2 \star P^+_1$ is equivalent to a solid-dashed-solid random walk.
In general, convolution does not have the commutative property.

Convex combination, denoted by $+$, is equivalent to a stochastic one-step random walk.
$0.6P^+_1 + 0.4 P^+_2$ means that we walk one step; with probability $0.6$ the step is on the solid edges, and with probability $0.4$ the step is on the dashed edges.
For example, if we are standing at $x_1$, then under $0.6P^+_1 + 0.4 P^+_2$, with probability $0.6 \times 0.3 + 0.4 \times 0.5 = 0.38$ we will walk to $x_2$, and with probability $0.62$ we will walk to $x_3$.
Convex combination can be combined with convolution.
For example, $P^+_2 \star (0.6P^+_1 + 0.4 P^+_2)$ means that we first take one step using $0.6P^+_1 + 0.4 P^+_2$, and then take another step using $P^+_2$.

We can show that the set of contexts forms a near-ring with scalar multiplication (also called a module), where convolution is multiplication and convex combination is addition with scalar multiplication.
``Near''-ring means that it only has the right distributive property but not the left one; that is, $(0.6P^+_1 + 0.4 P^+_2) \star P^+_2 = 0.6 P^+_1 \star P^+_2 + 0.4 P^+_2 \star P^+_2$, but $P^+_2 \star (0.6P^+_1 + 0.4 P^+_2)$ is not distributive.

There is a third operation for mixing multiple contexts called concatenation.
Concatenation is completely different from the other two operations, because convolution and convex combination act on the input space $\gX$, whereas concatenation acts on the output space of the encoder.
Specifically, given $r$ contexts, one trains an individual encoder $\Phi_j$ for each context, and then concatenate them as $\Phi(x) = [\Phi_1(x), \cdots, \Phi_r(x)]$.
Concatenation is a classical and very popular method, and it is connected to a whole field in data science known as feature engineering.
In feature engineering, people create different features based on different signals (formulated as contexts in this thesis), and then train a model on their concatenation.

Mixture of experts (MOE) \cite{jacobs1991adaptive} is a popular method that is a stronger version of concatenation.
In MOE, we first train individual encoders $\Phi_1,\cdots,\Phi_r$ for the $r$ contexts, and then at the downstream stage, we train a gating function $g: \gX \rightarrow \R^r$ that assigns weights to the $r$ encoders for each $x$.
The overall encoder is $\Phi(x) = \sum_{j=1}^r g(x)_j \Phi_j(x)$, and finally a linear probe is fit on top of $\Phi$.
The common practice is to implement $g$ as a small neural network, and its output space is usually restricted to $\Delta^{r-1}$, the $(r-1)$-dimensional unit simplex, so that $g(x)$ is a probability distribution over the $r$ encoders.
This can be easily done by adding a softmax layer to the end of the neural network.
Concatenation is a special case of MOE, where $g$ is restricted to be a constant function.

In what follows, we conduct a deep analysis on the three base operations for mixing multiple contexts: convolution, convex combination and concatenation.
Importantly, we discuss in what situations each operation should be used.
After that, we apply the three operations to real-world tabular datasets, and we find that they can improve the performance of state-of-the-art methods such as XGBoost \cite{chen2016xgboost}.

\section{Convolution}
Suppose we have $r$ contexts given by $P^+_1,\cdots,P^+_r$.
Let $\gA_j$ be the context space of $P^+_j$,
and let $Q^+_j$ be the heuristic inverse of $P^+_j$ (\Cref{def:heuristic-inverse}).
As mentioned earlier, convolution is similar to composing multiple data augmentation techniques.
Assuming that we have $T$-access to every context, we can transform $X$ by $X \xrightarrow{P_1^+} A_1 \xrightarrow{Q_1^+} X_1 \xrightarrow{P_2^+} A_2 \xrightarrow{Q_2^+} X_2 \xrightarrow{P_3^+} \cdots \xrightarrow{Q_{r-1}^+} X_{r-1} \xrightarrow{P_r^+} A_r$.
This $A_r$ is defined as the context variable of the convolution $P_r^+ \star \cdots \star P_1^+$.
The heuristic inverse of $P_r^+ \star \cdots \star P_1^+$ is defined as $Q_r^+$.
Then, it is easy to see that convolution has the associative property.

In the general case, we might not have $T$-access to all contexts, but we can assume that each context has either $k$-access or $T$-access.
This is a weak assumption since it has been shown previously that any access can be converted to $k$-access.

For every $j \in [r]$, define a kernel $k_j$ as follows:
\begin{itemize}
    \item If it has $k$-access, let $k_j$ be the kernel we have access to.
    \item If it has $T$-access, let $k_j(x,x') = \int Q_j^+(x'|a) dP_j^+(a|x) / \px(x')$ if $j < r$.
    If $j = r$, then define $k_r = \kx_r$ as the exact dual kernel of context $r$.
\end{itemize}

\begin{proposition}
The integral operator of the dual kernel of the convolution of the $r$ contexts is equal to $T_{k_1} T_{k_2} \cdots T_{k_r} \cdots T_{k_2} T_{k_1}$.
\end{proposition}
\begin{remark}
When each $T_{k_j}$ is a bounded compact self-adjoint operator, this operator is also bounded compact self-adjoint, so the Hilbert-Schmidt theorem still applies.
Note that this operator is not equal to $T_{k_r} \cdots T_{k_1}$, which is not necessarily a self-adjoint operator.
\end{remark}
\begin{proof}
We prove by induction on $r$. When $r=1$ this is obvious. Suppose the result holds for $r-1$. Let $k_{j}^+$ be the dual kernel between $X_1$ and $A_j$. Then, we have
\begin{equation*}
    k_r^+(x,x') = \iint k_1(x,z) k_{r-1}^+(z,z') k_1(z',x') d \px(z) d \px(z') ,
\end{equation*}
so it is easy to see that $T_{k_r^+} = T_{k_1} T_{k_{r-1}^+} T_{k_1}$, which shows that the result holds for $r$.
\end{proof}

We now discuss how to learn the contexture of a convolution.
If we have $T$-access to every context, then we can simply transform $X \rightarrow A_r$, and use SVME on $X$ and $A_r$.
If we have $T$-access to $P^+_1,\cdots,P^+_{r-1}$ but $k$-access to $P^+_r$, then it is also very simple.
We can transform $X \rightarrow X_{r-1}$, and use KISE on $X$ and $X_{r-1}$.

The more difficult scenario is when we have $k$-access to $P_j^+$ for some $j < r$.
Let $j_1 < j_2 < \cdots < j_l$ be all such $j$.
We now present an algorithm that learns the contexture.
First, initialize $\Phi^0 = \Phi$, and $X_0 = X$.
Second, transform $X_0 \rightarrow X_{j_1-1}$ using the $T$-access to $P^+_1,\cdots,P^+_{j_1-1}$.
Third, let $\Phi^1(X) = (T_{k_{j_1}} \Phi^0)(X_{j_1-1})$, which can be estimated with Monte Carlo using the $k$-access to $k_{j_1}$.
Fourth, let $X_{j_1} = X$, and transform $X_{j_1} \rightarrow X_{j_2-1}$ using the $T$-access to $P^+_{j_1+1}, \cdots, P^+_{j_2-1}$. 
Fifth, let $\Phi^2(X) = (T_{k_{j_2}} \Phi^1)(X_{j_2-1})$.
Repeat these two steps until we get $\Phi^l$, and transform $X_{j_l} \rightarrow X_{r-1}$.
Finally, the learning objective is SVME if we have $T$-access to $P_r^+$, or KISE if we have $k$-access to $P_r^+$.

The detailed algorithm is listed in \Cref{alg:convolution}.
One can prove that for any $t \in [0,l]$ and any $j \in [j_t, j_{t+1}-1]$, $\Phi^t(X_j)$ has the same distribution as $(T_{k_j} T_{k_{j-1}} \cdots T_{k_1} \Phi)(X)$.
Hence, $\Phi^l(X_{r-1})$ has the same distribution as $(T_{k_{r-1}} T_{k_{j-1}} \cdots T_{k_1} \Phi)(X)$.
By \Cref{thm:svme-spin}, this algorithm extracts the top-$d$ eigenspace of the operator.

\begin{algorithm}[t]
\caption{Extracting top-$d$ eigenspace of a convolution of contexts}
\label{alg:convolution}
\begin{algorithmic}[1]
\State Initialize encoder $\Phi: \gX \rightarrow \R^d$; If have $T$-access to context $r$, initialize $\Psi: \gA_r \rightarrow \R^d$
\For{each training step}
\State Sample a batch of samples $\oset{x_1,\cdots,x_m}$
\State Center $\Phi$: $\Phi \gets \Phi - \frac{1}{m} \sum_{i=1}^m \Phi(x_i)$
\State \textbf{for} $i \in [m]$ \textbf{do} $x_i^0 \gets x_i$; set $\mB \gets \vzero \in \R^{d \times m}$,  $k \gets$ \textbf{null}
\For{$j = 1,\cdots,r-1$}
    \If{have $k$-access to context $j$}
        \If{$k =$ \textbf{null}}
            \State $\mB \gets \brac{ \Phi(x_1^{j-1}), \cdots, \Phi(x_m^{j-1}) } \in \R^{d \times m}$
        \Else
            \State Set Gram matrix $\mG \in \R^{m \times m}$ as $\mG[p,q] = k(x_p, x_q^{j-1})$; $\mB \gets \frac{1}{m} \mB \mG$
        \EndIf
        \State $k \gets \kx_j$; \textbf{for} $i \in [m]$ \textbf{do} $x_i^j \gets x_i$ \Comment{reset $x_i^j$ to the original input} 
    \Else \Comment{$T$-access to context $j$}
        \State Sample $a_i^j \sim P_j^+(\cdot | x_i^{j-1}), \; x_i^j \sim Q_j^+(\cdot | a_i^j)$
    \EndIf
\EndFor

\If{$k =$ \textbf{null}} \Comment{for every $j \le r-1$, context $j$ has $T$-access}
    \State $\mB \gets [\Phi(x_1^{r-1}),\cdots,\Phi(x_m^{r-1})] \in \R^{d \times m}$
\Else
    \State Set Gram matrix $\mG \in \R^{m \times m}$ as $\mG[p,q] = k(x_p, x_q^{r-1})$; $\mB \gets \frac{1}{m} \mB \mG$
\EndIf

\If{has $k$-access to $r$}
    \State Set Gram matrix $\mG_r \in \R^{m \times m}$ as $\mG_r[p,q] = \kx_r(x_p, x_q)$; $\mC \gets \frac{1}{m} \mB \mG_r \in \R^{d \times m}$
    \State Define $\gL = \frac{1}{m} \brac{\sum_{i=1}^m  \norm{\Phi(x_i)}_2^2 - \dotp{\mB, \mC} }$ \Comment{$\dotp{\mB, \mC} = \Tr \paren{ \mB \mC^{\top} }$}
\Else \Comment{$T$-access to context $r$}
    \State Sample $a_i^r \sim P_r^+(\cdot | x_i)$; Center $\Psi$: $\Psi \gets \Psi - \frac{1}{m} \sum_{i=1}^m \Psi(a_i)$
    \State $\mC \gets \brac{\Psi(a_1^r), \cdots, \Psi(a_m^r)} \in \R^{d \times m}$ 
    \State Define $\gL = \frac{1}{m}  \brac{ \sum_{i=1}^m \norm{\Phi(x_i)}_2^2 + \norm{\mC}_F^2 - 2 \dotp{\mB, \mC} }$  \Comment{$\norm{\mC}_F^2 = \sum_{i,j} \mC[i,j]^2$}
\EndIf

\State Update $\Phi$ to minimize loss $\gL$, subject to $\Cov_{\px}[\Phi] = \mI$
\EndFor
\end{algorithmic}
\end{algorithm}

\paragraph{Standardizing kernels.}
If we only have $k$-access to some contexts, then it is important to standardize these kernels.
We say that a \psd{} kernel $k_j$ is \textbf{standardized}, if $\mu_0 \equiv 1$ is an eigenfunction of $T_{k_j}$ with eigenvalue $1$, and all eigenvalues of $T_{k_j}$ belong to $[0,1]$.
To see why standardizing $k_1,\cdots,k_r$ is necessary, consider a case where $k_1$ has much larger eigenvalues than the other kernels; then, when mixing these contexts, context $1$ could dominate over the other contexts.
We can standardize any \psd{} kernel in three steps:
\begin{enumerate}[label=(\roman*)]
    \item Center the kernel: $k(x,x') \gets \tilde{k}(x,x')$, which makes $\mu_0$ an eigenfunction with eigenvalue $0$. The centered kernel is defined as follows.
\begin{definition}
\label{def:centered-kernel}
For any \psd{} kernel $k: \gX \times \gX \rightarrow \R$, its \textbf{centered kernel} is given by
\begin{equation*}
    \tilde{k}(x,x') = k(x,x') - \int k(z,x') d \px(z) - \int k(x, z') d \px (z') + \iint k(z, z') d \px(z) d \px(z') ,
\end{equation*}
which is a \psd{} kernel that satisfies $\int k(x,z) d \px(z) = \int k(z,x) d \px(z) = 0$ for any $x$.
\end{definition}
    \item Divide $k$ by its largest eigenvalue, which makes all its eigenvalues at most $1$.
    \item $k(x,x') \gets k(x,x') + 1$, which makes $\mu_0$ an eigenfunction with eigenvalue $1$.
\end{enumerate}

The only ``hyperparameter'' we need to tune for a convolution is the order of the contexts, because convolution is not commutative.
In practice today, the order usually does not matter too much.
For example, it matters very little whether we first translate an image and then crop it, or we first crop it and then perform the translation.
However, as we obtain more complex contexts especially through context scaling, this order could become very important.

Apart from supervised learning with feature maps discussed in \Cref{sec:learn-mixture}, another example of convolution is supervised contrastive learning \cite{khosla2020supervised}.
Given a sample $X$, supervised contrastive learning first randomly samples $X_1$ that has the same class as $X$, and then augments $X_1 \rightarrow A_2$ via cropping, flipping, etc.
Supervised contrastive learning learns the convolution of the class context $P^+_1$ and the augmentation context $P^+_2$.

Finally, let us discuss when we should use convolution.
In practice, convolution is used to create ``harder'' pretraining tasks.
For example, in self-supervised learning, multiple weak data augmentations are composed together to create a stronger augmentation.
One great example is SimCLR \cite{chen2020simple}, whose success is largely due to its aggressive crop ratio and color distortion, both of which make the augmentation stronger.
When an augmentation is stronger, the association between $X$ and $A$ becomes weaker.
Hence, convolution should be used when all contexts have strong associations, because it always weakens the association.

\section{Convex Combination}
\label{sec:convex-combination}
A convex combination of $P_1^+, \cdots, P_r^+$ is written as $w_1 P_1^+ + \cdots + w_r P_r^+$.
Usually we require that $\vw = [w_1,\cdots,w_r] \in \Delta^{r-1}$, so that it is a probability distribution over the $r$ contexts.
In this case, the convex combination can be understood as follows:
given an input $X$, one first samples one $P_j^+$ from the probability distribution $\vw$ over the $r$ contexts, and then samples $A \sim P_j^+(\cdot | X)$.
This $A$ is the context variable of the convex combination.
Theoretically speaking, $w_1,\cdots,w_r$ can be any real values, but we assume that $\vw \in \Delta^{r-1}$ in this chapter unless stated otherwise.

Learning the contexture of a convex combination is simple: we only need to use a weighted sum of individual learning objectives.
Assume that we have either pair access or $k$-access to every context.
For each context $j$, we can learn its contexture using either SVME or KISE. Define its individual objective as
\begin{equation*}
    \gR_j = \left \{
\begin{aligned}
    \E_{X \sim \px} \E_{A_j \sim P_j^+(\cdot | X)} \brac{ \norm{ \Phi(X) - \Psi_j(A_j) }_2^2 }, \quad & \text{pair access to } P_j^+ ; \\ 
    \E_{X \sim \px} \brac{\norm{\tPhi(X)}_2^2 - \dotp{\tPhi(X), T_{k_j} \tPhi(X)}}, \quad & k\text{-access to } P^+_j .
\end{aligned}
    \right .
\end{equation*}
Then, we can learn the contexture of their convex combination by minimizing $\sum_j w_j \gR_j$ subject to $\Cov_{\px}[\Phi] = \mI$.
This approach might need more than two encoders, since it requires one $\Psi_j$ for each context $j$ with pair access.
An alternative approach is to convert every pair access to $k$-access first, but this leads to a huge overhead.

If we have pair access to context $j$, let $k_j = \kx_j$; if we have $k$-access to context $j$, let $k_j$ be the kernel we have access to, and assume that it has been standardized.
The following result shows that the dual kernel of the convex combination is $w_1 k_1 + \cdots + w_r k_r$, the linearly combined kernel.
If every $k_j$ is standardized and $\vw \in \Delta^{r-1}$, then the linearly combined kernel is also standardized.
Linearly combining multiple kernels is a classical technique in multiple kernel learning \cite{JMLR:v12:gonen11a}.

\begin{theorem}[Proof in \Cref{app:proof-thm-kernel-convex-combination}]
\label{thm:kernel-convex-combination}
Let $\Phi^*$ be a minimizer of the weighted sum of objectives $\sum_j w_j \gR_j$ subject to $\Cov_{\px}[\Phi] = \mI$.
Then, $\tPhi^*$ extracts the top-$d$ eigenspace of $\sum_j w_j k_j$.
\end{theorem}

Convolution and convex combination have the right distributive property.

\begin{proposition}
$ (w_1 P_1^+ + \cdots + w_r P_r^+) \star P_0^+ = w_1 P_1^+ \star P_0^+ + \cdots + w_r  P_r^+ \star P_0^+$.
\end{proposition}
\begin{proof}
    Let $k_L$ be the dual kernel of the left, and $k_R$ be the dual kernel of the right.
    It suffices to show that $T_{k_L} = T_{k_R}$.
    Using what was proved earlier, we have
    \begin{equation*}
        T_{k_L} = T_{k_0} (w_1 T_{k_1} + \cdots + w_r T_{k_r}) T_{k_0} = w_1 T_{k_0} T_{k_1} T_{k_0} + \cdots + w_r T_{k_0} T_{k_r} T_{k_0} = T_{k_R},
    \end{equation*}
    as desired.
\end{proof}
\begin{remark}
    The left distributive property does not hold, that is, $P_0^+ \star (w_1 P_1^+ + \cdots + w_r P_r^+)$ and $w_1 P_0^+ \star P_1^+ + \cdots + w_r P_0^+ \star P_r^+$ are not necessarily equal.
\end{remark}

The hyperparameters we need to tune for the convex combination are the weights $w_1,\cdots,w_r$, which should be chosen based on the  associations of the contexts.
Usually, if a context is very weak or strong, then we would give it a small weight to limit its impact.
For example, in RLHF we have two contexts: the alignment context and the reference model context.
The alignment context has a strong association---for a given prompt, there are multiple possible valid completions $A$, but alignment selects a small number of preferred completions from all valid ones.
Consequently, alignment reduces the conditional entropy $H(A|X)$, so $A$ has a stronger association with $X$.
The more selective the critic, the stronger the association.
On the other hand, the reference model context has a moderate association, provided that the reference model is well trained.
Therefore, in practice, people usually give a much larger weight to the reference model context than to the alignment context, so that the model will not be too different from the reference model.

Convex combination is usually used to balance strong and weak associations.
It can also be used when all contexts have very strong associations.
In this case, we want to select the $w_1,\cdots,w_r$ that weaken the association as much as possible.
We show that this can be achieved by playing a zero-sum game between a $\Phi$-player who learns the encoders, and a $\vw$-player who picks $\vw$ to maximize the loss of the $\Phi$-player.
The game has the following minimax form.
\begin{equation}
\label{eqn:convex-combination-game}
    \underset{\Phi: \gX \rightarrow \R^d; \; \Psi_j: \gA_j \rightarrow \R^d}{\textrm{minimize}} \; \; \max_{\vw \in \Delta^{r+1}} \; \gL(\Phi, \Psi_1,\cdots,\Psi_r; \vw) := \sum\nolimits_{j=1}^r w_j \gR_j \quad \text{s.t. }   \Cov_{\px} \brac{\Phi} = \mI  .
\end{equation}
$\Phi$ is allowed to be randomized, that is the $\Phi$-player can use a mixed strategy.
Obviously, the $\vw$-player has an optimal pure strategy, so $\vw$ need not to be randomized.
Let us analyze the Nash equilibrium of this game.
First, notice that the $\Phi$-player only needs to pick $\Phi$.
Once $\Phi$ is picked, the optimal $\Psi_j$'s can be determined as follows.
\begin{proposition}
    Suppose context $j$ has pair access. Then, when $\Phi$ is fixed, the optimal $\Psi_j$ that minimizes $\gR_j$ is $\Psi_j^* = T_{P_j^+}^* \Phi$.
\end{proposition}
\begin{proof}
    This is the same as the proof of \Cref{thm:svme-spin}.
\end{proof}

When both players play optimally and the game reaches Nash equilibrium,
the value of $\E_{\Phi}[\gL(\Phi, \Psi_1,\cdots,\Psi_r; \vw)]$ is called the value of this game, denoted by $\gL^*$. The expectation is taken over the randomness of $\Phi$.
The following result gives the optimal strategy of the $\vw$-player, and the formula for $\gL^*$.
\begin{theorem}
\label{thm:minimax-top-d}
Let $1 = \lambda_0(\vw) \ge \lambda_1(\vw) \ge \cdots$ be the eigenvalues of $T_{w_1 k_{1} + \cdots + w_r k_{r}}$.
Let $\vw^*$ be the optimal strategy of the $\vw$-player.
Then, the game value of \Cref{eqn:convex-combination-game} is
\begin{equation*}
    \gL^* = d - \sum\nolimits_{i=1}^d \lambda_i(\vw^*), \quad \text{ and } \vw^* \text{ minimizes } \sum\nolimits_{i=1}^d \lambda_i(\vw) \text{ over all } \vw \in \Delta^{r+1} .
\end{equation*}
\end{theorem}
\begin{proof}
    For a fixed $\vw$, when $\tPhi$ learns the contexture of $\sum_j w_j k_j$, the loss is $\gL = d - \sum_{i=1}^d \lambda_i(\vw)$. 
    Thus, the optimal $\vw^*$ that maximizes $\gL$ must minimize $\sum_{i=1}^d \lambda_i(\vw)$.
\end{proof}

\begin{algorithm}[!t]
\caption{Convex combination: Solving the minimax game}
\label{alg:convex-combination}
\begin{algorithmic}[1]
\Require Step size $\eta > 0$
\State Initialize encoder $\Phi: \gX \rightarrow \R^d$, and $\Psi_j: \gA_j \rightarrow \R^d$ if needed; $\vw \gets [1/r,\cdots,1/r]$
\For{training step $t=1,2,\cdots,T$ }
    \State Fix $\vw$, and find the optimal $\Phi, \Psi_j$ to \Cref{eqn:convex-combination-game}. Denote the optimal $\Phi$ at step $t$ by $\Phi^t$
    \State Compute $\gR_1,\cdots,\gR_r$; Update $w_j \gets w_j \cdot \exp(\eta \gR_j)$, then normalize $w_j \gets \frac{w_j}{\sum_{i=1}^r w_i}$
\EndFor
\State The $\Phi$-player picks the uniform distribution over $\Phi^1,\Phi^2,\cdots,\Phi^T$ (a randomized $\Phi$)
\end{algorithmic}
\end{algorithm}

To solve this game, we can use an algorithm similar to the Hedge algorithm in online learning \cite{freund1997decision}.
The algorithm is listed in \Cref{alg:convex-combination}.
A standard result in online learning shows that this algorithm can find the value of this game.

\begin{theorem}[Proof in \Cref{app:proof-thm-convergence-nash}]
\label{thm:convergence-nash}
Let $\gR_j^t$ be the loss $\gR_j$ in step $t$.
Suppose there exists a constant $C > 0$ such that $\gR_j^t \le C$ holds for all $t, j$.
Denote $\gL(\vw) = \frac{1}{T} \sum_{t=1}^T \paren{\sum_{j=1}^r w_j \gR_j^t} $.
If $\eta = \frac{\sqrt{\log r}}{C\sqrt{T}}$ where $T > \log r$, then 
    \begin{equation*}
        \sup_{\vw \in \Delta^{r+1}} \gL(\vw) \leq  \max_{\vw \in \Delta^{r+1}} \; \min_{\Phi,\Psi_1,\cdots,\Psi_r} \; \sum_{j=1}^r w_j \gR_j  + \frac{2C\sqrt{\log r}}{\sqrt{T}} = \gL^* + \frac{2C\sqrt{\log r}}{\sqrt{T}} ,
    \end{equation*}
which implies that $\gL(\vw) \rightarrow \gL^*$ as $T \rightarrow \infty$.
\end{theorem}

Convex combination cannot be used when all contexts have weak associations, because the decay rate of its eigenfunctions is upper bounded by that of the strongest context. That is, convex combination cannot make the association stronger.

\section{Concatenation}

Concatenation is easy to implement: For each context $j$, one trains an encoder $\Phi_j$, and then concatenates them into a single encoder by $\Phi(x) = [\Phi_1(x), \cdots, \Phi_r(x)]$.
Concatenation is used when all contexts have weak associations.
We cannot learn rich features from a context with a weak association, but if we have several such contexts and concatenate them, then the features will be richer.
Since concatenation could lead to redundant features that increase the sample complexity, it should not be used for contexts with very strong associations.

To further elaborate on concatenation, let us suppose $\gX$ is a finite set, and $|\gX| = N$.
Then, the dual kernel of each context $j$ is an $N \times N$ matrix denoted by $\mK_j$.
The dual kernel of the concatenation can be understood as $\paren{\begin{matrix} \mK_1 & & & 0 \\ & \mK_2 & & \\ & & \ddots & \\ 0 & & & \mK_r \end{matrix}}$, whose eigenvalues are the union of the eigenvalues of every individual $\mK_j$.
Therefore, the eigenvalues of the concatenation decay more slowly than those of any individual context.

The hyperparameters of concatenation are the dimensions of $\Phi_1,\cdots,\Phi_r$, denoted by $d_1,\cdots,d_r$.
One good way to select them is the following: first, estimate the singular values $s_1^{(j)},s_2^{(j)},\cdots$ of each context $j$ using the post-hoc approach in \Cref{sec:extract-exact};
then, select $d_1,\cdots,d_r$ such that $s_{d_1+1}^{(1)},s_{d_2+1}^{(2)},\cdots,s_{d_r+1}^{(r)}$ are all close and small.

\Cref{tab:mix-when} summarizes when to use each of the three base operations.
We can see that these three operations cover all possible scenarios.

\begin{table}[t]
    \centering
    \begin{tabular}{c|ccc}
    \toprule
       Base operation  & Concatenation & Convolution & Convex combination \\
       When to use?  & All weak associations & All strong associations & Mixed weak/strong \\ 
       \bottomrule
    \end{tabular}
    \caption{When to use each base operation of mixing contexts.}
    \label{tab:mix-when}
\end{table}

\section{Application to Tabular Data}

The methods of mixing contexts provide us with an opportunity to create better contexts almost effortlessly.
Here we test these methods on real-world tabular datasets. We focus on tabular data for two reasons:
\begin{enumerate}[label=(\roman*)]
    \item It is important: Tabular data is the most common type of data in industry.
    \item It is challenging: So far deep learning has not been as successful on tabular data as it is on other modalities. In particular, XGBoost \cite{chen2016xgboost} has long been the state of the art on tabular data.
\end{enumerate}

There are several complications of tabular data that make it more difficult for deep learning than image and text:
\begin{enumerate}[label=(\roman*)]
    \item Heterogeneous features, such as categorical and numerical columns.
    \item Features have different meanings, unlike \eg{} images where all features are pixels.
    \item Much lower signal-to-noise ratio than other modalities, and missing values (NaNs).
\end{enumerate}

\paragraph{Setup.}
We use 118 datasets and run a grid search on tens of thousands of hyperparameter combinations for each method and dataset.
Using a large number of datasets and hyperparameters ensures that the bias of dataset and hyperparameter selection is reduced as much as possible.
We focus on prediction tasks, and use 98 classification datasets and 20 regression datasets (labeled \texttt{reg-20}). 
The 98 classification datasets further consist of 56 smaller ones (labeled \texttt{cls-56}) with fewer than 1500 samples, and 42 larger ones (labeled \texttt{cls-42}).
All datasets are real-world datasets from OpenML \cite{OpenML2013}, and they cover a wide range of domains.
We use 10 train-val-test splits (which we call 10 folds) for each dataset and report the standard deviation of performance.
For evaluation, our metrics are the accuracy for classification, and the $R^2$-score for regression.
We consider two performances close if their difference is less than one standard deviation.

\paragraph{Baselines.}
We start by evaluating the baseline methods and two recent methods that are widely compared to in the literature---FT-Transformer \cite{gorishniy2021revisiting} and TabPFN \cite{hollmann2023tabpfn}.
\Cref{tab:baselines} reports the average performances and rankings of eight baseline methods.
For the rankings, we rank the methods from $1$ to $8$ on each dataset (ties get the same ranking) and take the average for every method.
From the table we can see that:
\begin{itemize}
    \item Compared to the prior work \cite{mcelfresh2023neural}, most methods (except CatBoost and TabPFN) get much higher performance in our experiments. In their paper, ResNet was reported to be much better than MLP, which is not very reasonable. On the other hand, in our experiments, MLP and ResNet have almost the same performance. This shows that careful hyperparameter tuning gives us more accurate baselines.
    \item To our surprise, MLP and ResNet perform the best on the 56 small classification data sets, though conventional wisdom suggests that deep learning is bad on small data sets. However, XGBoost has a much higher performance than MLP on the 42 larger classification data sets and the regression data sets. Based on this observation, we use MLP and XGBoost as our main baselines.
\end{itemize}

\begin{table}[t]
    \centering
    \resizebox{\columnwidth}{!}{%
    \begin{tabular}{l|llllllll}
    \toprule
         & Linear & Rand-Forest & CatBoost & XGBoost & MLP & ResNet & FT-Transformer & TabPFN \\
    \midrule
        \texttt{cls-56} Perf  & \res{79.03}{0.28} & \res{82.01}{0.31} & \res{82.79}{0.29} & \res{82.03}{0.30} &  \res{83.59}{0.29} & \bres{83.63}{0.29} & \res{82.84}{0.30} & \res{82.96}{0.30} \\ 
        \texttt{cls-56} Rank  & 5.38 & 4.93 & 4.11 & 5.00 & 3.14 & \textbf{2.95} & 3.84 & 3.75 \\ 
        Prior work$^*$  & 61 & 76 & 85 & 74 & 57 & 77 & 75 & 84 \\ 
    \midrule
        \texttt{cls-42} Perf  & \res{85.05}{0.09} & \res{89.03}{0.08} & \res{89.71}{0.09} & \bres{90.43}{0.08} & \res{90.04}{0.10} & \res{90.03}{0.09} & \res{90.42}{0.09} & \res{87.98}{0.08} \\ 
        \texttt{cls-42} Rank  & 6.26 & 4.36 & 4.07 & 2.83 & 3.17 & 3.05 & \textbf{2.81} & 5.10 \\ 
    \midrule
        \texttt{reg-20} Perf  & \res{56.60}{1.20} & \res{78.76}{0.43} & \bres{80.35}{0.37} & \res{80.26}{0.41} & \res{78.22}{0.41} & \res{78.37}{0.44} & \res{57.41}{0.36} & N/A \\ 
        \texttt{reg-20} Rank  & 6.40 & 4.10 & 3.20 & \textbf{2.55} & 3.30 & 3.65 & 3.95 & N/A \\ 
    \bottomrule
    \end{tabular}
    }
    \caption{Baseline average performances (accuracy or $R^2$-score) (\%) and rankings. Standard deviations reported in the subscripts. $^*$For \texttt{cls-56}, we compare with the numbers reported in the prior work by \cite[Table~2]{mcelfresh2023neural}, whose caption says ``57 data sets'' but one data set is actually duplicated.}
    \label{tab:baselines}
\end{table}

\paragraph{Methods.}
We consider the following four types of contexts:
\begin{enumerate}[label=(\roman*)]
    \item Y-Linear kernel (Y-Lin): We use the centered linear kernel on $Y$ defined after \Cref{thm:obj-regression}, that is $k(y,y') = \dotp{\tilde{y}, \tilde{y}'}$. More specifically, we use an STK of the kernel, which transforms all eigenvalues above threshold $c = 0.1$ to $1$ and the rest to $0$.
    \item XGBoost (XGB): An XGBoost model consists of $d_t$ trees. We construct a teacher model $\Phi_t: \gX \rightarrow \R^{d_t}$, where each dimension is the output of one tree.
    \item SCARF \cite{bahri2022scarf} (SF): Randomly masks some columns and replaces them with random values sampled from $\text{Unif}[c_{\min}, c_{\max}]$, where $c_{\min}$ and $c_{\max}$ are the smallest and largest values of this column in the training set. 
    \item Cutmix (CM): Randomly masks some columns and replace them with values from the same column but other random rows.
\end{enumerate}

We use $\oplus$ to denote concatenation. For example, ``$\oplus$ XGB'' refers to concatenation with the XGBoost teacher model.
We use ``$+$'' to denote convex combination, where the weights $\vw$ are the minimax weights in \Cref{eqn:convex-combination-game}.

\begin{table}[t]
    \centering
    \resizebox{\columnwidth}{!}{%
    \begin{tabular}{l|lllllll}
    \toprule
         & XGBoost & MLP & Y-Lin & SCARF & Cutmix & SCARF $+$ Y-Lin & Cutmix $+$ Y-Lin \\
    \midrule
        \texttt{cls-56} Perf & \res{82.03}{0.30} & \res{83.59}{0.29} & \res{83.58}{0.31} & \res{82.08}{0.30} & \res{81.93}{0.30} & \res{83.49}{0.32} & \bres{83.64}{0.32} \\
        \texttt{cls-56} Rank & 4.89 & \textbf{2.96} & 3.25 & 4.00 & 4.18 & 3.36 & 3.05 \\ 
    \midrule
        \texttt{cls-42} Perf & \res{90.43}{0.08} & \res{90.04}{0.10} & \bres{90.60}{0.07} & \res{88.15}{0.09} & \res{88.06}{0.09} & \res{90.55}{0.08} & \res{90.56}{0.08} \\ 
        \texttt{cls-42} Rank & 3.19 & 3.12 & \textbf{2.45} & 4.81 & 4.88 & \textbf{2.45} & 2.50 \\ 
    \bottomrule
    \end{tabular}
    }
    \caption{Results on the 98 classification datasets.}
    \label{tab:cls-results}
\end{table}

\begin{table}[t]
    \centering
    \resizebox{\columnwidth}{!}{%
    \begin{tabular}{l|lllllll}
    \toprule
         & XGBoost & MLP & Y-Lin & SF $+$ Y-Lin & CM $+$ Y-Lin & (SF $+$ Y-Lin) $\oplus$ XGB & (CM $+$ Y-Lin) $\oplus$ XGB \\
    \midrule
        \texttt{reg-20} Perf & \res{80.26}{0.41} & \res{78.22}{0.41} & \res{78.80}{0.61} & \res{79.17}{0.45} & \res{78.99}{0.53} & \bres{81.07}{0.47} & \res{81.04}{0.52} \\ 
        \texttt{reg-20} Rank & 3.55 & 3.95 & 3.90 & 3.60 & 4.10 & \textbf{2.50} & 2.65 \\ 
    \bottomrule
    \end{tabular}
    }
    \caption{Results on the 20 regression datasets. We omit SCARF and Cutmix because their performances are low.}
    \label{tab:reg-results}
\end{table}

\paragraph{Results.}
\Cref{tab:cls-results} reports the performance of seven methods on the classification datasets.
We can see that only using SCARF or Cutmix leads to pretty bad performance, largely because their association is too strong, as shown in \Cref{sec:context-usefulness}.
However, when they are mixed with Y-Lin, the mixture has a weaker association, and thus the performance becomes much better.
The performance of Y-Lin is very close to MLP on the 56 small datasets, but much higher than MLP on the 42 large datasets.
This suggests that representation learning usually works better with larger datasets.

\Cref{tab:reg-results} reports the performance of seven methods on the regression datasets.
We can see that Y-Lin is better than MLP but worse than XGBoost.
Mixing Y-Lin with SCARF or Cutmix slightly improves the performance, and concatenating with XGBoost further significantly improves the performance. The mixture of SCARF, Y-Lin and XGBoost achieves a much higher average performance than XGBoost.

The above results show that if we mix the right set of contexts, then we can achieve a higher performance than a single context.
Moreover, on all three benchmarks our methods significantly improve over XGBoost, and on all but \texttt{cls-56} our methods significantly improve over MLP.
These experiments showcase the practical value of the contexture theory.
Note that these experiments are only a start, and we expect there to be a large room of improvement if we can find better contexts for tabular data.

In summary, when we have multiple contexts with either strong or weak associations, we can mix them to obtain a better context with a moderate association.
This chapter introduced three base operations: convolution, convex combination and concatenation.
While mixing multiple contexts is a useful method, to achieve a revolutionary breakthrough, we still need to create new contexts that are completely different from existing ones.
In other words, we are not suggesting that context scaling can be achieved by solely mixing existing contexts.

\chapter{Statistical Learning Bounds for Representation Learning}
\label{chap:bounds}

This chapter studies representation learning in the finite sample regime, that is when there are only finite pretraining and downstream samples, how well one can learn the encoder and the downstream predictor.
Recall that the prediction error can be decomposed as the sum of the approximation error and the estimation error.
The key takeaways of this chapter are summarized as follows.

\begin{enumerate}[label=(\roman*)]
    \item When the embedding dimension $d$ increases, the approximation error decreases, but the estimation error increases.
    \item The approximation error consists of two parts: (a) the distance from the target function to the function class; (b) how well the function class can be approximated.
    \item Contexture learning (extracting the top-$d$ eigenspace) can be viewed as a spectrally transformed kernel (STK). It transforms all eigenvalues other than the top-$d$ to be zero. It loses some information, but achieves the fastest eigenvalue decay.
    \item Other STKs such as the inverse Laplacian are popular in semi-supervised learning, because they are more efficient than extracting the top-$d$ eigenspace.
\end{enumerate}

This chapter first defines the context complexity, and then proves the generalization bounds for contexture learning.
After that, we generalize contexture learning to spectrally transformed kernel regression (STKR), and prove generalization bounds for STKR.

\section{Context Complexity}

So far, our intrinsic evaluation of a context has only relied on the singular values of the context.
However, the singular functions of the context, particularly their smoothness, also have a great impact on the generalization performance.
Consider two contexts with similar spectra, but the first one has smoother singular functions than the second one.
Then, given the same number of pretrain samples, it is easier to approximate the top-$d$ singular functions of the first context.
The context complexity is defined to mathematically characterize such smoothness.

\begin{definition}
\label{def:context-complexity}
The \textbf{context complexity} of $P^+$ is defined as $\bt := \norm{\kx}_{\infty}^{1/2}$, such that
\begin{equation*}
    \kx(x,x) = \sum s_i^2 \mu_i(x)^2  = \int \frac{P^+(x|a) P^+(a|x)}{\px(x)} da = D_{\chi^2}\paren{ P^+(\cdot | x) \parallel \pa } + 1 \le \bt^2
\end{equation*}
holds for $\px$-almost all $x$, where $D_{\chi^2}(P \parallel Q) = \int (\frac{dP}{dQ} - 1)^2 dQ$ is the $\chi^2$-divergence.
\end{definition}

This $\bt$ was initially introduced as the \textit{augmentation complexity} by \cite{zhai2023understanding} in the context of self-supervised learning.
If $\bt$ is finite, then $\bt^2 \ge \int \kx(x,x) d \px(x) = \sum s_i^2$, which means that $\tkx$ is a trace-class operator.
If $k$ is the centered kernel of $\kx$, then $k(x,x) = \sum_{i \ge 1} s_i^2 \mu_i(x)^2 \le \bt^2 - 1$. Usually $\bt \gg 1$, so we use $k(x,x) \le \bt^2$ for simplicity.

Now let us see some examples of masking, and estimate their context complexity.
Intuitively, the context of a data augmentation is more complex if the augmentation is stronger, that is the association between $X$ and $A$ is weaker.
For masking, the mask ratio clearly controls the context complexity.
In addition, the complexity also depends on the type of masking.
For example, consider a checkerboard-style masking, where for any two adjacent pixels, exactly one of them is masked.
This masking has a mask ratio of 50\%, but clearly it is much weaker than the one that puts all 50\% masking on the center of the image where the object is located.
The context complexity provides a quantitative way to measure the strength of a data augmentations of different types.

\paragraph{Hypercube data model.}
Consider the hypercube data model introduced by \cite{saunshi2022understanding}: $\gX = \oset{-1,1}^{d_{\gX}}$, and $\px$ is the uniform distribution over $\gX$.
Consider three random masking methods similar to those studied in \cite{cabannes2023ssl}:
(i) Independent random masking;
(ii) Cutout-like block masking \cite{devries2017improved};
(iii) BERT-like masking.
See \Cref{fig:demo-aug} (left) for an illustration.
Denote the mask ratio by $\alpha$. Let us compute the $\bt$ for these three masking methods, which are denoted by $\bt_r, \bt_c, \bt_b$, respectively.

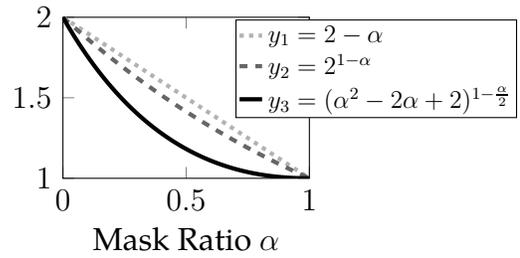
\begin{figure}[t]
    \centering
    \begin{tikzpicture}
  \node[anchor=west] at (-1.5,0) {3. Block Mask + Flip};
  \node[anchor=west] at (-1.5,0.7) {2. Block Mask};
  \node[anchor=west] at (-1.5,1.4) {1. Random Mask};
  \node[anchor=west] at (-1.5,2.1) {Original Sample};

  \node[] at (3.2,2.1) {1};
  \node[] at (3.6,2.1) {1};
  \node[] at (4.0,2.1) {-1};
  \node[] at (4.4,2.1) {-1};
  \node[] at (4.8,2.1) {1};
  \node[] at (5.2,2.1) {-1};
  \node[] at (5.6,2.1) {1};
  \node[] at (6.0,2.1) {1};

  \node[draw,inner sep=2pt] at (3.2,1.4) {0};
  \node[draw,inner sep=2pt] at (3.6,1.4) {0};
  \node[] at (4.0,1.4) {-1};
  \node[draw,inner sep=2pt] at (4.4,1.4) {0};
  \node[] at (4.8,1.4) {1};
  \node[] at (5.2,1.4) {-1};
  \node[] at (5.6,1.4) {1};
  \node[draw,inner sep=2pt] at (6.0,1.4) {0};

  \node[] at (3.2,0.7) {1};
  \node[] at (3.6,0.7) {0};
  \node[] at (4.0,0.7) {0};
  \node[] at (4.4,0.7) {0};
  \node[] at (4.8,0.7) {0};
  \node[] at (5.2,0.7) {-1};
  \node[] at (5.6,0.7) {1};
  \node[] at (6.0,0.7) {1};

  \draw (3.45,0.9) -- (4.95,0.9) -- (4.95,0.5) -- (3.45,0.5) -- (3.45,0.9);

  \node[] at (3.2,0.0) {1};
  \node[] at (3.6,0.0) {0};
  \node[] at (4.0,0.0) {0};
  \node[] at (4.4,0.0) {-1};
  \node[] at (4.8,0.0) {1};
  \node[draw, ultra thick,rounded corners=1mm,inner sep=2pt] at (5.2,0.0) {1};
  \node[] at (5.6,0.0) {1};
  \node[draw, ultra thick,rounded corners=1mm,inner sep=2pt] at (6.0,0.0) {-1};

  \draw (3.45,0.2) -- (4.15,0.2) -- (4.15,-0.2) -- (3.45,-0.2) -- (3.45,0.2);

  \node[draw, inner sep=2pt] at (1.5, -0.8) {Masked};
  \node[draw,ultra thick, rounded corners=1mm,inner sep=2pt] at (4,-0.8) {Flipped};

\begin{axis}[
    xmin = 0, xmax = 1,
    ymin = 1, ymax = 2,
    xtick distance = 0.5,
    ytick distance = 0.5,
    xlabel={Mask Ratio $\alpha$},
    x label style={at={(axis description cs:0.5,-0.3)},anchor=north,inner sep=0},
    major grid style = {lightgray},
    minor grid style = {lightgray!25},
    width = .3\linewidth,
    height = .23\linewidth,
    legend style = {nodes={scale=0.8, transform shape},at={(0.7,0.68)},anchor=west,inner sep=.5pt},
    legend cell align=left,
    at={(8.5cm, 0)},
    ]
    \addplot[domain = 0:1,
        samples = 200,
        smooth,
        ultra thick,
        dotted,
        black!30
    ] {2-x};
    \addlegendentry{$y_1=2-\alpha$};
    \addplot[domain = 0:1,
        samples = 200,
        smooth,
        ultra thick,
        dashed,
        black!60
    ] {2^(1-x)};   
    \addlegendentry{$y_2=2^{1 - \alpha}$};
    \addplot[domain = 0:1,
        samples = 200,
        smooth,
        ultra thick,
    ] {(x^2-2*x+2)^(1-x/2)};
    \addlegendentry{$y_3=(\alpha^2 - 2\alpha + 2)^{1 - \frac{\alpha}{2}}$};
\end{axis}

\end{tikzpicture}
    \caption{\textbf{Left:} Three mask-type data augmentations on the hypercube data model. \textbf{Right:} Their theoretical $\bt^{2/d_\gX}$ with different mask ratio $\alpha$.}
    \label{fig:demo-aug}
\end{figure}

\begin{example}
Consider a \emph{random masking} augmentation, \ie{} for any $x \in \gX$, each coordinate $x^{(i)}$ is randomly and independently masked to be $0$ (\ie{} $0$ denotes the \texttt{[MASK]} token)
with probability $\alpha \in (0,1)$ .
Then, its context complexity is given by $\bt_r^2 = (2 - \alpha)^{d_\gX}$.
\end{example}

\begin{example}
Consider \emph{random block masking},
\ie{} masking $x^{(i)},x^{(i+1)},\cdots,x^{(i+r-1)}$ for $r = \lceil \alpha d_{\gX} \rceil$ and a uniformly random $i \in [d_{\gX}-r]$, for any $x \in \gX$.
Then, $\bt_c^2 \le [2^{(1-\alpha) }]^{d_{\gX}}$.
\end{example}

\begin{example}
Consider \emph{random block masking with flipping}, where for any $x \in \gX$,
first mask $x^{(i)},\cdots,x^{(i+r-1)}$ to be $0$ 
for $r = \lceil \alpha d_{\gX} \rceil$ and a uniformly random $i \in [d_{\gX}-r]$,
then randomly flip the sign of each remaining coordinate independently with probability $\frac{\alpha}{2}$.
Then, its context complexity is bounded by $\bt_b^2 \le \left [(\alpha^2-2\alpha+2)^{(1-\alpha/2)} \right ]^{d_{\gX}}$.
\end{example}

See \Cref{app:mask-kappa} for the derivation of the above $\bt$.
\Cref{fig:demo-aug} (right) plots the $\bt^{2 / d_{\gX}}$ for all three examples.
We can see that $\bt$ becomes lower as the mask ratio $\alpha$ increases.
Moreover, when $\alpha \in (0,1)$ is fixed, Cutout-like masking has a lower $\bt$ than independent random masking, and BERT-like masking has a lower $\bt$ than Cutout-like masking.
Cutout has a weaker association than random masking, and BERT has an even weaker association.
Thus, these examples show that $\bt$ is lower when the association is weaker.

Another observation is that all three $\bt$ have an exponential dependency on $d_{\gX}$.
This is a manifestation of the typical curse of dimensionality in high-dimensional statistics.
One way to make $\bt$ polynomial in $d_{\gX}$ is to use a context with very weak association, such as a very strong data augmentation.
For example, if $\gA$ is a finite set with a small size, then $\bt$ will be polynomial.
However, contexts with such weak association usually lead to substantially worse performance in practice.
The bounds to be proved in this chapter depend on $\bt$ polynomially, meaning that they are not really useful in the high-dimensional scenario.
In practice, however, representation learning can still achieve good performances when the data dimension is high.
How to address this discrepancy is posed as an open problem.

\paragraph{Real language models.}

Now let us estimate the $\bt$ of some real language models.
We use the NLP dataset \texttt{wikipedia-simple}, and consider the context of masking tokens, where $x$ is a complete text whereas $a$ is a masked version of $x$.
Recall that $\bt^2$ is an upper bound of $\kx(x,x) = \int \frac{P^+(x|a) P^+(a|x)}{\px(x)} da$.
For a fixed $x$, this integration can be estimated with Monte Carlo (by sampling a set of $a \sim P^+(\cdot|x)$), and then $\bt^2$ can be estimated by its maximum over $x \in \gX$.
For $x = [x^{(1)}, \cdots, x^{(l)}]$ where $x^{(i)}$ is the $i^{th}$ token, we have
\begin{equation*}
    \log P^+ \left (x \middle \vert a \right ) = \log P^+ \left (x^{(1)} \middle \vert a \right ) + \log P^+ \left (x^{(2)} \middle \vert a,x^{(1)} \right ) + \cdots + \log P^+ \left (x^{(l)} \middle \vert a, x^{(1)}, \cdots, x^{(l-1)} \right ) .
\end{equation*}
We can leverage a bi-directional masked language model such as a BERT, and then compute $P^+(x^{(i)} | a, x^{(<i)})$ auto-regressively:
For each $i \in [l]$, use the BERT to output $P^+(x^{(i)}|a,x^{(<i)})$, and then replace $a^{(i)}$ with $x^{(i)}$ for $i+1$.
As such, we can estimate $P^+(x|a)$, and $\px(x)$ can be estimated by $P^+(x|a_0)$ where $a_0$ is a fully masked text.

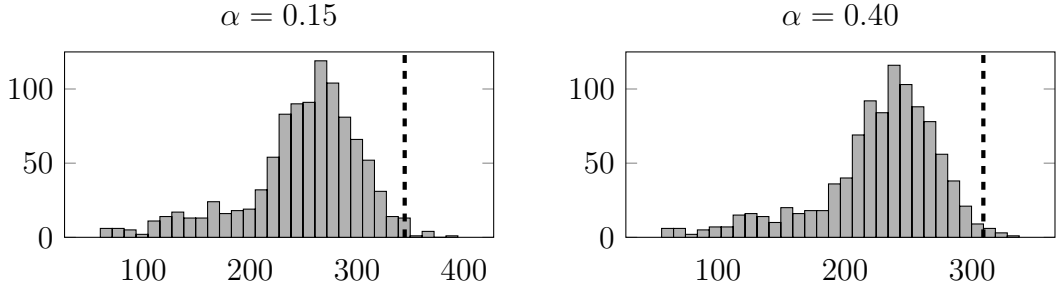
\begin{figure}[t]
    \centering
    \begin{tikzpicture}
\begin{axis}[
    name=plot1,
    ybar,
    ymin=0,
    ymax=125,
    width=.45\textwidth,
    height=.25\textwidth,
    xtick style={draw=none},
    title = {$\alpha=0.15$}
]
\addplot [hist={bins=30}, fill=black!30] table [y index=0] {data/s_arr_015.csv};
\draw [ultra thick, dashed] (axis cs:344.595,0) -- (axis cs:344.595,125);
\end{axis}

\begin{axis}[
    name=plot2,
    at={($(plot1.south east)+(50,0)$)},
    ybar,
    ymin=0,
    ymax=125,
    width=.45\textwidth,
    height=.25\textwidth,
    xtick style={draw=none},
    title = {$\alpha=0.40$}
]
\addplot [hist={bins=30}, fill=black!30] table [y index=0] {data/s_arr_040.csv};
\draw [ultra thick, dashed] (axis cs:308.805,0) -- (axis cs:308.805,125);
\end{axis}
\end{tikzpicture}
    \caption{Histograms of $\log \kx(x,x)^2$ for random masking on \texttt{wikipedia-simple} with mask ratio $\alpha$. The dashed vertical line in each plot indicates the $99^{\text{th}}$ percentile.}
    \label{fig:kappa-hist}
\end{figure}

A natural idea is to estimate $\kx(x,x)$ for a random subset of samples, and output their maximum as an estimate of $\bt^2$.
However, this approach has two issues.
First, $\sup_x \kx(x,x)$ is statistically impossible to estimate from a subset of data without any extra assumptions on the distribution of $\kx(x,x)$.
Second, almost all real datasets contain outliers, which are very different from most samples.
These outliers have very large $\kx(x,x)$, but given that these outliers have little impact on the actual pretraining, we do not want to take these large $\kx(x,x)$ into account.
To fix these two issues, we can use the $99^{\text{th}}$ percentile of $\kx(x,x)$.
First, the percentile can be estimated with a finite confidence interval via sampling regardless of the distribution of $\kx(x,x)$ \cite[Section~5.2]{hahn2011statistical}.
Second, we get rid of the outliers if they are fewer than 1 percent.
\Cref{fig:kappa-hist} plots the histograms of $\log \kx(x,x)^2$ for random masking on \texttt{wikipedia-simple}.
The dashed line in each plot indicates the $99^{\text{th}}$ percentile.
We can see from the plots that the $99^{\text{th}}$ percentile is a good choice, as it picks out the outliers where $\kx(x,x)$ is too large.

\Cref{fig:aug-comp-nlp} (left) plots the $99^{\text{th}}$ percentile of $\log \bt^2$ of four contexts: random masking, random masking with flipping, block masking, and block masking with flipping.
Masking randomly masks $\alpha$ of the tokens.
Masking with flipping masks $\alpha/2$ of the tokens and replaces another $\alpha/2$ of the tokens with random tokens.
Note that this replace rate is higher than the common 80-10-10 strategy in NLP, because we want to magnify the effect of flipping.
From the plot, we can see that the complexity drops as $\alpha$ increases as expected.
One observation is that the ``Random + Flip'' curve intersects with ``Block'' and ``Block + Flip'',
suggesting that block masking has a stronger effect when $\alpha$ is small, whereas flipping has a stronger effect when $\alpha$ is large.

\begin{figure}[t]
    \centering
    \begin{tikzpicture}
\begin{axis}[
name=plot0,
    xmin = 0, xmax = 0.8,
    ymin = 80, ymax = 360,
    xtick distance = 0.2,
    ytick distance = 100,
    yticklabel shift = -2.5,
    xlabel={Mask Ratio $\alpha$},
    x label style={at={(axis description cs:0.5,-0.28)},anchor=north,inner sep=0},
    x tick label style={font=\footnotesize},
    ylabel={$\log \bt^2$},
    y label style={at={(axis description cs:-0.41,0.2)},anchor=west,inner sep=0},
    y tick label style={font=\footnotesize},
    width = .24\linewidth,
    height = .22\linewidth,
    legend style = {nodes={scale=0.8, transform shape},at={(1.05,0.5)},anchor=west,inner sep=.5pt},
    legend cell align=left,
    ]
    \addplot[
        ultra thick,
        solid,
        opacity=0.3
    ] 
    coordinates {
    (0.05,353.9439)(0.10,348.45206)(0.15,346.39916)(0.20,339.02706)(0.30,325.07554)(0.40,307.0025)(0.50,278.59496)(0.60,250.01232)(0.70,214.14242)(0.80,160.24052)
    };
    \addlegendentry{Random};
    \addplot[
        solid,
        semithick
    ] 
    coordinates {
    (0.05,350.98668)(0.10,349.13778)(0.15,342.77486)(0.20,335.01478)(0.30,312.20536)(0.40,280.7551)(0.50,241.58934)(0.60,198.1497)(0.70,151.57564)(0.80,84.51188)
    };
    \addlegendentry{Random + Flip};
    \addplot[
        ultra thick,
        dashed,
        opacity=0.7
    ] 
    coordinates {
    (0.05,346.77186)(0.10,339.8619)(0.15,327.13244)(0.20,313.89546)(0.30,280.6536)(0.40,254.40494)(0.50,217.35124)(0.60,187.59812)(0.70,152.70328)(0.80,113.85798)
    };
    \addlegendentry{Block};
    \addplot[
        ultra thick,
        dotted,
        opacity=0.6
    ] 
    coordinates {
    (0.05,350.4642)(0.10,343.733)(0.15,336.8131)(0.20,322.86658)(0.30,302.60346)(0.40,279.2389)(0.50,240.72884)(0.60,205.4515)(0.70,164.48324)(0.80,108.44108)
    };
    \addlegendentry{Block + Flip};
\end{axis}

\begin{axis}[
name=plot1,
at={($(plot0.south east)+(140,0)$)},
    xmin = 0, xmax = 0.8,
    ymin = 70, ymax = 92,
    xtick distance = 0.2,
    ytick distance = 10,
    xtick pos=left,
    yticklabel shift = -2.5,
    axis y line*=left,
    xlabel={Mask Ratio $\alpha$},
    x label style={at={(axis description cs:0.5,-0.26)},anchor=north,inner sep=0},
    x tick label style={font=\footnotesize},
    ylabel={Test Accuracy},
    ylabel style = {font=\footnotesize},
    y label style={at={(axis description cs:-0.29,-0.1)},anchor=west,inner sep=0},
    width = .25\linewidth,
    height = .22\linewidth,
    ]
    \addplot[
        very thick
    ] 
    coordinates {
    (0.05,85.5208)(0.10,85.4109)(0.15,86.8021)(0.20,86.4543)(0.30,86.3262)(0.40,86.2713)(0.50,84.7703)(0.60,82.354)(0.70,78.3452)(0.80,73.2748)
    };
    \node[circle,fill=black,inner sep=0pt,minimum size=6pt]  at (0.15,86.8021) {};
    \node[font=\footnotesize] at (0.47,89.5) {$(0.15,86.8)$};
\end{axis}
\begin{axis}[
at={($(plot0.south east)+(140,0)$)},
    xmin = 0, xmax = 0.8,
    ymin = -0.4, ymax = 4.4,
    ytick distance = 2,
    hide x axis,
    axis y line*=right,
    width = .25\linewidth,
    height = .22\linewidth,
    ]
    \addplot[
        ultra thick,
        dashed
    ] 
    coordinates {
    (0.05,4.2055)(0.10,3.8285)(0.15,3.4483)(0.20,3.2481)(0.30,2.9752)(0.40,2.0678)(0.50,0.9194)(0.60,0.9785)(0.70,0.8668)(0.80,-0.386)
    };
\end{axis}

\begin{axis}[
name=plot2,
at={($(plot1.east)+(30,0)$)},
anchor=west,
    xmin = 0, xmax = 0.8,
    ymin = 80, ymax = 97,
    xtick distance = 0.2,
    ytick distance = 5,
    xtick pos=left,
    yticklabel shift = -2.5,
    axis y line*=left,
    xlabel={Mask Ratio $\alpha$},
    x label style={at={(axis description cs:0.5,-0.26)},anchor=north,inner sep=0},
    x tick label style={font=\footnotesize},
    width = .25\linewidth,
    height = .22\linewidth,
    ]
    \addplot[
        very thick
    ] 
    coordinates {
    (0.05,88.5321)(0.10,89.6789)(0.15,91.8578)(0.20,91.7431)(0.30,91.3991)(0.40,92.4312)(0.50,89.2202)(0.60,86.6972)(0.70,84.289)(0.80,80.1606)
    };
    \node[circle,fill=black,inner sep=0pt,minimum size=6pt]  at (0.40,92.4312) {};
    \node[font=\footnotesize] at (0.375,95) {$(0.40,92.4)$};
\end{axis}
\begin{axis}[
at={($(plot1.east)+(30,0)$)},
anchor=west,
    xmin = 0, xmax = 0.8,
    ymin = 5.5, ymax = 12.5,
    xtick distance = 0.2,
    ytick distance = 3,
    yticklabel shift = -2.5,
    axis y line*=right,
    hide x axis,
    ylabel={Train-Test Gap},
    ylabel style = {font=\footnotesize},
    y label style={at={(axis description cs:1.3,-0.15)},anchor=west,inner sep=0},
    width = .25\linewidth,
    height = .22\linewidth,
    ]
    \addplot[
        ultra thick,
        dashed
    ] 
    coordinates {
    (0.05,10.6408)(0.10,9.6411)(0.15,7.3939)(0.20,7.4581)(0.30,7.2631)(0.40,5.7573)(0.50,8.0032)(0.60,9.1022)(0.70,9.8549)(0.80,11.6106)
    };
\end{axis}

\node[font=\footnotesize] at (7.7,0.2) {QNLI};
\node[font=\footnotesize] at (11.25,0.2) {SST-2};

\end{tikzpicture}
    \caption{\textbf{Left:} Estimated $\log \bt^2$ ($99^{\text{th}}$ percentile) on \texttt{wikipedia-simple}, which is the average of five runs with different random seeds. \textbf{Right:} Downstream performance on QNLI and SST-2. The solid line is the test accuracy and the dashed line is the train-test gap. The highest test accuracy is labeled on each plot.}
    \label{fig:aug-comp-nlp}
\end{figure}
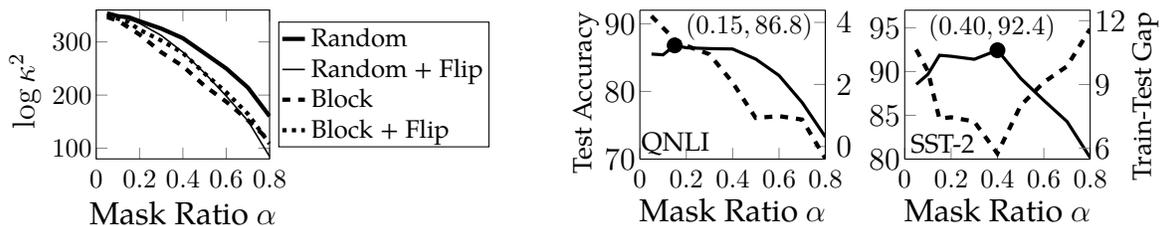

\Cref{fig:aug-comp-nlp} (right) plots the real downstream performance of BERT with different mask ratios on QNLI~\cite{wang2018glue} and SST-2~\cite{SST2}.
The models are \texttt{roberta-large} trained with the fast pretraining recipe in \cite{wettig2022should}. The context is random masking without the 80-10-10 strategy.
At downstream, the encoder is fine-tuned along with the linear head following common practice.
From the plot, we can see that the highest test accuracy (solid line) is achieved at $\alpha = 0.15$ on QNLI and at $\alpha = 0.40$ on SST-2. This is because the association between $X$ and $A$ is moderate when $\alpha$ is neither too big nor too small.
The dashed line is the gap between the train accuracy and test accuracy.
On QNLI, the gap monotonically decreases with $\alpha$;
on SST-2, it is U-shaped, with the lowest at $\alpha = 0.40$.
We will come back to explain this observation after we prove the generalization bounds.

\section{Generalization Bounds for Contexture Learning}

Suppose there are $m$ pretrain samples $x_1,\cdots,x_m$ and $n$ downstream samples $\tilde{x}_1,\cdots,\tilde{x}_n$, jointly \iid{} sampled from $\px$.
Usually $m \gg n$.
This section proves an error bound when an encoder is learned with $x_1,\cdots,x_m$, and then a predictor is fit with $\tilde{x}_1,\cdots,\tilde{x}_n$.
Assume that the context has kernel access, and assume that the kernel $k$ that we have access to is the centered kernel of $\kx$.
In practice there could be a difference between $\kx$ and $k$, but since this difference depends on what $k$ we have, this is not something we can analyze in our study of generalization.
Therefore, for simplicity we ignore this difference. 
The scenario where the context has pair or transformation access will be discussed later.

Let $s_i^2, \mu_i$ be the actual eigenvalues and eigenfunctions of $T_k$.
The \textbf{\boldmath empirical top-$d$ eigenfunctions} of $T_k$ can be estimated with kernel PCA as introduced at the beginning of \Cref{chap:objectives}, which is reiterated as follows.
\begin{enumerate}[label=(\roman*)]
    \item Compute the Gram matrix $\mG \in \R^{m \times m}: \mG[i,j] = k(x_i,x_j)$.
    \item Let $\lambda_1 \ge \lambda_2 \ge \cdots \ge \lambda_m \ge 0$ be the eigenvalues of $\mG$, with orthonormal eigenvectors $\vv_1,\cdots,\vv_m$. Note that $\lambda_1$ has the same order as $m$, so it can be much greater than $1$. 
    \item Assume that $\lambda_d > 0$. Let $\phi_i(x) = \frac{1}{\sqrt{\lambda_i}} \sum_{j=1}^m \vv_i[j] k(x,x_j) $, for $i \in [d]$.
\end{enumerate}

Although ideally we would like a bound for any target function in $\fep$, we cannot prove a universal bound for the entire $\fep$.
To see why, consider the following example.
Suppose $f^*$ is a $(1-\epsilon/2)$-compatible function that is easy to learn, and $\norm{f^*} = 1$.
Suppose there is a $\mu_D$ with an extremely large $D$, such that the functiion is very noisy and it has large values on all the samples we have.
Let $f' = f^* + \frac{\epsilon}{4} \mu_D$. It is easy to show that $f' \in \fep$, but learning $f'$ with the samples we have is impossible.
In order to make learning possible, we have to get rid of the extremely noisy components.

To this end, we define a variant of the induced RKHS in \Cref{sec:induced-rkhs}.

\begin{definition}
\label{def:hk}
Define the \textbf{induced RKHS} of $k$ as $\hk = \sset{f = \sum s_i u_i \mu_i}{\sum u_i^2 < \infty}$.
For $f_1 = \sum u_i \mu_i, f_2 = \sum v_i \mu_i \in \hk$, define their inner product as $\dotp{f_1, f_2}_{\hk} = \sum \frac{u_i v_i}{s_i^2}$.
\end{definition}

Note that $\hk$ is a subset of $\lxp$, rather than $\lap$.
Obviously $\hk$ is the RKHS of $k$, that is for any $f \in \hk$, we have $\dotp{f, k(\cdot,x)}_{\hk} = f(x)$.
In particular, we have $\dotp{k(\cdot,x), k(\cdot,x')}_{\hk} = k(x,x')$.
This implies that for all $i, j \in [d]$, we have
\begin{equation*}
    \dotp{\phi_i, \phi_j}_{\hk} = \frac{1}{\sqrt{\lambda_i \lambda_j}} \sum_{q,r=1}^m \vv_i[q] \vv_j[r] \dotp{k(\cdot,x_q), k(\cdot,x_r)}_{\hk} = \frac{1}{\sqrt{\lambda_i \lambda_j}} \vv_i^{\top} \mG \vv_j = \sI[i=j] .
\end{equation*}
That is, $\phi_1,\cdots,\phi_d$ are orthonormal in the Hilbert space $\hk$.
Note that when there are finite samples, it is impossible to make $\phi_1,\cdots,\phi_d$ orthonormal in $\lxp$ since we have no access to $\px$.
Given $\hk$, define the following set of compatible functions.

\begin{definition}
\label{def:compatible-hk}
The set of functions \textbf{\boldmath $(1-\epsilon)$-compatible with $\hk$} is defined as
\begin{equation*}
    \fhkp = \sset{f \in \hk}{\norm{\tilde{f}}_{\hk} \ge \norm{\tilde{f}}_{\px} \ge (1-\epsilon) \norm{\tilde{f}}_{\hk}}  .
\end{equation*}
\end{definition}

\begin{proposition}
$\fhkp \subseteq \fep$.
\end{proposition}
\begin{proof}
    Let $f = \sum_i u_i \mu_i \in \fhkp$.
    Then, $\sum_i u_i^2 \ge (1-\epsilon) \sum_i \frac{u_i^2}{s_i^2}$.
    By Cauchy-Schwarz inequality, we have $\paren{\sum_i s_i^2 u_i^2} \paren{\sum_i \frac{u_i^2}{s_i^2}} \ge \paren{\sum_i u_i^2}^2$.
    Thus, $\sum_i s_i^2 u_i^2 \ge \sum_i u_i^2$.
\end{proof}

The following lemma can be proved in the same way as \Cref{lem:ratio-trace}, which is left as an exercise to the reader.
\begin{lemma}
\label{lem:hk}
For any $f_1,\cdots,f_d \in \hk$ such that $\dotp{f_i, f_j}_{\hk} = \sI[i=j]$, we have
\begin{equation*}
    \norm{f_1}_{\px}^2 + \cdots + \norm{f_d}_{\px}^2 \le s_1^2 + \cdots + s_d^2 .
\end{equation*}
\end{lemma}

\paragraph{Approximation error bound.}
Our goal is to prove a universal bound for $\fhkp$.
First, let us bound the approximation error.
Previously, the approximation error was defined as the distance from $f^*$ to the span of $\Phi$ in space $\lxp$.
However, since we cannot compute the distance in space $\lxp$ with only finite samples, here we define it with the distance in space $\hk$.
Specifically, let $\fPhi$ be the projection of $f^*$ onto the span of $\Phi$ in space $\hk$, that is $\dotp{f^* - \fPhi, \fPhi}_{\hk} = 0$.
Let $f^* - \fPhi = \beta f_0$, where $\beta \ge 0$ and $\norm{f_0}_{\hk} = 1$.
The above lemma implies that $\norm{\phi_1}_{\px}^2 + \cdots + \norm{\phi_d}_{\px}^2 + \norm{f_0}_{\px}^2 \le s_1^2 + \cdots + s_{d+1}^2$.

Hence, to bound $\norm{f_0^2}$, it suffices to prove a lower bound for $\norm{\phi_1}_{\px}^2 + \cdots + \norm{\phi_d}_{\px}^2$.
Using the definition of eigenvectors, it is not hard to show that $\frac{1}{m}\sum_{j=1}^m \phi_i(x_j)^2 = \frac{\lambda_i}{m}$ for all $i \in [d]$, which is the empirical $L^2$ norm of $\phi_i$.
Therefore, we need to bound two things: (i) the gap between $\frac{\lambda_1 + \cdots + \lambda_d}{m}$ and $s_1^2 + \cdots + s_d^2$; (ii) the gap between the empirical $L^2$ norm and the actual $L^2$ norm.

Using \cite[Theorem~3.2]{blanchard2007statistical}, for any $\delta \in (0,1)$, gap (i) can be bounded by
\begin{equation}
\label{eqn:approx-1}
\frac{\lambda_1 + \cdots + \lambda_d}{m} \ge s_1^2 + \cdots + s_d^2 - \frac{\bt^2}{\sqrt{m}} \sqrt{\frac{1}{2} \log \frac{6}{\delta}} \qquad \text{with probability at least } 1 - \frac{\delta}{2}  .
\end{equation}

Gap (ii) can be bounded using classical generalization bounds with the Rademacher complexity.

\begin{definition}
\label{def:rademacher}
Let $\gF$ be a function class.
Let $S = \oset{x_1,\cdots,x_m}$ be \iid{} drawn from $\px$.
The \textbf{empirical Rademacher complexity} of $\gF$ on $S$ is defined as
\begin{equation*}
    \hrad_S(\gF)  = \E_{\sigma_1,\cdots,\sigma_m} \brac{ \sup_{f \in \gF} \frac{1}{m} \sum_{i=1}^m \sigma_i f(x_i) } ,
\end{equation*}
where $\sigma_1,\cdots,\sigma_m$ are Rademacher variables, which are \iid{} uniform random variables taking values in $\oset{-1, +1}$.
The \textbf{Rademacher complexity} of $\gF$ is defined as
\begin{equation*}
    \rad_m(\gF) = \E_S[\hrad_S(\gF)]  .
\end{equation*}
\end{definition}

Define the following function class:
\begin{equation*}
    \gF_d := \sset{F = f_1^2 + \cdots + f_d^2}{f_i \in \hk; \; \forall i, j \in [d], \; \dotp{f_i, f_j}_{\hk} = \sI[i=j]}  .
\end{equation*}
Since $\phi_1^2 + \cdots + \phi_d^2 \in \gF_d$, it suffices to bound $\abs{\frac{1}{m} \sum_{i=1}^m F(x_i) - \E[F(X)] }$ for all $F \in \gF_d$, which requires $\rad_m(\gF_d)$ and a bound for $|F(x)|$.
Let $\phi_i = \sum_j u_{ij} s_j \mu_j$, and $\mU = (u_{ij}) = [\vu_1,\cdots,\vu_d]$, which is a matrix with $d$ columns and infinitely many rows. Then, $\mU^{\top} \mU = \mI_d$.
Let $\mM(x) = [s_1 \mu_1(x), s_2 \mu_2(x), \cdots]$.
Then, by \Cref{def:context-complexity}, we have $\norm{\mM(x)}_2^2 = \sum s_i^2 \mu_i(x)^2 \le \bt^2$ for $\px$-almost all $x$.
Thus, for $\px$-almost all $x$, we have
\begin{equation*}
    \abs{F(x)} = \abs{\mM(x)^{\top} \mU \mU^{\top} \mM(x) } \le \norm{\mM(x)}_2^2 \norm{\mU \mU^{\top}}_2 \le \bt^2 \quad \text{for all } F \in \gF_d  .
\end{equation*}

Regarding $\rad_m(\gF_d)$, we have the following result.
\begin{lemma}[Proof in \Cref{app:proof-lem-rad-gfd}]
\label{lem:rad-gfd}
$\rad_m(\gF_d) \le \frac{\sqrt{d}}{\sqrt{m}} \bt^2$.
\end{lemma}

Hence, by \cite[Theorem~4.10]{Wainwright_2019}, with probability at least $1 - \frac{\delta}{2}$ we have
\begin{equation*}
    \abs{\frac{1}{m} \sum_{i=1}^m F(x_i) - \E_{\px}[F(X)] } \le \frac{\bt^2}{\sqrt{m}} \paren{ 2 \sqrt{d} + \sqrt{2 \log \frac{2}{\delta}} } \qquad \text{for all } F \in \gF_d ,
\end{equation*}
which implies that
\begin{equation}
\label{eqn:approx-2}
    \norm{\phi_1}_{\px}^2 + \cdots + \norm{\phi_d}_{\px}^2 \ge \frac{\lambda_1 + \cdots + \lambda_d}{m} - \frac{\bt^2}{\sqrt{m}} \paren{ 2 \sqrt{d} + \sqrt{2 \log \frac{2}{\delta}} }  .
\end{equation}

Combining \Cref{eqn:approx-1,eqn:approx-2}, with probability at least $1 - \delta$ we have
\begin{equation*}
    \norm{\phi_1}_{\px}^2 + \cdots + \norm{\phi_d}_{\px}^2 \ge s_1^2 + \cdots + s_d^2 - \frac{\bt^2}{\sqrt{m}} \paren{ 2 \sqrt{d} +3\sqrt{ \log \frac{6}{\delta}} }  .
\end{equation*}
Thus, by \Cref{lem:hk}, we get
\begin{equation*}
\norm{f_0}_{\px}^2 \le s_{d+1}^2 + \frac{\bt^2}{\sqrt{m}} \paren{ 2 \sqrt{d} +3\sqrt{ \log \frac{6}{\delta}} }  .
\end{equation*}

Let $\alpha = \norm{\fPhi}_{\hk}^2 \ge s_1^{-2} \norm{\fPhi}_{\px}^2$. Then, we have
\begin{equation*}
    \alpha^2 + \beta^2 = \norm{f^*}_{\hk}^2 \le \frac{\norm{f^*}_{\px}^2}{(1-\epsilon)^2} \le \frac{\paren{ \norm{\fPhi}_{\px} + \beta \norm{f_0}_{\px}  }^2}{(1-\epsilon)^2} \le \frac{ s_1^2 \alpha^2 + \beta^2 \norm{f_0}_{\px}^2 + (\alpha^2 + \beta^2) s_1 \norm{f_0}_{\px} }{(1-\epsilon)^2}  .
\end{equation*}
With some simple algebra, we get
\begin{equation*}
  (s_1^2 - \norm{f_0}_{\px}^2) \beta^2 \le \brac{s_1^2 - (1-\epsilon)^2 + s_1 \norm{f_0}_{\px}} (\alpha^2 + \beta^2) \le \brac{s_1^2 - (1-\epsilon)^2 + s_1 \norm{f_0}_{\px}} \frac{\norm{f^*}_{\px}^2}{(1-\epsilon)^2}  .
\end{equation*}
Note that $\norm{f^* - \fPhi}_{\hk} = \beta$, and $\norm{f^* - \fPhi}_{\px} = \beta \norm{f_0}_{\px}$.
Hence, we have proved the following bound for the approximation error.
\begin{theorem}
\label{thm:approx-error}
Let $\fPhi$ be the projection of any $f^* \in \fhkp$ onto the span of $\Phi$, such that $\dotp{f^* - \fPhi, \fPhi}_{\hk} = 0$.
Suppose $(1-\epsilon)^2 < s_1^2$.
Then, for any $\delta \in (0,1)$, with probability at least $1 - \delta$ we have
\begin{align*}
    \norm{f^* - \fPhi}_{\hk}^2 & \le \frac{s_1^2 - (1-\epsilon)^2 + s_1 \sqrt{s_{d+1}^2 + \frac{\bt^2}{\sqrt{m}} \paren{ 2 \sqrt{d} +3\sqrt{ \log \frac{6}{\delta}} }} }{s_1^2 - s_{d+1}^2 - \frac{\bt^2}{\sqrt{m}} \paren{ 2 \sqrt{d} +3\sqrt{ \log \frac{6}{\delta}} } } \cdot \frac{\norm{f^*}_{\px}^2}{(1-\epsilon)^2} ; \\ 
    \norm{f^* - \fPhi}_{\px}^2 & \le  \frac{s_1^2 - (1-\epsilon)^2 + s_1 \sqrt{s_{d+1}^2 + \frac{\bt^2}{\sqrt{m}} \paren{ 2 \sqrt{d} +3\sqrt{ \log \frac{6}{\delta}} }}}{s_1^2 - s_{d+1}^2 - \frac{\bt^2}{\sqrt{m}} \paren{ 2 \sqrt{d} +3\sqrt{ \log \frac{6}{\delta}} } } \cdot \frac{s_1^2 \norm{f^*}_{\px}^2}{(1-\epsilon)^2} ,
\end{align*}
assuming that the denominator is positive.
\end{theorem}

Comparing to \Cref{thm:top-d-optimal}, we can see that this bound is fairly tight when $m$ is sufficiently large. The only differences are that the numerator has an extra term that is close to $s_1 s_{d+1}$ ($s_{d+1}$ should be pretty small), and $\norm{f^*}_{\px}^2$ is multiplied by $\frac{s_1^2}{(1-\epsilon)^2}$.

\paragraph{Estimation error bound.}
Next, we bound the estimation error of the predictor fit on top of $\Phi$ using the $n$ labeled downstream samples. The bound consists of two parts:
\begin{enumerate}[label=(\roman*)]
    \item Assume that the labels are generated from $\fPhi(\tilde{x}_i)$. Let $\hat{f}$ be the learned predictor. Then, the gap $\norm{\hat{f} - \fPhi}_{\px}$ can be bounded using standard generalization bounds.
    \item However, the real labels are generated from $f^*(\tilde{x}_i)$ instead of $\fPhi(\tilde{x}_i)$. Thus, we need to bound the gap between $f^*(\tilde{x}_i)$ and $\fPhi(\tilde{x}_i)$.
\end{enumerate}

Let the downstream training set be $\oset{(\tilde{x}_i, y_i)}_{i=1}^n$ \iid{} drawn from a distribution $P_{\gX,\gY}$ with marginal distribution $\px$, and we assume the following \textbf{moment condition} \cite{fischer2020sobolev}:
\begin{equation}
\label{eqn:moment-condition}
    \E_{P_{\gX,\gY}} \brac{ \abs{ Y - f^*(X) }^r } \le \frac{1}{2} r! \sigma^2 L^{r-2} \quad \text{for all } r \ge 2 \text{ and } \px\text{-almost all } x  ,
\end{equation}
for some $\sigma, L > 0$. For example, if $Y \sim \gN(f^*(X), \sigma^2)$, then this condition holds with $L = \sigma$.
Denote $\vy = [y_1,\cdots,y_n]$, $\vy^* = [f^*(\tilde{x}_1), \cdots, f^*(\tilde{x}_n)]$, and $\vy^*_\Phi = [\fPhi(\tilde{x}_1), \cdots, \fPhi(\tilde{x}_n)]$.
Then, $\vy_\Phi = \vy - \vy^* + \vy_\Phi^*$ consists of the original labels shifted by $\fPhi - f^*$.
Part (i) bounds the estimation error assuming that the labels are $\vy_\Phi$.
Then, part (ii) bounds the gap between  $\vy_{\Phi}$ and $\vy$.

\paragraph{Part (i) Estimation error bound for kernel ridge regression.}
Let the downstream predictor be $f(x) = \vw^{\top} \Phi(x)$.
Here we consider an unbiased linear predictor for simplicity.
The proof when the predictor is biased is the same but a bit more verbose.
Consider fitting $\vw$ with ridge regression given by
\begin{equation*}
    \tilde{\vw} = \argmin_{\vw} \oset{ \frac{1}{n} \paren{ \vy_{\Phi}[i] - \vw^{\top}\Phi(\tilde{x}_i) }^2 + \beta_n \norm{\vw}_2^2 } , 
\end{equation*} 
where the regularization parameter $\beta_n$ can change with $n$.
When $n$ is larger, typically we would like $\beta_n$ to be smaller, because we have more samples so can use a smaller regularization to achieve the same level of generalization.
Let $\tilde{f}(x) = \tilde{\vw}^{\top} \Phi(x)$.
Note that $\tilde{f} \in \hk$, and $\norm{\tilde{f}}_{\hk}^2 = \sum \tilde{w}_i^2 \norm{\phi_i}_{\hk}^2 = \norm{\tilde{\vw}}_2^2$.
Thus, the above regression is equivalent to a kernel ridge regression on $\hk$, where the regularization term is $\beta_n \norm{\tilde{f}}_{\hk}^2$.

Denote $\mPhi = [\Phi(\tilde{x}_1),\cdots,\Phi(\tilde{x}_n)] \in \R^{d \times n}$.
A standard result in statistics shows that
\begin{equation}
\label{eqn:kernel-ridge-regression}
    \tilde{\vw} = \paren{ \mPhi \mPhi^{\top} + n \beta_n \mI_d }^{-1} \mPhi \vy_\Phi  .
\end{equation}

$\norm{\tilde{f} - \fPhi}_{\px}^2$ can be bounded using the results in \cite{fischer2020sobolev}.
To use their results, apart from the moment condition \Cref{eqn:moment-condition}, three additional conditions are required:
\begin{itemize}
    \item \textbf{Eigenvalue decay (EVD):} Let $\zeta_1 \ge \zeta_2 \ge \cdots$ be the eigenvalues of the RKHS. Then, $\zeta_i \le c_1 i^{-\frac{1}{p}}$ for some constant $c_1 > 0$ and $p \in (0, 1]$.
    \item \textbf{Embedding condition (EMB):} There exists a constant $c_2 > 0$ such that for any $f \in \hk$, $\norm{f}_{\infty} \le c_2 \norm{f}_{\hk}$. Here $\norm{f}_{\infty}$ is the smallest $B$ such that $|f(x)| \le B$ for $\px$-almost all $x$.
    \item \textbf{Source condition (SRC):} There exists a constant $c_3 > 0$ such that $\norm{f}_{\hk} \le c_3$.
\end{itemize}

Space $\hk$ obviously satisfies (EVD) because its rank is $d$, so the condition holds for any $p \in (0,1]$.
For $f = \fPhi$, it satisfies (EMB) with $c_2 = \bt$, because for any $x_0 \in \gX$, we have $f(x_0) = \dotp{f, k(x_0, \cdot)}_{\hk} \le \norm{f}_{\hk} \norm{k(x_0, \cdot)}_{\hk}$,
and $\norm{k(x_0, \cdot)}_{\hk}^2 = \norm{\sum s_i^2 \mu_i(x_0) \mu_i(\cdot)}_{\hk}^2 = \sum s_i^2 \mu_i(x_0)^2 \le \bt^2$ for $\px$-almost all $x_0$.
Moreover, assuming that $f^*$ is fixed, $f = \fPhi$ also satisfies (SRC) with $c_3 = \norm{f^*}_{\px} / (1-\epsilon)$ because $\norm{\fPhi}_{\hk} \le \norm{f^*}_{\hk}$.

With all these conditions, invoking Theorem 3.1 in \cite{fischer2020sobolev}, we get the following.
\begin{theorem}
\label{thm:fischer-1}
For any $p \in (0,1]$ and sufficiently large $n \ge 1$ and $\eta > 0$, suppose we choose $\beta_n = \Theta\paren{n^{-\frac{1}{1+p}}}$. Then, there exists a constant $A > 0$ independent of $n$ and $\eta$, such that with probability at least $1 - 4 e^{\eta}$, we have $\norm{\tilde{f} - \fPhi}_{\px} \le s_1 \norm{\tilde{f} - \fPhi}_{\hk}$, and
\begin{equation*}
\norm{\tilde{f} - \fPhi }_{\hk}^2 \le 2\sqrt{\beta_n} \frac{\norm{f^*}_{\px}^2}{(1-\epsilon)^2} + A \eta^2 \brac{ \frac{\bt^2 \paren{ \sigma^2 + \frac{\norm{f^*}_{\px}^2}{(1-\epsilon)^2} }}{n \beta_n^{\frac{1}{2} + p}} +  \frac{\bt^2 \max \oset{ L^2, 4 \bt^2 \frac{\norm{f^*}_{\px^2}}{(1-\epsilon)^2}  }}{ n^2 \sqrt{\beta_n} }  }   .
\end{equation*}
\end{theorem}

We can see that as $n \rightarrow \infty$, there is $\norm{\tilde{f} - \fPhi }_{\px} \rightarrow 0$.
when $p$ is very small, we can think of this bound as $O(\frac{\bt^2}{\sqrt{n}})$.
In the ideal case, all singular values except the top-$d$ are $0$; then, $\bt^2$ can be viewed as an approximation of $\sum_{i=1}^{d} s_i^2$.
This is the bound that we used to derive the metric in \Cref{sec:context-usefulness}.

\paragraph{Part (ii) Shift in the labels.}
The actual downstream predictor is given by
\begin{equation}
   \hat{\vw} = \paren{ \mPhi \mPhi^{\top} + n \beta_n \mI_d }^{-1} \mPhi \vy  ,
\end{equation}
similar to \Cref{eqn:kernel-ridge-regression}. Let $\hat{f}(x) = \hat{\vw}^{\top} \Phi(x)$.
Then, we have
\begin{equation*}
    \norm{\hat{f} - \tilde{f}}_{\hk}^2 = \norm{ \brac{ \paren{ \mPhi \mPhi^{\top} + n \beta_n \mI_d }^{-1} \mPhi (\vy - \vy_{\Phi}) }^{\top} \Phi }_{\hk}^2 = \norm{\paren{ \mPhi \mPhi^{\top} + n \beta_n \mI_d }^{-1} \mPhi (\vy - \vy_{\Phi})}_2^2  .
\end{equation*}

Let $\mQ = \mPhi \mPhi^{\top} + n \beta_n \mI_d$.
It suffices to bound $\norm{\mQ^{-1} \mPhi}_2^2$ and $\norm{\vy - \vy_\Phi}_2^2$.
To bound these, we need the following lemma, which can be proved using the Rademacher complexity.
\begin{lemma}[Proof in \Cref{app:proof-lem-unit-rad}]
\label{lem:unit-rad}
Let $\gF = \sset{f_1 f_2}{f_1,f_2 \in \hk; \norm{f_1}_{\hk}, \norm{f_2}_{\hk} \le 1}$.
Then, $\rad_n(\gF) \le \frac{\bt^2}{\sqrt{n}}$.
If $m \ge n$, then for any $\delta \in (0,1)$, with probability at least $1 - \delta$,
\begin{equation*}
    \norm{\mPhi^{\top} \vu}_2^2 \ge \frac{n}{m} \lambda_d - \bt^2 \sqrt{n} \paren{4 + 2 \sqrt{2 \log \frac{2}{\delta}}} \quad \text{for any unit vector } \vu \in \R^d  .
\end{equation*}
\end{lemma}

$\norm{\mQ^{-1} \mPhi}_2^2$ is equal to the largest eigenvalue of $\mPhi^{\top} \mQ^{-2} \mPhi$,
which is also the largest eigenvalue of $\mQ^{-2}\mPhi \mPhi^{\top}$ by Sylvester's theorem.
Let $\xi_1 \ge \cdots \ge \xi_d \ge 0$ and $\valpha_1,\cdots,\valpha_d$ be the eigenvalues and orthonormal eigenvectors of $\mPhi \mPhi^{\top} \in \R^{d \times d}$.
If $\xi_i = 0$, then $\valpha_i$ is also an eigenvector of $\mQ^{-2} \mPhi \mPhi^{\top}$ with eigenvalue $0$.
If $\xi > 0$, then we have
\begin{equation*}
    \mQ \valpha_i = \paren{\mPhi \mPhi^{\top} + n \beta_n \mI_d} \valpha_i = \paren{ \xi_i + n \beta_n } \valpha_i ,
\end{equation*}
which implies that $\mQ^2 \valpha_i = \paren{ \xi_i + n \beta_n }^2 \valpha_i = \frac{\paren{ \xi_i + n \beta_n }^2}{\xi_i} \mPhi \mPhi^{\top} \valpha_i$.
Thus, $\valpha_i$ is an eigenvector of $\mQ^{-2}\mPhi \mPhi^{\top}$ with eigenvalue $\frac{\xi_i}{\paren{ \xi_i + n \beta_n }^2}$.
We can therefore conclude that $\paren{\frac{\xi_i}{\paren{ \xi_i + n \beta_n }^2}}$ are all the eigenvalues of $\mQ^{-2}\mPhi \mPhi^{\top}$.
Meanwhile, by \Cref{lem:unit-rad}, we have
\begin{equation*}
    \xi_d = \valpha_d^{\top} \mPhi \mPhi^{\top} \valpha_d = \norm{\mPhi^{\top} \valpha_d}_2^2 \ge \frac{n}{m} \lambda_d - \bt^2 \sqrt{n} \paren{4 + 2 \sqrt{2 \log \frac{2}{\delta}}} .
\end{equation*}
When $n \ge \frac{4 \bt^4 m^2}{\lambda_d^2} \paren{ 4 + 2 \sqrt{2 \log \frac{2}{\delta}} }^2$, we have $\xi_d \ge \frac{n \lambda_d}{2m}$,
which implies that $\norm{\mQ^{-1} \mPhi}_2^2 \le \frac{2m}{n \lambda_d}$.
Remember that $\lambda_1$ has the same order as $m$, so we should view $\lambda_d$ as also having roughly the same order as $m$.

Regarding $\norm{\vy - \vy_{\Phi}}_2^2$, it is equal to $\sum_{i=1}^n F(\tilde{x}_i)$ where $F(x) = (f^*(x) - \fPhi(x))^2$,
and we have already bounded $\norm{f^* - \fPhi}_{\hk}^2$ in \Cref{thm:approx-error}.
We can bound $F$ using \Cref{lem:unit-rad}.

\begin{corollary}[Proof in \Cref{app:proof-lem-est-1}]
\label{cor:est-1}
With probability at least $1 - \delta$, we have
\begin{equation*}
    \frac{\norm{\vy - \vy_\Phi}_2^2}{n} \le \norm{f^* - \fPhi}_{\px}^2 + \norm{f^* - \fPhi}_{\hk}^2 \frac{\bt^2}{\sqrt{n}} \paren{ 2 + \sqrt{2 \log \frac{2}{\delta}}}  .
\end{equation*}
\end{corollary}

Next, note that $G(x) = \frac{(\hat{f}(x) - \tilde{f}(x))^2}{\norm{\hat{f} - \tilde{f}}_{\hk}^2} \in \gF$ defined in \Cref{lem:unit-rad}, and
\begin{equation*}
    \sum_{i=1}^n \paren{\hat{f}(\tilde{x}_i) - \tilde{f}(\tilde{x}_i)}^2 =  \norm{\mPhi^{\top} \paren{ \mPhi \mPhi^{\top} + n \beta_n \mI_d }^{-1} \mPhi (\vy - \vy_{\Phi})}_2^2 \le \norm{\vy - \vy_\Phi}_2^2 
\end{equation*}
because $\norm{\mPhi^{\top} \paren{ \mPhi \mPhi^{\top} + n \beta_n \mI_d }^{-1} \mPhi}_2 \le 1$.
Thus, using \Cref{cor:est-1}, we get
\begin{align*}
    \norm{\hat{f} - \tilde{f}}_{\px}^2 & \le \frac{1}{n} \sum_{i=1}^n \paren{\hat{f}(\tilde{x}_i) - \tilde{f}(\tilde{x}_i)}^2 + \norm{\hat{f} - \tilde{f}}_{\hk}^2 \frac{\bt^2}{\sqrt{n}} \paren{2 + \sqrt{2 \log \frac{2}{\delta}}} \\ 
    & \le \frac{\norm{\vy - \vy_\Phi}_2^2 }{n} \brac{1 + \frac{2 \bt^2}{\lambda_d \sqrt{n}} \paren{2 + \sqrt{2 \log \frac{2}{\delta}}} } \\ 
    & \le \frac{3}{2} \paren{\norm{f^* - \fPhi}_{\px}^2 + \frac{\lambda_d}{4}\norm{f^* - \fPhi}_{\hk}^2}  ,
\end{align*}
where the last step uses $n \ge \frac{4 \bt^4}{\lambda_d^2} \paren{ 4 + 2 \sqrt{2 \log \frac{2}{\delta}} }^2$.
Combining the above inequality with \Cref{thm:fischer-1}, we obtain a bound for $\norm{\hat{f} - \fPhi}_{\px}^2$.

\paragraph{Dependency on the context complexity.}
Our prediction error bound decreases with $\bt$, meaning that the weaker the association of the context, the better the generalization will be.
Let us revisit the experimental results in \Cref{fig:aug-comp-nlp} (right).
On QNLI, the train-test gap decreases with the mask ratio $\alpha$, which matches our bound.
However, on SST-2, the train-test gap first decreases and then increases.
The reason is that our bound assumes that $f^* \in \fhkp$, but as the association becomes weaker, $\fhkp$ becomes smaller and smaller.
When the association becomes too weak, $f^*$ might no longer belong to $\fhkp$, and our bound will no longer hold.
In this case, further weakening the association might make the train-test gap larger.

\paragraph{Pair/transformation access.}
The above analysis assumed that we have access to the centered dual kernel of the context.
What if we have pair or transformation access instead? 
Recall that the dual kernel is $\kx(x,x') = \int \frac{P^+(a|x) P^+(a|x')}{\pa(a)} da$.
If we have transformation access, then we can assume that we have full access to $P^+(a|x)$ for any given $x$.
The main problem appears in the denominator, because we can only estimate $\pa(a)$ by $\frac{1}{m}\sum_{i=1}^m P^+(a|x_i)$.
Although their difference can be shown to be small if $P^+(a|x)$ is assumed to be smooth in $x$, if $\pa(a)$ itself is very small, then a tiny estimation error of $\pa(a)$ will have a huge impact on $\kx$ because it is in the denominator.
For this reason, we also need to assume that $\pa(a)$ is bounded away from zero for all $a$, in order to get a reasonable generalization bound.
Such a bound was proved in \cite{zhai2023understanding}.

The case with pair access to the context is even more challenging, because in this case we have to estimate $P^+(a|x)$, so we need to consider the sample complexity of $a$ as well.
Deriving a bound for this case is posed as an open problem.

\section{Spectrally Transformed Kernel Regression}
\label{sec:stkr}
The previous section showed that fitting the linear probe with ridge regression is equivalent to kernel ridge regression on a subspace of the induced RKHS, spanned by $\phi_1,\cdots,\phi_d$.
In fact, this is a special case of spectrally transformed kernel regression (STKR) \cite{zhai2024spectrally}.
\begin{definition}
    Let $k: \gX \times \gX \rightarrow \R$ be a \psd{} kernel such that $k(x,x') = \sum_{i=1}^{\infty} \lambda_i \mu_i(x) \mu_i(x')$, where $\lambda_1 \ge \lambda_2 \ge \cdots$ and $\mu_1,\mu_2,\cdots$ are the eigenvalues and eigenfunctions of $T_k$.
    Then, a \textbf{spectrally transformed kernel (STK)} of $k$ is defined as $k_s(x,x') = \sum_{i=1}^{\infty}s(\lambda_i) \mu_i(x) \mu_i(x')$, for some transformation function $s: [0,+\infty) \rightarrow [0,+\infty)$ such that $s(0) = 0$.
\end{definition}

An STK has the same eigenfunctions as the original kernel.
In general one would like $s$ to be a monotonically non-decreasing function, so that the spectral transformation does not change the order of the eigenfunctions,

\begin{example}
    If $\phi_1,\cdots,\phi_d$ are obtained by kernel PCA described in the previous section, and the downstream linear probe is fit using ridge regression, then this process is equivalent to doing kernel ridge regression on an STK $k_s$, where $s(\lambda_i) = \lambda_i \sI[i \le d]$ is called the truncation function.
\end{example}

\begin{figure}[t]
    \centering
    \begin{tikzpicture}

\node[circle, draw, fill=black!40, thick, align=center] (a) at (0,0) {$x_1$};
\node[circle, draw, thick, align=center] (b) at (0.6,1) {$x_4$};
\node[circle, draw, fill=black!40, thick, align=center] (c) at (1.2,2) {$x_2$};
\node[circle, draw, thick, align=center] (d) at (1.2,0) {$x_5$};
\node[circle, draw, thick, align=center] (e) at (1.8,1) {$x_6$};
\node[circle, draw, fill=black!40, thick, align=center] (f) at (2.4,0) {$x_3$};

\draw[thick] (a) -- (b);
\draw[thick] (c) -- (b);
\draw[thick] (d) -- (b);
\draw[thick] (c) -- (e);
\draw[thick] (d) -- (e);
\draw[thick] (e) -- (f);

\node[circle, draw, fill=black!40, thick, align=center] at (5,0) {$x_1$};
\node[circle, draw, fill=black!40, thick, align=center] at (6.2,2) {$x_2$};
\node[circle, draw, fill=black!40, thick, align=center] at (7.4,0) {$x_3$};

\node[circle, draw, fill=black!40, thick, align=center] (x1) at (10,0) {$x_1$};
\node[circle, draw, fill=black!40, thick, align=center] (x2) at (11.2,2) {$x_2$};
\node[circle, draw, fill=black!40, thick, align=center] (x3) at (12.4,0) {$x_3$};
\node[circle, draw, dashed, thick, align=center] (x4) at (10.6,1) {$x_4$};
\node[circle, draw, dashed, thick, align=center] (x6) at (11.8,1) {$x_6$};
\draw[thick] (x1) -- (x2);
\draw[thick] (x2) -- (x3);

\draw [<->,ultra thick] (x4) to [out=180,in=120] (x1);
\draw [<->,ultra thick] (x2) to [out=180,in=120] (x4);

\draw [<->,ultra thick] (x6) to [out=0,in=60] (x3);
\draw [<->,ultra thick] (x2) to [out=0,in=60] (x6);

\node[] at (1.2,-0.8) {(a)};
\node[] at (6.2,-0.8) {(b)};
\node[] at (11.2, -0.8) {(c)};

\end{tikzpicture}
    \caption{(a) A graph example where the kernel is the adjacency matrix. Shaded nodes are labeled and white nodes are unlabeled. (b) In KRR, the unlabeled nodes are useless and can be removed, so the graph becomes three isolated nodes. (c) With a two-step random walk $k^2$, $x_1$ and $x_2$ are connected, and $x_2$ and $x_3$ are connected.}
    \label{fig:stkr-motivation}
\end{figure}
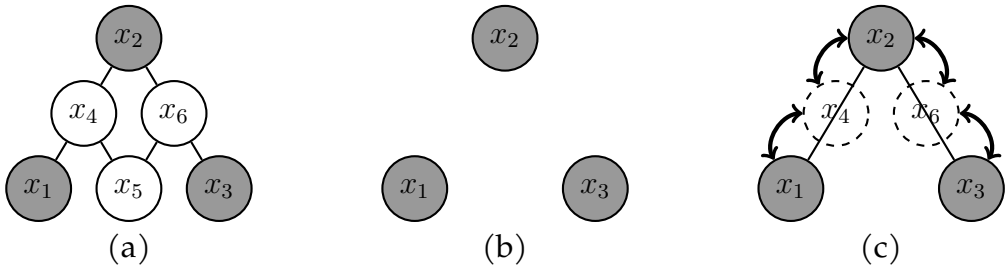

STKR is extremely useful in a semi-supervised learning setting where there are much more unlabeled samples than labeled samples.
One might ask what is the point of using an STK or extracting the top-$d$ eigenspace, instead of directly performing kernel ridge regression with $k$.
The answer is that in many cases, regression with $k$ would fail, but with $k_s$ it would not.
This can be demonstrated with a concrete example.
Consider the graph in \Cref{fig:stkr-motivation} (a), where the shaded nodes are labeled but the white nodes are not, and the context is that connected nodes are similar. Then, $k$ can be the adjacency matrix of this graph.
Let $\hk$ be the RKHS of $k$.
Recall that given a labeled dataset $\oset{(\tilde{x}_i, y_i)}_{i=1}^n$,
kernel ridge regression (KRR) with $k$ is given by
\begin{equation*}
    \hat{f} \in \argmin_{f \in \hk} \; \oset{\frac{1}{n} \sum_{i=1}^n \paren{f(\tilde{x}_i) - y_i}^2 + \beta_n \norm{f}_{\hk}^2 }  ,
\end{equation*}
for which there is a classical \textbf{Representer Theorem} \cite[Theorem~4.2]{scholkopf2002learning}:
\begin{theorem}
\label{thm:representer-theorem}
    All minimizers of KRR admit the form $\hat{f}^*(x) = \sum_{j=1}^n \alpha_i^* k(x,\tilde{x}_j)$, where
    \begin{equation*}
        \valpha^* \in \arginf_{\valpha \in \R^n} \; \oset{ \frac{1}{n} \sum_{i=1}^n \brac{ \sum_{j=1}^n \alpha_j k(\tilde{x}_i,\tilde{x}_j) - y_i}^2 + \beta_n \sum_{i,j=1}^n \alpha_i \alpha_j k(\tilde{x}_i,\tilde{x}_j) }  .
    \end{equation*}
\end{theorem}

What this theorem implies is that KRR only uses the labeled samples, and the large number of unlabeled samples are not used at all.
This means that in \Cref{fig:stkr-motivation} (a), the white nodes are not used by KRR.
When all the white nodes are removed, the graph becomes (b)---a graph with three isolated nodes, which is a useless context.

So what goes wrong here?
The key is that the graph only says that $x_1$ and $x_4$ are similar, and $x_4$ and $x_2$ are similar. However, it does not say that $x_1$ and $x_2$ are similar, that is it cannot imply the transitivity of similarity on its own.
Therefore, in the eyes of this graph, $x_1,x_2,x_3$ are three completely independent nodes, which is why the graph is useless in KRR.
However, STKR assumes that similarity is transitive, and this allows us to start at any point $x$ and do a random walk up to some number of steps, and the node we land on is still similar to $x$.
For example, if we allow random walks up to two steps, then the graph becomes \Cref{fig:stkr-motivation} (c), which no longer consists of three isolated nodes.

To formalize this idea of random walk, define the following power spaces.
\begin{definition}
    A \textbf{power space} is an RKHS associated with $k^p$ for any $p \ge 1$, where
    \begin{align*}
        k^p(x,x') & = \sum_{i=1}^{\infty} \lambda_i^p \mu_i(x) \mu_i(x') ; \\ 
        \hkp & = \sset{f = \sum u_i \mu_i}{\sum \frac{u_i^2}{\lambda_i^p} < \infty}, \; \dotp{\sum u_i \mu_i, \sum v_i \mu_i}_{\hkp} = \sum \frac{u_i v_i}{\lambda_i^p}  .
    \end{align*}
\end{definition}

The proof of the following proposition is left as an exercise.
\begin{proposition}
    For any $p \ge 1$, there is $k^{p+1}(x,x') = \int k^p(x,z) k(z,x') d \px (z)$.
\end{proposition}

This formula shows that when $p$ is an integer, $k^p$ can be viewed as a $p$-step random walk.
When $p$ is a real number, $k^p$ is essentially a continuous random walk, which is called a \textbf{diffusion process}.
This definition requires $p \ge 1$, because when $p < 1$, $\hkp$ is not necessarily an RKHS, due to the following classical result \cite[p.~36]{scholkopf2002learning}.
\begin{proposition}
\label{prop:rkhs-operator}
Let $\gH$ be a Hilbert space of real-valued functions on $\gX$. For any $x \in \gX$, define an evaluation functional $L_x: \gH \rightarrow \R$ as $(L_x f) = f(x)$.
Then, $\gH$ is an RKHS if and only if for all $x \in \gX$, $L_x$ is a continuous operator, that is there exists a constant $M_x > 0$ such that $|f(x)| \le M_x \norm{f}_{\gH}$ for all $f \in \gH$. (Note: $\sup_x M_x = \infty$ is allowed.)
\end{proposition}

In the last section, it has been shown that $M_x = \bt$ when $p = 1$.
Obviously, for any $p > 1$ and $f \in \hkp$, we have $\norm{f}_{\hkp}^2 \ge \lambda_1^{1-p} \norm{f}_{\hk}^2$. Thus, $\hkp$ is still an RKHS when $p > 1$.
However, when $p < 1$, it is easy to construct an example where $\hkp$ is not an RKHS.

\begin{figure}[t]
    \centering
    \begin{tikzpicture}

\draw[rounded corners, semithick] (0, 0) rectangle (8, 2.4) {};
\draw[rounded corners, semithick] (1.5, 0.2) rectangle (7.8, 2.2) {};
\draw[rounded corners, semithick] (3, 0.4) rectangle (7.6, 2) {};
\draw[rounded corners, semithick] (6, 0.8) rectangle (7.2, 1.6) {};

\draw[rounded corners, ultra thick] (4.5, 0.6) rectangle (7.4, 1.8) {};

\node[anchor=west] at (0.05,1.2) {$\hk$};
\node[anchor=west] at (1.55,1.2) {$\gH_{k^{1.5}}$};
\node[anchor=west] at (3.05,1.2) {$\gH_{k^2}$};
\node[anchor=west] at (6.05,1.2) {$\gH_{k^3}$};

\node[anchor=west] at (4.55,1.2) {{\boldmath $\hks$}};

\end{tikzpicture}
    \caption{Illustration of the multiscale smoothness induced by diffusion, producing the chain $\hk \supseteq \gH_{k^{1.5}} \supseteq \gH_{k^2} \supseteq \gH_{k^3} \supseteq \cdots$. The RKHS of the spectrally transformed kernel $\ks$ is $\hks$ marked in bold, which is in this chain but not necessarily equal to any $\hkp$.}
    \label{fig:multiscale-smoothness}
\end{figure}
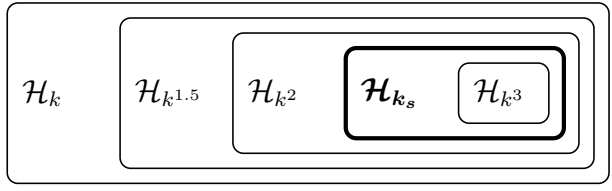

$k^p$ is an example of an STK, and $\oset{k^p}_{p \ge 1}$ forms a chain of function classes: $\lxp \supset \gH_{k^1} \supset \gH_{k^{1.5}} \supset \gH_{k^2} \supset \cdots$, as illustrated in \Cref{fig:multiscale-smoothness}.
We say that any function $f \in \hks$ is \textbf{smooth} \wrt{} the kernel $\ks$, and the diffusion process induces \textbf{multiscale smoothness}.
The kernel metric of $k^p$ is $d_{k^p}(x,x') = \norm{k^p(x,\cdot) - k^p(x',\cdot)}_{\hkp} = \sum \lambda_i^p (\mu_i(x) - \mu_i(x'))^2$, which is equivalent to the diffusion distance defined in \cite{coifman2006diffusion}.

The key result of this section is the generality of STK. Suppose the target function is smooth \wrt{} a certain measure of smoothness.
Then, under some mild conditions, the class of all smooth functions under this measure must be the RKHS of some STK.
This is a quite insightful result because it essentially says that if we know how to deal with $\hks$, then we can deal with almost any kind of downstream task.

Specifically speaking, each $k^p$ defines a measure of smoothness.
We are interested in a certain measure of smoothness called the \textbf{target smoothness},
which satisfies the following condition:
For any two functions $f_1$ and $f_2$, if all $k^p$ say that $f_1$ is smoother than $f_2$, then the target smoothness must also say that $f_1$ is smoother than $f_2$.
This is called the condition of \textbf{preserving relative smoothness}.
This is a weak condition, because usually $f_1$ is smoother at some scale $p$ while $f_2$ is smoother at some other scale $q$.
However, if $f_1$ is smoother at all scales, then it must also be smoother under the target smoothness.
In addition, there are two more assumptions: (i) All smooth functions under the target smoothness form a Hilbert space $\gH_t$;
(ii) $\gH_t \subseteq \hk$.

To prove the result, we need to formally define the term ``smoothness''.
\begin{definition}
For any Hilbert space $\gH \subset \lxp$, the \textbf{smoothness} of $f \in \gH$ \wrt{} $\gH$ is defined as $r_{\gH}(f) = \frac{\norm{\tilde{f}}_\px^2}{\norm{\tilde{f}}_\hks^2}$.
\end{definition}

Smoothness is an alias of compatibility, because $r_{\gH}(f) \ge (1 - \epsilon)^2$ is equivalent to $f$ being $(1-\epsilon)$-compatible with $\gH$ as per \Cref{def:compatible-hk}.
The key result is as follows.

\begin{theorem}[Proof in \Cref{app:proof-thm-stkr-generality}]
\label{thm:stkr-generality}
Suppose $\gH_t \subseteq \hk$ preserves relative smoothness: for any $f_1, f_2 \in \lxp$, if $r_{\hkp}(f_1) \ge r_{\hkp}(f_2)$ for all $p \ge 1$, then $r_{\gH_t}(f_1) \ge r_{\gH_t}(f_2)$.
Then, $\gH_t$ is an RKHS, whose reproducing kernel is an STK for a transformation function $s$ such that:
\begin{enumerate}[label=(\roman*)]
    \item $s$ is monotonically non-decreasing;
    \item $s(\lambda) \le M \lambda$ for some constant $M > 0$;
    \item $s$ is continuous on $[0, +\infty)$;
    \item $s$ is $C^{\infty}$ on $(0, +\infty)$.
\end{enumerate}
\end{theorem}

STKR is the generalization of contexture learning from representation learning to semi-supervised learning.
In semi-supervised learning, STKR is usually much more efficient than extracting the top-$d$ eigenspace.
The following is a popular example.

\begin{example}[Inverse Laplacian]
     For $\eta \in (0, \lambda_1^{-1})$, define $\ks$ such that $\ks^{-1}(x,x') =  k^{-1}(x,x') - \eta k^0(x,x')$. $k^{-1}$ and $k^0$ are STKs with $s(\lambda) = \lambda^{-1}$ and $s(\lambda) = \lambda^0$.
    Then, the reciprocal of $s$ is given by $s^{-1}(\lambda) = \lambda^{-1} - \eta > 0$ for $\lambda \in (0, \lambda_1]$, 
    which means that $s(\lambda) = \frac{\lambda}{1 - \eta \lambda} = \sum_{p=1}^{\infty} \eta^{p-1} \lambda^p$,
and $\| f \|_\hks^2 = \| f \|_\hk^2 - \eta \| f \|_\px^2$.
\end{example}

The inverse Laplacian used to be very popular for semi-supervised learning, because it can be implemented very efficiently via a method called label propagation \cite{Zhu2002LearningFL,zhou2003learning}.

\section{Implementation and Generalization Analysis of STKR}

Now we develop the algorithms for STKR for a wide variety of STKs, including the inverse Laplacian.
After that, we derive generalization bounds for the algorithms.
In particular, we consider polynomial STKs with $s(\lambda) = \sum_{p=1}^{\infty} \pi_p \lambda^p$, where $\pi_p \ge 0$ for all $p$.

Since $\ks = \sum_{p=1}^{\infty} \pi_p k^p$, it suffices to show how to estimate $k^p$.
For example, when $p = 2$, we have $k^2(x,x') = \int k(x,z) k(x',z) d \px(z) \approx \frac{1}{m+n} \sum_{i=1}^{m+n} k(x,x_i) k(x',x_i) $, that is we estimate $k^2$ by Monte Carlo.
Here, $x_{m + i} = \tilde{x}_i$.
Recall that KRR does not use the $m$ unlabeled samples at all, but here we can make use of the unlabeled samples when estimating $k^p$.
Similarly, we can estimate $k^p$ for all positive integer $p$ as follows:
\begin{enumerate}
    \item Compute Gram matrix $\gk \in \R^{(m+n) \times (m+n)}$, where $\gk[i,j] = k(x_i, x_j)$.
    \item Let $\hat{k}^1 = k$. Define $\vv_k(x) \in \R^{m+n}$ as $\vv_k(x)[i] = k(x,x_i)$.
    \item Compute $\hatkp(x,x') = \frac{\vv_k(x)^{\top} \gk^{p-2} \vv_k(x') }{(m+n)^{p-1}}$ iteratively for $p=2,3,\cdots$.
\end{enumerate}

Then, an estimate of $\ks$ is given by $\hatks = \sum_{p=1}^{\infty} \pi_p \hatkp$.
One difference is that previously we only use the $m$ unlabeled samples to estimate the top-$d$ eigenfunctions, but here we use all $(m+n)$ samples to estimate $\hatks$.
This difference is small in practice provided that $m \gg n$.
Later we will see why we can use all $(m+n)$ samples in this situation.

We use $\gkn \in \R^{n \times n}$ to denote the Gram matrix on the $n$ labeled samples,
and define $\vkn(x) \in \R^n$ as $\vkn(x)[i] = k(x,\tilde{x}_i)$.
Similarly, define $\gksn, \vksn, \ghatksn$ and $\vhatksn$.
Let $\tilde{f}$ and $\hat{f}$ be the predictor obtained from KRR with $\ks$ and $\hatks$, respectively. The following closed-form formulas can be derived from the Representer Theorem:
\begin{equation}
\label{eqn:stkr-formula}
\left \{ 
\begin{aligned}
    & \tilde{f}(x) = \vksn(x)^{\top} \tilde{\valpha}, \quad \tilde{\valpha} = \paren{\gksn + n \beta_n \mI_n}^{-1} \vy ; \\ 
    & \hat{f}(x) = \vhatksn(x)^{\top} \hat{\valpha}, \quad \hat{\valpha} = \paren{\ghatksn + n \beta_n \mI_n}^{-1} \vy  . 
\end{aligned}
\right . 
\end{equation}

Here, $\vy = [y_1,\cdots,y_n]$.
To obtain $\hat{f}$, it suffices to solve $\mA \hat{\valpha} = \vy$ for $\mA = \ghatksn + n \beta_n \mI_n$.
Let us consider two scenarios: (i) $s$ is simple: For some $q$, $\pi_p = 0$ for all $p > q$;
(ii) $s^{-1}$ is simple: $s^{-1}(\lambda) = \sum_{p=0}^{q-1} \xi_p \lambda^{p-r}$, such as the inverse Laplacian.
Here ``simple'' means that the polynomial contains only a few terms.

For scenario (i), directly computing $\mA$ is slow because it involves lots of matrix-matrix multiplications.
A faster alternative is iterative methods, such as Richardson iteration \cite{richardson1911ix}, which solves a linear system $\mA \vx = \vb$ by iteratively computing $\vx^{(t+1)} = \vx^{(t)} + \gamma(\vb - \mA \vx^{(t)})$ for some $\gamma > 0$.
Richardson iteration is guaranteed to converge to the solution when $\gamma$ is chosen correctly.
While computing $\mA$ is slow, computing $\mA \vx^{(t)}$ is very efficient because it only involves matrix-vector multiplication.
This method is called STKR propagation (STKR-Prop), because it is a generalization of label propagation (Label-Prop) for the inverse Laplacian.
Define $\mF \in \R^{(m+n) \times n}$ as $\mF[i,j] = k(x_i, \tilde{x}_j)$.
The algorithm is listed in \Cref{alg:richardson}.

\begin{figure}[t]
\begin{minipage}[t]{0.5\textwidth}
\begin{algorithm}[H]
\caption{STKR-Prop for simple $s^{\phantom{0}}$ }
\label{alg:richardson}
\begin{algorithmic}[1]
\Require $\gk$, $\gkn$, $\mF$, $s$, $\beta_n$, $\vy$, $\gamma$, $\epsilon$
\State Initialize: $\hat{\valpha} \gets \vzero \in \R^n$
\While{True}
\Statex  \textit{\# Compute $\vu = (\ghatksn + n\beta_n \mI_n)\hat{\valpha}$}
\State $\tilde{\valpha} \gets \frac{1}{m+n} \mF \hat{\valpha}$, $\vv \gets \vzero \in \R^{m+n}$ 
\State \textbf{for} $p = q, \cdots, 2$ \textbf{do} $\vv \gets \frac{\gk \vv}{m+n} + \pi_p \tilde{\valpha}$
\State $\vu \gets \mF^{\top} \vv + \pi_1 \gkn \hat{\valpha} + n \beta_n \hat{\valpha}$ 
\State \textbf{if} $\| \vu - \vy\| _2 < \epsilon \| \vy \|_2 $ \textbf{then return} $\hat{\valpha}$
\State $\hat{\valpha} \gets \hat{\valpha} - \gamma (\vu - \vy)$
\vspace{.085 in}
\EndWhile
\end{algorithmic}
\end{algorithm}
\end{minipage}
\hfill
\begin{minipage}[t]{0.51\textwidth}
\begin{algorithm}[H]
\caption{STKR-Prop for simple $s^{-1}$}
\label{alg:richardson-inverse}
\begin{algorithmic}[1]
\Require $\gk$, $s^{-1}(\lambda)$, $\beta_n$, $\vy$, $\gamma$, $\epsilon$
\State Initialize: $\vtheta \gets \vzero \in \R^{m+n}$, $\tilde{\vy} \gets [\vy, \vzero_m]^{\top}$
\While{True}
\Statex  \textit{\# Compute $\vu = \mM\vtheta$}
\State $\vv \gets \vzero \in \R^{m+n}$  
\State \textbf{for} $p = q-1, \cdots, 0$ \textbf{do} $\vv \gets \frac{\gk \vv}{m+n} + \xi_p \vtheta$
\State $\vu \gets  \brac{ \paren{ \frac{\gk^r}{(m+n)^{r-1}} \vtheta} [1:n] , \vzero_m }^{\top} + n \beta_n \vv $ 
\State $\va \gets \vu - \tilde{\vy}$, $\vtheta \gets \vtheta - \gamma \va$  
\State \textbf{if} $\|\va\|_2 < \epsilon \| \vy \|_2$ \textbf{then return} $\vtheta$
\EndWhile
\end{algorithmic}
\end{algorithm}
\end{minipage}
\end{figure}

The next question is, given $\hat{\valpha}$, how to efficiently compute $\hat{f}(x)$ for a test input $x$?
We do not want to compute $\vhatksn(x)$, which involves another set of matrix-vector multiplications.
An efficient way is that we can store the $\vv$ computed in line 4 of \Cref{alg:richardson} in the memory.
Then, $\hat{f}(x) = \sum_{i=1}^{m+n} k(x_i, x) \vv[i] + \pi_1 \sum_{j=1}^n k(\tilde{x}_j, x) \hat{\valpha}[j]$ for any $x \in \gX$, which only needs $O(m+n)$ time to compute.

For scenario (ii) where $s$ could be complex but $s^{-1}(\lambda) = \sum_{p=0}^{q-1} \xi_p \lambda^{p-r}$ is simple, we can no longer estimate $\ghatksn \valpha$.
However, we can do the following transformation:
Let $\mQ = \sum_{p=0}^{q-1} \xi_p \paren{\frac{\gk}{m+n}}^p$.
Then, we have $\ghatks \mQ = (m+n) \paren{\frac{\gk}{m+n}}^r$, where $\ghatks \in \R^{(m+n) \times (m+n)}$ is the Gram matrix of $\hatks$ on all $(m+n)$ samples.
Therefore, we can efficiently compute $\ghatks \mQ \vx$ for any vector $\vx$, and this motivates us to find a $\vtheta \in \R^{m + n}$ such that $\mQ \vtheta = [\hat{\valpha}, \vzero_m]^{\top}$.
To solve for $\vtheta$, we need $(m+n)$ linear equations.
The last $m$ elements of $\mQ \vtheta$ are all zero, which provides us with $m$ linear equations.
Since $\mA \hat{\valpha} = \vy$, the first $n$ elements of $\mA \mQ \vtheta$ must be $\vy$, which provides us with another $n$ linear equations.
Overall, we can find $\vtheta$ by solving $\mM \theta = \tilde{\vy}$, where
\begin{equation*}
    \mM = (m+n) \tilde{\mI}_n \paren{\frac{\gk}{m+n}}^r + n \beta_n \mQ, \quad \tilde{\vy} = [\vy, \vzero_m]^{\top}  .
\end{equation*}
Here, $\tilde{\mI}_n = \diag \oset{1,\cdots,1,0,\cdots,0}$, with $n$ ones and $m$ zeros.
Once again, we can solve for $\vtheta$ using Richardson iteration.
The algorithm is listed in \Cref{alg:richardson-inverse}.
After running this algorithm, we can store $\paren{\frac{\gk}{m+n}}^{r-1}\vtheta$ in the memory.
Then, for any test input $x$, we can compute $\hat{f}(x)$ with $\hat{f}(x) = \vv_k(x)^{\top} \paren{\frac{\gk}{m+n}}^{r-1}\vtheta$ in $O(m+n)$ time.

Next, let us study the time complexity of these two algorithms.
Assume that computing $k(x,x')$ for any $x,x' \in \gX$ takes $O(1)$ time.
To start with, let us review a classical result about Richardson iteration.

\begin{lemma}
\label{lem:richardson}
Consider solving $\mA \vx = \vb$ with Richardson iteration, where $\mA \in \R^{n \times n}$ is positive definite, and $\vx, \vb \in \R^n$.
Let $\lambda_{\max}$ and $\lambda_{\min}$ be the largest and smallest eigenvalue of $\mA$, and let $\tau = \frac{\lambda_{\max}}{\lambda_{\min}}$ be the condition number of $\mA$.
Then, by choosing $\gamma = \frac{2}{\lambda_{\max} + \lambda_{\min}}$ and setting the stop criterion as $\norm{\vx^{(t+1)} - \vx^{(t)}}_2 < \epsilon \norm{\vb}_2$, the iteration stops in $O\paren{\tau \log \frac{1}{\epsilon}}$ steps.
\end{lemma}
\begin{proof}
Let $\vx^*$ be the solution such that $\mA \vx^* = \vb$.
    By $\vx^{(t+1)} = \vx^{(t)} + \gamma (\vb - \mA \vx^{(t)})$, we have $\vx^{(t+1)} - \vx^* = (\mI_n - \gamma \mA)(\vx^{(t)} - \vx^*)$.
    If we start with $\vx^{(0)} = \vzero$, then this implies that $\vx^* - \vx^{(t)} = (\mI_n - \gamma \mA)^t \vx^*$.
    When $\gamma = \frac{2}{\lambda_{\max} + \lambda_{\min}}$, we have
    \begin{equation*}
        \norm{\vx^* - \vx^{(t)}}_2 \le \norm{\mI_n - \gamma \mA}_2^t \norm{\vx^*}_2 = \paren{ 1 - \frac{2 \lambda_{\min}}{\lambda_{\max} + \lambda_{\min}} }^t \norm{\vx^*}_2 \le \exp \paren{ - \frac{2 \lambda_{\min} t}{\lambda_{\max} + \lambda_{\min}} } \norm{\vx^*}_2  .
    \end{equation*}
This implies that $\norm{\vx^{(t+1)} - \vx^{(t)}}_2 \le 2 \exp \paren{ - \frac{2 \lambda_{\min} t}{\lambda_{\max} + \lambda_{\min}} } \norm{\vx^*}_2 \le 2 \exp (-\frac{t}{\tau}) \norm{\vx^*}_2$.
Thus, the iteration stops when $t = O\paren{\tau \log \frac{1}{\epsilon}}$.
Moreover, when the iteration stops, we have $\norm{\mA \vx^{(t)} - b}_2 = \gamma^{-1} \norm{\vx^{(t+1)} - \vx^{(t)}}_2 < \gamma^{-1} \epsilon \norm{\vb}_2$.
\end{proof}

Let $\lambda_1$ be the largest eigenvalue of $T_k$.
When $m$ and $n$ are sufficiently large, the largest eigenvalue of $\ghatksn$ is close to $n s(\lambda_1)$.
Meanwhile, the smallest eigenvalue of $\mA = \ghatksn + n \beta_n \mI_n$ is at least $n \beta_n$.
Therefore, $\tau$ can be upper bounded by $O(\beta_n^{-1} s(\lambda_1))$.
Moreover, each iteration in \Cref{alg:richardson} has a time complexity of $O((m+n)^2 q)$.
Thus, the overall time complexity of \Cref{alg:richardson} is $O((m+n)^2 q \beta_n^{-1} s(\lambda_1) \log \frac{1}{\epsilon})$.

For \Cref{alg:richardson-inverse}, the analysis is similar but much more complex. We have the following result regarding its time complexity.

\begin{theorem}[Proof in \Cref{app:proof-thm-rich-inverse-time}]
\label{thm:rich-inverse-time}
Let $\rho(\lambda) = \frac{\lambda^r}{s(\lambda)} = \sum_{p=0}^{q-1} \xi_p \lambda_p$, where $\rho(0) = \xi_0 > 0$.
Then, $\rho(\lambda)$ is a continuous function on $[0,+\infty)$.
Denote its maximum and minimum on $[0, \lambda_1]$ by $\rho_{\max}$ and $\rho_{\min}$.
Then, with $\gamma = (n \lambda_1^r)^{-1}$, \Cref{alg:richardson-inverse} has a total time complexity of $O\paren{ (m+n)^2 \frac{\max \oset{q,r} \lambda_1^r \beta_n^{-1}}{\rho_{\min}} \log \brac{ \max \oset{ \frac{1}{\epsilon}, \frac{\lambda_1^r \rho_{\max} \norm{\vy}_2 }{ n \beta_n^2 \rho_{\min}^2 \norm{\hat{\valpha}_*}_2 } } } }$, where $\hat{\valpha}_*$ is the exact solution.
\end{theorem}

We have seen how to implement STKR when $s$ is a polynomial.
Now let us derive the generalization bound for the general STKR.
Define the following function class:
\begin{equation*}
    \gF_{\epsilon}(\hks) = \sset{f \in \hks}{ \norm{f}_{\hks}^2 \le \epsilon \norm{f}_{\px}^2 }  .
\end{equation*}
This class is similar to $\fhkp$, and it contains functions with smoothness at least $\epsilon^{-1}$ \wrt{} $\hks$.
Our result consists of an approximation error bound and an estimation error bound.
Recall the definition of $\tilde{f}$ and $\hat{f}$ in \Cref{eqn:stkr-formula}.
The approximation error is the gap between $\tilde{f}$ and $f^*$, and the estimation error is the gap between $\hat{f}$ and $\tilde{f}$.
The approximation error can be bounded by the result in \cite{fischer2020sobolev}.

\begin{theorem}[Proof in \Cref{app:proof-thm-approx-error-stkr}]
\label{thm:approx-error-stkr}
    Let $\lambda_1 \ge \lambda_2 \ge \cdots$ be the eigenvalues of $T_k$.
    Let $M$ be given by \Cref{thm:stkr-generality}.
    Suppose the moment condition \Cref{eqn:moment-condition} holds, and the eigenvalues decay by order $p^{-1}$ for some $p \in (0,1]$, that is $s(\lambda_i) = O(i^{-\frac{1}{p}})$ for all $i$.
    Choose $\beta_n = \Theta(n^{-\frac{1}{1+p}})$.
    Then, there exists a constant $c_0 > 0$ independent of $n \ge 1$ and $\tau \ge \bt^{-1} M^{-\frac{1}{2}}$ (recall that $\bt^2 = \norm{k}_{\infty}$), such that
    \begin{equation*}
        \norm{\tilde{f} - f^*}_{\px}^2 \le c_0 \tau^2 \bt^2 M \brac{  \paren{\epsilon \norm{f^*}_{\px}^2 + \sigma^2 } n^{-\frac{1}{1+p}} + \max \oset{L^2, \bt^2 M \epsilon \norm{f^*}_{\px}^2 } n^{-\frac{1+2p}{1+p}}  }
    \end{equation*}
    holds for all $f^* \in \gF_{\epsilon}(\hks)$ with probability at least $1 - 4 e^{-\tau}$, given that $n$ is sufficiently large.
\end{theorem}

\begin{remark}
    \cite{fischer2020sobolev} showed that the learning rate $O(n^{-\frac{1}{1+p}})$ is minimax optimal, which means that one can construct an example where the learning rate is at most $O(n^{-\frac{1}{1+p}})$.
    When $\bt^2 = \norm{k}_{\infty} < \infty$, one can always choose $p = 1$ because $i \cdot s(\lambda_i) \le \sum_{j=1}^i s(\lambda_j) \le M \sum \lambda_j \le M \bt^2$.
\end{remark}

Regarding the estimation error, we prove the following result.
\begin{theorem}[Proof in \Cref{app:thm-estimat-error-stkr}]
\label{thm:estimat-error-stkr}
Let $\hat{\lambda}_1$ be the largest eigenvalue of $\frac{\gk}{m+n}$, and denote $\lambda_{\max} = \max \oset{ \lambda_1, \hat{\lambda}_1 }$.
Then, for any $\delta \in (0,1)$, with probability at least $1 - \delta$ we have
\begin{equation*}
    \norm{\hat{f} - \tilde{f}}_{\px}^2 \le 8s(\lambda_{\max}) \left . \nabla_{\lambda} \paren{\frac{s(\lambda)}{\lambda}} \right |_{\lambda = \lambda_{\max}} \frac{\beta_n^{-2} \bt^4}{\sqrt{m+n}} \paren{2 + \sqrt{2 \log \frac{1}{\delta}}} \frac{\norm{\vy}_2^2}{n}  .
\end{equation*}
\end{theorem}

This result requires $\hat{\lambda}_1$, for which there is the following result.
\begin{lemma}[\cite{shawe2005eigenspectrum}, Theorem~2]
\label{lem:shawe-taylor}
    For any $\delta \in (0,1)$, with probability at least $1-\delta$ we have
\begin{equation*}
    \hat{\lambda}_1 \le \lambda_1 + \frac{\bt^2}{\sqrt{m+n}} \brac{ 2 \sqrt{2} + \sqrt{19 \log \frac{2(m+n+1)}{\delta}} }  .
\end{equation*}
\end{lemma}

The key to prove \Cref{thm:estimat-error-stkr} is using the complexity of $\hks$ to prove a uniform deviation bound for $\abs{ \hatks(x,x_i) - \ks(x,x_i) }$ for all $x$ and $i$.
Thus, we can see why we only used the $m$ unlabeled samples to extract the top-$d$ eigenspace, but can use all $(m+n)$ samples in STKR.
The reason is that for uniform deviation bounds, the function class must be independent of the training samples.
For the top-$d$ eigenfunctions, the downstream function class is the span of $\phi_1,\cdots,\phi_d$, and the proof of \Cref{thm:fischer-1} was based on this function class.
This function class depends on the $m$ unlabeled samples.
If the $n$ labeled samples were also used, then the uniform deviation bound would not hold.
On the other hand, in \Cref{thm:estimat-error-stkr}, the function class if $\hks$, which is independent of all $(m+n)$ samples.
Hence, we can use all the $(m+n)$ samples in STKR.

\section{Empirical Study of Contexture Learning and STKR}
This section conducts some experiments on STKR with the following goals:
\begin{enumerate}[label=(\roman*)]
    \item Verify that STKR-Prop (\Cref{alg:richardson,alg:richardson-inverse}) works with general polynomial $s$, such as the inverse Laplacian, and compare them to label propagation (Label-Prop).
    \item Explore possible reasons why the inverse Laplacian works so well in practice, by examining the effect of $p$ on the performance when using STKR with $s(\lambda) = \lambda^p$.
    \item Compare extracting the top-$d$ eigenspace with STKR and Label-Prop.
\end{enumerate}

The experiments here focus on graph node classification tasks.
The datasets used are listed in \Cref{tab:stkr-datasets}, and they all come from the \textit{PyTorch Geometric} library \citep{Fey/Lenssen/2019}.
Each dataset is split into four sets: train, validation (val), test and other.
Among them, train and val contain labeled samples,
while test and other contain unlabeled samples.
Both the transductive and the inductive settings are tested.

\begin{table}[t]
\centering
\vskip -.1in
\begin{tabular}{@{}lccccc@{}}
\toprule
 & \multicolumn{1}{l}{Classes} & \multicolumn{1}{l}{Nodes} & \multicolumn{1}{l}{Edges} & \multicolumn{1}{l}{Train} & \multicolumn{1}{l}{Validation} \\ \midrule
Cora & 7 & 2,708 & 10,556 & 5.17 & 18.46 \\
CiteSeer & 6 & 3,327 & 9,104 & 3.61 & 15.03 \\
PubMed & 3 & 19,717 & 88,648 & 0.3 & 2.54 \\
Amazon - Computers & 10 & 13,752 & 491,722 & 1.45 & 1.45 \\
Amazon - Photos & 8 & 7,650 & 238,162 & 2.09 & 2.09 \\
Coauthor - CS & 15 & 18,333 & 163,788 & 1.64 & 1.64 \\
Coauthor - Physics & 5 & 34,493 & 495,924 & 0.29 & 0.29 \\
DBLP & 4 & 17,716 & 105,734 & 0.45 & 0.45 \\
CoraFull & 70 & 19,793 & 126,842 & 7.07 & 7.07 \\ \bottomrule
\end{tabular}
\caption{Number of classes, nodes, edges, and fractions (\%) of train/validation sets.}
\label{tab:stkr-datasets}
\end{table}

\begin{itemize}
    \item In the \textbf{transductive setting}, the samples in all four sets are available at train time. The learner hides the labels of val samples for validation. Thus, $n$ is the size of the train set, while $m$ is the size of all other three sets combined. 
    \item In the \textbf{inductive setting}, samples in the test set are invisible at train time. The learner hides the entire val set (samples and labels) for validation. Thus, $n$ is the size of the train set, while $m$ is the size of the other set.
\end{itemize}

Label propagation (Label-Prop) only works for the transductive setting, and is implemented as follows: 
Let $\mW$ be the adjacency matrix of the graph, such that $\mW[i,j] = 1$ if $x_i$ and $x_j$ are connected, and $0$ otherwise.
Let $\mD$ be a diagonal matrix such that $\mD[i,i] = \sum_j \mW[i,j]$.
Let $\mS = \mD^{-\frac{1}{2}} \mW \mD^{-\frac{1}{2}}$.
Following \cite{zhou2003learning}, Label-Prop solves
\begin{equation*}
    \paren{\mI_{m+n} - \eta \mS} \hat{\vy} = \tilde{\vy}, \quad \text{where } \vy = [\vy, \vzero_m].
\end{equation*}
Then, $\hat{\vy}$ contains the predicted labels for all $(m+n)$ samples.
On the other hand, STKR works for both transductive and inductive settings. Its base kernel $k$ is defined as
\begin{equation*}
    k(x,x') = (m+n) \frac{W(x,x')}{\sqrt{D(x) D(x')}},
\end{equation*}
where $W(x_i,x_j) = \mW[i,j]$.
For the transductive setting, $D(x_i) = \mD[i,i]$;
for the inductive setting, $D(x_i) = \sum_{j \notin \text{ test nodes}}W(x_i,x_j)$, that is the sum is taken over only visible nodes.

\begin{table}[t]
\centering
\resizebox{\columnwidth}{!}{%
\begin{tabular}{@{}llllllllll@{}}
\toprule
 & \textbf{CS} & \textbf{CiteSeer} & \textbf{Computers} & \textbf{Cora} & \textbf{CoraFull} & \textbf{DBLP} & \textbf{Photo} & \textbf{Physics} & \textbf{PubMed} \\ \midrule
LP (t) & $79.07_{2.19}$ & {\boldmath $52.73_{7.72}$} & $77.30_{3.05}$ & \underline{$73.33_{6.00}$} & {\boldmath $54.47_{3.24}$} & {\boldmath $66.44_{3.78}$} & $83.95_{5.78}$ & $84.33_{4.86}$ & {\boldmath $72.28_{5.55}$} \\
Lap (t) & $78.96_{2.53}$ & \underline{$52.12_{7.67}$} & $77.81_{3.94}$ & {\boldmath $77.04_{5.74}$} & \underline{$53.81_{2.34}$} & \underline{$65.42_{5.02}$} & $84.08_{6.52}$ & $84.22_{4.86}$ & $71.93_{4.86}$ \\
Poly (t) & {\boldmath $79.13_{2.29}$} & $48.79_{8.51}$ & $76.72_{4.12}$ & $71.48_{5.80}$ & $53.25_{3.54}$ & $64.52_{4.20}$ & $79.21_{7.20}$ & \underline{$84.45_{4.89}$} & \underline{$72.18_{4.66}$} \\
Topd (t)  & $78.80_{3.22}$ & $46.06_{1.08}$ & {\boldmath $80.80_{3.06}$} & $69.26_{7.82}$ & $50.36_{2.85}$ & $64.86_{4.60}$ & {$84.61_{6.30}$} & $83.20_{2.25}$ & $65.38_{5.66}$ \\
Lap (i) &$78.42_{2.81}$ & $46.06_{6.97}$ & $77.15_{2.64}$ & $67.78_{7.62}$ & $53.30_{3.24}$ & $65.20_{4.92}$ & \underline{$84.87_{5.66}$} & $83.11_{5.09}$ & $70.36_{4.80}$ \\
Poly (i) & ${79.02}_{2.42}$ & ${44.55}_{9.15}$ & ${71.97}_{4.13}$ & ${65.19}_{9.11}$ & ${51.98}_{3.88}$ & ${64.52}_{4.05}$ & ${78.42}_{7.80}$ & {\boldmath ${84.68}_{4.83}$} & ${70.76}_{4.28}$ \\
Topd (i) & \underline{$79.13_{3.35}$} & $41.52_{6.71}$ & \underline{$80.80_{3.28}$} & $63.70_{6.00}$ & $47.41_{3.39}$ & $63.16_{3.41}$ & {\boldmath $85.53_{5.68}$} & $82.44_{3.88}$ & $64.31_{4.95}$\\
KRR (i) & $13.11_{2.29}$ & $13.64_{5.93}$ & $26.35_{4.34}$ & $28.52_{8.56}$ & $19.80_{2.22}$ & $44.80_{3.86}$ & $33.95_{7.07}$ & $19.74_{1.46}$ & $20.76_{2.06}$ \\
\bottomrule
\end{tabular}
}
\caption{The test accuracy (\%) of Label-Prop (LP), STKR-Prop with inverse Laplacian (Lap), with polynomial $s(\lambda) = \lambda^8$ (Poly), with kernel PCA (Topd), and with $s(\lambda) = \lambda$ (KRR). (t) and (i) indicate transductive and inductive. Standard deviations are given across ten random seeds. Best/second-best results are in bold/underlined.
}
\label{tab:stkr-exp}
\end{table}

We use 1\% of the samples as test samples.
Each experiment is run with 10 random seeds for splitting the dataset.
The results are reported in \Cref{tab:stkr-exp}, from which we make the following observations:
\begin{itemize}
    \item STKR works pretty well with general polynomial $s(\lambda)$ in the inductive setting.
In the transductive setting, the performance of STKR-Prop with the inverse Laplacian is similar to Label-Prop, and STKR-Prop with polynomial $s$ is slightly worse. 
The performance under the inductive setting is slightly worse than that under the transductive setting,
which is reasonable since there is less information at train time for the inductive setting. 
It is also worth noting that the running time of STKR-Prop is similar to that of Label-Prop with the same number of iterations.
\item STKR with $s(\lambda) = \lambda^p$ for $p > 1$ is much better than KRR (where $p = 1$).
Moreover, we test STKR with $s(\lambda) = \lambda^p$ for more choices of $p \in \oset{1,2,4,6,8}$, and report the results on three data sets in \Cref{fig:stkr-poly-p}. It is clear that a larger $p$ leads to higher performance.
This suggests one possible reason why inverse Laplacian works so well empirically: 
It contains $k^p$ for $p=1,2,\cdots$, so it can use multi-step similarity information up to infinitely many steps. 
\item Extracting the top-$d$ eigenspace with kernel PCA can achieve pretty high performance.
Specifically, on 3 of the 9 datasets we use, such as \texttt{Computers}, STKR with top-$d$ truncation is better than Label-Prop and STKR with inverse Laplacian.
This shows that STKR with inverse Laplacian and STKR with top-$d$ truncation plus kernel PCA are two parallel methods---neither is superior to the other one.
\end{itemize}

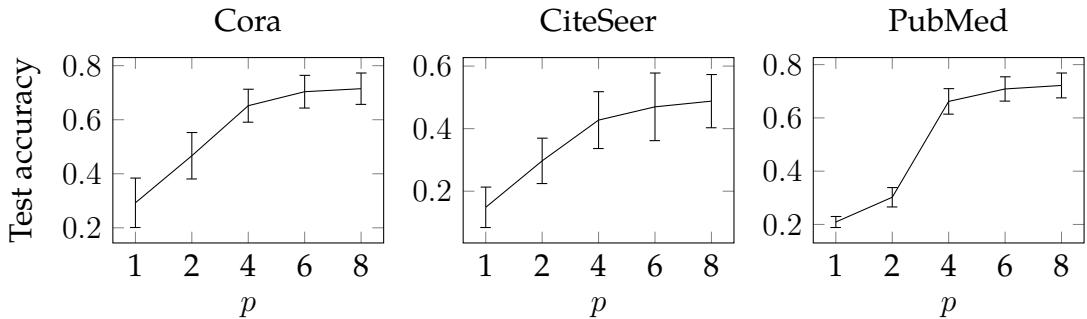
\begin{figure}[t]
    \centering
    \begin{tikzpicture}
    \begin{axis}[name=plot0, width=0.32\textwidth,height=0.25\textwidth, symbolic x coords={1,2,4,6,8},xtick=data,
    title={Cora}, xlabel={$p$}, ylabel={Test accuracy}]
        \addplot [error bars/.cd, y dir=both, y explicit]
            table [col sep=comma, x=p, y=Mean, y error=STDev]{stkr-poly-p-cora.csv};
    \end{axis}
    \begin{axis}[name=plot1, width=0.32\textwidth,height=0.25\textwidth, symbolic x coords={1,2,4,6,8},xtick=data,
    title={CiteSeer}, xlabel={$p$}, at={($(plot0.south east)+(30,0)$)},]
        \addplot [error bars/.cd, y dir=both, y explicit]
            table [col sep=comma, x=p, y=Mean, y error=STDev]{stkr-poly-p-citeseer.csv};
    \end{axis}
    \begin{axis}[name=plot2, width=0.32\textwidth,height=0.25\textwidth, symbolic x coords={1,2,4,6,8},xtick=data,
    title={PubMed}, xlabel={$p$}, at={($(plot1.south east)+(30,0)$)},]
        \addplot [error bars/.cd, y dir=both, y explicit]
            table [col sep=comma, x=p, y=Mean, y error=STDev]{stkr-poly-p-pubmed.csv};
    \end{axis}
\end{tikzpicture}
    \caption{Performance of STKR-Prop with $s(\lambda) = \lambda^p$ and 8 iterations. The best $\beta_n$ with the highest test accuracy is selected. Each experiment is run with 10 random seeds.}
    \label{fig:stkr-poly-p}
\end{figure}

In conclusion, this chapter first derived generalization bounds for contexture learning, that is extracting the top-$d$ eigenspace with finite samples.
The key quantity is the context complexity, which reflects the smoothness of the eigenfunctions.
Next, this chapter introduced a more general formulation called STKR, and presented its implementation as well as generalization bounds.
The key takeaway from this chapter is how the context affects the sample complexity of representation learning.

\chapter{Generalization Under Distribution Shift}
\label{chap:dist-shift}

The contexture theory characterizes the mechanism of representation learning, thereby advancing the science of foundation models.
However, one main assumption of the contexture theory is that the data distribution $\px$ is fixed.
This is a very strong assumption that is hardly true in practice.
In reality, the distributions of the pretrain data and the downstream data are always different.
Whether a model trained on one distribution can still achieve good performance on another distribution is called the problem of \textbf{out-of-distribution (OOD) generalization}.
The main takeaway from this chapter is that OOD generalization is extremely hard, both in theory and in practice.

There is a rich body of work on OOD generalization in machine learning, statistics, applied probability and optimization \cite{shimodaira2000improving,huang2006correcting,bickel2007discriminative,quionero2009dataset}.
In deep learning, there are two types of research on OOD generalization.
The first type studies how to transfer a model trained on one domain to another domain, also known as transfer learning \cite{pan2010transfer} or domain adaptation \cite{wang2018deep}.
The second type studies how to train robust models against distribution shift, that is preserving the model's performance on the new distribution.
Such research is valuable in safety-critical applications or domains where the data is constantly changing, such as finance.
Both types of research are relevant to foundation models.
For example, how to apply an LLM trained on Wikipedia to a dataset of Python codes is the first type of research.
How to make sure that an LLM always generates proper responses to prompts it has never seen is the second type of research.
The first type is more related to the fine-tuning stage rather than the representation learning stage of foundation model training.
Therefore, this chapter focuses on the second type.

In representation learning, a foundation model is pretrained on one distribution $P$, and then applied to another distribution $Q$.
This chapter studies an easier problem, which is the standard problem in the literature:
A predictor is trained on $P$ and then evaluated on $Q$.
We assume that the distribution shift only contains \textbf{covariate shift}, where the ground truth target function is always fixed, and only the distribution of $X$ changes.
In practice, there are two types of covariate shift \cite{pmlr-v139-koh21a}.
\begin{itemize}
    \item In \textbf{domain generalization}, the support of $Q$ might contain samples that are not in the support of $P$. This is the typical scenario in domain adaptation.
    \item In \textbf{subpopulation shift}, the support of $Q$ is a subset of the support of $P$. Alternatively we can write $Q \ll P$, that is $Q$ is absolutely continuous to $P$. Mathematically this means that for any set $A$, $P(A) = 0$ implies $Q(A) = 0$.
\end{itemize}

In the context of foundation models, it suffices to study subpopulation shift.
There are two main reasons.
First, the pretraining set is very large, so it is very unlikely to have test samples that do not appear in the pretraining set.
Second, even in domain adaptation, the common practice is to fine-tune the foundation model on the new distribution before using it.
It is rarely the case that a model is used on a completely new distribution without any fine-tuning.
For these reasons, this chapter focuses on subpopulation shift exclusively.

\section{Reweighting and DRO}

In subpopulation shift, $P$ and $Q$ are different distributions on the same set of samples.
Thus, one can view $Q$ as assigning different weights to the samples than $P$.
For example, given a training set $\oset{(x_i,y_i)}_{i=1}^n$, $P$ is usually defined as the uniform distribution over the $n$ samples, while $Q$ gives different weights to these samples.
Since the model is evaluated on $Q$, we care more about those samples where $Q$ place larger weights than $P$, because these samples have greater impact on the evaluation. We call them \textbf{upweighted samples}.

Empirical risk minimization (ERM) is the standard training algorithm in deep learning. It assumes that the training samples are \iid{} sampled from the data distribution, and minimizes the average model risk over the training samples.
Let $\ell(\hat{y}, y)$ be the loss function. Then, ERM minimizes the following empirical risk:
\begin{equation*}
    \hat{\gR}_{\erm}(f) = \frac{1}{n} \sum_{i=1}^n \ell(f(x_i), y_i)  .
\end{equation*}

However, if $P$ and $Q$ give different weights to the samples, then ERM could achieve low performance on $Q$ if its risk on the upweighted samples on higher than average \cite{hovy2015tagging,blodgett2016demographic,tatman2017gender}.
One common example of this is \textbf{class imbalance}, which was briefly discussed in \Cref{sec:illustrative-examples}.
In a classification task, if some classes are significantly smaller than the other classes, then a model trained via ERM will typically have high risk on these small classes.
However, the model is required to have good performance on every class, meaning that $Q$ places larger weights on these small classes than $P$.
As a result, the ERM model is poor in the class imbalance situation.

Given the above discussion,
the most straightforward way to tackle the subpopulation shift is reweighting, also known as \textbf{importance weighting} \cite{shimodaira2000improving}.
The idea is to assign a different weight $P'$ to the samples in the training loss, such that $P' = Q$.
The importance weighting (IW) empirical risk is
\begin{equation*}
    \hat{\gR}_{\iw}(f) = \frac{1}{n} \sum_{i=1}^n \frac{Q(x_i)}{P(x_i)}  \ell(f(x_i), y_i)  . 
\end{equation*}
Note that $P(x)$ and $Q(x)$ are the weights of $x$, and they are not necessarily the density functions.
For example, in class imbalance, $P(x)$ is proportional to the size of the class of $x$.
If the goal is to have good performance on every class, then $Q(x)$ is the same for all $x$.
Consequently, $\gR_{\iw}$ divides the loss on each sample by its class size, thereby assigning larger weights to samples in smaller classes.
Note that this is different from \Cref{eqn:obj-classification-balanced}, which divides the sample loss by the square root of its class size.

Importance weighting assumes that we know $Q$ at train time.
What if $Q$ is unknown?
This is the more common case in training foundation models.
The pretraining dataset is a very large and comprehensive dataset, while the downstream task usually focuses on only one domain, that is a small part of the pretraining data.
However, since we may not know the downstream task at pretrain time, we do not know which data it will focus on.

\textbf{Distributionally robust optimization (DRO)} is the most popular approach when $Q$ is unknown.
It aims to minimize the model risk on the worst $Q$ whose distance to $P$ is bounded by some $\rho > 0$.
The distance from one distribution to another is also known as a divergence function, and is denoted by $D(Q \parallel P)$.
Note that $D(Q \parallel P)$ might not be symmetric, that is we could have $D(Q\parallel P) \neq D(P \parallel Q)$.
The DRO risk is defined as
\begin{equation}
\label{eqn:dro-risk}
\gR_{D,\rho}(f;P) = \sup_Q \sset{\E_{(X,Y) \sim Q} [\ell \paren{f(X),Y}]}{D(Q \parallel P) \leq \rho}  .
\end{equation}
The constraint $D(Q \parallel P) \leq \rho$ comes from our prior knowledge about $Q$.
The following are two examples of DRO.

\begin{example}
    \textbf{Conditional value at risk (CVaR)} \cite[Example~3]{duchi2018learning} aims to maximize the model's performance on the worst $\alpha$ fraction of the data, for some fixed $\alpha \in (0,1)$.
    For example, CVaR is widely used in finance, where the model needs to perform well when the market is at its lowest point.
    In this scenario, we have $Q(x) \le \alpha^{-1}P(x)$ for all $x$, that is the sample weight under $P$ is at least $\alpha$ times the weight under $Q$. Thus, $D(Q \parallel P) = \sup \frac{Q(x)}{P(x)}$, and $\rho = \alpha^{-1}$.
\end{example}

\begin{example}
    In \textbf{group DRO (GDRO)} \cite{sagawa2019distributionally}, the data is divided into a number of groups, and the model is required to perform well on the worst group. 
    This is a common scenario in the field of machine learning fairness.
    For example, a credit approval model is required to be fair across all races.
    Then, the groups are defined by the races.
    In this case, $P(x)$ and $Q(x)$ are the weights on the group of $x$, $D(Q \parallel P) = \sup \frac{Q(x)}{P(x)}$, and $\rho$ is the number of groups.
\end{example}

Both reweighting and DRO are quite heuristic, but do they actually work in practice?
The results are mixed.
For example, balancing the classes can usually give the model higher performance on small classes, but at the cost of the average performance.
Meanwhile, \cite{sagawa2019distributionally} showed that reweighting and DRO methods can overfit very easily, and they usually require much stronger regularization than ERM, or early stopping.
Furthermore, \cite{gulrajani2021in} conducted a large-scale empirical study, and showed the surprising negative result that reweighting are not better than ERM on most real datasets.
The authors argued that reweighting were reported better in prior work only because the ERM baseline had not been sufficiently tuned.

The rest of this chapter will show two reasons why reweighting and DRO might fail.
The first reason is related to the training dynamics of these methods.
The second reason is related to their sensitivity to the outliers in the dataset.
Possible solutions to each failure mode will be discussed.
These analyses are based on my work \cite{pmlr-v139-zhai21a,zhai2021,zhai2023grw}.

\section{Generalized Reweighting (GRW) Versus ERM}
Let the input space be $\gX \subseteq \R^{\dx}$, and we assume that all $x \in \gX$ satisfies $\norm{x}_2 \le 1$.
Consider learning a target function $f^*: \gX \rightarrow \R$ using a training set $\oset{(x_i, y_i)}_{i=1}^n$.
Denote $\mX = (x_1,\cdots,x_n) \in \R^{\dx \times n}$, and $\mY = (y_1,\cdots,y_n)\in \R^n$. For any function $g: \gX \mapsto \R^m$, we overload notation and denote $g(\mX) = (g(x_1),\cdots,g(x_n)) \in \R^{m \times n}$.

The difference between reweighting and DRO is that $Q$ is fixed in reweighting are fixed, while it is not in DRO.
Here we present a general formulation called \textbf{generalized reweighting (GRW)}.
At training step $t$, GRW minimizes the following weighted empirical risk:
\begin{equation}
\label{eqn:w-emp}
\hat{\gR}_{\vq^{(t)}}(f) = \sum_{i=1}^n q_i^{(t)} \ell(f(x_i), y_i)  ,
\end{equation}
where $\vq^{(t)} = (q_1^{(t)},\cdots,q_n^{(t)})$ is the sample weight vector, such that $q_1^{(t)}+\cdots+q_n^{(t)} = 1$.
If $\vq^{(t)}$ does not change with $t$, we call it static GRW;
if $\vq^{(t)}$ can change with $t$, we call it dynamic GRW.
Note that ERM is a special case of static GRW.

Importance weighting is obviously an example of static GRW.
Group DRO is an example of dynamic GRW, and here is how it is usually implemented.
Let there be $K$ groups.
Denote the empirical risk over group $k$ by $\hat{\gR}_k(f)$, and the model at time $t$ by $f^{(t)}$. 
For all $k \in [K]$, group DRO iteratively sets $q_i^{(t)} = g_k^{(t)} / n_k$ for all $(x_i,y_i)$ in group $k$, where $g_k^{(t)}$ is the group weight that is updated by
\begin{equation*}
    g_k^{(t)} \propto g_k^{(t-1)} \exp \brac{ \nu \hat{\gR}_k \paren{f^{(t-1)}} }
\end{equation*}
for some $\nu > 0$. The group weights are normalized so that $q_1^{(t)}+\cdots+q_n^{(t)}=1$. 
\cite[Proposition~2]{sagawa2019distributionally} showed that for convex settings, the Group DRO risk of iterates converges to the global minimum with the rate $O(t^{-1/2})$ if $\nu$ is sufficiently small.

The key result of this section is that GRW cannot improve over ERM, because the models they produce are too similar.
It relies on two key assumptions.
First, the model is over-parameterized, meaning that the number of parameters is (much) greater than the number of samples. This is the usual case in deep learning. 
Second, the optimizer is gradient descent with a sufficiently small learning rate. Our results also hold for other gradient methods such as momentum SGD and Adam, as long as the training loss converges to zero.

Let us first gain some insights from linear models, using an analysis similar to \cite{pmlr-v80-gunasekar18a}.
Then, we will study neural networks.

\paragraph{Insights from linear models.}
Consider a regression task, where the loss $\ell(\hat{y}, y) = \frac{1}{2}(\hat{y} - y)^2$ is the squared loss.
Consider using a linear model denoted by $f(x) = \dotp{\theta, x}$ for $\theta \in \R^{\dx}$.
We assume that the model is over-parameterized, that is $\dx > n$.
The weight update rule of GRW under gradient descent (GD) is the following:
\begin{equation}
\label{eqn:update-linear}
    \theta^{(t+1)} = \theta^{(t)} - \eta \sum_{i=1}^n q_i^{(t)} \nabla_\theta \ell(f^{(t)}(x_i), y_i)  = \theta^{(t)} - \eta \sum_{i=1}^n q_i^{(t)}  (f^{(t)}(x_i) - y_i) x_i  ,
\end{equation}
where $\eta > 0$ is the learning rate. 
We now show that, under some assumptions, as $t \rightarrow \infty$, $\theta^{(t)}$ must converge to a common $\theta^*$ for all GRW and ERM methods.
The proof consists of two steps.
First, we prove that the training loss will converge to zero.
In this case, as long as $x_1,\cdots,x_n$ are linearly independent, $\theta^*$ must be an interpolator, which means that $\langle \theta^*, x_i \rangle = y_i$ for all $i \in [n]$.
Second, we prove that there is a unique interpolator $\theta^*$. The interpolator only depends on the starting point $\theta^{(0)}$ and the training samples, but it does not depend on the sample weights $\vq^{(t)}$.
These results require the following assumption about the sample weights.
\begin{assumption}
\label{ass:qstar}
There are constants $q_1,\cdots,q_n$ such that for all $i \in [n]$, there is $q_i^{(t)} \rightarrow q_i$ as $t \rightarrow \infty$.
Moreover, $\min_{i} q_i = q^* > 0$.
\end{assumption}

This assumption says two things.
First, the sample weights become stable after a sufficient amount of training. It guarantees that the training loss will not change too much between two consecutive training steps.
Second, all samples have positive weights, which means that no sample is deleted from the training set.
This is important because the unique $\theta^*$ depends on the training samples.
If samples were deleted, $\theta^*$ would not be the same.
The following theorem shows that the training loss will converge to zero.
\begin{theorem}[Proof in \Cref{app:proof-thm-grw-reg-linear}]
\label{thm:grw-reg-linear}
If $x_1,\cdots,x_n$ are linearly independent, then there exists a constant $\eta_0 > 0$ such that: For any GRW algorithm satisfying \Cref{ass:qstar}, under the update rule \Cref{eqn:update-linear} with $\eta \le \eta_0$, the empirical ERM risk $\hat{\gR}_{\erm}(f^{(t)}) \rightarrow 0$ as $t \rightarrow \infty$.
\end{theorem}

When $x_1,\cdots,x_n$ are linearly independent,
\Cref{eqn:update-linear} implies that $\theta^{(t+1)} - \theta^{(t)}$ is always a linear combination of $x_1,\cdots,x_n$.
As a result, for all $t$, $\theta^{(t)} - \theta^{(0)} \in \sspan \oset{x_1,\cdots,x_n}$.
This is an $n$-dimensional subspace of $\R^d$.
By Cramer's rule, there is exactly one $\tilde{\theta}$ in this subspace such that $\langle \tilde{\theta} + \theta^{(0)}, x_i \rangle = y_i$ for all $i \in [n]$.
Therefore, $\theta^* = \tilde{\theta} + \theta^{(0)}$ is unique and independent of the sample weights.
It only relies on $\theta^{(0)}$ and $x_1,\cdots,x_n$.

Finally, note that ERM is a special case of GRW.
Thus, GRW produces the exact same model as ERM, so it cannot improve over ERM.
One way to solve this problem is adding regularization.
The regularization has two effects: (i) moving $\theta^{(t+1)} - \theta^{(t)}$ out of the span of $x_1,\cdots,x_n$; (ii) preventing the model from achieving zero loss, that is interpolating the training samples.
Another solution is to add some new samples to (via data augmentation for example), or delete som samples from the dataset.

Let us demonstrate this result with a simple experiment.
The experiment is conducted on a training set of six MNIST images, five of which are digit 0 and one is digit 1. 
The two different digits define two groups.
We use a 784-dimensional linear model and run ERM, importance weighting and group DRO (with $\nu = 1$).
The results are reported in \Cref{fig:grw-reg}, in which (a) and (b) are the results when no regularization is applied.
From (a) we can see that the three models will converge to the same $\theta$;
from (b) we can see that the training loss of all three models will converge to zero.
Then, we apply $L^2$ regularization and run the experiment again. From (c) and (d) we can see that when the regularization is small, the training loss will still converge to zero, and the three models will still converge to the same $\theta$.
From (e) and (f) we can see that the three models converge to different points only when the regularization is so large that the training loss does not converge to zero.

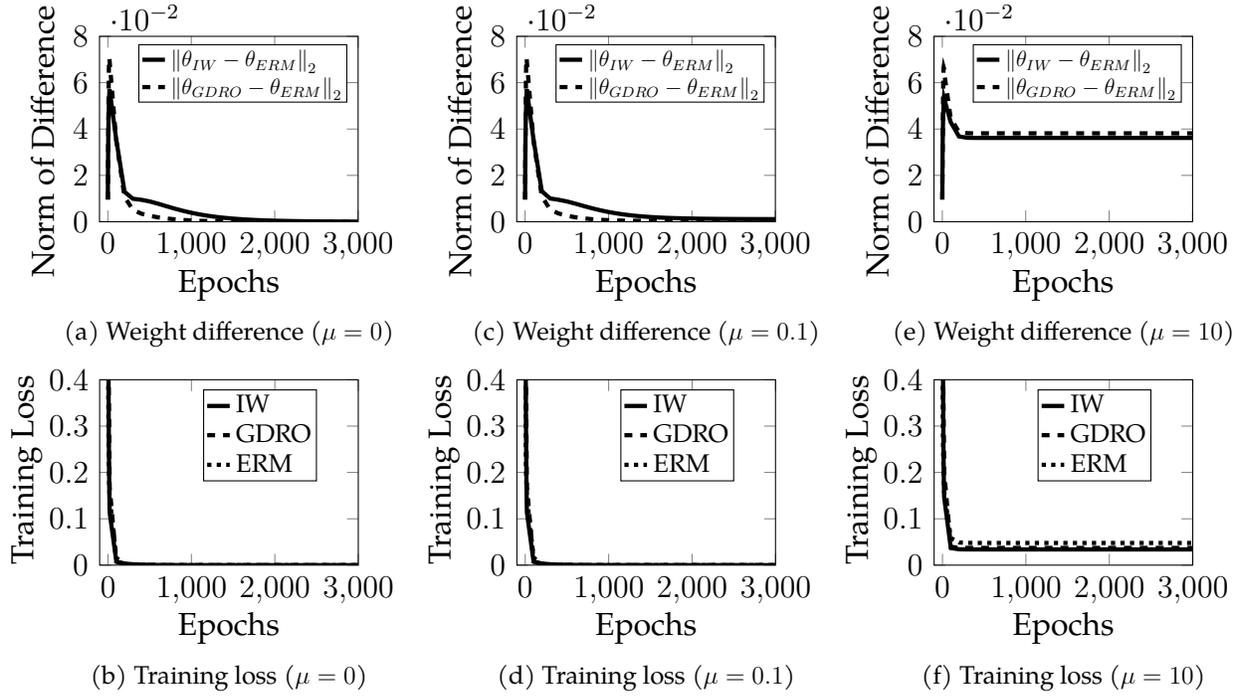
\begin{figure}[t]
    \centering
    \begin{tikzpicture}

\begin{axis}[
name=plot0,table/col sep=comma,no markers,
    xmin = -100, xmax = 3000,
    ymin = 0, ymax = 0.08,
    xlabel={Epochs},
    x label style={at={(axis description cs:0.5,-0.25)},anchor=north,inner sep=0},
    ylabel={Norm of Difference},
    y label style={at={(axis description cs:-0.22,-0.3)},anchor=west,inner sep=0},
    width = .31\linewidth,
    height = .25\linewidth,
    legend style = {nodes={scale=0.7, transform shape},at={(0.15,0.8)},anchor=west,inner sep=.5pt},
    legend cell align=left,
    ]
    \addplot [opacity=0.4, black, ultra thick, solid] table [x={epochs}, y={iw}] {data/grw_reg_00_diff.csv};
    \addplot [black, ultra thick, dashed] table [x={epochs}, y={gdro}] {data/grw_reg_00_diff.csv};
    \addlegendentry{$\norm{\theta_{IW} - \theta_{ERM}}_2$}
    \addlegendentry{$\norm{\theta_{GDRO} - \theta_{ERM}}_2$}
\end{axis}

\begin{axis}[
name=plot1,table/col sep=comma,no markers,
at={($(plot0.south west)+(0,-130)$)},
    xmin = -100, xmax = 3000,
    ymin = 0, ymax = 0.4,
    xlabel={Epochs},
    x label style={at={(axis description cs:0.5,-0.25)},anchor=north,inner sep=0},
    ylabel={Training Loss},
    y label style={at={(axis description cs:-0.28,0)},anchor=west,inner sep=0},
    width = .31\linewidth,
    height = .25\linewidth,
    legend style = {nodes={scale=0.8, transform shape},at={(0.4,0.7)},anchor=west,inner sep=.5pt},
    legend cell align=left,
    ]
    \addplot [opacity=0.4, black, ultra thick, solid] table [x={epochs}, y={iw}] {data/grw_reg_00_train.csv};
    \addplot [opacity=0.6, black, ultra thick, dashed] table [x={epochs}, y={gdro}] {data/grw_reg_00_train.csv};
    \addplot [black, ultra thick, dotted] table [x={epochs}, y={erm}] {data/grw_reg_00_train.csv};
    \addlegendentry{IW}
    \addlegendentry{GDRO}
    \addlegendentry{ERM}
\end{axis}

\begin{axis}[
name=plot2,table/col sep=comma,no markers,
at={($(plot0.south east)+(60,0)$)},
    xmin = -100, xmax = 3000,
    ymin = 0, ymax = 0.08,
    xlabel={Epochs},
    x label style={at={(axis description cs:0.5,-0.25)},anchor=north,inner sep=0},
    ylabel={Norm of Difference},
    y label style={at={(axis description cs:-0.22,-0.3)},anchor=west,inner sep=0},
    width = .31\linewidth,
    height = .25\linewidth,
    legend style = {nodes={scale=0.7, transform shape},at={(0.15,0.8)},anchor=west,inner sep=.5pt},
    legend cell align=left,
    ]
    \addplot [opacity=0.4, black, ultra thick, solid] table [x={epochs}, y={iw}] {data/grw_reg_0.1_diff.csv};
    \addplot [black, ultra thick, dashed] table [x={epochs}, y={gdro}] {data/grw_reg_0.1_diff.csv};
    \addlegendentry{$\norm{\theta_{IW} - \theta_{ERM}}_2$}
    \addlegendentry{$\norm{\theta_{GDRO} - \theta_{ERM}}_2$}
\end{axis}

\begin{axis}[
name=plot3,table/col sep=comma,no markers,
at={($(plot2.south west)+(0,-130)$)},
    xmin = -100, xmax = 3000,
    ymin = 0, ymax = 0.4,
    xlabel={Epochs},
    x label style={at={(axis description cs:0.5,-0.25)},anchor=north,inner sep=0},
    ylabel={Training Loss},
    y label style={at={(axis description cs:-0.28,0)},anchor=west,inner sep=0},
    width = .31\linewidth,
    height = .25\linewidth,
    legend style = {nodes={scale=0.8, transform shape},at={(0.4,0.7)},anchor=west,inner sep=.5pt},
    legend cell align=left,
    ]
    \addplot [opacity=0.4, black, ultra thick, solid] table [x={epochs}, y={iw}] {data/grw_reg_0.1_train.csv};
    \addplot [opacity=0.6, black, ultra thick, dashed] table [x={epochs}, y={gdro}] {data/grw_reg_0.1_train.csv};
    \addplot [black, ultra thick, dotted] table [x={epochs}, y={erm}] {data/grw_reg_0.1_train.csv};
    \addlegendentry{IW}
    \addlegendentry{GDRO}
    \addlegendentry{ERM}
\end{axis}

\begin{axis}[
name=plot4,table/col sep=comma,no markers,
at={($(plot2.south east)+(60,0)$)},
    xmin = -100, xmax = 3000,
    ymin = 0, ymax = 0.08,
    xlabel={Epochs},
    x label style={at={(axis description cs:0.5,-0.25)},anchor=north,inner sep=0},
    ylabel={Norm of Difference},
    y label style={at={(axis description cs:-0.22,-0.3)},anchor=west,inner sep=0},
    width = .31\linewidth,
    height = .25\linewidth,
    legend style = {nodes={scale=0.7, transform shape},at={(0.15,0.8)},anchor=west,inner sep=.5pt},
    legend cell align=left,
    ]
    \addplot [opacity=0.4, black, ultra thick, solid] table [x={epochs}, y={iw}] {data/grw_reg_10_diff.csv};
    \addplot [black, ultra thick, dashed] table [x={epochs}, y={gdro}] {data/grw_reg_10_diff.csv};
    \addlegendentry{$\norm{\theta_{IW} - \theta_{ERM}}_2$}
    \addlegendentry{$\norm{\theta_{GDRO} - \theta_{ERM}}_2$}
\end{axis}

\begin{axis}[
name=plot5,table/col sep=comma,no markers,
at={($(plot4.south west)+(0,-130)$)},
    xmin = -100, xmax = 3000,
    ymin = 0, ymax = 0.4,
    xlabel={Epochs},
    x label style={at={(axis description cs:0.5,-0.25)},anchor=north,inner sep=0},
    ylabel={Training Loss},
    y label style={at={(axis description cs:-0.28,0)},anchor=west,inner sep=0},
    width = .31\linewidth,
    height = .25\linewidth,
    legend style = {nodes={scale=0.8, transform shape},at={(0.4,0.7)},anchor=west,inner sep=.5pt},
    legend cell align=left,
    ]
    \addplot [opacity=0.4, black, ultra thick, solid] table [x={epochs}, y={iw}] {data/grw_reg_10_train.csv};
    \addplot [opacity=0.6, black, ultra thick, dashed] table [x={epochs}, y={gdro}] {data/grw_reg_10_train.csv};
    \addplot [black, ultra thick, dotted] table [x={epochs}, y={erm}] {data/grw_reg_10_train.csv};
    \addlegendentry{IW}
    \addlegendentry{GDRO}
    \addlegendentry{ERM}
\end{axis}

\node[scale=0.8] at ($(plot0.south)+(0,-1.5)$) {(a) Weight difference ($\mu = 0$)};
\node[scale=0.8] at ($(plot1.south)+(0,-1.5)$) {(b) Training loss ($\mu = 0$)};
\node[scale=0.8] at ($(plot2.south)+(0,-1.5)$) {(c) Weight difference ($\mu = 0.1$)};
\node[scale=0.8] at ($(plot3.south)+(0,-1.5)$) {(d) Training loss ($\mu = 0.1$)};
\node[scale=0.8] at ($(plot4.south)+(0,-1.5)$) {(e) Weight difference ($\mu = 10$)};
\node[scale=0.8] at ($(plot5.south)+(0,-1.5)$) {(f) Training loss ($\mu = 10$)};

\end{tikzpicture}
    \caption{Experiment results of ERM, importance weighting (IW) and group DRO (GDRO) with the squared loss and $L^2$ regularization on six MNIST images with a linear model. $\mu$ is the regularization coefficient. All norms are $L^2$ norms.}
    \label{fig:grw-reg}
\end{figure}

To conclude, we gain the following insight from the above analysis:
Without a very large regularization, early stopping or altering the training set, GRW and ERM will produce very similar models, so GRW cannot be better than ERM.

\paragraph{Wide neural networks, regression tasks.}

With this insight, we now study neural networks.
In particular, we focus on \textit{sufficiently wide fully-connected neural networks} within the neural tangent kernel (NTK) regime \cite{jacot2018ntk}.
The neural network is defined as
\begin{equation*}
    \vh^{l+1} = \frac{\mW^{l}}{\sqrt{d_l}} \vx^l + \beta\vb^{l} \; ; \qquad \vx^0 = x, \;  \vx^{l+1} = \sigma(\vh^{l+1}) . \quad (l = 0,\cdots,L)  
\end{equation*}
Here $\sigma$ is a non-linear activation function, $\mW^{l} \in \R^{d_{l+1} \times d_l}$ and $\mW^{L} \in \R^{1 \times d_L}$, and $d_0 = \dx$. The parameter vector $\theta$ consists of $\mW^0,\cdots,\mW^L$ and $\vb^0,\cdots,\vb^L$ ($\theta$ is the concatenation of all flattened weights and biases). The final output is $f(x) = \vh^{L+1}$.
A wide neural network has large $d_1,\cdots,d_L$.
Moreover, the neural network is initialized as
\begin{equation*}
       \left \{
    \begin{aligned}  
    &\mW_{i,j}^{l(0)} \sim \gN(0,1) \\ 
    &\vb_j^{l(0)} \sim \gN(0,1)
    \end{aligned}
    \right . \quad
    (l = 0,\cdots,L-1) \qquad \text{and} \qquad  \left \{
    \begin{aligned}  
    &\mW_{i,j}^{L(0)} = 0 \\ 
    &\vb_j^{L(0)} \sim \gN(0,1)
    \end{aligned}
    \right .
\end{equation*}

Finally, we assume that $\sigma$ is differentiable everywhere; and both $\sigma$ and its first-order derivative $\dot{\sigma}$ are Lipschitz, meaning that there exists a constant $L > 0$ such that $| f(x_1) - f(x_2) | \le L \norm{x_1 - x_2}_2$ for all $x_1,x_2$.
In the rest of this chapter, we will use \textit{wide NN} to refer to a neural network that satisfies all the above conditions.

The \textit{neural tangent kernel} (NTK) is defined as $\Theta^{(0)}(x,x') = \nabla_\theta f^{(0)}(x)^{\top}  \nabla_\theta f^{(0)}(x')$.
Our result is based on the following NTK theorem proved in \cite{jacot2018ntk}.
\begin{lemma}
\label{lem:ntk}
If $\sigma$ is Lipschitz and $d_l \rightarrow \infty$ for $l = 1,\cdots,L$ sequentially, then $\Theta^{(0)}(x, x')$ converges in probability to a non-degenerate deterministic limiting kernel $\Theta(x, x')$. Here ``non-degenerate'' means that $\Theta(x, x')$ depends on $x$ and $x'$ and is not a constant.
\end{lemma}

The \textit{kernel Gram matrix} $\Theta = \Theta(\mX, \mX) \in \R^{n \times n}$ is a positive semi-definite symmetric matrix. Denote its largest and smallest eigenvalues by $\lambda^{\max}$ and $\lambda^{\min}$. Note that $\Theta$ is non-degenerate, so we can assume that $\lambda^{\min} > 0$ (which is almost surely true when $d_L \gg n$). 

\begin{theorem}[Proof in \Cref{app:proof-thm-grw-reg-nn}]
\label{thm:grw-reg-nn}
Let $f^{(t)}$ be a wide NN trained by any GRW method satisfying \Cref{ass:qstar} with the squared loss.
Let $f^{(t)}_{\erm}$ be the same model trained by ERM from the same initial point. 
Suppose $d_1=\cdots=d_L=\dl$, $\nabla_{\theta} f^{(0)}(x_1) ,\cdots,\nabla_{\theta} f^{(0)}(x_n) $ are linearly independent, and $\lambda^{\min} > 0$.
Then, with a sufficiently small $\eta$, for any $\delta > 0$, there exists $\tilde{D}(\lambda) > 0$ such that when $\tilde{D}(\lambda) \le \dl \rightarrow \infty$, with probability at least $(1-\delta)$ over random initialization, there is
\begin{equation*}
    \limsup_{t \rightarrow \infty } \left | f^{(t)}(x) - f^{(t)}_{\erm}(x) \right | = O(\dl^{-1/4}) \rightarrow 0 \qquad \text{for all } x \in \R^d \text{ such that } \norm{x}_2 \le 1 .
\end{equation*}
\end{theorem}

This theorem says that on any test point $x$ in the unit ball, the GRW model and the ERM model produce almost the same output.
Thus, the two models must have similar OOD generalization performance.
Note that for simplicity, we only prove for $d_1=\cdots=d_L=\dl \rightarrow \infty$, but the result can be very easily extended to the case where $d_l / d_1 \rightarrow \alpha_l$ for $l = 2,\cdots,L$ for some constants $\alpha_2,\cdots,\alpha_L$, and $d_1 \rightarrow \infty$.

The key of proving this theorem is to consider the following \textit{linearized neural network}:
\begin{equation}
\label{eqn:def-linearized-nn}
    f^{(t)}_{\lin}(x) = f^{(0)}(x) + \dotp{ \theta^{(t)} - \theta^{(0)}, \nabla_{\theta} f^{(0)}(x)},
\end{equation}
which is a linear model \wrt{} $\nabla_\theta f^{(0)}(x)$.
If $\nabla_{\theta} f^{(0)}(x_1) ,\cdots,\nabla_{\theta} f^{(0)}(x_n) $ are linearly independent (which is almost surely true when the model is overparameterized so that $\theta$ has a very high dimension), then our previous insight tells us that the linearized network will converge to the unique interpolator.
It then suffices to show that the wide NN can be approximated by its linearized version uniformly throughout training.

Now let us study the effect of $L^2$ regularization, with which the GRW learning objective becomes
\begin{equation}
\label{eqn:grw-l2-pen}
    \hat{\gR}_{\vq^{(t)}}^{\mu}(f) = \sum_{i=1}^n q_i^{(t)} \ell(f(x_i),y_i) + \frac{\mu}{2} \norm{ \theta - \theta^{(0)} }_2^2  .
\end{equation}

Adding regularization does make a difference to the model regardless of how big $\mu$ is.
However, to make it possible for GRW to improve over ERM,
the regularization must be large enough to \textit{significantly lower the training performance}. Otherwise, the final model would still be too close to the unregularized ERM model.

\begin{theorem}[Proof in \Cref{app:proof-thm-reg-nn-regu}]
\label{thm:grw-reg-nn-regu}
Let $f_{\reg}^{(t)}$ be a regularized wide NN trained by any GRW algorithm satisfying \Cref{ass:qstar} with the squared loss. Suppose there exists $M_0 > 0$ such that $\norm{\nabla_\theta f^{(0)}(x)}_2 \le M_0$ for all $\norm{x}_2 \le 1$. 
Suppose $\lambda^{\min} > 0$, $\mu > 0$, $d_1=\cdots=d_L = \dl$, $\nabla_{\theta} f^{(0)}(x_1) ,\cdots,\nabla_{\theta} f^{(0)}(x_n) $ are linearly independent, and the learning rate is sufficiently small.
If the empirical training risk of $f_{\reg}^{(t)}$ satisfies $\limsup_{t \rightarrow \infty} \hat{\gR}(f_{\reg}^{(t)}) < \epsilon$ for some $\epsilon > 0$, then as $\dl \rightarrow \infty$, with probability close to $1$ over random initialization there is
\begin{equation*}
    \limsup_{t \rightarrow \infty} \left | f_{\reg}^{(t)}(x) - f_{ \erm}^{(t)}(x) \right | = O(\dl^{-1/4} + \sqrt{\epsilon}) \rightarrow O(\sqrt{\epsilon}) \quad \text{for all } x \in \R^d \text{ such that } \norm{x}_2 \le 1 .
\end{equation*}
\end{theorem}

\paragraph{Linear models, classification tasks.}
So far we have been discussing regression tasks.
We now move on to binary classification, where the label space is $\gY = \{ +1,-1 \}$, and the loss is the logistic loss $\ell(\hat{y}, y) = \log(1 + \exp(- \hat{y} y))$. 
The big difference here is that \textit{the logistic loss does not have finite minimizers}. The logistic loss converging to zero means that the model weight ``explodes'' to infinity instead of converging to a finite point. 

Again consider the linear model $f(x) = \dotp{\theta, x}$.
Prior work has shown a couple of negative results under this setting.
For example, \cite{byrd2019effect} empirically observed that importance weighting does not improve over ERM for linear models, and \cite{xu2021understanding} proved that for importance weighting, as $t \rightarrow \infty$, $\| \theta^{(t)} \|_2 \rightarrow \infty$ and $\theta^{(t)} / \| \theta^{(t)} \|_2$ converges to a unit vector that does not depend on the sample weights, so it does not improve over ERM.

We extend these results to GRW. First, we show that when the training error goes to zero, $\theta^{(t)}$ will converge to the \textit{max-margin classifier} defined as
\begin{equation*}
    \hat{\theta}_{\mm} = \argmax_{\theta: \norm{\theta}_2 = 1} \oset{\min_{i=1,\cdots,n} y_i \cdot \dotp{\theta,x_i} } .
\end{equation*}

\begin{theorem}[Proof in \Cref{app:proof-thm-grw-cls-linear}]
\label{thm:grw-cls-linear}
    Suppose $x_1,\cdots,x_n$ are linearly independent. Suppose we use GRW such that for all $i \in [n]$, $\liminf_{t \rightarrow \infty} q_i^{(t)} > 0$.
    As $t \rightarrow \infty$, if the empirical risk $\hat{\gR}(f^{(t)})$ converges to zero and $\theta^{(t)} / \norm{\theta^{(t)}}_2 \rightarrow \vu$ for some unit vector $\vu$, then $\vu = \hat{\theta}_{\mm}$.
\end{theorem}

This result is an extension of \cite{soudry2018implicit}. It implies that all GRW methods including ERM, if converge, must converge to the same point $\hat{\theta}_{\mm}$ that does not depend on the sample weights $q_i^{(t)}$.
Next, we show that any GRW satisfying \Cref{ass:qstar} does converge.

\begin{definition}
\label{def:smooth}
A first-order differentiable function $f$ on $\gD$ is called \textbf{\boldmath $L$-smooth} for $L > 0$ if
\begin{equation*}
    f(y) \le f(x) + \langle \nabla f(x), y-x \rangle + \frac{L}{2} \left \| y-x \right \|_2^2 \qquad \text{for all } x,y \in \gD .
\end{equation*}
An equivalent definition is that $f$ is $L$-smooth if it satisfies
\begin{equation*}
    \norm{\nabla f(x) - \nabla f(y)}_2 \le L \norm{x-y}_2 \qquad \text{for all } x, y \in \gD .
\end{equation*}
\end{definition}

\begin{theorem}[Proof in \Cref{app:proof-thm-grw-cls-converge}]
\label{thm:grw-cls-converge}
Suppose the loss function $\ell(\hat{y}, y)$ is convex, $L$-smooth in $\hat{y}$ and strictly monotonically decreasing to zero as $y\hat{y} \rightarrow +\infty$.
Consider the linear model $f(x) = \dotp{\theta, x}$.
Suppose $x_1,\cdots,x_n$ are linearly independent.
For any GRW that satisfies \Cref{ass:qstar} with $q_i$, denote $F(\theta) = \sum_{i=1}^n q_i \ell(\langle \theta,x_i \rangle, y_i)$.
Then, for a sufficiently small learning rate $\eta$:
\begin{enumerate}[label=(\roman*)]
    \item $F(\theta^{(t)}) \rightarrow 0$ as $t \rightarrow \infty$.
    \item $\left\| \theta^{(t)} \right\|_2 \rightarrow \infty$ as $t \rightarrow \infty$.
    \item Let $\theta_R = \argmin_{\theta} \{ F(\theta): \left\| \theta \right\|_2 \le R \}$. $\theta_R$ is unique for any $R$ such that $\min_{\| \theta \|_2 \le R} F(\theta) < \min_i q_i \ell(0, y_i)$. And if $\lim_{R \rightarrow \infty} \frac{\theta_R}{R}$ exists, then $\lim_{t \rightarrow \infty} \frac{\theta^{(t)}}{\left\| \theta^{(t)} \right\|_2}$ also exists and they are equal.
\end{enumerate}
\end{theorem}

This result is an extension of Theorem 1 in \cite{pmlr-v125-ji20a}.
We show that the logistic loss satisfies the conditions of the above theorem and $\lim_{R \rightarrow \infty} \frac{\theta_R}{R} = \hat{\theta}_{\mm}$ in \Cref{app:proof-logistic-loss}.
Thus, \Cref{thm:grw-cls-linear} and \Cref{thm:grw-cls-converge} imply that for a linear model, all GRW methods satisfying \Cref{ass:qstar} (including ERM) will converge to the same point.

\paragraph{Wide neural networks, classification tasks.}

We now study wide NNs with regularization.
But before that, we have to point out that it is impossible to extend \Cref{thm:grw-cls-converge} to a wide NN without regularization.
This is because for a neural network, if $\norm{\theta^{(t)}}_2$ goes to infinity, then $\norm{\nabla_\theta f}_2$ will also go to infinity, unlike a linear model where this gradient is a constant. 
Consequently, the gap between the neural networks and its linearized counterpart will ``explode'' under gradient descent, so we cannot prove that the wide NN can always be approximated by its linearized version similar to the previous section.
However, with regularization, an approximation theorem can be proved.

Consider minimizing the regularized risk \Cref{eqn:grw-l2-pen} with a wide NN, and $\ell$ is the logistic loss. 
Define the max-margin linearized NN as
\begin{equation*}
    f_{\mm}(x) = \dotp{ \hat{\theta}_{\mm}, \nabla_\theta f^{(0)}(x) } \quad \text{where} \quad \hat{\theta}_{\mm} = \argmax_{\| \theta \|_2 = 1} \oset{ \min_{i=1,\cdots,n} y_i \cdot \dotp {\theta, \nabla_\theta f^{(0)}(x_i) } }
\end{equation*}

Note that $f_{\mm}$ does not depend on $q_i^{(t)}$.
We can again show that regularization only works when it is large enough to sufficiently downgrade the training performance.
\begin{theorem}[Proof in \Cref{app:proof-thm-grw-cls-nn}]
\label{thm:grw-cls-nn}
Suppose $\norm{\nabla_\theta f^{(0)}(x)}_2$ is bounded.
Under the same conditions of \Cref{thm:grw-reg-nn-regu}, when $\ell$ is the logistic loss, for any $\delta > 0$ there exists a constant $C(\delta) > 0$ such that with probability at least $1-\delta$, the following holds as $\dl \rightarrow \infty$:
\begin{center}
    For any $\epsilon \in (0,\frac{1}{4})$, if the training error has $\limsup_{t \rightarrow \infty} \hat{\gR}_{\vq^{(t)}}^{\mu}(f_{\reg}^{(t)}) < \epsilon$, then for any $x$ such that $|f_{\mm}(x) | > C(\delta) \cdot (-\log 2 \epsilon)^{-1/2}$, $f_{\reg}^{(t)}(x)$ has the same sign as $f_{\mm}(x)$ for a sufficiently large $t$.
\end{center}
\end{theorem}

Apart from adding a large regularization, is there any other way to prevent GRW from obtaining almost the same model as ERM in classification?
The main reason why GRW always converges to the max-margin classifier regardless of the weights $q_i^{(t)}$ is that the logistic loss is exponentially tailed.
Thus, one way to ensure that the sample weights have an impact is to use a polynomially tailed loss.
For example, \cite{wang2022is} defined the following polynomially tailed loss for linear classifiers:
\begin{equation*}
    \ell_{\alpha,\beta}(\hat{y}, y) = 
    \left \{ 
\begin{aligned}
    \ell_{\textrm{left}}(\hat{y}, y), \quad & \text{if } \hat{y} y < \beta ; \\
    \frac{1}{\brac{\hat{y} y - (\beta - 1)}^{\alpha}} , \quad & \text{if } \hat{y} y \ge \beta ,
\end{aligned}
    \right .
\end{equation*}
where $\ell_{\textrm{left}}$ is any function such that the overall loss function $\ell_{\alpha,\beta}$ is convex, differentiable and strictly decreasing.
Here we empirically compare between the logistic loss and this polynomially tailed loss on the six MNIST images we used earlier.
The results are plotted in \Cref{fig:grw-poly}, and we can observe the following:
\begin{itemize}
    \item For either loss function, the training loss of each method converges to 0.
    \item In theory the norm of the ERM model will explode to infinity, but in reality it will not because once the training loss becomes extremely small, it will turn into zero in the floating number representation, and thus the training halts.
    \item We can see a fundamental difference between the logistic loss and the polynomially tailed loss. For the logistic loss, the norm of the gap between importance weighting (or Group DRO) and ERM will converge to around 0.06 when the training stops, while for the polynomially tailed loss, the norm will be larger than 0.22 and will keep growing, which shows that for the polynomially tailed loss the normalized model weights do not converge to the same point.
\end{itemize}

\begin{figure}[t]
    \centering
    \begin{tikzpicture}

\begin{axis}[
name=plot0,table/col sep=comma,no markers,
    xmin = -30, xmax = 1000,
    ymin = 0.04, ymax = 0.12,
    xtick distance={400},
    xlabel={$\times 10000$ Epochs},
    x label style={at={(axis description cs:0.5,-0.25)},anchor=north,inner sep=0},
    ylabel={Norm of Difference},
    y label style={at={(axis description cs:-0.35,-0.3)},anchor=west,inner sep=0},
    yticklabel style={
        /pgf/number format/fixed,
        /pgf/number format/precision=5
    },
    scaled y ticks=false,
    width = .31\linewidth,
    height = .25\linewidth,
    legend style = {nodes={scale=0.7, transform shape},at={(0.2,0.72)},anchor=west,inner sep=.5pt, inner ysep=5pt},
    legend cell align=left,
    ]
    \addplot [opacity=0.4, black, ultra thick, solid] table [x expr=\coordindex, y={diw}] {data/linear_logistic.csv};
    \addplot [black, ultra thick, dashed] table [x expr=\coordindex, y={dgdro}] {data/linear_logistic.csv};
    \addlegendentry{$\norm{\tilde{\theta}_{\iw} - \tilde{\theta}_{\erm}}_2$}
    \addlegendentry{$\norm{\tilde{\theta}_{\gdro} - \tilde{\theta}_{\erm}}_2$}
\end{axis}

\begin{axis}[
name=plot1,table/col sep=comma,no markers,
at={($(plot0.south west)+(0,-130)$)},
    xmin = -30, xmax = 1000,
    xtick distance={400},
    ymin = 0.205, ymax = 0.29,
    xlabel={$\times 10000$ Epochs},
    x label style={at={(axis description cs:0.5,-0.25)},anchor=north,inner sep=0},
    ylabel={Norm of Difference},
    y label style={at={(axis description cs:-0.35,-0.3)},anchor=west,inner sep=0},
    yticklabel style={
        /pgf/number format/fixed,
        /pgf/number format/precision=5
    },
    scaled y ticks=false,
    width = .31\linewidth,
    height = .25\linewidth,
    legend style = {nodes={scale=0.6, transform shape},at={(0.1,0.43)},anchor=west,inner sep=.5pt,inner ysep=4pt},
    legend cell align=left,
    ]
    \addplot [opacity=0.4, black, ultra thick, solid] table [x expr=\coordindex, y={diw}] {data/linear_poly.csv};
    \addplot [black, ultra thick, dashed] table [x expr=\coordindex, y={dgdro}] {data/linear_poly.csv};
    \addlegendentry{$\norm{\tilde{\theta}_{\iw} - \tilde{\theta}_{\erm}}_2$}
    \addlegendentry{$\norm{\tilde{\theta}_{\gdro} - \tilde{\theta}_{\erm}}_2$}
\end{axis}

\begin{axis}[
name=plot2,table/col sep=comma,no markers,
at={($(plot0.south east)+(45,0)$)},
    xmin = -30, xmax = 1000,
    xtick distance={400},
    ymin = 1.1, ymax = 2.5,
    xlabel={$\times 10000$ Epochs},
    x label style={at={(axis description cs:0.5,-0.25)},anchor=north,inner sep=0},
    ylabel={Norm},
    y label style={at={(axis description cs:-0.28,0.24)},anchor=west,inner sep=0},
    yticklabel style={
        /pgf/number format/fixed,
        /pgf/number format/precision=5
    },
    scaled y ticks=false,
    width = .31\linewidth,
    height = .25\linewidth,
    legend style = {nodes={scale=0.7, transform shape},at={(0.15,0.8)},anchor=west,inner sep=.5pt},
    legend cell align=left,
    ]
    \addplot [black, ultra thick, solid] table [x expr=\coordindex, y={nerm}] {data/linear_logistic.csv};
\end{axis}

\begin{axis}[
name=plot3,table/col sep=comma,no markers,
at={($(plot2.south west)+(0,-130)$)},
    xmin = -30, xmax = 1000,
    xtick distance={400},
    ymin = 2, ymax = 18,
    xlabel={$\times 10000$ Epochs},
    x label style={at={(axis description cs:0.5,-0.25)},anchor=north,inner sep=0},
    ylabel={Norm},
    y label style={at={(axis description cs:-0.28,0.24)},anchor=west,inner sep=0},
    yticklabel style={
        /pgf/number format/fixed,
        /pgf/number format/precision=5
    },
    scaled y ticks=false,
    width = .31\linewidth,
    height = .25\linewidth,
    legend style = {nodes={scale=0.8, transform shape},at={(0.4,0.7)},anchor=west,inner sep=.5pt},
    legend cell align=left,
    ]
    \addplot [black, ultra thick, solid] table [x expr=\coordindex, y={nerm}] {data/linear_poly.csv};
\end{axis}

\begin{axis}[
name=plot4,table/col sep=comma,no markers,
at={($(plot2.south east)+(60,0)$)},
    xmin = -30, xmax = 1000,
    xtick distance={400},
    ymin = -0.001, ymax = 0.006,
    xlabel={$\times 10000$ Epochs},
    x label style={at={(axis description cs:0.5,-0.25)},anchor=north,inner sep=0},
    ylabel={Training Loss},
    y label style={at={(axis description cs:-0.4,-0.1)},anchor=west,inner sep=0},
    yticklabel style={
        /pgf/number format/fixed,
        /pgf/number format/precision=5
    },
    scaled y ticks=false,
    width = .31\linewidth,
    height = .25\linewidth,
    legend style = {nodes={scale=0.9, transform shape},at={(0.4,0.7)},anchor=west,inner sep=.5pt},
    legend cell align=left,
    ]
    \addplot [black!30, black, ultra thick, solid] table [x expr=\coordindex, y={lerm}] {data/linear_logistic.csv};
    \addplot [black!60, ultra thick, dashed] table [x expr=\coordindex, y={liw}] {data/linear_logistic.csv};
    \addplot [black, ultra thick, dotted] table [x expr=\coordindex, y={lgdro}] {data/linear_logistic.csv};
    \addlegendentry{ERM}
    \addlegendentry{IW}
    \addlegendentry{GDRO}
\end{axis}

\begin{axis}[
name=plot5,table/col sep=comma,no markers,
at={($(plot4.south west)+(0,-130)$)},
    xmin = -30, xmax = 1000,
    xtick distance={400},
    ymin = 0, ymax = 0.12,
    xlabel={$\times 10000$ Epochs},
    x label style={at={(axis description cs:0.5,-0.25)},anchor=north,inner sep=0},
    ylabel={Training Loss},
    y label style={at={(axis description cs:-0.4,-0.1)},anchor=west,inner sep=0},
    yticklabel style={
        /pgf/number format/fixed,
        /pgf/number format/precision=5
    },
    scaled y ticks=false,
    width = .31\linewidth,
    height = .25\linewidth,
    legend style = {nodes={scale=0.9, transform shape},at={(0.4,0.7)},anchor=west,inner sep=.5pt},
    legend cell align=left,
    ]
    \addplot [black!30, black, ultra thick, solid] table [x expr=\coordindex, y={lerm}] {data/linear_poly.csv};
    \addplot [black!60, ultra thick, dashed] table [x expr=\coordindex, y={liw}] {data/linear_poly.csv};
    \addplot [black, ultra thick, dotted] table [x expr=\coordindex, y={lgdro}] {data/linear_poly.csv};
    \addlegendentry{ERM}
    \addlegendentry{IW}
    \addlegendentry{GDRO}
\end{axis}

\node[scale=0.9] at ($(plot0.south)+(0,-1.5)$) {(a) Weight difference};
\node[scale=0.9] at ($(plot1.south)+(0,-1.5)$) {(d) Weight difference};
\node[scale=0.9] at ($(plot2.south)+(0,-1.5)$) {(b) Norm of ERM model};
\node[scale=0.9] at ($(plot3.south)+(0,-1.5)$) {(e) Norm of ERM model};
\node[scale=0.9] at ($(plot4.south)+(0,-1.5)$) {(c) Training loss};
\node[scale=0.9] at ($(plot5.south)+(0,-1.5)$) {(f) Training loss};

\end{tikzpicture}
    \caption{Experiment results of ERM, importance weighting (IW) and group DRO (GDRO) with the logistic loss \textbf{(top row)} and the polynomially tailed loss \textbf{(bottom row)} on a linear model. All norms are $L^2$ norms. $\tilde{\theta} = \theta / \norm{\theta}_2$.}
    \label{fig:grw-poly}
\end{figure}
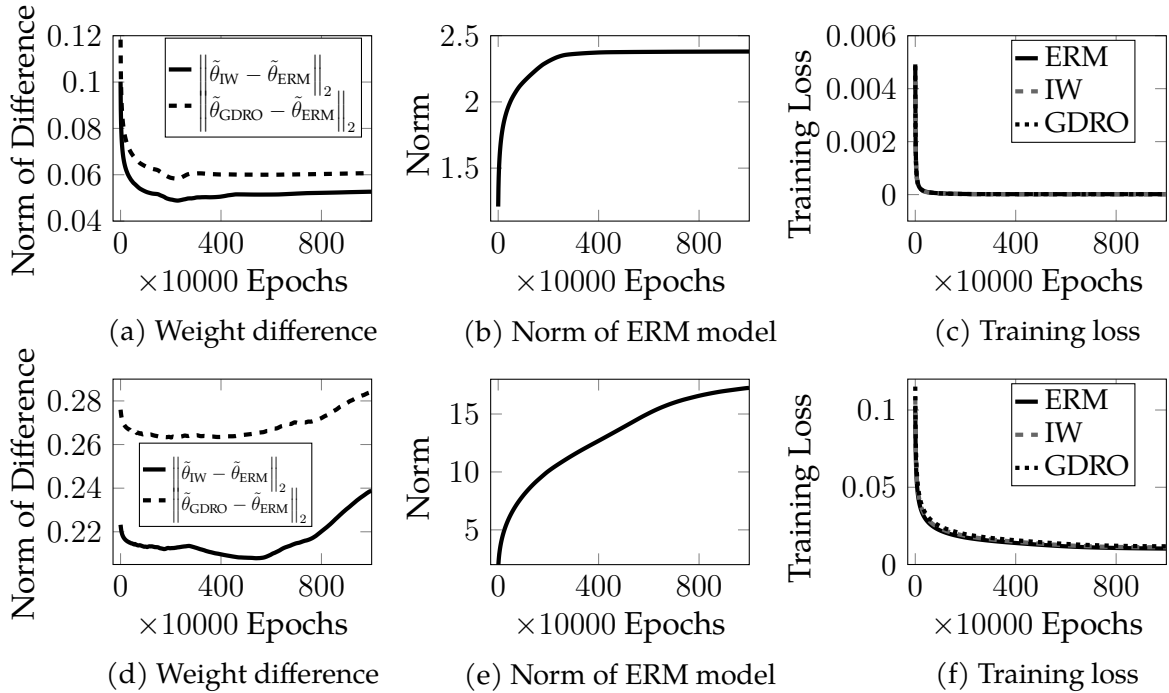

\section{Sensitivity to Outliers}

Another issue with DRO is its sensitivity to outliers, which are samples significantly different from most of the sample in the dataset.
Let us use CVaR as an example to see why DRO is particularly sensitive to outliers.
CVaR places all the weights on the worst $\alpha$ fraction of the training samples.
``Worst'' here means that the model gets the highest loss on these samples.
However, almost all real datasets contain outliers, and by their very definition, the model tends to have high loss on the outliers.
This means that CVaR is very likely to place large weights on many outliers, which will make training very unstable, and the final performance pretty bad.

\begin{figure}[t]
    \centering
    \begin{tikzpicture}
    \begin{axis}[name=plot0, width=0.31\textwidth,height=0.26\textwidth, 
    title={(a) Average accuracy}, title style={at={(0.5,-0.6)},anchor=north,yshift=-0.1}, xlabel={Epochs}, ylabel={Test accuracy},
        legend style = {nodes={scale=0.7, transform shape},at={(0.03,0.25)},anchor=west,inner sep=.5pt},
    legend cell align=left]
        \addplot [black!60, ultra thick, dashed] table [col sep=comma, x=epochs, y=erm]{compas-original-avg.csv};
        \addplot [black!30, ultra thick, solid] table [col sep=comma, x=epochs, y=cvar]{compas-original-avg.csv};
        \addplot [black, ultra thick, solid, dotted] table [col sep=comma, x=epochs, y=chisq]{compas-original-avg.csv};
        \addlegendentry{ERM};
        \addlegendentry{CVaR};
        \addlegendentry{$\chi^2$-DRO};
    \end{axis}
    \begin{axis}[name=plot1, width=0.31\textwidth,height=0.26\textwidth, at={($(plot0.south east)+(55,0)$)},
    title={(b) Worst-group accuracy}, title style={at={(0.5,-0.6)},anchor=north,yshift=-0.1}, xlabel={Epochs}, ylabel={Test accuracy},
        legend style = {nodes={scale=0.7, transform shape},at={(0.03,0.25)},anchor=west,inner sep=.5pt},
    legend cell align=left]
        \addplot [black!60, ultra thick, dashed] table [col sep=comma, x=epochs, y=erm]{compas-original-worst.csv};
        \addplot [black!30, ultra thick, solid] table [col sep=comma, x=epochs, y=cvar]{compas-original-worst.csv};
        \addplot [black, ultra thick, solid, dotted] table [col sep=comma, x=epochs, y=chisq]{compas-original-worst.csv};
        \addlegendentry{ERM};
        \addlegendentry{CVaR};
        \addlegendentry{$\chi^2$-DRO};
    \end{axis}
    \begin{axis}[name=plot2, width=0.31\textwidth,height=0.26\textwidth, at={($(plot1.south east)+(55,0)$)},
    title={(c) Training loss}, title style={at={(0.5,-0.6)},anchor=north,yshift=-0.1}, xlabel={Epochs}, ylabel={Training loss},
        legend style = {nodes={scale=0.7, transform shape},at={(0.52,0.75)},anchor=west,inner sep=.5pt},
    legend cell align=left]
        \addplot [black!60, ultra thick, dashed] table [col sep=comma, x=epochs, y=erm]{compas-original-train-loss.csv};
        \addplot [black!30, ultra thick, solid] table [col sep=comma, x=epochs, y=cvar]{compas-original-train-loss.csv};
        \addplot [black, ultra thick, solid, dotted] table [col sep=comma, x=epochs, y=chisq]{compas-original-train-loss.csv};
        \addlegendentry{ERM};
        \addlegendentry{CVaR};
        \addlegendentry{$\chi^2$-DRO};
    \end{axis}
\end{tikzpicture}
    \caption{Results of ERM and two DRO methods on the original COMPAS data set.}
    \label{fig:compas-demo-original}
\end{figure}
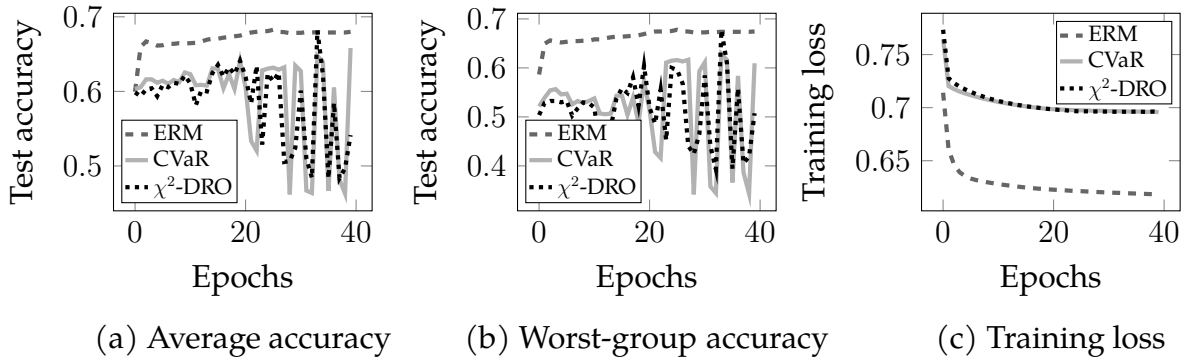

Let us use an experiment on a real dataset to show that this is a real problem in practice.
We use the COMPAS dataset \cite{larson2016we}, a recidivism prediction data set with 5049 training instances (after pre-processing and train-test splitting).
We construct four groups on this data set with two sensitive features: race and sex. These two features define four overlapping groups (demographic groups): \textit{White}, \textit{Non-white}, \textit{Male} and \textit{Female}.
We use a two-layer feed-forward neural network with ReLU activation, and train it with three methods: ERM, CVaR, and $\chi^2$-DRO, which is a DRO method to be introduced later.
The results are plotted in \Cref{fig:compas-demo-original}.
We report the average test accuracy, the the minimum test accuracy on any group (the worst-group accuracy).
From figures (a) and (b), we can see that for both average and worst-group test accuracies, the two DRO methods are worse than ERM.
Moreover, the two DRO curves are jumping up and down, showcasing the huge volatility of DRO training, in stark contrast to the stable curve of ERM.
In addition, the training loss is plotted in figure (c), and we can see that the loss curve is pretty stable, meaning that this high volatility is not caused by optimization.

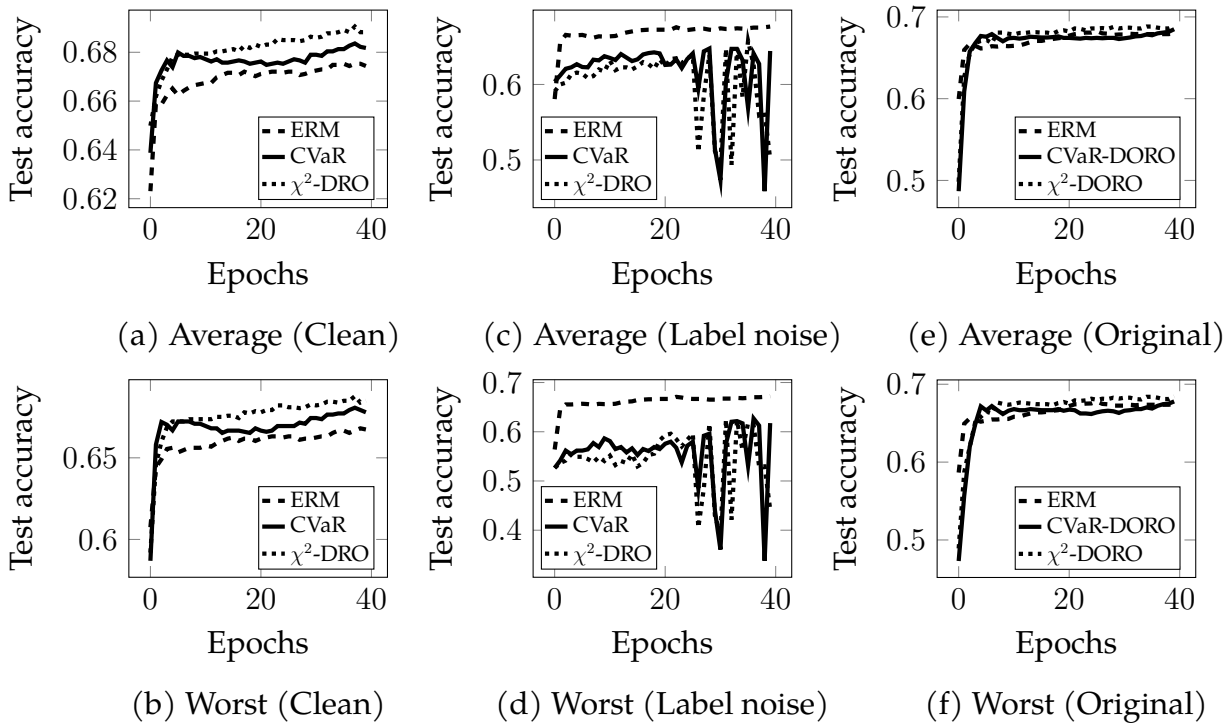
\begin{figure}[t]
    \centering
    \begin{tikzpicture}
    \begin{axis}[name=plot0, width=0.31\textwidth,height=0.26\textwidth, 
    title={(a) Average (Clean)}, title style={at={(0.5,-0.6)},anchor=north,yshift=-0.1}, xlabel={Epochs}, ylabel={Test accuracy},
        legend style = {nodes={scale=0.7, transform shape},at={(0.5,0.25)},anchor=west,inner sep=.5pt},
    legend cell align=left]
        \addplot [black, ultra thick, dashed] table [col sep=comma, x=epochs, y=erm]{compas-clean-avg.csv};
        \addplot [opacity=0.4, black, ultra thick, solid] table [col sep=comma, x=epochs, y=cvar]{compas-clean-avg.csv};
        \addplot [black, ultra thick, solid, dotted] table [col sep=comma, x=epochs, y=chisq]{compas-clean-avg.csv};
        \addlegendentry{ERM};
        \addlegendentry{CVaR};
        \addlegendentry{$\chi^2$-DRO};
    \end{axis}
    \begin{axis}[name=plot1, width=0.31\textwidth,height=0.26\textwidth, at={($(plot0.south west)+(0,-140)$)},
    title={(b) Worst (Clean)}, title style={at={(0.5,-0.6)},anchor=north,yshift=-0.1}, xlabel={Epochs}, ylabel={Test accuracy},
        legend style = {nodes={scale=0.7, transform shape},at={(0.5,0.25)},anchor=west,inner sep=.5pt},
    legend cell align=left]
        \addplot [black, ultra thick, dashed] table [col sep=comma, x=epochs, y=erm]{compas-clean-worst.csv};
        \addplot [opacity=0.4, black, ultra thick, solid] table [col sep=comma, x=epochs, y=cvar]{compas-clean-worst.csv};
        \addplot [black, ultra thick, solid, dotted] table [col sep=comma, x=epochs, y=chisq]{compas-clean-worst.csv};
        \addlegendentry{ERM};
        \addlegendentry{CVaR};
        \addlegendentry{$\chi^2$-DRO};
    \end{axis}
    \begin{axis}[name=plot2, width=0.31\textwidth,height=0.26\textwidth, at={($(plot0.south east)+(55,0)$)},
    title={(c) Average (Label noise)}, title style={at={(0.5,-0.6)},anchor=north,yshift=-0.1}, xlabel={Epochs}, ylabel={Test accuracy},
        legend style = {nodes={scale=0.7, transform shape},at={(0.03,0.25)},anchor=west,inner sep=.5pt},
    legend cell align=left]
        \addplot [black, ultra thick, dashed] table [col sep=comma, x=epochs, y=erm]{compas-noisy-avg.csv};
        \addplot [opacity=0.4, black, ultra thick, solid] table [col sep=comma, x=epochs, y=cvar]{compas-noisy-avg.csv};
        \addplot [black, ultra thick, solid, dotted] table [col sep=comma, x=epochs, y=chisq]{compas-noisy-avg.csv};
        \addlegendentry{ERM};
        \addlegendentry{CVaR};
        \addlegendentry{$\chi^2$-DRO};
    \end{axis}
    \begin{axis}[name=plot3, width=0.31\textwidth,height=0.26\textwidth, at={($(plot2.south west)+(0,-140)$)},
    title={(d) Worst (Label noise)}, title style={at={(0.5,-0.6)},anchor=north,yshift=-0.1}, xlabel={Epochs}, ylabel={Test accuracy},
        legend style = {nodes={scale=0.7, transform shape},at={(0.03,0.25)},anchor=west,inner sep=.5pt},
    legend cell align=left]
        \addplot [black, ultra thick, dashed] table [col sep=comma, x=epochs, y=erm]{compas-noisy-worst.csv};
        \addplot [opacity=0.4, black, ultra thick, solid] table [col sep=comma, x=epochs, y=cvar]{compas-noisy-worst.csv};
        \addplot [black, ultra thick, solid, dotted] table [col sep=comma, x=epochs, y=chisq]{compas-noisy-worst.csv};
        \addlegendentry{ERM};
        \addlegendentry{CVaR};
        \addlegendentry{$\chi^2$-DRO};
    \end{axis}
    \begin{axis}[name=plot4, width=0.31\textwidth,height=0.26\textwidth, at={($(plot2.south east)+(55,0)$)},
    title={(e) Average (Original)}, title style={at={(0.5,-0.6)},anchor=north,yshift=-0.1}, xlabel={Epochs}, ylabel={Test accuracy},
        legend style = {nodes={scale=0.7, transform shape},at={(0.3,0.25)},anchor=west,inner sep=.5pt},
    legend cell align=left]
        \addplot [black, ultra thick, dashed] table [col sep=comma, x=epochs, y=erm]{compas-doro-avg.csv};
        \addplot [opacity=0.4, black, ultra thick, solid] table [col sep=comma, x=epochs, y=cvar]{compas-doro-avg.csv};
        \addplot [black, ultra thick, solid, dotted] table [col sep=comma, x=epochs, y=chisq]{compas-doro-avg.csv};
        \addlegendentry{ERM};
        \addlegendentry{CVaR-DORO};
        \addlegendentry{$\chi^2$-DORO};
    \end{axis}
    \begin{axis}[name=plot5, width=0.31\textwidth,height=0.26\textwidth, at={($(plot4.south west)+(0,-140)$)},
    title={(f) Worst (Original)}, title style={at={(0.5,-0.6)},anchor=north,yshift=-0.1}, xlabel={Epochs}, ylabel={Test accuracy},
        legend style = {nodes={scale=0.7, transform shape},at={(0.3,0.25)},anchor=west,inner sep=.5pt},
    legend cell align=left]
        \addplot [black, ultra thick, dashed] table [col sep=comma, x=epochs, y=erm]{compas-doro-worst.csv};
        \addplot [opacity=0.4, black, ultra thick, solid] table [col sep=comma, x=epochs, y=cvar]{compas-doro-worst.csv};
        \addplot [black, ultra thick, solid, dotted] table [col sep=comma, x=epochs, y=chisq]{compas-doro-worst.csv};
        \addlegendentry{ERM};
        \addlegendentry{CVaR-DORO};
        \addlegendentry{$\chi^2$-DORO};
    \end{axis}
\end{tikzpicture}
    \caption{Average and worst-group test accuracies on: (a), (b) the ``clean'' dataset where potential outliers are removed; (c), (d) the noisy set where label noise is added to the ``clean'' set; (e), (f) the original COMPAS dataset, but with DORO algorithms.}
    \label{fig:compas-doro-more}
\end{figure}

Next, we examine if this instability and poor performance is caused by the outliers in the dataset.
For this purpose, we ``clean'' the dataset by removing from it 1000 ``potential outliers'', which are the samples on which the ERM model has a high loss.
Note that these are not 100\% outliers, but removing them has a huge impact on the performance of the two DRO methods.
We ``clean'' the dataset in five rounds. In each round, we train a model from scratch on the samples with ERM, and then remove 200 training samples that incur the highest loss on this model.
After five rounds, 1000 instances are removed in total, and we get a ``clean'' data set with 4049 samples.
\Cref{fig:compas-doro-more} (a), (b) show the results on the ``clean'' dataset, from which we can see that the two DRO methods become very stable and better than ERM.
It should be emphasized that up to this point, we have only removed samples from the dataset and added nothing into it.
The outliers naturally exist in the original data set.

One might argue that these high-loss samples are not necessarily outliers.
Thus, to further substantiate our claim, we conduct a third experiment where we add some outliers to the ``clean'' dataset.
We use a common source of noise: incorrect labels. 
Specifically, we randomly flip 20\% of the labels of the ``clean'' COMPAS dataset, and run the three methods again.
The results are plotted in \Cref{fig:compas-doro-more} (c), (d).
We can see that while the label noise only slightly influences ERM, it significantly downgrades the performance of the two DRO methods, and makes their training highly unstable again.

The experiment results so far should have convinced the reader that outliers have a much greater impact on DRO than ERM, and they can make DRO poor and highly unstable.
In the following section, we will introduce a solution called DORO.
\Cref{fig:compas-doro-more} (e), (f) plot the performance of DORO on the original COMPAS dataset (that is, before cleaning).
We can see that DORO is much better and more stable than DRO.

\section{Distributionally and Outlier Robust Optimization}

The idea of distributionally and outlier robust optimization, or DORO, can be summarized by \Cref{fig:doro}.
We know that CVaR places large weights on the worst samples, which may contain lots of outliers.
Instead, CVaR-DORO ignores the worst of the worst samples, and places weights on the second-worst fraction of the samples, in order to avoid potential outliers.
We can extend this idea to a family of DRO method called the Cressie-Read family.

\begin{figure}[t]
    \centering
    \begin{tikzpicture}

\begin{axis}[
    samples=100,
    domain=0.000001:0.999999, xmax=1, xmin=0,
    ymax=1, ymin=0,
    width = .5\linewidth,
    height = .25\linewidth,
    xlabel=Loss,
    ylabel={Cumulative density},
    ytick={0,0.5,1}
]
\addplot [ultra thick] {1 - 1/(1 - 0.2 * ln(x))};
\draw [thick] (axis cs:0.6,0) -- (axis cs:0.6,0.4);
\draw [thick, dashed] (axis cs:0.46,0) -- (axis cs:0.46,0.74);
\draw [thick, dashed] (axis cs:0.8,0) -- (axis cs:0.8,0.74);
\node[] at (0.8,0.3) {DRO};
\node[] at (0.63,0.61) {DORO};
\draw [thick, ->] (axis cs:0.7,0.3) -- (axis cs:0.6,0.3);
\draw [thick, ->] (axis cs:0.9,0.3) -- (axis cs:1.0,0.3);
\draw [thick, ->] (axis cs:0.52,0.61) -- (axis cs:0.46,0.61);
\draw [thick, ->] (axis cs:0.74,0.61) -- (axis cs:0.8,0.61);
\draw [thick, dashed] (axis cs:0,0.092694) -- (axis cs:0.6,0.092694);
\node[] at (0.05,0.17) {$0.1$};
\end{axis}

\end{tikzpicture}
    \caption{Comparison between DRO and DORO for CVaR.}
    \label{fig:doro}
\end{figure}
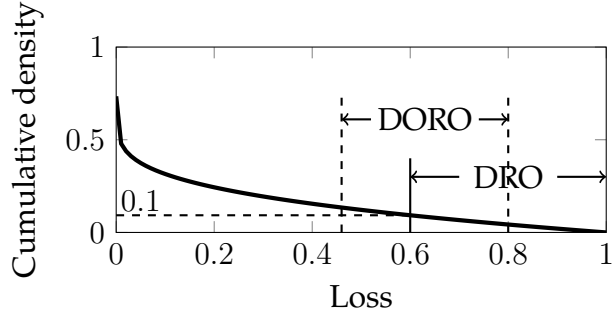

Recall that a DRO method is defined by the constraint $D(Q \parallel P) \le \rho$, where the divergence function $D(Q \parallel P)$ measures the difference between $Q$ and $P$.
There are two large families of divergence functions, namely integral probability metrics (IPMs) and $f$-divergences.
An IPM is defined as $D(Q \parallel P) = \sup_{f \in \gF} \brac{\E_{X \sim Q} f(X) - \E_{Y \sim P} f(Y)}$ for some function class $\gF$.
It is symmetric, so it is also written as $D(Q,P)$.
Examples include the total variation (TV) distance defined as $D(Q,P) = \frac{1}{2} \int |Q(x) - P(x)| dx$,
maximum mean discrepancy (MMD) defined as $D(Q,P) = \norm{\E_{X \sim Q} \pi(X) - \E_{Y \sim P} \pi(Y) }_{\gH}$ for some feature map $\pi$ and some RKHS $\gH$, and the Wasserstein distance defined as $D(Q,P) = \inf_{\gamma \in \Gamma(Q,P)} \int \rho(x,y) d \gamma(x,y)$ for some metric function $\rho(\cdot, \cdot)$.

An $f$-divergence is defined as $D_f(Q \parallel P) = \int f \paren{\frac{dQ}{dP}}dP$ for some function $f$. An $f$-divergence is not necessarily symmetric.
For example, when $f(t) = - \log t$, then the $f$-divergence becomes the popular KL-divergence $D_f(Q \parallel P) = D_{KL}(P \parallel Q) = \int \log \paren{\frac{P(x)}{Q(x)}} dP(x)$ (note that $P$ and $Q$ are reversed).
The TV distance is the only non-trivial divergence function that is both an IPM and an $f$-divergence.

The Cressie-Read family of R\'enyi divergence \cite{cressie1984multinomial} is a family of $f$-divergences.
For any $\beta > 1$, the divergence is defined as
\begin{equation*}
    D_\beta(Q \parallel P) = \int f_\beta \paren{\frac{dQ}{dP}} dP, \qquad \text{where } \; f_\beta(t) = \frac{1}{\beta (\beta-1)} \paren{t^\beta - \beta t + \beta - 1}  .
\end{equation*}

The reason why we are interested in this family is that the DRO risk \Cref{eqn:dro-risk} \wrt{} $D_\beta$ has the following dual characterization (see \cite[Lemma~1]{duchi2018learning}):
\begin{equation}
\label{eqn:dual-cressie-read}
    \gR_{D_\beta, \rho}(f; P) = \inf_{\eta \in \R} \oset{ c_\beta(\rho) \E_{(X,Y) \sim P} \brac{ \paren{ \ell \paren{f(X), Y} - \eta }_+^{\beta_*} }^{\frac{1}{\beta_*}} + \eta }  ,
\end{equation}
where $(x)_+ = \max \oset{x, 0}$, $\beta_* = \frac{\beta}{\beta - 1}$, and $c_\beta(\rho) = \paren{1 + \beta(\beta-1) \rho}^{1/\beta}$.

We first show the relationship between the DRO loss and the worst-group risk.
Suppose there are $K$ groups $\gD_1,\cdots,\gD_K$.
Let $P_k(x,y) = P((x,y) | (x,y) \in \gD_k)$.
Then, the worst-group risk is given by
\begin{equation*}
    \gR_{\max}(f;P) = \max_{k=1,\cdots,K} \gR(f;P_k) = \max_{k=1,\cdots,K} \E_{(X,Y) \sim P} [\ell(f(X), Y) | (X,Y) \in \gD_k] .
\end{equation*}

Consider the scenario where we do not know $\gD_1,\cdots,\gD_K$ at train time, but only know that the smallest group is at least $\alpha \in (0,1)$ of the population size.
Then, we can use the DRO risk given by the Cressie-Read family to obtain a surrogate of the worst-group risk, as shown by the following result.

\begin{proposition}
\label{prop:rmax-dbeta}
	Let $\alpha = \min_{k=1,\cdots,K} P(\gD_k) \leq \exp(-1) \approx 36.8\%$ be the minima group size, and define $\rho = f_{\beta}(\frac{1}{\alpha})$, then we have $\gR_{\max}(f;P) \le \gR_{D_\beta, \rho}(f; P)$.
\end{proposition}
\begin{proof}
    Note that $f_\beta'(t) = \frac{1}{\beta - 1} \paren{t^{\beta - 1} -1}$. Thus, $f_\beta'(t)$ is decreasing when $t \in [0,1]$ and increasing when $t \in \brac{1, \frac{1}{\alpha}}$.
    This implies that $f_\beta(t) \le \max \oset{ f_\beta(0), f_\beta \paren{\frac{1}{\alpha}} }$.
    We can further verify that $f_\beta \paren{\frac{1}{\alpha}} - f_\beta(0) = \frac{1}{\beta (\beta-1)} \paren{\frac{1}{\alpha^\beta} - \frac{\beta}{\alpha}}$, which is non-negative when $\alpha \le \beta^{-\frac{1}{\beta - 1}}$.
    Since $\beta > 1$, we have $\beta^{-\frac{1}{\beta - 1}} \ge \exp(-1)$. Thus, we have essentially proved that when $\alpha \le \exp(-1)$, there is
    \begin{equation*}
        \forall t \in \brac{0, \frac{1}{\alpha}}, \; f_\beta(t) \le f_\beta \paren{\frac{1}{\alpha}}  .
    \end{equation*}
    For any $k$, there is $\frac{dP_k}{dP} \le \frac{1}{\alpha}$. Thus, we have
    \begin{equation*}
        D_\beta(P_k \parallel P) = \int f_\beta \paren{ \frac{d P_k}{dP} } dP \le \int f_\beta \paren{ \frac{1}{\alpha} } dP = f_\beta \paren{\frac{1}{\alpha}}  ,
    \end{equation*}
    which combined with the definition of $\gR_{D_\beta, \rho}$ completes the proof.
\end{proof}

While the Cressie-Read family only defines the $f$-divergence for finite $\beta \in (1, +\infty)$, it can be shown that the dual characterization \Cref{eqn:dual-cressie-read} is valid for $\beta = \infty$ as well, for which $D_\beta$ becomes CVaR.
Our analysis below focuses on two cases:
(i) $\beta = \infty$, which corresponds to CVaR;
(ii) $\beta = 2$, which corresponds to the $\chi^2$-DRO risk used in \cite{pmlr-v80-hashimoto18a}.
\Cref{tab:dro} summarizes the relevant information about these two DRO risks.
We will denote the CVaR risk by $\cvar_\alpha(f;P)$, and the $\chi^2$-DRO risk by $\gR_{D_{\chi^2, \rho}}(f;P)$.

\begin{table}[t]
\centering
\begin{tabular}{l|cccccc}
\toprule
\textbf{Method} & $\beta$ & $\beta_*$ & $\rho$ & $c_\beta(\rho)$ & $D_\beta(Q \parallel P)$ & Risk notation \\ 
\midrule
\textbf{CVaR} & $\infty$ & $1$ & $-\log \alpha$ & $\alpha^{-1}$ & $\sup \log \frac{dQ}{dP}$ & $\cvar_\alpha(f;P)$ \\ 
{\boldmath $\chi^2$-\textbf{DRO}} & $2$ & $2$ & $\frac{1}{2}\paren{\frac{1}{\alpha} - 1}^2$ & $\sqrt{1 + \paren{\frac{1}{\alpha} - 1}^2}$ & $\frac{1}{2} \int \paren{\frac{dQ}{dP} - 1}^2 dP$ & $\gR_{D_{\chi^2, \rho}}(f;P)$ \\ 
\bottomrule
\end{tabular}
\caption{The two DRO risks studied in our analysis.}
\label{tab:dro}
\end{table}

From \Cref{eqn:dual-cressie-read}, we derive the following dual form of CVaR:
\begin{equation}
\label{eqn:cvar-dual}
    \cvar_\alpha(f;P) = \inf_{\eta \in \R} \oset{ \alpha^{-1} \E_{(X,Y) \sim P} \brac{ \paren{ \ell \paren{f(X), Y} - \eta }_+ } + \eta } .
\end{equation}

It is easy to see that the optimal $\eta$ of \Cref{eqn:cvar-dual} is the $\alpha$-quantile of the loss defined as
\begin{equation}
    q_{f;P}(\alpha) = \inf_q \oset{ P_{(X,Y) \sim P} \brac{ \ell \paren{ f(X), Y } > q } \le \alpha } .
\end{equation}

\Cref{prop:rmax-dbeta} implies that both CVaR and $\chi^2$-risks are upper bounds of $\gR_{\max}$.

\begin{corollary}
\label{cor:surrogate-dro}
    Let $\alpha = \min_{k=1,\cdots,K} P(\gD_k) \leq \exp(-1)$ be the minimal group size, and $\rho = \frac{1}{2} \paren{\frac{1}{\alpha} - 1}^2$. Then, we have $\gR_{\max}(f;P) \le \cvar_\alpha (f;P) \le \gR_{D_{\chi^2}, \rho}(f; P)$.
\end{corollary}

\begin{proof}
    Denote $p_k = P(\gD_k)$ for all $k$.
    Then, $P(x,y) = p_k P(x,y|\gD_k) + (1-p_k) P(x,y|\overline{\gD_k})$.
    Let $Q = P_k$ and $Q'(x,y) = \frac{p_k - \alpha}{1 - \alpha} P(x,y|\gD_k) + \frac{1-p_k}{1-\alpha}P(x,y|\overline{\gD_k})$.
    Then, $P = \alpha Q + (1-\alpha) Q'$, which by the definition of CVaR implies that $\E_{P_k}[\ell (f(X), Y)] \le \cvar_\alpha (f;P)$. Thus, $\gR_{\max}(f;P) \le \cvar_\alpha (f;P)$.
    On the other hand, for any $Q$ such that there exists $Q'$ that satisfies $P = \alpha Q + (1-\alpha) Q'$, there is $\frac{dQ}{dP}(x,y) \le \frac{1}{\alpha}$ \aew{}.
    Thus, $D_{\chi^2}(Q \parallel P) \le \frac{1}{2} \paren{ \frac{1}{\alpha} - 1 }^2 = \rho$.
    This implies that $\cvar_\alpha (f;P) \le \gR_{D_{\chi^2}, \rho}(f; P)$.
\end{proof}

\paragraph{Method.}
We model the outliers in the dataset with \textbf{\boldmath Huber's $\epsilon$-contamination model} \cite{huber1992robust}.
Let $P$ be the clean distribution without outliers, and $\ptrain$ be the observed contaminated training distribution.
Then, this model assumes that
\begin{equation*}
\ptrain = (1-\epsilon) P + \epsilon \tilde{P}, \quad \text{where } \tilde{P} \text{ is an arbitrary distribution, and } 0 < \epsilon < \frac{1}{2}.
\end{equation*}

As we saw in \Cref{fig:doro}, DORO ignores the worst samples where the loss is the highest.
With this insight, we define the DORO risk as follows.

\begin{definition}
The \textbf{\boldmath expected $\epsilon$-DORO risk} is defined as
\begin{equation}
\label{eqn:doro-risk-def}
    \gR_{D,\rho,\epsilon}(f;\ptrain) = \inf_{P'} \sset{\gR_{D,\rho(f;P')}}{\exists \; \tilde{P}' \; \text{s.t.} \; \ptrain = (1-\epsilon) P' + \epsilon \tilde{P}'}  .
\end{equation}
\end{definition}

The DORO risk has the following relationship with the total variation.
\begin{lemma}
\label{lem:robust-dro}
Let $\tv(P,Q)=\frac{1}{2}\int_{\gX \times \gY} |P(z) - Q(z)|dz$ be the total variation,  and $\ptrain$ be given by the Huber's model. Then the DORO risk can be lower bounded by
\begin{equation*}
    \gR_{D,\rho,\epsilon}(\theta;\ptrain) \geq \inf_{P''} \oset{ \gR_{D,\rho}(\theta;P''):  \tv (P, P'') \leq \frac{\epsilon}{1-\epsilon} }  .
\end{equation*}
\end{lemma}
\begin{proof}
    For any $P'$ such that $\ptrain = (1-\epsilon) P' + \epsilon \tilde{P'}$ for some $\tilde{P'}$. Let $U(z) = \min \oset{P(z), P'(z)}$ for all $z \in \gX \times \gY$.
    Then, we have
    \begin{equation*}
        (1-\epsilon) U(z) + \epsilon \tilde{P}(z) + \epsilon \tilde{P'}(z) \geq \ptrain(z) \quad \text{for any } z \in \gX \times \gY ,
    \end{equation*}
    as both $\tilde{P}(z)$ and $\tilde{P'}(z)$ are non-negative. Integrating both sides produces $\int_{\gX \times \gY} U(z) dz \geq \frac{1-2\epsilon}{1-\epsilon}$, which implies that $\tv (P,P') \leq \frac{\epsilon}{1-\epsilon}$. Thus, we have
    \begin{equation*}
        \gR_{D,\rho}(\theta;P') \geq \inf_{P''}\oset{ \gR_{D,\rho}(\theta,P''):  \tv (P, P'') \leq \frac{\epsilon}{1-\epsilon} },
    \end{equation*}
    which combined with the definition of the DORO risk proves the result.
\end{proof}

With the Cressie-Read family, the DORO risk has the following dual formula:
\begin{proposition}[Proof in \Cref{app:proof-prop-doro-dual}]
\label{prop:doro-dual}
Let $\ell$ be a continuous non-negative loss function, and suppose $\ptrain$ is a continuous distribution. 
Then, we have
\begin{align*}
 \gR_{D,\rho,\epsilon}(f;\ptrain) =  \inf_\eta \oset{c_\beta(\rho) \underset{(X,Y) \sim \ptrain}{\E} \sbrac{ \paren{ \ell - \eta }_+^{\beta_*} }{\Pr_{(X', Y') \sim \ptrain} \oset{ \ell > \ell'} \ge \epsilon}^{\frac{1}{\beta_*}} + \eta}  ,
\end{align*}
where $\ell = \ell(f(X), Y)$, and $\ell' = \ell(f(X'), Y')$.
\end{proposition}

With this dual formula, the DORO risk can be minimized using \Cref{alg:doro}.
For each batch of samples, this algorithm first sorts the samples by their training losses, and then finds the optimal $\eta^*$ in the above dual form.
For example, we can use Brent's method \cite{brent1971algorithm} to find $\eta^*$.
Then, this $\eta^*$ is fixed and $\theta$ is updated to minimize the dual form.
This algorithm is inspired by the ITLM algorithm \cite{shen2019learning}, in which an alternative approach to making DRO more robust to outliers was proposed---removing the potential outliers from the dataset via data preprocessing.
In comparison, DORO does not throw away any data.
In addition, preprocessing methods cannot cope with online data where new samples are received in a stream, but DORO is still feasible.

\begin{algorithm}[t]
\caption{Minimizing the DORO risk with $D_\beta$ Divergence}
\label{alg:doro}
\begin{algorithmic}
\State \textbf{Input:} Batch size $n$, outlier fraction $\epsilon$, minimal group size $\alpha$, initial model weight $\theta$
\For{each iteration}
\State Sample a batch $(x_1,y_1),\cdots,(x_n,y_n) \sim \ptrain$
\State Compute the sample losses: $\ell_i = \ell(f_\theta(x_i), y_i)$ for $i=1,\cdots,n$
\State Sort the sample losses: $\ell_{i_1} \geq \cdots \geq \ell_{i_n}$
\State Define $F(\theta,\eta) = c_\beta(\rho) \cdot [ \frac{1}{n-\lfloor \epsilon n \rfloor} \sum_{j=\lfloor \epsilon n \rfloor+1}^{n} (\ell(f_\theta(x_{i_j}), y_{i_j})-\eta)_+^{\beta_*}]^{\frac{1}{\beta_*}} +\eta$
\State Find $\eta^* = \argmin_{\eta} F(\theta,\eta)$
\State Update $\theta$ by one step to minimize $\ell(\theta)= F(\theta,\eta^*)$ with some gradient method
\EndFor
\end{algorithmic}
\end{algorithm}

\paragraph{Theoretical guarantee.}
We now show that the DORO risk is a surrogate of the worst-group risk, meaning that the DORO risk is an upper bound.
This result parallels \Cref{cor:surrogate-dro} in the uncontaminated setting, and guarantees that minimizing the DORO risk over $\ptrain$ efficiently minimizes $\gR_{\max}$ over $P$.

\begin{theorem}[Proof in \Cref{app:proof-thm-doro-guarantee}]
\label{thm:doro-guarantee}
Let $\alpha = \min_{k=1,\cdots,K}P(\gD_k)$, and $\rho = \frac{1}{2}\paren{\frac{1}{\alpha}-1}^2$. Suppose $\ptrain$ satisfies the Huber's $\epsilon$-contamination model.
Suppose $\ell$ is a non-negative loss function with a uniformly bounded second moment: $\E_{(X,Y) \sim P} \brac{\ell(f_\theta(X), Y)^2} \le \sigma^2$ for all $\theta$. Then, we have
\begin{equation}
\label{eqn:thm-doro-guarantee}
\begin{aligned}
    \gR_{\max}(f_\theta;P) & \le \max \oset{ 3 \cvar_{\alpha,\epsilon}(f_\theta;\ptrain), 3 \alpha^{-1} \sigma \sqrt{\frac{\epsilon}{1-\epsilon}} } \\ 
    & \le \max \oset{ 3 D_{\chi^2,\rho,\epsilon}(f_\theta;\ptrain), 3 \alpha^{-1} \sigma \sqrt{\frac{\epsilon}{1-\epsilon}} }  .
\end{aligned}
\end{equation}
\end{theorem}

\paragraph{Empirical evaluation.}

We test DORO on three datasets: COMPAS, CelebA \cite{liu2015deep} and CivilComments-Wilds \cite{borkan2019nuanced,pmlr-v139-koh21a}.
CelebA is a facial recognition dataset, where the target is whether the person has blond hair.
CivilComments-Wilds is a toxicity identification NLP dataset, where the target is whether an online post contains toxic contents.
All targets are binary.
For COMPAS, we randomly sample 70\% of the samples to be the training data.
The other two datasetse have official train-val-test splits.
On COMPAS we define 4 groups, and on each of the other datasets we define 16 groups.

We use a two-layer ReLU-activated feed-forward neural network on the COMPAS dataset, a ResNet18 \cite{He_2016_CVPR} on CelebA, and a BERT-base-uncased \cite{devlin2018bert} on CivilComments-Wilds.
Each algorithm is run 300 epochs on COMPAS, 30 epochs on CelebA and 5 epochs on CivilComments-Wilds.
The best model is selected based on the worst-group accuracy on the validation set.
Note that in reality, the worst-group accuracy is not available because we might now know the group membership of each sample.
Therefore, this model selection strategy is an oracle one.
We will discuss more on this point later.

\begin{table}[t]
\begin{center}
\begin{tabular}{ cccc } 
 \toprule
 \textbf{Dataset} & \textbf{Method} & \textbf{Average Accuracy}  & \textbf{Worst-group Accuracy}\\ 
 \midrule
 \multirow{5}{*}{COMPAS} & ERM & $69.31 \pm 0.19$  & $68.83 \pm 0.18$ \\ 
 & CVaR & $68.52 \pm 0.31$ & $68.22 \pm 0.30$ \\ 
 & CVaR-DORO & $69.38 \pm 0.10$ & $69.11 \pm 0.05$ \\ 
 & $\chi^2$-DRO & $67.93 \pm 0.40$ & $67.32 \pm 0.60$  \\ 
 & $\chi^2$-DORO & $69.62 \pm 0.16$  & $69.22 \pm 0.11$ \\ 
 \midrule
 \multirow{5}{*}{CelebA} & ERM & $95.01 \pm 0.38$ & $53.94 \pm 2.02$ \\ 
 & CVaR & $82.83 \pm 1.33$ & $66.44 \pm 2.34$ \\ 
 & CVaR-DORO & $92.91 \pm 0.48$ & $72.17 \pm 3.14$  \\ 
 & $\chi^2$-DRO & $83.85 \pm 1.42$ & $67.76 \pm 3.22$ \\ 
 & $\chi^2$-DORO & $82.18 \pm 1.17$ & $68.33 \pm 1.79$ \\ 
 \midrule
 \multirow{5}{*}{CivilComments-Wilds} & ERM & $92.04 \pm 0.24$ & $64.62 \pm 2.48$ \\ 
 & CVaR & $89.11 \pm 0.76$ & $63.90 \pm 4.42$ \\ 
 & CVaR-DORO & $90.45 \pm 0.70$ & $68.00 \pm 2.10$ \\ 
 & $\chi^2$-DRO & $90.08 \pm 0.92$ & $65.55 \pm 1.51$ \\ 
 & $\chi^2$-DORO & $90.11 \pm 1.09$ & $67.19 \pm 2.51$ \\ 
 \bottomrule
\end{tabular}
\end{center}
\caption{Average/worst-group test accuracies of the selected models. (\%)}
\label{tab:result-doro}
\end{table}

\Cref{tab:result-doro} reports the 95\% confidence intervals of the mean test accuracies on each dataset.
For DRO and DORO, we do a grid search to pick the best $\alpha$ and $\epsilon$ that achieve the best worst-group accuracy.
Each experiment is repeated 10 times on COMPAS and CelebA, and 5 times on CivilComments-Wilds with different random seeds.
From the table, we can conclude that DORO consistently outperforms DRO in terms of the average and worst-group test accuracies.

\begin{table}[t]
\begin{center}
\begin{tabular}{c|cc}
\toprule
\textbf{Method} & \textbf{Average} & \textbf{Worst-group}   \\
\midrule
ERM & $0.73 \pm 0.06$ & $8.59 \pm 0.90$  \\ 
CVaR & $11.53 \pm 1.72$ & $21.47 \pm 0.71$ \\ 
CVaR-DORO & $4.03 \pm 1.57$ & $16.84 \pm 0.91$ \\ 
$\chi^2$-DRO & $8.88 \pm 2.98$ & $19.06 \pm 1.18$ \\ 
$\chi^2$-DORO & $1.60 \pm 0.34$ & $13.01 \pm 1.40$ \\ 
\bottomrule
\end{tabular}
\end{center}
\caption{Standard deviations of test accuracies during training on CelebA.}
\label{tab:acc-std-celeba}
\end{table}

Next, we show that DORO enhances training stability.
We compute the standard deviations of the average and worst-group test accuracies across epochs during training on CelebA, and the results are reported in \Cref{tab:acc-std-celeba}. 
We can see that DORO lowers the standard deviations, which means that its training is more stable.

\paragraph{Difficulty of model selection.}
In our experiments, we select the best models based on their worst-group validation accuracies, which is an oracle strategy.
However, model selection without access to the worst-group accuracy is extremely hard.
We tried the following three strategies of model selection in our experiments:
\begin{itemize}
    \item Max average accuracy
    \item Min CVaR risk
    \item Min CVaR-DORO risk
\end{itemize}

All of these strategies are significantly worse than the oracle strategy.
The main reason is that though we have shown that CVaR and CVaR-DORO risks are surrogates of the worst-group risk, they do not necessarily have a monotonic relationship.
Therefore, we pose model selection in this scenario as an open problem.

To summarize this chapter, we first introduced the most popular solutions to subpopulation shift---reweighting and DRO.
Then, we demonstrated two issues they have.
First, they might not be able to improve over ERM because the models they lead to are too similar. Thus, a large regularization, early stopping or a different loss function (such as the polynomially tailed loss) is neccessary for them to work.
Second, DRO is very sensitive to outliers because it places large weights on them.
To solve this problem, we proposed the DORO risk which avoids the potential outliers.

The key takeaway from this chapter is that generalization under distribution shift, or \textbf{distributionally robust generalization (DRG)}, is extremely hard, much harder than DRO.
Over the past decade a lot of methods for DRO have been proposed, but whether these models can lead to better DRG is questionable.
In fact, as mentioned earlier, \cite{gulrajani2021in} showed that many of these methods do not have better DRG.
The same applies to representation learning.
For example, upweighting the data in a certain domain during pretraining does not necessarily lead to a higher downstream performance on that domain.
Thus, we need to be very cautious of our heuristics.

\chapter{Conclusion}
\label{chap:conclusion}

Here is a summary of the key results from the contexture theory:
\begin{enumerate}
    \item Representations are learned from the association between input $X$ and a context variable $A$, and we call this association the contexture.
    \item Learning the contexture, or extracting the top-$d$ eigenspace, can preserve the most information of $T_{P^+}$, the expectation operator.
    \item The top-$d$ eigenspace can be extracted by training a large encoder to optimize certain variational objectives. This is how deep representation learning works.
    \item Learning the contexture is optimal if the task is compatible with the context.
    \item Making models larger inevitably produces diminishing returns, and creating better contexts is necessary for further advancements of pretraining.
    \item A good context should have a moderate association between $X$ and $A$.
    \item Mixing multiple existing contexts with convolution, convex combination and concatenation can lead to better contexts.
    \item The context complexity affects the sample complexity of representation learning.
    \item Learning the contexture can be viewed as STKR with the truncation function.
    \item Analyzing the distribution shift from pretraining to downstream is very hard.
\end{enumerate}

Let us conclude this thesis by discussing its limitations, and posing some open problems for future work.

\paragraph{Effect of optimization and model architecture.}
Throughout this thesis, we did not analyze the effect of optimization and model architecture on representation learning, though in \Cref{sec:scaling-laws} we showed that scaling up the model size brings the learned representation closer to the top-$d$ singular functions of $\tp$.
In reality, the \textit{implicit bias} of optimization and the \textit{inductive bias} of model architecture are both very important to the encoder.
For optimization, \cite{cohen2021gradient} showed that if the model is trained with popular gradient methods such as gradient descent or Adam \cite{DBLP:journals/corr/KingmaB14}, then after a sufficient amount of training time, the model weights will oscillate around what they called the \textit{edge of stability}, instead of converging to any optimal weights.
Thus, the first open problem is whether the representation is always close to the top-$d$ singular functions when it is oscillating.
And even better, can we characterize this oscillating representation as a dynamical system?
The inductive bias of the model architecture contains our prior knowledge about the task, and thus should be considered as a part of the context.
For example, in \Cref{sec:learn-mixture} we proved that when the model is an encoder composed with a feature map, then the representation trained via supervised learning will be the contexture of the convolution of two contexts---the feature map and the label.
Thus, the second open problem is how to express the inductive bias of any arbitrary model architecture as a context, and how the context affects the learned representation.

\paragraph{Context scaling.}
The contexture theory makes an important prediction that the next major breakthrough in pretraining will be the result of context scaling, where a much better and more complex context is obtained from the real world rather than human heuristics.
For example, base large language models are pretrained with masked token prediction.
To further improve these models, RLHF \cite{ouyang2022training} created new contexts by collecting human preferences of different completions of the same prompt.
Such large-scale data collection serves as a major source of new contexts, and deep learning has been proved to be quite capable of finding the patterns within these contexts.
For example, AlphaFold \cite{jumper2021highly} can learn from the context between protein sequences and their structures through data collected from lab experiments.
The major downside is that these experiments are usually quite expensive, and could sometimes take years to conduct.

Another source of new contexts that has been explored quite a lot recently is multi modalities \cite{liang2024foundations}.
In \Cref{sec:svme-spin} we showed that multi-modal models such as CLIP can learn from the context between different modalities.
Apart from image and text, there are many modalities to be explored, such as videos, tables, graphs, etc. The problem is how to merge multiple modalities through the contexture theory in order to learn a ``world model''.

\paragraph{Towards system 2 thinking (reasoning).}

It was mentioned in the introduction that the contexture learning only covers system 1 thinking, that is the fast, automatic and associative thinking.
Currently it is not applicable to system 2 thinking, such as logical reasoning.
Since system 2 thinking is slow and effortful for humans, we hypothesize that pretraining alone is not sufficient for learning system 2 thinking---post-training is necessary.
Indeed, this has been a hot topic recently in LLM research \cite{wei2022chain,jaech2024openai,guo2025deepseek}.

The high-level idea is to train a model that does not produce outputs so quickly.
In pretraining, $\Phi$ is nothing more than a function, and computing $\Phi(x)$ on a computer usually requires only a fraction of a second.
For system 2 thinking, we would like the computation of $\Phi(x)$ to take a much longer time, which is called \textit{test-time scaling}.
It usually involves a chain of thinking steps, such that $a_1 = \Phi(x), a_2 = \Phi(x;a_1), a_3 = \Phi(x;a_1,a_2)$ and so on, and the final output is $a_L$.
Here, $a_1,a_2,\cdots$ are the intermediate results, also known as a scratch pad.
Recent theoretical studies \cite{feng2023towards,li2024chain} also showed that such a sequential procedure is necessary for LLMs to carry out certain logical reasoning.

The open problem is: Will test-time scaling always improve the performance, or will it achieve diminishing returns after some point?
For example, when an LLM thinks for 3 minutes, its performance is usually significantly better than if it only thinks for 1 minute.
But is an LLM that thinks for 3 weeks significantly better than an LLM thinking for 1 week? What about 3 years versus 1 year?
If an LLM thinks for 3 years, will it be able to solve problems that no human beings can solve, such as proving $P \neq NP$?
To answer these questions, we need a scientific understanding of the mechanism of system 2 thinking, similar to what we have done for system 1 thinking in this thesis.

\begin{acknowledgments}
First and foremost, I would like to express my greatest gratitude to my two amazing PhD advisors Zico Kolter and Pradeep Ravikumar.
I formed the early ideas of the contexture theory in my fourth year after writing my two ICLR 2024 papers, but it was Pradeep who helped me refine these ideas and put them into a structured and well-written thesis. 
In the past two years, this thesis has been rewritten for almost twenty times, and Pradeep put great effort into this process.
Meanwhile, this thesis could never be in place without the guidance and encouragement from Zico. 
Zico focuses on the application side more so I spent less time with him, but I still remember him telling me ``I am not convinced by your theory, but I think this is great work and you should go for it''.
From Zico, I learned how non-theory ML people would view my work, which helped me make it more approachable and readable.

I am grateful to all my collaborators, especially Bingbin Liu and her co-advisor Andrej Risteski, who offered me valuable advice during the formation of my ideas.
Congratulations to Bingbin for getting married in June!
I also would like to thank Yuandong Tian for agreeing to be a member of my thesis committee, and wish him all the best at Meta GenAI.

Moreover, I would not have come to the US to pursue my PhD without Liwei Wang, my undergraduate advisor at Peking University.
In November 2018, I was at a loss about my future, so I went to Liwei's office to seek advice.
He told me firmly, ``for someone like you, there is no better option than pursuing a PhD in the US''.
He also taught me to only aim for two types of research: Those that lead to real life-improving products, and those that dig deep into the fundamentals of mathematics and science.
This has been my research philosophy during my PhD.
I want to thank my undergraduate mentor at UCLA Cho-Jui Hsieh, my mentor at MSRA Di He who is now a professor at Peking University, and my mentors at Amazon Stefan Schroedl, Aram Galstyan and Anoop Kumar.
I want to thank all members at Locus Lab and RAIL who provided feedback on my papers.
I also appreciate the assistance from my PhD program administrators Deb Cavlovich and Matthew Stewart, my OIE advisor Nick Hernandez, and all administrative staff at CMU.

I would like to thank all the students I have mentored during PhD: Yuzhe Lu, Zhenlin Wang, Yilong Qin, Roger Jin, Yash Gupta, Kai Yang, Chenhao Zhang, Hugo Contant, Zihao Ye and Allan Pais.
As the Chinese idiom goes, teaching others makes oneself better.
Working with these students helped me see my weaknesses, and made me a better mentor and a better person.

Finally, I would like to thank my family. My PhD has been quite crazy, starting in the middle of a global pandemic, followed by the advent of ChatGPT that changed everything.
I could not have persevered through all these challenges without the love and encouragement from my parents, as well as my other relatives and close friends.
This thesis is dedicated to my parents.

\textbf{Collaborations:} Chapter 2 is a joint work with Kai Yang. Chapters 3 and 4 is a joint work with Kai Yang, Che-Ping Tsai, Burak Var\i c\i , Hugo Contant and Chenhao Zhang. Chapter 5 is a joint work with Bingbin Liu, Rattana Pukdee, Roger Jin, Andrej Risteski and Maria-Florina Balcan.
Chapter 6 is a joint work with Chen Dan.
Xiaoyu Huang helped check the proofs.
All works are supervised by my advisors Zico Kolter and Pradeep Ravikumar.
\end{acknowledgments}

\appendix
\chapter{Proofs for Chapter 2}
\label{app:proof-chap2}

\section{Proof of Theorem \ref{thm:obj-classification}}
\label{app:proof-thm-obj-classification}

Let us first prove the following lemma.

\begin{lemma}
\label{lem:lem-supervised-kern}
    $\tp \Lambda \tpstar$ is the integral kernel operator of the following kernel
    \begin{equation*}
        k(x, x') =  \iint \kp(a, a') P^+(a| x) P^+(a'| x') da da'.
    \end{equation*}
\end{lemma}

\begin{proof}

By definition, we have $(\tpstar h)(a') = \int h(x') P^+(x'|a')dx'$, which implies that
\begin{equation*}
\begin{aligned}
    (\Lambda \tpstar h)(a) &= \int (\tpstar h)(a') \kp(a, a') \pa(a') da'\\ 
    &= \iint h(x') P^+(x'|a') \kp(a,a') \pa(a') dx' da' \\
    &= \iint h(x') P^+(a'|x') \kp(a,a') \px(x') dx' da'.
\end{aligned}
\end{equation*}
This further implies that
\begin{equation*}
\begin{aligned}
    ( \tp \Lambda \tpstar h)(x) &=  \int (\Lambda \tpstar h)(a) P^+(a|x) da \\
    &= \iiint h(x')  \kp(a,a') P^+(a|x) P^+(a'|x') \px(x')  da da' dx' \\
    &= \int h(x') k(x,x') \px(x') dx',
\end{aligned}
\end{equation*}
as desired.
\end{proof}

Now, we prove \Cref{thm:obj-classification}.

\begin{proof}
Denote $\gR(\Phi, \mW) = \underset{(X,A) \sim P^+}{\E} \; \brac{   \norm{A - \mW \Phi(X)}_2^2 }$.
Assuming, without loss of generality, that $\E_{X \sim \px} [ \Phi_i  \Phi_j] = \delta_{ij}$; otherwise one can perform Gram-Schmidt process on $\Phi_i$ and change $\mW$ respectively. Thus, it amounts to minimize 
\begin{equation*}
\begin{aligned}
    \gR(\Phi, \mW) &=  \underset{X \sim \px}{\E} \; \underset{A \sim P^+(\cdot | X)}{\E} \; \brac{   \norm{ A - \mW  \Phi(X) }_2^2 } \\
    &= \underset{X \sim \px}{\E} \norm{\mW \Phi(X)}_2^2 - 2  \underset{(X,A) \sim P^+}{\E} \; \dotp{ A, \mW  \Phi(X)} + \underset{A \sim \pa}{\E} \norm{A}_2^2 \\
    &= \norm{\mW}_F^2 - 2  \underset{(X,A) \sim P^+}{\E} \; \dotp{ A, \mW  \Phi(X)} + \underset{A \sim \pa}{\E} \norm{ A}_2^2.
\end{aligned}
\end{equation*}

Denote $\mW = (w_{ij})_{1\leq i\leq d_A, 1\leq j \leq d}$. We have 
\begin{equation*}
    \frac{\partial  \gR}{\partial w_{ij}} = 2 w_{ij} - 2 \underset{(X,A) \sim P^+}{\E} \; \brac{ A_i  \Phi_j(X)}  ,
\end{equation*}
which implies that for a fixed $\Phi$, the optimal $\mW$ that minimizes this loss should satisfy
\begin{equation*}
    w_{ij} = \underset{(X,A) \sim P^+}{\E} \; \brac{ A_i  \Phi_j(X)}  .
\end{equation*}
Combining the minimizer of $\mW$ with $ \gR$ and notice that $\E_{A \sim \pa} \norm{ A}_2^2$ is a constant, it suffices to \textbf{maximize}
\begin{equation*}
\begin{aligned}
    F(\Phi) &= \sum_{i,j} \brac{\underset{(X,A) \sim P^+}{\E} \;  A_i  \Phi_j(X)}^2 \\
    &= \int \sum_j  \Phi_j(x_1)  \Phi_j(x_2) \dotp{ a_1,  a_2} \px(x_1) P^+(a_1 | x_1) \px(x_2) P^+(a_2 | x_2) dx_1 da_1 dx_2 da_2 \\
    &= \iint \sum_j  \Phi_j(x_1)  \Phi_j(x_2) \hat k(x_1, x_2) \px(x_1)  \px(x_2) dx_1 dx_2,
\end{aligned}
\end{equation*}
where 
\begin{align}
    \hat k(x_1, x_2)  &= \iint \dotp{ a_1,  a_2} P^+(a_1 | x_1) P^+(a_2 | x_2) da_1 da_2 \label{eqn:supervised-unbiased-kernel} \\
    &= \iint \sI [a_1 = a_2] P^+(a_1 | x_1) P^+(a_2 | x_2) da_1 da_2. \nonumber
\end{align}
Thus $\Phi^*$ is a minimizer of $\gR(\Phi)$ if $\Phi^*$ extracts the top-$d$ eigenfunctions of $\hat k(x_1, x_2)$. Combining with \Cref{lem:lem-supervised-kern} yields that $k_\Lambda(a, a') = \sI[a = a']$.
Furthermore, we have $(\Lambda g)(a) = \int g(a') k_\Lambda(a, a') d\pa(a') = g(a)\pa(a)$, as desired.

If all classes have the same size, we have $\pa(a) \equiv c \in (0, 1)$ where $c$ is a constant. Thus $(\Lambda g)(a) = g(a)\pa(a) = cg(a)$, which implies that $ \tp \Lambda \tpstar = c \tp \tpstar$. This concludes that $\tp \Lambda \tpstar$ and $\tp \tpstar$ share the same top-$d$ eigenfunctions.
\end{proof}

\section{Proof of Theorem \ref{thm:obj-classification-balanced}}
\label{app:proof-thm-obj-classification-balanced}
\begin{proof}
Denote $\gR(\Phi, \mW) = \E_{(X,A) \sim P^+}\brac{   \frac{1}{\sqrt{\pa (A)}}\norm{A - \mW \Phi(X)}_2^2 }$.
Assuming without loss of generality that $\E_{(X,A) \sim P^+}\brac{ \frac{1}{\sqrt{\pa (A)}} \Phi_i \Phi_j } = \sI[i=j]$;
otherwise we can perform Gram-Schmidt process on $\Phi_i$ and change the value of $\mW$ respectively. Thus, it amounts to minimize 
\begin{equation*}
\begin{aligned}
    \gR(\Phi, \mW) &=  \underset{(X,A) \sim P^+}{\E} \; \brac{  \frac{1}{\sqrt{\pa (A)}} \norm{ A - \mW  \Phi(X) }_2^2 } \\
    &= \underset{(X,A) \sim P^+}{\E} \; \brac{  \frac{1}{\sqrt{\pa (A)}} \norm{ \mW  \Phi(X) }_2^2 } \\
    & \quad - 2 \underset{(X,A) \sim P^+}{\E} \;  \dotp{ \frac{A}{\sqrt{\pa (A)}}, \mW  \Phi(X)} + \underset{A \sim \pa}{\E} \brac{\frac{ \norm{ A}_2^2}{\sqrt{\pa (A)}}} \\
    &= \norm{\mW}_F^2 - 2 \underset{(X,A) \sim P^+}{\E} \;  \dotp{ \frac{A}{\sqrt{\pa (A)}}, \mW  \Phi(X)} + \underset{A \sim \pa}{\E} \brac{\frac{ \norm{ A}_2^2}{\sqrt{\pa (A)}}}  .
\end{aligned}
\end{equation*}

Denote $\mW = (w_{ij})_{1\leq i\leq d_A, 1\leq j \leq d}$. We have 
\begin{equation*}
    \frac{\partial  \gR}{\partial w_{ij}} = 2 w_{ij} - 2 \underset{(X,A) \sim P^+}{\E} \; \brac{  \frac{A_i}{\sqrt{\pa (A)}} \Phi_j(X)}  ,
\end{equation*}
which implies that for a fixed $\Phi$, the minimizer of $\mW$ satisfies 
\begin{equation*}
    w_{ij} = \underset{(X,A) \sim P^+}{\E} \; \brac{  \frac{A_i}{\sqrt{\pa (A)}} \Phi_j(X)}  .
\end{equation*}

Combining the minimizer of $\mW$ with $ \gR$, it suffices to maximize
\begin{equation*}
    \gR' = \sum_{i,j} \brac{\underset{(X,A) \sim P^+}{\E} \;  \frac{A_i}{\sqrt{\pa (A)}} \Phi_j(X)  }^2 
    = \iint \sum_j  \Phi_j(x_1)  \Phi_j(x_2) \hat k(x_1, x_2) \px(x_1)  \px(x_2) dx_1 dx_2,
\end{equation*}
where 
\begin{equation*}
\begin{aligned}
    \hat k(x_1, x_2) &= \iint \frac{\dotp{ a_1,  a_2}}{\sqrt{\pa(a_1) \pa(a_2)}} P^+(a_1 | x_1) P^+(a_2 | x_2) da_1 da_2 \\
    & = \iint \frac{\sI[a_1 = a_2]}{\sqrt{\pa(a_1) \pa(a_2)}} P^+(a_1 | x_1) P^+(a_2 | x_2) da_1 da_2 \\
    & = \int \frac{P^+(a | x_1) P^+(a | x_2)}{\pa(a)} dy.
\end{aligned}
\end{equation*}
Thus $\Phi^*$ is a minimizer of $\gR(\Phi)$ if $\Phi^*$ extracts the top-$d$ eigenfunctions of $\hat k(x_1, x_2)$. Combining with \Cref{def:two-kernels,def:learn-contexture} yields the desired results.
\end{proof}

\section{Proof of Theorem \ref{thm:obj-regression}}
\label{app:proof-thm-obj-regression}

\begin{proof}
    If the linear predictor is unbiased, then the proof is the same as \Cref{app:proof-thm-obj-classification}.
    Consider the case where the linear predictor is biased. Then the pretraining objective can be rewritten as
    \begin{equation*}
        \gR(\Phi, \mW, \vb) =  \underset{(X,A) \sim P^+}{\E} \; \brac{   \norm{\tilde A - \mW \tilde \Phi(X) }_2^2 } + \norm{\hat \vb}_2^2  ,
    \end{equation*}
    where $\hat \vb = \mW \E_{X \sim \px}[\Phi(X)] - \E_{A \sim \pa}[A] + \vb$.
    Thus, if $\Gamma$ and $\mW$ are fixed, then the optimal $\vb^* = \E_{A \sim \pa}[A] - \mW \E_{X \sim \px}[\Phi(X)]$.

Assuming, without loss of generality, that $\E_{X \sim \px} [\tilde \Phi_i \tilde \Phi_j] = \delta_{ij}$; otherwise we can perform Gram-Schmidt process on $\tphi_i$ and change $\mW$ accordingly. Then, it amounts to minimize 
\begin{equation*}
    \gR(\Phi, \mW, \vb^*) = \norm{\mW}_F^2 - 2 \underset{(X,A) \sim P^+}{\E} \; \dotp{\tilde A, \mW \tilde \Phi(X)} + \underset{A \sim \pa}{\E} \norm{\tilde A}_2^2   ,
\end{equation*}
for which we have
\begin{equation*}
    \frac{\partial \gR}{\partial w_{ij}} = 2 w_{ij} - 2 \underset{(X,A) \sim P^+}{\E} \; \brac{\tilde A_i \tilde \Phi_j(X)}  .
\end{equation*}
Thus, the optimal $\mW$ is given by $w^*_{ij} = \E_{(X,A) \sim P^+}  \brac{\tilde A_i \tilde \Phi_j(X)} $. Note that $\E_{A \sim \pa}\norm{\tilde{A}}_2^2$ is a constant.
Thus, minimizing $\gR(\Phi, \mW^*, \vb^*)$ is equivalent to maximizing
\begin{equation*}
\begin{aligned}
    \gJ &= \sum_{i,j} \brac{\underset{(X,A) \sim P^+}{\E} \; \tilde A_i \tilde \Phi_j(X)}^2 \\
    &= \int \sum_j \tilde \Phi_j(x_1) \tilde \Phi_j(x_2) \dotp{\tilde a_1, \tilde a_2} \px(x_1) P^+(a_1 | x_1) \px(x_2) P^+(a_2 | x_2) dx_1 da_1 dx_2 da_2 \\
    &= \iint \sum_j \tilde \Phi_j(x_1) \tilde \Phi_j(x_2) \hat k(x_1, x_2) \px(x_1)  \px(x_2) dx_1 dx_2,
\end{aligned}
\end{equation*}
where 
\begin{equation*}
    \hat k(x_1, x_2) = \iint \dotp{\tilde a_1, \tilde a_2} P^+(a_1 | x_1) P^+(a_2 | x_2) da_1 da_2.
\end{equation*}
Then, we can complete the proof in the same way as \Cref{app:proof-thm-obj-classification}.
\end{proof}

\section{Proof of Theorem \ref{thm:node-repre}}
\label{app:proof-thm-node-repre}

\begin{proof}
Without loss of generality, suppose $\bPhi = \vzero$.
We have
\begin{equation*}
    \paren{\tp f} (u) = \sum_v f(v) \frac{w(u,v)}{D(u)}; \quad \dotp{\tp f, g}_{\px} = \sum_{u,v} f(u) g(v) \frac{w(u,v)}{D_{\Sum}} = \dotp{f, \tp g}_{\px},
\end{equation*}
which implies that $\tp$ is self-adjoint.
Therefore, the eigenfunctions of $\tp$ are the same as those of $\tpstar \tp$, with square root eigenvalues.

For the objective of \Cref{eqn:node-repre-obj}, we have
\begin{equation*}
\begin{aligned}
    \frac{1}{2} \E_{(u,v) \sim P^+} \brac{ \norm{\Phi(u) - \Phi(v)}_2^2 } & = \underset{(u,v) \sim P_w}{\E} \brac{\norm{\Phi(u)}_2^2 - \dotp{\Phi(u), \Phi(v)}} \\
    & = \sum_{i=1}^d \paren{ \norm{\phi_i}_\px^2 - \dotp{\phi_i, \tp \phi_i}_\px } \\
    & = d - \sum_{i=1}^d \dotp{\phi_i, \tp \phi_i}_\px .
\end{aligned}
\end{equation*}
Note that $(u,v)$ and $(v,u)$ can be drawn from $P^+$ with equal probability.
We conclude that $\Phi$ extracts the top-$d$ eigenfunctions of $\tp$, which are the same as the top-$d$ eigenfunctions of $\tpstar \tp$. This implies that $\tPhi$ learns the contexture of $\tp$.
\end{proof}

\section{Proof of Theorem \ref{thm:obj-multi-view}}
\label{app:proof-thm-obj-multi-view}

\begin{proof}
    \begin{enumerate}[label=(\roman*)]
    \item The spectral contrastive loss is 
    \begin{equation*}
    \gR(\Psi) = \underset{X \sim \px}{\E} \; \underset{A,A^+ \sim P^+(\cdot |X)}{\E} \brac{ - \dotp{\tPsi(A), \tPsi(A^+)} + \frac{1}{2} \underset{A^- \sim \pa}{\E} \brac{\dotp{\tPsi(A), \tPsi(A^-)} ^2} }   .
    \end{equation*}
    Suppose $\psi_i = \sum_{j\geq 0} c_{ij}\nu_j$ where $\nu_j$ is the ONB of $L^2(\pa)$ in \Cref{lem:duality}. Since $\nu_j$ is the ONB of $L^2(\pa)$ and $\nu_0 \equiv 1$, we can get for $j\geq 1$, $\E_{\pa} [\nu_j(a)] = \delta_{0,j} = 0$. Thus we can get $\tilde\psi_i = \psi_i - \E[\psi_i] = \sum_{j\geq 1} c_{ij}\nu_j$.
    
    Denote matrix $\mC = (c_{ij})_{1\leq i\leq d,j\geq 1}$, matrix $\mB = (b_{ij}) := \mC^\top \mC$, and matrix $\mD = \diag(s_1^2,s_2^2,\cdots)$ where $s_i$ is the singular value of $\tp$. We have 
    \begin{equation*}
    \begin{aligned}
        & \quad \underset{X \sim \px}{\E} \underset{A,A^+ \sim P^+(\cdot |X)}{\E} \brac{\dotp{\tPsi(A), \tPsi(A^+)}} \\
        &= \iiint \dotp{\tPsi(a), \tPsi(a^+)} P^+(a|x) P^+(a^+|x) \px(x) dxdada^+ \\
        &= \int \dotp{\int \tPsi(a) P^+(a|x) dy, \int \tPsi(a^+) P^+(a^+|x) da^+} p(x) dx\\
        &= \int \dotp{\tp \tPsi(x), \tp \tPsi(x)} p(x) dx = \|\tp \tPsi\|^2_{\px} \\
        &= \sum_i s_i^2 b_{ii}   ;
    \end{aligned}
    \end{equation*}
    and 
    \begin{equation*}
    \begin{aligned}
        \underset{A,A^- \sim \pa}{\E}  \brac{\dotp{\tPsi(A), \tPsi(A^-)} ^2}
        &= \iint \brac{\sum_{i=1}^d \tilde \psi_i(a)\tilde \psi_i(a^-)}^2 d\pa(a) d\pa(a^-)  \\
        &= \sum_{1\leq i,j\leq d} \brac{\int \tilde\psi_i(a)\tilde\psi_j(a) d\pa(a)}^2 \\
        &= \sum_{i,j} b_{ij}^2   .
    \end{aligned}
    \end{equation*}
    Thus, we have
    \begin{equation*}
    \gR(\Psi) = - \sum_i s_i^2 b_{ii} + \frac12 \sum_{i,j} b_{ij}^2 \\= \|\mB - \mD\|_F^2 - \|\mD\|_F^2  .
    \end{equation*}
    So if suffices to minimize $\|\mB - \mD\|_F^2$ where $\text{rank}(\mB) \leq d$. By Eckart-Young-Mirsky Theorem, we know the minimizer of $\mB$ is $\mB^* = \text{diag}(s_1^2,\cdots,s_d^2)$. Thus the minimizer of $\mC$ should be $\mC^* = \mU\diag(s_1,\cdots,s_d)$ where $\mU \in \sR^{d\times d}$ is an orthonormal matrix. This indicates the minimizer $\tPsi^*$ extracts the top-$d$ singular functions of $\tp$, and hence $\tPhi^*$ learns the contexture of $P^+$.
    
    \item Non-contrastive learning is done by minimizing
    \begin{equation*}
        \gR(\Psi) = \underset{X \sim \px}{\E} \; \underset{A,A^+ \sim P^+(\cdot |X)}{\E} \brac{ \norm{\Psi(A) - \Psi(A^+)}_2^2 }  , 
    \end{equation*}
    subject to $\Cov_{\pa}\brac{\Psi} = \mI$. 
    This amounts to minimizing $F(\Psi) = -\E_{A,A^+}[\dotp{\tPsi(A), \tPsi(A^+)}]$, because $\gR(\Psi) - 2 F(\Psi) = 2$ is a constant under the constraint $\Cov_{\pa}\brac{\Psi} = \mI$. 
    Suppose $\psi_i = \sum_{j\geq 0} c_{ij}\nu_j$ where $\nu_j$ is the ONB of $L^2(\pa)$ in \Cref{lem:spectral-decomposition}. Since $\E_{\pa} [\nu_j(a)] = \delta_{0,j}$, we can get $\tilde\psi_i = \psi_i - \E[\psi_i] = \sum_{j\geq 1} c_{ij}\nu_j$. 
    
   Using the same calculation as in (i), we have
    \begin{equation*}
        F(\Psi) = -\underset{X \sim \px}{\E} \underset{A,A^+ \sim P^+(\cdot |X)}{\E} \brac{\dotp{\tPsi(A), \tPsi(A^+)}} = -\|\tp \tPsi\|^2_{\px} = - \sum_i s_i^2 b_{ii}   .
    \end{equation*}
    By $\E_{\pa}\brac{\tilde\psi_i \tilde\psi_j} = \delta_{ij}$, we have
    \begin{equation*}
        \sum_{i} b_{ii} = \sum_{i,j} c_{ij}^2 = d   .
    \end{equation*}
    Since $\nu_i$ is an ONB of $L^2(\pa)$, $\tilde\psi_1, \cdots, \tilde\psi_d$ are orthogonal, we have
    \begin{equation}
        b_{ii} = \sum_{j=1}^d c_{ji}^2 = \sum_{j=1}^d \dotp{\tilde \psi_j, \nu_i}_{\pa}^2 \leq \|\nu_i\|_{\pa}^2 = 1  .
        \label{eqn:proof-prop2-projection}
    \end{equation}
    Thus, we conclude that
    \begin{equation*}
        \gL_{\rm{N}}(\Psi) + \sum_{i=1}^d s_i^2 = \sum_{i=1}^d s_i^2(1 - b_{ii}) - \sum_{i>d} s_i^2 b_{ii} \geq \sum_{i=1}^d s_d^2 (1 - b_{ii}) - \sum_{i>d} s_d^2 b_{ii} = 0,
    \end{equation*}
    which implies that $\gL_{\rm{N}}(\Psi) \geq - \sum_{i=1}^d s_i^2$. To attain equality, we will have $b_{ii} = 1$ for $i=1,\cdots,d$, and $b_{ii} = 0$ for $i \geq d+1$. By \Cref{eqn:proof-prop2-projection}, we can know $\Psi^*$ extracts the span of $\nu_1,\cdots,\nu_d$, indicating that $\tPsi^*$ extracts the top-$d$ singular functions of $\tp$ and $\tPhi^*$ learns the contexture of $P^+$.
\end{enumerate}
\end{proof}

\section{Proof of Theorem \ref{thm:svme-spin}}
\label{app:proof-thm-svme-spin}
\begin{proof}
    We first show that when $\Phi$ is fixed in SVME, the optimal $\Psi$ that minimizes the objective is $\Psi = \tpstar \Phi$.
    The SVME objective is
    \begin{equation*}
        \gL = \underset{(X,A) \sim P^+}{\E} \brac{ \norm{\Phi(X)}_2^2 + \Psi(A)^{\top} \paren{\Psi(A) - 2 \Phi(X)} } = \sum_{i=1}^d \oset{\norm{\phi_i}_{\px}^2 + \dotp{\psi_i, \psi_i - 2 \tpstar \phi_i}_{\pa} } ,
    \end{equation*}
    which implies that 
    \begin{equation*}
        \frac{\partial \gL}{\partial \psi_i} = 2 \psi_i - 2 \tpstar \phi_i  .
    \end{equation*}
    Setting it to zero yields $\psi_i = \tpstar \phi_i$. With this, the SVME objective becomes
    \begin{equation*}
        \gL(\Phi) = \sum_{i=1}^d \oset{\norm{\phi_i}_{\px}^2 - \norm{\tpstar \phi_i}_{\pa}^2 } = \sum_{i=1}^d \dotp{\phi_i, \phi_i - \tp \tpstar \phi_i}_{\px}
    \end{equation*}
    because $\tpstar$ is the adjoint of $\tp$. This $\gL$ is exactly the objective of KISE.
    Under the orthonormality constraint, minimizing $\gL(\Phi)$ is equivalent to maximizing $\sum_{i=1}^d \dotp{\phi_i, \tkx \phi_i}_{\px}$.
    Then, we can use the proof of \Cref{thm:obj-multi-view} (ii) to show that when $\Phi^*$ is the optimal solution, $\tPhi^*$ learns the contexture of $P^+$.    
\end{proof}

\section{Proof of Theorem \ref{thm:nef}}
\label{app:proof-thm-nef}

\begin{proof}
    Without loss of generality assume that $\tkx$ has at least $d+1$ positive eigenvalues (including $\lambda_0$).
First we prove a simple result in linear algebra: For any positive definite matrix $\mA$, if vectors $\vu_1,\cdots,\vu_k$ are pairwise orthogonal and satisfy $\vu_i^{\top} \mA \vu_i = 1$ for all $i$, then $\norm{\vu_1}_2^2 + \cdots + \norm{\vu_k}_2^2 \le \lambda_1^{-1} + \cdots + \lambda_k^{-1}$, where $0 < \lambda_1 \le \lambda_2 \le \cdots$ are the eigenvalues of $\mA$.
To prove this, let $\mU$ be the matrix whose columns are $\vu_1,\cdots,\vu_k$. Then, $\mU^{\top} \mU$ is a diagonal matrix and all elements on the diagonal of $\mU^{\top} \mA \mU$ are $1$.
Denote the sorted diagonal elements of $\mU^{\top} \mU$ by $d_1 \ge d_2 \ge \cdots \ge d_k \ge 0$.
Let $\mQ = \mU \mU^{\top}$.
A freshman linear algebra exercise states that the eigenvalues of $\mQ^2$ are $d_1^2,\cdots,d_k^2$, and the rest are all zeros.
Now consider $\mU^{\top} \mU \mU^{\top} \mA \mU$. It is easy to see that the diagonal elements of this matrix are $d_1,d_2,\cdots,d_k$. This implies that $\Tr(\mU^{\top} \mU \mU^{\top} \mA \mU) = \Tr(\mU^{\top} \mU) = d_1+\cdots+d_k$.
Thus, $\Tr(\mQ^2 \mA) = \Tr(\mQ) = d_1+\cdots+d_k$, because $\Tr(\mA \mB) = \Tr(\mB \mA)$.
By von Neumann's trace inequality, there is $\Tr(\mQ^2 \mA) \ge \lambda_1 d_1^2 + \cdots + \lambda_k d_k^2$.
So by Cauchy-Schwarz inequality, $(d_1+\cdots+d_k)^2 \le (\lambda_1^{-1}+\cdots+\lambda_k^{-1})\Tr(\mQ^2 \mA) = (\lambda_1^{-1}+\cdots+\lambda_k^{-1})(d_1+\cdots+d_k)$, which implies that $\norm{\vu_1}_2^2 + \cdots + \norm{\vu_k}_2^2 = d_1 + \cdots + d_k \le \lambda_1^{-1}+\cdots+\lambda_k^{-1}$. The equality is only attained when the Cauchy-Schwarz inequality attains equality, that is $d_i = \lambda_i^{-1}$ for all $i$.

For simplicity, we assume that the eigenvalues of $\tkx$ satisfy $\lambda_1 > \lambda_2 > \cdots > \lambda_d > 0$. If an eigenvalue has more multiplicity, the proof will be the same but much more verbose.
Then, $\mu_1,\cdots,\mu_d$ are fixed. Let $\tphi^*_i = \sum_k u_{ik} \mu_k$.
Then, \Cref{eqn:nef} is equivalent to the following optimization problem:
\begin{align*}
    \underset{u_{ik}}{\text{maximize}} \qquad &  \sum_{i=1}^d \sum_{k} \lambda_k u_{ik}^2 \\ 
    \text{s.t.} \qquad & \forall i \in [d]: \; \sum_k u_{ik}^2 = 1 ;\\ 
    & \forall 1 \le i < j \le d: \; \sum_k \lambda_k u_{ik}u_{jk} = 0 .
\end{align*}

Obviously, for any $\lambda_k = 0$, the optimal $u_{ik}$ should be zero. So without loss of generality, we assume that all $\lambda_k$ in the above problem are positive. Define matrix $\mU$ whose $(i,k)$-th element is $\sqrt{\lambda_k} u_{ik}$. Denote the rows of $\mU$ by $\vu_1,\cdots,\vu_d$.
Let $\mD = \diag \oset{\lambda_1^{-1},\lambda_2^{-1},\cdots}$.
Then, the constraints of the above problem become $\vu_i^{\top} \mD \vu_i = 1$ for all $i$, and $\vu_1,\cdots,\vu_d$ are pairwise orthogonal.
Thus, by the result we have just proved, the objective of this problem satisfies $\norm{\vu_1}_2^2 + \cdots + \norm{\vu_d}_2^2 \le \lambda_1 + \cdots + \lambda_d$. The equality is only attained when $\norm{\vu_i}_2^2 = \lambda_i$ for all $i \in [d]$, assuming that $\norm{\vu_1}_2 \ge \cdots \ge \norm{\vu_d}_2$.
This is only possible when $u_{11}^2 = u_{22}^2 = \cdots = u_{dd}^2 = 1$, which proves the result.
\end{proof}

\section{Proof of Theorem \ref{thm:obj-classification-omega}}
\label{app:proof-thm-obj-classification-omega}

\begin{proof}
    It suffices to show that $\tpywstar \tp \Lambda \tpstar \tpyw$ is the integral kernel operator of
    \begin{equation*}
        k(\omega, \omega') =  \iint \kp(y, y') \pyw(y| \omega) \pyw(y'| \omega') dy dy'.
    \end{equation*}
After that, we can follow the proof in \Cref{app:proof-thm-obj-classification} to prove the result.

Since $x$ determines $\omega,y$ at the same time, we have $y \perp \omega|x$. This implies that 
\begin{equation*}
    P(\omega|x) = P(\omega|x,y), \quad P(y|x) = P(y|x,\omega).
\end{equation*}
By definition, we have
\begin{equation*}
    (\tpyw h)(x') = \int h(\omega')P(\omega'|x')d\omega',
\end{equation*}
which implies that 
\begin{equation*}
\begin{aligned}
        (\tpstar \tpyw h)(y') &= \int (\tpyw h)(x')P_{X|Y}(x'|y')dx'
    = \iint h(\omega')P(\omega'|x')P_{X|Y}(x'|y')dx'd\omega'\\
    &= \iint h(\omega')P(\omega'|x',y')P_{X|Y}(x'|y')dx'd\omega'
     = \int h(\omega') P_{\omega|Y}(\omega'|y') d\omega'.
\end{aligned}
\end{equation*}
Thus, we have
\begin{equation*}
\begin{aligned}
    (\Lambda \tpstar \tpyw h)(y) &= \int (\tpstar \tpyw h)(y') \kp(y, y') \py(y') dy' \\
    &= \iint h(\omega') P_{\omega|Y}(\omega'|y') \kp(y,y') \py(y') d\omega' dy'.
\end{aligned}
\end{equation*}
This implies that
\begin{equation*}
\begin{aligned}
    (\tpywstar \tp \Lambda \tpstar \tpyw h)(\omega) &= \int (\tp \Lambda \tpstar \tpyw h)(x) P(x|\omega)dx \\
    &= \iint (\Lambda \tpstar \tpyw h)(y)P(y|x)P_{X|\omega}(x|\omega)dydx \\
    &= \iint (\Lambda \tpstar \tpyw h)(y)P(y|x,\omega)P_{X|\omega}(x|\omega)dydx \\
    &= \int (\Lambda \tpstar \tpyw h)(y) \pyw(y|\omega) dy \\
    &= \iiint h(\omega') P_{\omega|Y}(\omega'|y') \kp(y,y') \py(y') \pyw(y|\omega) d\omega' dy' dy \\
    &= \iiint h(\omega') \kp(y, y') \pyw(y|w) \pyw(y'|w') \pw(\omega') dy dy' d\omega',
\end{aligned}
\end{equation*}
as desired.
\end{proof}

\chapter{Proofs for Chapter 3}

\section{Proof of Theorem \ref{thm:two-conditions-fep}}
\label{app:proof-two-conditions-fep}
\begin{proof}
    Let $f^* = \sum u_i \mu_i$, and $g^* = \sum s_i u_i \nu_i$.
    For this $f^*$, the maximum in \Cref{eqn:def-compatibility} is attained by $g^*$, so it suffices to show that $g^*$ satisfies \Cref{eqn:thm-two-conditions-fep}. By Bayes rule, we have
    \begin{equation*}
        P^+(A'|A=a) = \int P^+(A'|X=x) P^+(x|A'=a) dx,
    \end{equation*}
    which implies that $P^+(A'=a'|A=a) = \ka(a,a') \pa(a')$. Therefore, we have
     \begin{align*}
    & \E_{X \sim \px} \E_{A,A' \sim P^+(\cdot |X)} \brac{g^*(A) g^*(A')} = \E_{A \sim \pa} \E_{A' \sim P(\cdot|A)} \brac{ g^*(A) g^*(A') } \\ 
     = \; & \E_A \brac{ g^*(A)  \int g^*(a') P(a' | A) da' }  = \E_A \brac{ g^*(A) \int g^*(a') \ka(a,a') \pa(a') da' }  = \dotp{ g^*, \tka g^* }_\pa .
    \end{align*}
     Since $\tka g^* = \tpstar \tp g^* = \sum s_i^3 u_i \nu_i$, \Cref{eqn:thm-two-conditions-fep} is equivalent to $\sum(s_i^2 - s_i^4) u_i^2 \le 2 \epsilon \sum s_i^2 u_i^2$.
     Meanwhile, we have $\sum s_i^2 u_i^2 \ge (1-\epsilon)^2 \sum u_i^2 \ge (1-2\epsilon) \sum u_i^2$. By Cauchy-Schwarz inequality, we have $(\sum s_i^4 u_i^2) (\sum u_i^2) \ge (\sum s_i^2 u_i^2)^2 \ge (1-2\epsilon) (\sum u_i^2) (\sum s_i^2 u_i^2) $, which proves \Cref{eqn:thm-two-conditions-fep}.
\end{proof}

\section{Proof of Theorem \ref{thm:top-d-optimal}}
\label{app:proof-top-d-optimal}
\begin{proof}
Since $\sspan(\Phi)$ is at most rank-$d$, thus there exists $f_1 \in \sspan\{ \mu_1,\cdots, \mu_{d+1}\}$ with $\|f_1\|_\px = 1$ that is orthogonal to $\sspan(\Phi)$. Thus there exists $f_1, f_2 \in \sspan\{ \mu_1, \cdots, \mu_{d+1} \}$ with $\|f_1\|_\px = \|f_2\|_\px = 1$, $f_1$ is orthogonal to $\sspan(\Phi)$ and $f_2 \in \sspan(\Phi)$ (thus $f_1 \perp f_2$), and $\mu_1 \in \sspan\{f_1, f_2\}$.  
Suppose $\mu_1 = \alpha_1 f_1 + \alpha_2 f_2$ (without loss of generosity, assuming $\alpha_1, \alpha_2 \in [0, 1]$) and denote $f_0 = \alpha_2 f_1 - \alpha_1 f_2$. Then $\| f_0 \|_\px = 1$ and $\langle \mu_1, f_0\rangle_\px = 0$. Since $f_1, f_2 \in \sspan\{\mu_1, \cdots, \mu_{d+1}\}$, we have $f_0 \in \sspan\{\mu_2, \cdots, \mu_{d+1}\}$ and thus $\E[f_0] = 0$.

Consider $f = \beta_1 \mu_1 + \beta_2 f_0 \in \fep$ where $\beta_1^2 + \beta_2^2 = 1$, $\beta_1,\beta_2 \in [0,1]$. 
Denote $f = \sum_{i\geq 1} u_i \mu_i$
Then, we have $\sum_i u_i^2 =1$ and
\begin{equation*}
    \beta_2^2 \leq \frac{s_1^2 - (1-\eps)^2}{s_1^2 - s_{d+1}^2}  \implies
    \sum_{i\geq 1} s_i^2 u_i^2 \geq s_1^2\beta_1^2 + s_{d+1}^2 \beta_2^2 = s_1^2 - (s_1^2 - s_{d+1}^2) \beta_2^2 \geq (1 - \eps)^2 \sum_i u_i^2    .
\end{equation*}
Since $f = (\alpha_1 \beta_1 + \alpha_2 \beta_2) f_1 + (\alpha_2 \beta_1 - \alpha_1 \beta_2) f_2$, the approximation error of $f$ is $(\alpha_1 \beta_1 + \alpha_2 \beta_2)^2$.
Define a function $F(\alpha_1) = \alpha_1 \beta_1 + \alpha_2 \beta_2 = \alpha_1 \beta_1 + \sqrt{1 - \alpha_1^2} \beta_2$ ($\alpha_1 \in [0,1]$).
We can show that dis function first increases and then decreases with $\alpha_1$.
Thus, $F(\alpha_1)^2 \geq \min\{ F(0)^2, F(1)^2\} = \min \{\beta_1^2, \beta_2^2\}$.
If  $\beta_2^2 =  \frac{s_1^2 - (1-\eps)^2}{s_1^2 - s_{d+1}^2} \leq \frac{1}{2}$, then the approximation error is always at least $\frac{s_1^2 - (1-\eps)^2}{s_1^2 - s_{d+1}^2}$.
To attain this lower bound, we must have $\sum_{i\geq 1} s_i^2 u_i^2 = s_1^2\beta_1^2 + s_{d+1}^2 \beta_2^2$. This implies that $f_1 = \mu_{d+1}$, indicating that $\sspan(\phi_1, \cdots, \phi_d) = \sspan(\mu_1, \cdots, \mu_d)$. Thus, $\Phi$ must learn the contexture of $\tp$.

On the other hand, if $\Phi$ learns the contexture, then the approximation of $\Phi$ on $f$ will be $A := \sum_{i \ge d+1} u_i^2$.
Then, we have
\begin{equation*}
    (1 - \eps)^2 \leq  \sum_{i \geq 1} s_i^2 u_i^2 \leq s_1^2\sum_{i=1}^d u_i^2 + s_{d+1}^2\sum_{i \geq d+1} u_i^2 = s_1^2 - (s_1^2 - s_{d+1}^2) A   ,
\end{equation*}
and this implies that 
\begin{equation*}
    A = \min_{\vw \in \R^d, \; b \in \R} \; \norm{\vw^{\top} \Phi + b - f}_\px^2\leq \frac{s_1^2 - (1-\eps)^2}{s_1^2 - s_{d+1}^2}  .
\end{equation*}
When $u_1^2 = 1 - \frac{s_1^2 - (1-\eps)^2}{s_1^2 - s_{d+1}^2} ,  u_{d+1}^2 = \frac{s_1^2 - (1-\eps)^2}{s_1^2 - s_{d+1}^2} $, the equality holds.
Thus, the lower bound $\frac{s_1^2 - (1-\eps)^2}{s_1^2 - s_{d+1}^2}$ is attained if and only if $\Phi$ learns the contexture.

The converse part is obvious because for any $\Phi$ we can simply choose $f = \beta_1 \mu_1 + \beta_2 f_0$ as defined above.
Then, we have $f \in \fep$, and the approximation error of $\Phi$ on this $f$ is at least $\frac{s_1^2 - (1-\eps)^2}{s_1^2 - s_{d+1}^2}$.
\end{proof}

\chapter{Proofs for Chapter 4}
\section{Proof of Theorem \ref{thm:kernel-convex-combination}}
\label{app:proof-thm-kernel-convex-combination}

\begin{proof}
    By the proof of \Cref{thm:svme-spin}, $\gR_j$ can be written as
    \begin{equation*}
        \gR_j = d - \E_{A_j} \norm{ \paren{\tpstar \tPhi}(A_j) }_2^2 = d - \sum_{i=1}^d \dotp{\tphi_i, T_{k_j} \tphi_i}_{\px}  .
    \end{equation*}
    Therefore, the weighted sum $\gL = \sum_j w_j \gR_j$ is equivalent to
    \begin{equation*}
        \gL = \sum_{j=1}^r w_j \paren{ d - \sum_{i=1}^d \dotp{\tphi_i, T_{k_j} \tphi_i}_{\px} } = d - \sum_{i=1}^d \dotp{\tphi_i, \sum_{j=1}^r \paren{ w_j T_{k_j} } \tphi_i }_{\px} ,
    \end{equation*}
    where $\sum_j w_j T_{k_j}$ is equal to the integral operator of the linearly combined kernel $\sum_j w_j k_j$.
    Thus, minimizing this $\gL$ subject to the orthonormality constraint will make $\tPhi$ learn the contexture, as shown in \Cref{thm:svme-spin}.
\end{proof}

\section{Proof of Theorem \ref{thm:convergence-nash}}
\label{app:proof-thm-convergence-nash}

\begin{algorithm}[t]
\caption{Solving the game \Cref{eqn:convex-combination-game} (Rewritten)}
\label{alg:bi-level-restate}
\begin{algorithmic}[1]
\Require Embedding dimension $d$, priors $P_1^+,\cdots,P_r^+$, step size $\eta$
\State Initialize: $\rmL^0\gets [0,\cdots,0]$, $\Phi: \gX \rightarrow \R^d$, $\Psi_k: \gA_k \rightarrow \R^d$ for $k \in [r]$
\For{$t=1,\cdots,T$}
    \State $W^t = \sum_{k} \exp(\eta \rmL^{t-1}_i)$; $\rw^t_k = \frac{\exp(\eta \rmL^{t-1}_k)}{W^t}$
    \State $\Phi^t,\Psi^t_1, \cdots, \Psi^t_r \gets \arg\min \sum_k \rw^t_k \gR_k$, and get corresponding loss $\gR^t_k$, $l^t = \sum_{k} \rw^t_k \gR^t_k$
    \State Update the loss vector $\rmL^t \gets \rmL^{t-1} + \gR^t$
\EndFor
\end{algorithmic}
\end{algorithm}

For ease of reading and better use of notations, we restate our algorithm in \Cref{alg:bi-level-restate} and use the notation defined there in our proof. It is easy to verify that they are equivalent. We first prove the following lemma.
\begin{lemma}
    Suppose $\gR^t_k \leq C$ holds for all $t,k$, for some constant $C$. If $\eta C< 1$, then for any $\vw\in\Delta^r$,
    \begin{equation}
        \sum_{t=1}^T \left( \sum_{k=1}^r w_k\gR_k^t \right) - \sum_{t=1}^t l^t \leq C^2T\eta + \frac{1}{\eta}\log r.
        \label{eqn:nash-potential}
    \end{equation}
    Moreover, suppose $T > \log r$ and $\eta = \frac{\sqrt{\log r}}{C\sqrt{T}}$, \Cref{eqn:nash-potential} becomes
    \begin{equation*}
        \sum_{t=1}^T \left( \sum_{k=1}^r w_k\gR_k^t \right) - \sum_{t=1}^T l^t \leq 2C\sqrt{T\log r}.
    \end{equation*}
    \label{lem:bilevel-convergence}
\end{lemma}
\begin{proof}
    Consider the following potential function:
    \begin{equation*}
        \Omega(t) = \frac{1}{\eta}\log(W^t) = \frac{1}{\eta} \log \left(\sum_{k=1}^r \exp(\eta \rmL_k^t) \right).
    \end{equation*}
    Since $\mathrm{e}^x \leq 1 + x + x^2$ when $x\leq 1$, thus we have
    \begin{equation*}
    \begin{aligned}
        \Omega(t) - \Omega(t-1) &= \frac{1}{\eta} \log \frac{W^t}{W^{t-1}}  = \frac{1}{\eta} \log \left(\sum_{k=1}^r \rw_k^t \exp(\eta \gR^t_k) \right) \\ 
        & \leq \frac{1}{\eta} \log \left( \sum_{k=1}^r \rw_k^t [1 + \eta\gR^t_k + (\eta\gR^t_k)^2 ] \right) \\ 
        & = \frac{1}{\eta} \log \left( 1 + \eta\sum_{k=1}^r \rw_k^t [\gR^t_k + \eta(\gR^t_k)^2 ] \right) \\
        & \leq \sum_{k=1}^r \rw_k^t [\gR^t_k + \eta(\gR^t_k)^2]\\
        & = \sum_{k=1}^r \rw_k^t\gR^t_k + \eta\sum_{k=1}^r \rw_k^t(\gR^t_k)^2 \\ 
        & \leq \sum_{k=1}^r \rw_k^t\gR^t_k + C^2\eta = l^t + C^2\eta.
    \end{aligned}    
    \end{equation*}
    Summing over $t=1,\cdots,T$ and we can get 
    \begin{equation}
    \label{proof:potential-sum}
        \Omega(T) - \Omega(0) \leq \sum_{t=1}^t l^t + C^2T\eta.
    \end{equation}
    On the other hand, we know that for any $k\in[r]$, there is
    \begin{equation*}
        \Omega(T) \geq \rmL_k^T = \sum_{t=1}^T \gR_k^t.
    \end{equation*}
    Thus we can get for any $\vw\in\Delta^r$, 
    \begin{equation}
    \label{proof:potential-T}
        \Omega(T) \geq \sum_{k=1}^r w_k\rmL_k^T = \sum_{t=1}^T \left( \sum_{k=1}^r w_k\gR_k^t \right).
    \end{equation}
    Since $\Omega(0) = \frac{1}{\eta}\log r$, combining (\ref{proof:potential-sum}) and (\ref{proof:potential-T}) yields 
    \begin{equation*}
        \sum_{t=1}^T \left( \sum_{k=1}^r w_k\gR_k^t \right) - \sum_{t=1}^t l^t \leq C^2T\eta + \frac{1}{\eta}\log r
    \end{equation*}
    for any $w\in\Delta^r$.

    Moreover, take $\eta = \frac{\sqrt{\log r}}{C\sqrt{T}} < \frac{1}{C}$ and \Cref{eqn:nash-potential} becomes
    \begin{equation*}
        \sum_{t=1}^T \left( \sum_{k=1}^r w_k\gR_k^t \right) - \sum_{t=1}^T l^t \leq 2C\sqrt{T\log r},
    \end{equation*}
    as desired.
\end{proof}

Now, we finish the proof of \Cref{thm:convergence-nash}.
\begin{proof}
        For any $\hat\Phi,\hat\Psi_1,\cdots,\hat\Psi_r$, by the optimality of $\Phi^t,\Psi^t_1,\cdots,\Psi^t_r$, we have 
    \begin{equation*}
    \begin{aligned}
        \frac{1}{T}\sum_{t=1}^T l_t &= \frac{1}{T}\sum_{t=1}^T \sum_{k=1}^r \rw_k^t\gR_k^t 
        \leq \frac{1}{T}\sum_{t=1}^T \sum_{k=1}^r \rw_k^t \E_{X \sim \px} \E_{Y \sim P_k^+(\cdot | X)} \|\hat\Phi(X) - \hat\Psi_k(Y)\|_2^2 \\ 
        & = \sum_{k=1}^r \frac{\sum_{t=1}^T \rw_k^t}{T} \E_{X \sim \px} \E_{Y \sim P_k^+(\cdot | X)} \|\hat\Phi(X) - \hat\Psi_k(Y)\|_2^2 \\
        & \leq \max_{w\in\Delta^r} \sum_{k=1}^r w_k\E_{X \sim \px} \E_{Y \sim P_k^+(\cdot | X)} \|\hat\Phi(X) - \hat\Psi_k(Y)\|_2^2.
    \end{aligned}    
    \end{equation*}
    This implies that
    \begin{equation*}
        \frac{1}{T}\sum_{t=1}^T l_t \leq \min_{\Phi,\Psi} \max_{w\in\Delta^r} \sum_{k=1}^r w_k \gR_k.
    \end{equation*}
    On the other hand, for any $w\in \Delta^r$, 
    \begin{equation*}
        \gL(w) = \frac{1}{T}\sum_{t=1}^T \left( \sum_{k=1}^r w_k\gR_k^t \right) .
    \end{equation*}
    Applying \Cref{lem:bilevel-convergence}, we can get 
    \begin{equation*}
        \gL(w) \leq \frac{1}{T}\sum_{t=1}^T l_t + \frac{2C\sqrt{\log r}}{\sqrt{T}} \leq \min_{\Phi,\Psi} \max_{w\in\Delta^r} \sum_{k=1}^r w_k \gR_k + \frac{2C\sqrt{\log r}}{\sqrt{T}} = \gL^* + \frac{2C\sqrt{\log r}}{\sqrt{T}}
    \end{equation*}
    for any $w\in\Delta^r$, completing our proof.
\end{proof}

\chapter{Proofs for Chapter 5}

\section{Context Complexity of Masking}
\label{app:mask-kappa}

\begin{example}
Consider a \emph{random masking} augmentation, \ie{} for any $x \in \gX$, each coordinate $x^{(i)}$ is randomly and independently masked to be $0$ (\ie{} $0$ denotes the \texttt{[MASK]} token)
with probability $\alpha \in (0,1)$ .
Then, its context complexity is given by $\bt_r^2 = (2 - \alpha)^{d_\gX}$.
\end{example}
    \begin{proof}
We know that
$\bt^2 \ge \int \frac{P^+(a|x)^2}{\pa(a)} da$, whose right-hand side is a constant for all $x$ by symmetry.
Given an $a$, suppose $a$ has $r$ coordinates masked and $(d_{\gX} - r)$ coordinates unmasked.
Then, there are $2^r$ possible $x$ that can be masked to become $a$.
For each of these $x$, $P^+(a|x) = \alpha^r (1-\alpha)^{d_{\gX} - r}$.
So $p(a) = \int P^+(a|x) \px(x) dx = 2^{r - d_{\gX}} \alpha^r (1-\alpha)^{d_{\gX} - r}$.
Thus, we have
\begin{align*}
   \bt^2 & = \int \frac{P^+(a|x)^2}{\pa(a)} da = \sum_{r=0}^{d_\gX} \binom{d_\gX}{r} \frac{\alpha^{2r} (1-\alpha)^{2d_{\gX} - 2r}}{2^{r - d_{\gX}} \alpha^r (1-\alpha)^{d_{\gX} - r}} \\ 
   & = \sum_{r=0}^{d_\gX} \binom{d_\gX}{r} \alpha^r (2-2\alpha)^{d_\gX - r} \\ 
   & = (\alpha + 2 - 2 \alpha)^{d_\gX} = (2 - \alpha)^{d_\gX} ,
\end{align*}
which completes the proof.
\end{proof}

\begin{example}
Consider \emph{random block masking},
\ie{} masking $x^{(i)},x^{(i+1)},\cdots,x^{(i+r-1)}$ for $r = \lceil \alpha d_{\gX} \rceil$ and a uniformly random $i \in [d_{\gX}-r]$, for any $x \in \gX$.
Then, $\bt_c^2 \le [2^{(1-\alpha) }]^{d_{\gX}}$.
\end{example}
\begin{proof}
For any $a$, we have $\pa(a) = \frac{1}{d_\gX - r + 1} \frac{1}{2^{d_\gX - r}}$, and $P^+(a|x) = \frac{1}{d_\gX - r + 1}$ if $a$ is a masked version of $x$.
So there always is $\frac{P^+(a|x)}{\pa(a)} = 2^{d_\gX - r} \le 2^{(1-\alpha) d_\gX}$. Thus, we have $\bt^2 \le 2^{(1-\alpha)d_\gX}$.
\end{proof}

\begin{example}
Consider \emph{random block masking with flipping}, where for any $x \in \gX$,
first mask $x^{(i)},\cdots,x^{(i+r-1)}$ to be $0$ 
for $r = \lceil \alpha d_{\gX} \rceil$ and a uniformly random $i \in [d_{\gX}-r]$,
then randomly flip the sign of each remaining coordinate independently with probability $\frac{\alpha}{2}$.
Then, its context complexity is bounded by $\bt_b^2 \le \left [(\alpha^2-2\alpha+2)^{(1-\alpha/2)} \right ]^{d_{\gX}}$.
\end{example}
\begin{proof}
For any $a$, we have $\pa(a) = \frac{1}{d_\gX - r + 1} \frac{1}{2^{d_\gX - r}}$.
Suppose $a$ is a masked version of $x$, and among the unmasked $(d_\gX - r)$ coordinates, $a$ and $x$ have $k$ disagreeing coordinates.
For a given $k$,
there are $(d_\gX - r + 1) \binom{d_\gX - r}{k}$ possible $a$,
and we have $P^+(a|x) = \frac{1}{d_\gX - r + 1}  (\frac{\alpha}{2})^{k} (1 - \frac{\alpha}{2})^{d_\gX - r - k}$.
Thus, we have
\begin{align*}
    \int \frac{P^+(a|x)^2}{\pa(a)} da & = \sum_{k=0}^{d_\gX - r} (d_\gX - r + 1) \binom{d_\gX - r}{k} \frac{\frac{1}{(d_\gX - r + 1)^2}(\frac{\alpha}{2})^{2k} (1 - \frac{\alpha}{2})^{2d_\gX - 2r - 2k}}{\frac{1}{d_\gX - r + 1} \frac{1}{2^{d_\gX - r}}} \\ 
   & = \sum_{k=0}^{d_\gX - r}\binom{d_\gX - r}{k} 2^{d_\gX - r} \left ( \frac{\alpha^2}{4} \right )^k \left ( 1 - \alpha + \frac{\alpha^2}{4}\right )^{d_\gX - r - k} \\ 
   & = 2^{d_\gX - r} \left ( \frac{\alpha^2}{4} + 1 - \alpha + \frac{\alpha^2}{4} \right )^{d_\gX - r} \\ 
   & \le \left ( \alpha^2 - 2 \alpha + 2 \right )^{d_\gX - r} \le \left ( \alpha^2 - 2 \alpha + 2 \right )^{(1-\alpha/2)d_{\gX}}  ,
\end{align*}
which proves the bound.
\end{proof}

\section{Proof of Lemma \ref{lem:rad-gfd}}
\label{app:proof-lem-rad-gfd}

\begin{proof}
    Let $f_i = \sum_j u_{ij} s_j \mu_j$, and $\mU = (u_{ij}) = [\vu_1,\cdots,\vu_d]$.
    $\mU$ is a matrix with $d$ columns and infinitely many rows.
    Then, since $\dotp{f_i, f_j}_{\hk} = \sI[i=j]$, we have $\mU^{\top} \mU = \mI_d$.
Let $\mM(x) = [s_1 \mu_1(x),s_2 \mu_2(x),\cdots]$, and for a set of samples $S = \oset{x_1,\cdots,x_m}$ denote $\mM_j = \mM(x_j)$.
Then, we have
\begin{align*}
    \hrad_S(\gF_d) & =  \E_{\sigma_1,\cdots,\sigma_m} \brac{ \sup_{F \in \gF_d} \frac{1}{m} \sum_{j=1}^m \sigma_i F(x_j) }  \\ 
    & \le \E_{\sigma_1,\cdots,\sigma_m} \brac{ \sup_{\mU: \mU^{\top} \mU = \mI_d} \abs{ \frac{1}{m} \sum_{j=1}^m \sum_{i=1}^d \sigma_j \vu_i^{\top} \mM_j \mM_j^{\top} \vu_i } } \\ 
    & = \E_{\sigma_1,\cdots,\sigma_m} \brac{ \sup_{\mU: \mU^{\top} \mU = \mI_d} \abs{ \Tr \oset{ \mU^{\top} \paren{ \frac{1}{m} \sum_{j=1}^m \sigma_j \mM_j \mM_j^{\top} } \mU } } } \\ 
    & = \E_{\sigma_1,\cdots,\sigma_m} \brac{ \sup_{\mU: \mU^{\top} \mU = \mI_d} \abs{ \Tr \oset{  \paren{ \frac{1}{m} \sum_{j=1}^m \sigma_j \mM_j \mM_j^{\top} } \mU \mU^{\top} } } }  ,
\end{align*}
where $\sigma_1,\cdots,\sigma_m$ are Rademacher variables, which are \iid{} uniform random variables taking values in $\oset{-1, +1}$.
Let $\beta_1 \ge \beta_2 \ge \cdots$ be the singular values of $\frac{1}{m} \sum_{j=1}^m \sigma_j \mM_j \mM_j^{\top}$.
For any $x$, $\mM(x)^{\top} \mM(x) = \sum s_i^2 \mu_i^2(x) \le \bt^2$, which implies that $\mM_j^{\top} \mM_j \le \bt^2$.

For any $\mU$, the singular values of $\mU \mU^{\top}$ are $d$ ones and lots of zeros.
Moreover, $\norm{\frac{1}{m}\sum_{j=1}^m \sigma_j \mM_j \mM_j^{\top}}_F^2 = \sum_{i=1}^{\infty} \beta_i^2$.
So by von Neumann's trace inequality, we have
\begin{equation*}
     \sup_{\mU: \mU^{\top} \mU = \mI_d} \abs{ \Tr \oset{  \paren{ \frac{1}{m} \sum_{j=1}^m \sigma_j \mM_j \mM_j^{\top} } \mU \mU^{\top} } } \le \sum_{i=1}^d \beta_i \le \sqrt{d \sum_{i=1}^d \beta_i^2}  \le \frac{\sqrt{d}}{m} \norm{\sum_{j=1}^m \sigma_j \mM_j \mM_j^{\top}}_F  .
\end{equation*}

Thus, for any $S$, we have
\begin{align*}
     \hrad_S(\gF_d) & \le \frac{\sqrt{d}}{m} \E_{\sigma_1,\cdots,\sigma_m} \brac{\norm{\sum_{j=1}^m \sigma_j \mM_j \mM_j^{\top}}_F } \\ 
     & = \frac{\sqrt{d}}{m}  \E_{\sigma_1,\cdots,\sigma_m} \brac{ \Tr \oset{ \paren{\sum_{j=1}^m \sigma_j \mM_j \mM_j^{\top}}^{\top} \paren{ \sum_{l=1}^m \sigma_l \mM_l \mM_l^{\top} } }^{1/2} } \\ 
     & \le \frac{\sqrt{d}}{m} \sqrt{  \E_{\sigma_1,\cdots,\sigma_m} \brac{ \Tr \paren{\sum_{j,l=1}^m \sigma_j \sigma_l \mM_j \mM_j^{\top}\mM_l \mM_l^{\top}   }  } } \qquad (\textrm{Jensen}) \\ 
     & = \frac{\sqrt{d}}{m} \sqrt{ \Tr \paren{ \sum_{j,l=1}^m \E[\sigma_j \sigma_l] \mM_j \mM_j^{\top}\mM_l \mM_l^{\top} } } \\ 
     & = \frac{\sqrt{d}}{m} \sqrt{ \Tr \paren{\sum_{j=1}^m \mM_j \mM_j^{\top}\mM_j \mM_j^{\top} } } = \frac{\sqrt{d}}{m} \sqrt{\sum_{j=1}^m \paren{\mM_j^{\top}\mM_j}^2} \le \frac{\sqrt{d}}{\sqrt{m}} \bt^2  .
\end{align*}
Since $\rad_m(\gF_d) = \E_S[\hrad_S(\gF_d)]$, we have $\rad_m(\gF_d) \le \frac{\sqrt{d}}{\sqrt{m}} \bt^2$.
\end{proof}

\section{Proof of Lemma \ref{lem:unit-rad}}
\label{app:proof-lem-unit-rad}

\begin{proof}
    Let $f_1 = \sum u_i s_i \mu_i$ and $f_2 = \sum v_i s_i \mu_i$.
    Let $\vu = [u_1,u_2,\cdots]$ and $\vv = [v_1,v_2,\cdots]$.
    Then, $\norm{\vu}_2 \le 1$ and $\norm{\vv}_2 \le 1$.
    For any $S = \oset{x_1,\cdots,x_n}$, let $\mM(x) = [s_1 \mu_1(x),s_2\mu_2(x)]$ and $\mM_j = \mM(x_j)$. Then, we have
    \begin{align*}
        \hrad_S(\gF) & \le \E_{\sigma_1,\cdots,\sigma_n} \brac{ \sup_{\norm{\vu}_2 \le 1, \norm{\vv}_2 \le 1} \abs{  \frac{1}{n} \sum_{j=1}^n \sigma_j \vu^{\top} \mM_j \mM_j^{\top} \vv  }  } \\ 
        & \le \frac{1}{n}   \E_{\sigma_1,\cdots,\sigma_n} \brac{ \norm{ \sum_{j=1}^n \sigma_j \mM_j \mM_j^{\top} }_2 } \le \frac{1}{n}   \E_{\sigma_1,\cdots,\sigma_n} \brac{ \norm{ \sum_{j=1}^n \sigma_j \mM_j \mM_j^{\top} }_F }  \\ 
        & = \frac{1}{n} \E_{\sigma_1,\cdots,\sigma_n} \brac{ \Tr \oset{ \paren{\sum_{j=1}^n \sigma_j \mM_j \mM_j^{\top}}^{\top} \paren{ \sum_{l=1}^n \sigma_l \mM_l \mM_l^{\top} } }^{1/2}  } \\ 
        & \le \frac{1}{n} \sqrt{\E_{\sigma_1,\cdots,\sigma_n} \brac{\Tr\oset{ \sum_{j,l=1}^n \sigma_j \sigma_l \mM_j \mM_j^{\top} \mM_l \mM_l^{\top} }} } \qquad (\text{Jensen}) \\ 
        & = \frac{1}{n} \sqrt{\Tr\oset{ \sum_{j,l=1}^n \E\brac{\sigma_j \sigma_l} \mM_j \mM_j^{\top} \mM_l \mM_l^{\top} }}  = \frac{1}{n} \sqrt{\Tr \oset{ \sum_{j=1}^n \mM_j \mM_j^{\top} \mM_j \mM_j^{\top} }} \le \frac{1}{n} \sqrt{n \bt^4}   .
    \end{align*}
    Here the first line is not equality because of the absolute value. Since this holds for any $S$, we have $\rad_n(\gF) = \E_s[\hrad_S(\gF)] \le \frac{\bt^2}{\sqrt{n}}$.

    By \cite[Theorem~4.10]{Wainwright_2019}, for any $\delta \in (0,1)$, with probability at least $1-\delta$, both of the following hold simultaneously for any $f \in \gF$:
    \begin{equation*}
    \left \{ 
\begin{aligned}
   &  \abs{ \frac{1}{n} \sum_{i=1}^n f(\tilde{x}_i) - \E_{X \sim \px}[f(X)] } \le \frac{\bt^2}{\sqrt{n}} \paren{2 + \sqrt{2 \log \frac{2}{\delta}}}  ; \\ 
   &  \abs{ \frac{1}{n} \sum_{i=1}^m f(x_i) - \E_{X \sim \px}[f(X)]  } \le \frac{\bt^2}{\sqrt{m}} \paren{2 + \sqrt{2 \log \frac{2}{\delta}}}  .
\end{aligned}
    \right . 
    \end{equation*}
    For any unit vector $\vu \in \R^d$, let $f_\vu(x) = \vu^{\top}\Phi(x)$. Then, $\norm{f_\vu}_{\hk} = 1$, so $f_\vu^2 \in \gF$.
    Moreover, we have
    \begin{equation*}
        \sum_{i=1}^m f_\vu(x_i)^2 = \norm{\mG [\vv_1,\cdots,\vv_d] \vu}_2^2 = \sum_{i=1}^d \lambda_i u_i^2 \ge \lambda_d .
    \end{equation*}
    Thus, using the above inequalities, assuming that $m \ge n$, we obtain
    \begin{equation*}
        \frac{1}{n} \sum_{j=1}^n f(\tilde{x}_j)^2 \ge \frac{\lambda_d}{m} - \frac{\bt^2}{\sqrt{n}} \paren{4 + 2 \sqrt{2 \log \frac{2}{\delta}}}  ,
    \end{equation*}
    which implies the result since $\norm{\mPhi \vu}_2^2 = \sum_{j=1}^n f(\tilde{x}_j)^2$.
\end{proof}

\section{Proof of Corollary \ref{cor:est-1}}
\label{app:proof-lem-est-1}
\begin{proof}
    Denote $F = f^* - f_\Phi$. Then, $\vy - \vy_\Phi = [F(\tilde{x}_1),\cdots,F(\tilde{x}_n)]$, and $\frac{F^2}{\norm{F}_{\hk}^2} \in \gF$, where $\gF$ was defined in \Cref{lem:unit-rad}.
    Therefore, by \Cref{lem:unit-rad}, we have
    \begin{equation*}
        \frac{1}{n} \sum_{j=1}^n F(\tilde{x}_j) \le \E_{X \sim \px} [F(X)^2] + \norm{F}_{\hk}^2 \frac{\bt^2}{\sqrt{n}} \paren{2 + \sqrt{2 \log \frac{2}{\delta}}},
    \end{equation*}
    as desired.
\end{proof}

\section{Proof of Theorem \ref{thm:stkr-generality}}
\label{app:proof-thm-stkr-generality}

\begin{proof}
    Let $\mu_1,\mu_2,\cdots$ be the eigenfunctions of $T_k$.
    First, let us show that $\gH_t$ must be an RKHS.
    since $\mu_1$ is the common top-$1$ eigenfunction of $T_{k^p}$ for all $p \ge 1$, we have $r_{k^p}(\mu_1) \ge r_{k^p}(f)$ for all $f \in \hk$.
    By the condition of preserving relative smoothness, this implies that for all $f \in \gH_t \subset \hk$, we have $r_t(\mu_1) \ge r_t(f)$.
    Let $C_0 =r_t(\mu_1)$. Then, for any $f \in \gH_t$, we have $\norm{f}_{\px} \le \sqrt{C_0} \norm{f}_{\gH_t}$.
    In other words, $\norm{\cdot}_{\gH_t}$ is a stronger norm than $\norm{\cdot}_{\px}$ on $\gH_t$.
    Thus, for any sequence $(h_i) \in \gH_t$ such that $\norm{h_i - h}_{\hk} \rightarrow 0$: first, we have $h \in \hk$ because $\hk$ is a Hilbert space; second, we have $\norm{h_i - h}_{\px} \rightarrow 0$.
    Similarly, if $\norm{h_i - h'}_{\hk} \rightarrow 0$, then $\norm{h_i - h'}_{\px} \rightarrow 0$.

    Consider the inclusion map $I: \gH_t \rightarrow \hk$, where $I h = h$.
    For any sequence $(h_i) \in \gH_t$ such that $\norm{h_i - h}_{\hk} \rightarrow 0$ and $\norm{h_i - h'}_{\hk} \rightarrow 0$, $h_i$ converges to both $h$ and $h'$ under $\norm{\cdot}_{\px}$, so we must have $h' = h = I h$.
    This means that the graph of $I$ is closed, so the closed graph theorem \cite[Chapter~2]{brezis2011functional} guarantees that $I$ must be a bounded operator, meaning that there exists a constant $C$ such that $\norm{f}_{\hk} \le C \norm{f}_{\gH_t}$ for all $f \in \gH_t$.

    Let $\delta_x:f \mapsto f(x)$ be the evaluation functional at point $x$.
    Since $\hk$ is an RKHS, there exists a constant $M_x > 0$ such that $|f(x)| \le M_x \norm{f}_{\hk}$ for all $f \in \hk$. 
    Thus, for any $f \in \gH_t \subset \hk$, we have $|f(x)| \le M_x \norm{f}_{\hk} \le M_x C \norm{f}_{\gH_t}$. Thus, by \Cref{prop:rkhs-operator}, $\gH_t$ is also an RKHS.
    Let $\ks$ be the reproducing kernel of $\gH_t$.
    From now on, we will use $\hks$ to denote $\gH_t$.

    Second, we prove by induction that $\mu_1,\cdots,\mu_d$ are the top-$d$ eigenfunctions of $T_{\ks}$, and $\dotp{\mu_i, \mu_j}_{\hks} = 0$ for any $i \neq j$.
    We have already shown that $\mu_1$ maximizes $r_{\hks}(f)$ over all $f \in \hks$.
    Thus, $\mu_1$ must be the top-$1$ eigenfunction of $\hks$.
    Suppose $d \ge 2$, and $\mu_1,\cdots,\mu_{d-1}$ are the top-$(d-1)$ eigenfunctions and are orthogonal to each other in $\hks$.
    Let $\gH_0 = \sset{h}{\forall i \in [d-1]: \dotp{h, \mu_i}_{\px} = 0}$.
    Obviously, $\gH_0 \cap \gH_{k^p}$ is a closed subspace of $\gH_{k^p}$ for any $p \ge 1$.
    Moreover, for any $f \in \gH_0 \cap \hks$ and any $i \in [d-1]$, we have $\dotp{f,\mu_i}_{\hks} = s_i^{-1} \dotp{f,\mu_i}_{\px} = 0$, where $s_i$ is the eigenvalue of $T_{\ks}$ corresponding to $\mu_i$.
    Thus, $\hks \cap \gH_0$ is a closed subspace of $\hks$.
    By the condition of preserving relative smoothness, $\mu_d$ maximizes $r_{\hks}(f)$ over $f \in \hks \cap \gH_0$. Thus, $\mu_d$ is the $d$-th eigenfunction of $T_{\ks}$, and is orthogonal to $\mu_1,\cdots,\mu_{d-1}$ in $\hks$.

    Third, we prove by contradiction that $s_i \le M \lambda_i$ for all $i$.
    If this is false, then obviously one can find $t_1 < t_2 < \cdots$ such that $s_{t_i} \ge 1 \cdot \lambda_{t_i}$ for all $i$.
    Consider $f = \sum_{i=1}^{\infty} \sqrt{i^{-1} \cdot \lambda_{t_i}} \mu_{t_i} $.
    Then, $\norm{f}_{\hk}^2 = \sum_i i^{-1} = +\infty$.
    Since $\hks \subset \hk$, this implies that $\norm{f}_{\hks}^2 = +\infty = \sum_i \frac{\lambda_{t_i}}{i \cdot s_{t_i}} \le \sum_i \frac{1}{i^2} < +\infty$, which is a contradiction.

    Fourth, we find a function $s(\lambda)$ that satisfies the conditions in the theorem to interpolate $(\lambda_i, s_i)$ for all $i$. 
    We first point out that we can without loss of generality assume that $\lambda_i < 2 \lambda_{i+1}$ for all $i$:
If there is an $i$ that does not satisfy this condition, we simply insert some new $\lambda$'s between $\lambda_i$ and $\lambda_{i+1}$,
whose corresponding $s$'s are the linear interpolations between $s_i$ and $s_{i+1}$, so that $s_i \le M \lambda_i$ still holds.
With this assumption, it suffices to construct a series of bump functions $\{ f_i \}_{i=1}^{\infty}$,
    where $f_i \equiv 0$ if $\lambda_i = \lambda_{i+1}$;
    otherwise, $f_i(\lambda) = s_i - s_{i+1}$ for $\lambda \ge \lambda_i$ and $f_i(\lambda) = 0$ for $\lambda \le \lambda_{i+1}$.
    Such bump functions are $C^{\infty}$ and monotonically non-decreasing.
    Then, define $s(\lambda) = \sum_i f_i(\lambda)$ for $\lambda > 0$, and $s(0) = 0$.
    This sum of bump functions converges everywhere on $(0, +\infty)$, since it is a finite sum locally everywhere.
    Clearly this $s$ is monotonic, interpolates all the points, continuous on $[0, +\infty)$ and $C^{\infty}$ on $(0, +\infty)$.
    And for all $\lambda$ that is not $\lambda_i$, for instance $\lambda \in (\lambda_{i+1}, \lambda_i)$,
    there is $s(\lambda) \le s(\lambda_i) \le M \lambda_i \le 2M \lambda_{i+1} \le 2M \lambda$.
    Thus, $s(\lambda) = O(\lambda)$ for $\lambda \in [0, +\infty)$. 
\end{proof}

\begin{remark}
In general, we cannot guarantee that $s(\lambda)$ is differentiable at $\lambda = 0$.
Here is a counterexample: $\lambda_i = 3^{-i}$, and $s_i = 3^{-i}$ if $i$ is odd and $2 \cdot 3^{-i}$ if $i$ is even.
Were $s(\lambda)$ to be differentiable at $\lambda = 0$, its derivative would be $1$ and also would be $2$, a contradiction.
\end{remark}

\section{Proof of Theorem \ref{thm:rich-inverse-time}}
\label{app:proof-thm-rich-inverse-time}

\begin{proof}
Let $\hat{\lambda}_1 \ge \cdots \ge \hat{\lambda}_{m+n}$ be the eigenvalues of $\frac{\mG_k}{m+n}$.
It is easy to show that $\mQ$ has the same eigenvectors as $\frac{\mG_k}{m+n}$, with eigenvalues $g(\hat{\lambda}_1),\cdots,g(\hat{\lambda}_2)$.
By \Cref{lem:shawe-taylor} and Borel-Cantelli lemma, as $n \rightarrow \infty$, $\hat{\lambda}_1 \xrightarrow{\textit{a.s.}} \lambda_1$.
For simplicity, let us assume that $\lambda$ is slightly larger than $\lambda_1$, so almost surely there is $\hat{\lambda}_1 \le \lambda$.
Then, all eigenvalues of $\mQ$ are in $[\rho_{\min}, \rho_{\max}]$.

The first part of this proof is to bound $\norm{\vu_t}_2$, where $\vu_t := (m+n) \tilde{\mI}_n \paren{\frac{\mG_k}{m+n}}^{r} (\vtheta_* - \vtheta_t)$.
Let $\vtheta_t$ be the $\vtheta$ at iteration $t$, and $\vtheta_*$ be the optimal solution. Since $\vtheta_0 = \vzero$, we have
\begin{equation}
\label{eqn:proof-richardson-1}
\begin{aligned}
    \vtheta_* - \vtheta_t = & \brac{\paren{\mI_{m+n} - \gamma \brac{(m+n) \tilde{\mI}_n \paren{\frac{\mG_k}{m+n}}^{r} + n \beta_n \mQ }} \vtheta_* + \gamma \tilde{\vy} } \\
&  - \brac{\paren{\mI_{m+n} - \gamma \brac{(m+n) \tilde{\mI}_n \paren{\frac{\mG_k}{m+n}}^{r} + n \beta_n \mQ }} \vtheta_{t-1} + \gamma \tilde{\vy} } \\ 
= & \paren{\mI_{m+n} - \gamma \brac{(m+n) \tilde{\mI}_n \paren{\frac{\mG_k}{m+n}}^{r} + n \beta_n \mQ }} \paren{ \vtheta_* - \vtheta_{t-1} } \\ 
= & \paren{\mI_{m+n} - \gamma \brac{(m+n) \tilde{\mI}_n \paren{\frac{\mG_k}{m+n}}^{r} + n \beta_n \mQ }}^t \vtheta_*  .
\end{aligned}
\end{equation}

Note that
\begin{align*}
& \paren{\frac{\mG_k}{m+n}}^{r/2} \paren{\mI_{m+n} - \gamma \brac{(m+n) \tilde{\mI}_n \paren{\frac{\mG_k}{m+n}}^{r} + n \beta_n \mQ }} \\
= \; & \paren{\mI_{m+n} - \gamma \brac{(m+n) \paren{\frac{\mG_k}{m+n}}^{r/2} \tilde{\mI}_n \paren{\frac{\mG_k}{m+n}}^{r/2} + n \beta_n \mQ }} \paren{\frac{\mG_k}{m+n}}^{r/2}  .
\end{align*}

Thus, by propagating $\paren{\frac{\mG_k}{m+n}}^{r/2}$ from left to right, we get
\[
\paren{\frac{\mG_k}{m+n}}^{r/2} (\vtheta_* - \vtheta_t) = \paren{\mI_{m+n} - \gamma \mR }^t \paren{\frac{\mG_k}{m+n}}^{r/2}  \vtheta_*  ,
\]
where $\mR := (m+n) \paren{\frac{\mG_k}{m+n}}^{r/2} \tilde{\mI}_n \paren{\frac{\mG_k}{m+n}}^{r/2} + n \beta_n \mQ$ is a \psd{} matrix.
Denote the smallest and largest eigenvalues of $\mR$ by $\tilde{\lambda}_{\min}$ and $\tilde{\lambda}_{\max}$.
Then, $\tilde{\lambda}_{\min} \ge n \beta_n \rho_{\min}$.
In terms of $\tilde{\lambda}_{\max}$, we have
\begin{equation*}
    (m+n) \paren{\frac{\mG_k}{m+n}}^{r/2} \tilde{\mI}_n \paren{\frac{\mG_k}{m+n}}^{r/2}  = \paren{\frac{\mG_k}{m+n}}^{\frac{r - 1}{2} } \paren{\mG_k^{\frac{1}{2}} \tilde{\mI}_n \mG_k^{\frac{1}{2}}} \paren{\frac{\mG_k}{m+n}}^{\frac{r - 1}{2}}  .
\end{equation*}

By Sylvester's theorem, all non-zero eigenvalues of $\mG_k^{\frac{1}{2}} \tilde{\mI}_n \mG_k^{\frac{1}{2}}$ are the eigenvalues of $\tilde{\mI}_n \mG_k \tilde{\mI}_n$,
\ie{} the non-zero eigenvalues of $\mG_{K,n}$.
By \Cref{lem:shawe-taylor}, $\frac{1}{n} \norm{\mG_{K,n}}_2 \xrightarrow{\textit{a.s.}} \lambda_1$,
so suppose $\norm{\mG_{K,n}}_2 \le n \lambda$.
Then, $\tilde{\lambda}_{\max} \le n \lambda^r + n \beta_n \rho_{\max}$.

Since $\mM \vtheta_* = \tilde{\vy}$, and $\paren{\frac{\mG_k}{m+n}}^{r/2} \mM = \mR \paren{\frac{\mG_k}{m+n}}^{r/2}$, we have $\mR \paren{\frac{\mG_k}{m+n}}^{r/2} \vtheta_* = \paren{\frac{\mG_k}{m+n}}^{r/2} \tilde{\vy}$.
Note that $\mR(\mI_{m+n} - \gamma \mR) = (\mI_{m+n} - \gamma \mR) \mR$.
Thus, we have
\begin{align*}
\paren{\frac{\mG_k}{m+n}}^{r/2} (\vtheta_* - \vtheta_t) & = \paren{\mI_{m+n} - \gamma \mR }^t \paren{\frac{\mG_k}{m+n}}^{r/2}  \vtheta_* \\ 
& = \mR^{-1} (\mI_{m+n} - \gamma \mR)^t \mR \paren{\frac{\mG_k}{m+n}}^{r/2}  \vtheta_* \\ 
& = \mR^{-1} (\mI_{m+n} - \gamma \mR)^t \paren{\frac{\mG_k}{m+n}}^{r/2} \tilde{\vy}  .
\end{align*}

Now we bound $\| \vu_t \|_2$.
First, note that for any matrices $\mA, \mB \in \R^{d \times d}$ where $\mB$ is \psd{},
there is $\vu^{\top} \mA^{\top} \mB \mA \vu \le \| \mB \|_2 \| \mA \vu \|_2^2 \le \| \mB \|_2 \| \mA^{\top} \mA \|_2 \| \vu \|_2^2$ for any $\vu \in \R^d$,
so $\| \mA^{\top} \mB \mA \|_2 \le \| \mB \|_2 \| \mA^{\top} \mA \|_2$.
Second, note that the last $m$ elements of $\tilde{\vy}$ are zeros,
which means that $\tilde{\vy} = \tilde{\mI}_n \tilde{\vy}$. Thus, we have
\begin{align*}
\| \vu_t \|_2 & = \norm{ (m+n) \tilde{\mI}_n \paren{\frac{\mG_k}{m+n}}^r (\vtheta_* - \vtheta_t) }_2 \\
& = \norm{(m+n) \tilde{\mI}_n \paren{\frac{\mG_k}{m+n}}^{r/2}  \mR^{-1}  \paren{\mI_{m+n} - \gamma \mR }^t \paren{\frac{\mG_k}{m+n}}^{r/2}  \tilde{\vy} }_2 \\ 
& = \norm{(m+n) \tilde{\mI}_n \paren{\frac{\mG_k}{m+n}}^{r/2}  \paren{\mI_{m+n} - \gamma \mR }^{t/2} \mR^{-1} \paren{\mI_{m+n} - \gamma \mR }^{t/2} \paren{\frac{\mG_k}{m+n}}^{r/2} \tilde{\mI}_n \tilde{\vy} }_2 \\ 
& \le \norm{\paren{\mI_{m+n} - \gamma \mR }^{t/2} \mR^{-1} \paren{\mI_{m+n} - \gamma \mR }^{t/2}}_2 \norm{(m+n) \tilde{\mI}_n \paren{\frac{\mG_k}{m+n}}^r \tilde{\mI}_n }_2 \norm{\tilde{\vy}}_2 \\
& \le \frac{1}{\tilde{\lambda}_{\min}} \norm{\mI_{m+n} - \gamma \mR }_2^t (n \lambda_1^r) \| \vy \|_2 ,
\end{align*}
where the last step is because we have already proved $\norm{(m+n) \tilde{\mI}_n \paren{\frac{\mG_k}{m+n}}^r \tilde{\mI}_n }_2 \le n \lambda_1^r$.

Now, for $\gamma = \frac{1}{n \lambda^r }$, when $n$ is sufficiently large it is less than $\frac{2}{ \tilde{\lambda}_{\max} + \tilde{\lambda}_{\min} }$, because $\beta_n = o(1)$.
Thus, $\norm{\mI_{m+n} - \gamma \mR }_2 \le 1 - \frac{\tilde{\lambda}_{\min}}{n \lambda^r} \le 1 - \frac{\beta_n \rho_{\min} }{\lambda^r} $.
Thus, we have

\[
\| \vu_t \|_2 \le \paren{1 - \frac{\beta_n \rho_{\min} }{\lambda^r} }^t   \frac{\lambda^r}{\beta_n \rho_{\min} } \| \vy \|_2  .
\]

The second part of this proof is to bound $\| \mQ(\vtheta_* - \vtheta_t)  \|_2$.
Let us return to \cref{eqn:proof-richardson-1}, which says that
\begin{align*}
    \norm{ \mQ( \vtheta_* - \vtheta_{t+1} ) }_2 & = \norm{ \paren{ \mI_{m+n} - \gamma n \beta_n \mQ } \mQ (\vtheta_* - \vtheta_t ) - \gamma \mQ \vu_t }_2 \\
    & \le \paren{ 1 - \frac{\beta_n \rho_{\min}}{\lambda^r} } \norm{ \mQ (\vtheta_* - \vtheta_{t}) }_2 + \frac{\rho_{\max}}{n \lambda^r} \| \vu_t \|_2  .
\end{align*}
Here again, we assume that $n$ is large enough so that $\lambda^r > \beta_n \rho_{\min}$. This implies that
\begin{align*}
& \norm{\mQ (\vtheta_* - \vtheta_{t+1}) }_2 - t \paren{1 - \frac{\beta_n \rho_{\min} }{\lambda^r} }^t \frac{ \rho_{\max} \| \vy \|_2}{n \beta_n \rho_{\min}} \\
\le \; & \paren{1 - \frac{\beta_n \rho_{\min} }{\lambda^r} } \brac{ \norm{\mQ(\vtheta_* - \vtheta_{t})}_2 - (t-1) \paren{1 - \frac{\beta_n \rho_{\min} }{\lambda^r} }^{t-1} \frac{ \rho_{\max} \| \vy \|_2}{n \beta_n \rho_{\min}} } \\ 
\le \; & \cdots \le \paren{1 - \frac{\beta_n \rho_{\min} }{\lambda^r} }^t \brac{ \paren{1 - \frac{\beta_n \rho_{\min} }{\lambda^r} } \| \mQ \vtheta_* \|_2 + \frac{ \rho_{\max} \| \vy \|_2}{n \beta_n \rho_{\min}} }  .
\end{align*}
Thus, there is $\norm{ \mQ ( \vtheta_* - \vtheta_{t} ) }_2 \le \paren{1 - \frac{\beta_n \rho_{\min} }{\lambda^r} }^t \| \mQ \vtheta_* \|_2 + t \paren{1 - \frac{\beta_n \rho_{\min} }{\lambda^r} }^{t-1}  \frac{ \rho_{\max} \| \vy \|_2}{n \beta_n \rho_{\min}} $.
Using $1-x \le e^{-x}$, we have
\[
\norm{ \mQ ( \vtheta_* - \vtheta_{t} ) }_2 \le \exp \paren{ - \frac{\beta_n \rho_{\min} t }{\lambda^r} } \| \mQ \vtheta_* \|_2 + t \exp \paren{ - \frac{\beta_n \rho_{\min} (t-1) }{\lambda^r} } \frac{ \rho_{\max} \| \vy \|_2}{n \beta_n \rho_{\min}}  .
\]
When $t = t_0 :=  \frac{4 \lambda^r}{\beta_n \rho_{\min}} \log \frac{2 \lambda^r \rho_{\max } \| \vy \|_2 }{ n \beta_n^2 \rho_{\min}^2 \| \mQ \vtheta_* \|_2 }  $, by $\log (2x) \le x$ for $x > 0$, we have
\[
\exp \paren{  \frac{\beta_n \rho_{\min}  }{\lambda^r}  \frac{t}{2} } \ge \paren{\frac{2 \lambda^r \rho_{\max } \| \vy \|_2 }{ n \beta_n^2 \rho_{\min}^2 \| \mQ \vtheta_* \|_2 }}^2 \ge \frac{4 \lambda^r \rho_{\max} \| \vy \|_2 }{ n \beta_n^2 \rho_{\min}^2 \| \mQ \vtheta_* \|_2 } \log \paren{ \frac{2 \lambda^r \rho_{\max} \| \vy \|_2 }{ n \beta_n^2 \rho_{\min}^2 \| \mQ \vtheta_* \|_2 } } .
\]
Let $F(t) := \exp \paren{ \frac{\beta_n \rho_{\min}}{2 \lambda^r} t } - \frac{\rho_{\max } \| \vy \|_2 }{ n \beta_n \rho_{\min} \| \mQ \vtheta_* \|_2  } t $. Then we have $F(t_0) \ge 0$.
And it is easy to show that for all $t \ge \frac{t_0}{2}$, there is $F'(t) \ge 0$.
This means that when $t \ge t_0$, there is $F(t) \ge 0$, so we have
\[
\norm{ \mQ ( \vtheta_* - \vtheta_{t} ) }_2 \le \exp \paren{ - \frac{\beta_n \rho_{\min} t }{\lambda^r} } \| \mQ \vtheta_* \|_2 + \exp \paren{ - \frac{\beta_n \rho_{\min}  }{\lambda^r} \paren{\frac{t}{2} - 1} } \| \mQ \vtheta_* \|_2  .
\]
Hence, when $t \ge \max \oset{ \frac{2 \lambda^r}{\beta_n \rho_{\min}} \log \frac{2}{\epsilon} + 2 , t_0 }$, we have $\norm{ \mQ ( \vtheta_* - \vtheta_{t} ) }_2 \le \epsilon \| \mQ \vtheta_* \|_2$, which implies that the relative estimation error of $\hat{\valpha}$ is less than $\epsilon$. 
\end{proof}

\section{Proof of Theorem \ref{thm:approx-error-stkr}}
\label{app:proof-thm-approx-error-stkr}
\begin{proof}
    Let us look at the three conditions used in \Cref{thm:fischer-1}.
    \begin{itemize}
        \item \textbf{Eigenvalue decay (EVD):} This is a condition of the theorem.
        \item \textbf{Embedding condition (EMB):} For any $f = \sum u_i \mu_i \in \hk$, for $\px$-almost all $x$ we have $f(x)^2 = \paren{\sum u_i \mu_i(x)}^2 \le \paren{\sum \frac{u_i^2}{\lambda_i}} \paren{\sum \lambda_i \mu_i(x)^2} \le \norm{f}_{\hk}^2 \bt^2 \le \bt^2 M \norm{f}_{\hks}^2$. Thus, EMB holds with $c_2 = \bt \sqrt{M}$.
        \item \textbf{Source condition (SRC):} This holds with $c_3 = \sqrt{\epsilon} \norm{f^*}_{\px}$.
    \end{itemize}
    Thus, the theorem can be proved using the result in \cite{fischer2020sobolev}.
\end{proof}

\section{Proof of Theorem \ref{thm:estimat-error-stkr}}
\label{app:thm-estimat-error-stkr}

\begin{proposition}
\label{prop:kp-gap-1}
For any \psd{} matrices $\mA, \mB \in \R^{d \times d}$, we have $\Tr(\mA \mB) \le \norm{\mA}_2 \Tr(\mB)$.
\end{proposition}
\begin{proof}
    An elementary proof can be found at \url{https://math.stackexchange.com/questions/2241879/reference-for-trace-norm-inequality}.
\end{proof}

\begin{lemma}
\label{lem:kp-gap}
For any $\delta \in (0,1)$, with probability at least $1-\delta$ the following holds for all $p \ge 1$:
\begin{equation*}
    \abs{ \hatkp(x,x_j) - k^p(x,x_j) } \le \paren{p-1} \lambda_{\max}^{p-2} \frac{\bt^4}{\sqrt{m+n}} \paren{2 + \sqrt{2 \log \frac{1}{\delta}}} \quad \text{for all } j \in [m+n], x \in \gX  ,
\end{equation*}
which implies that
\begin{equation*}
    \abs{ \hatks(x,x_j) - \ks(x,x_j) } \le \left . \nabla_\lambda \paren{\frac{s(\lambda)}{\lambda}} \right |_{\lambda = \lambda_{\max}} \frac{\bt^4}{\sqrt{m+n}}\paren{2 + \sqrt{2 \log \frac{1}{\delta}}}  .
\end{equation*}
\end{lemma}
\begin{proof}
    For any $x' \in \gX$ and any $p \ge 1$, we have
    \begin{equation*}
        \norm{k^p(\cdot,x')}_{\hk}^2 = \norm{\sum_i \lambda_i^p \mu_i(x') \mu_i(\cdot)}_{\hk}^2 = \sum_i \frac{\lambda_i^{2p} \mu_i(x')^2 }{\lambda_i} \le \lambda_1^{2p-2} \bt^2  .
    \end{equation*}
    Let $F_p(x) = \vu^{\top} \paren{\frac{\mG_k}{m+n}}^p \vv_k(x)$, where $\vu \in \R^{m+n}$ satisfies $\norm{\vu}_1 \le 1$.
    Since $\dotp{k(x_i,\cdot), k(x_j,\cdot)}_{\hk} = k(x_i,x_j)$, we have $\dotp{\vv_k, \vv_k}_{\hk} = \mG_k$. Thus,
    \begin{equation*}
        \norm{F_p}_{\hk}^2 = \dotp{ \vu^{\top} \paren{\frac{\mG_k}{m+n}}^p \vv_k, \vu^{\top} \paren{\frac{\mG_k}{m+n}}^p \vv_k }_{\hk} = \vu^{\top} \frac{\mG_k^{2p+1}}{(m+n)^{2p}} \vu  . 
    \end{equation*}
    Since $\mG_k$ is \psd{}, we can define $\mG_k^{1/2}$. 
    Using \Cref{prop:kp-gap-1}, we have
\begin{align*}
    \norm{F_p}_{\hk}^2 & = \vu^{\top} \frac{\mG_k^{2p+1}}{(m+n)^{2p}} \vu = \Tr \oset{ \vu^{\top} \mG_k^{1/2} \paren{\frac{\mG_k}{m+n}}^{2p} \mG_k^{1/2} \vu } \\ 
    & = \Tr \oset{ \paren{\frac{\mG_k}{m+n}}^{2p} \mG_k^{1/2} \vu \vu^{\top} \mG_k^{1/2} } \le \hat{\lambda}_1^{2p} \Tr \paren{\mG_k^{1/2} \vu \vu^{\top} \mG_k}  .
\end{align*}
Moreover, $\Tr \paren{\mG_k^{1/2} \vu \vu^{\top} \mG_k} = \vu^{\top} \mG_k \vu = \sum_{i,j=1}^{m+n} u_i u_j k(x_i,x_j) \le \sum_{i,j=1}^{m+n}\abs{u_i u_j k(x_i,x_j)} \le \bt^2 \norm{\vu}_1^2 \le \bt^2$. Thus, we have $\norm{F_p}_{\hk} \le \hat{\lambda}_1^p \bt$ for all $p \ge 1$.

For any $p \ge 1$, define $\vv_{k^p}(x) \in \R^{m+n}$ as $\vv_{k^p}(x)[i] = k^p(x,x_i)$ for $i \in [m+n]$. Then, 
\begin{align*}
    & \abs{k^p(x,x_j) - \hatkp(x,x_j)} \\ 
    = \; & \abs{ k^p(x,x_j) - \frac{1}{(m+n)^{p-1}} \vv_k(x)^{\top} \mG_k^{p-2} \vv_k(x_j) } \\ 
    \le \; & \abs{ k^p(x,x_j) - \frac{1}{m+n} \vv_{k^{p-1}}(x)^{\top} \vv_k(x_j) } \\ 
    & + \sum_{q=1}^{p-2} \frac{1}{(m+n)^q} \abs{ \vv_{k^{p-q}}(x)^{\top} \mG_k^{q-1} \vv_k(x_j) - \vv_{k^{p-q-1}}(x)^{\top} \frac{\mG_k^q}{m+n} \vv_k(x_j) }   .
\end{align*}
Since $f(z) = k^{p-1}(x,z) k(x_j,z) \in \gF$, where $\gF$ was defined in \Cref{lem:unit-rad},
the first term can be bounded as
\begin{align*}
    & \abs{ k^p(x,x_j) - \frac{1}{m+n} \vv_{k^{p-1}}(x)^{\top} \vv_k(x_j) } \\ 
    = \; & \abs{ \int k^{p-1}(x,z) k(x_j,z) d \px(z) - \frac{1}{m+n} \sum_{i=1}^{m+n} k^{p-1}(x,x_i) k(x_j,x_i) } \\ 
    \le \; & \lambda_1^{p-2} \frac{\bt^4}{\sqrt{m+n}} \paren{2 + \sqrt{2 \log \frac{1}{\delta}}} ,
\end{align*}
where the last step uses $\norm{k^{p-1}(x,\cdot)}_{\hk} \le \lambda_1^{p-2} \bt$, and $\norm{k(x_j,\cdot)}_{\hk} \le \bt$.

For the second term, since $\vv_k(x_j) = \mG_k \ve_j$ where $\ve_j = [0,\cdots,0,1,0,\cdots,0]$, we have
\begin{align*}
  &   \sum_{q=1}^{p-2} \frac{1}{(m+n)^q} \abs{ \vv_{k^{p-q}}(x)^{\top} \mG_k^{q-1} \vv_k(x_j) - \vv_{k^{p-q-1}}(x)^{\top} \frac{\mG_k^q}{m+n} \vv_k(x_j) } \\ 
  = \; & \abs{ \int k^{p-q-1}(x,z) \brac{\ve_j^{\top} \paren{\frac{\mG_k}{m+n}}^q \vv_k(z)} d \px (z) - \frac{1}{m+n} \sum_{j=1}^{m+n} k^{p-q-1}(x,x_j)\brac{\ve_j^{\top} \paren{\frac{\mG_k}{m+n}}^q \vv_k(x_j)}  } \\ 
  \le \; & \lambda_1^{p-q-2} \hat{\lambda}_1^q \frac{\bt^4}{\sqrt{m+n}} \paren{2 + \sqrt{2 \log \frac{1}{\delta}}}  , 
\end{align*}
where the last step uses $\norm{k^{p-q-1}(x,\cdot)}_{\hk} \le \lambda_1^{p-q-2} \bt$, and $\norm{\ve_j^{\top} \paren{\frac{\mG_k}{m+n}}^q \vv_k}_{\hk} \le \hat{\lambda}_1^q \bt$ since $\norm{\ve_j}_1 = 1$.
Finally, note that $\nabla_\lambda \paren{\frac{s(\lambda)}{\lambda}} = \sum_{p=1}^{\infty} \pi_p (p-1) \lambda^{p-2}$. Combining all of the above yields the result.
\end{proof}

\begin{corollary}
\label{cor:kp-gap-2}
Under the settings of \Cref{lem:kp-gap}, we have
\begin{align*}
    & \left |  \kssquare(x_i,x_j) - \langle \hatks(x_i, \cdot), \ks(x_j, \cdot) \rangle_\px \right | + \left | \langle \hatks(x_i, \cdot), \hatks(x_j, \cdot) \rangle_\px - \langle \hatks(x_i, \cdot), \ks(x_j, \cdot) \rangle_\px \right | \\ 
    \le \; & 2 s(\lambda_{\max}) \left . \nabla_\lambda \paren{\frac{s(\lambda)}{\lambda}} \right |_{\lambda = \lambda_{\max}}  \frac{  \bt^4  }{\sqrt{m+n}}  \paren{ 2 + \sqrt{2 \log \frac{1}{\delta}} }    
\end{align*}
holds for all $i,j \in [m+n]$, where $\lambda_{\max} = \max \oset{\lambda_1, \hat{\lambda}_1}$.
\end{corollary}
\begin{proof}
    Let $F_{p,q}(x) = \vu^{\top} \paren{\frac{\mG_k}{m+n}}^p \vv_{k^q}(x)$
    for any $\| \vu \|_1 \le 1$ and any $p \ge 0 , q \ge 1$.
    By \Cref{prop:kp-gap-1}, we have
        \begin{align*}
        \| F_{p,q} \|_\hk^2 & = \vu^{\top} \paren{\frac{\mG_k}{m+n}}^p \mG_{k^{2q-1}} \paren{\frac{\mG_k}{m+n}}^p \vu \\ 
        & = \Tr \paren{ \paren{\frac{\mG_k}{m+n}}^{p - 1/2} \frac{\mG_{k^{2q-1}}}{m+n} \paren{\frac{\mG_k}{m+n}}^{p - 1/2} \mG_k^{1/2} \vu  \vu^{\top} \mG_k^{1/2} } \\ 
        & \le \hat{\lambda}_1^{2p - 1} \norm{\frac{\mG_{k^{2q-1}}}{m+n}}_2 \Tr \paren{\mG_k^{1/2} \vu  \vu^{\top} \mG_k^{1/2}} \\ 
        & = \hat{\lambda}_1^{2p - 1} \norm{\frac{\mG_{k^{2q-1}}}{m+n}}_2 \vu^{\top} \mG_k \vu \le \hat{\lambda}_1^{2p - 1} \norm{\frac{\mG_{k^{2q-1}}}{m+n}}_2 \bt^2  .
    \end{align*}

    For any unit vector $\vw \in \R^{m+n}$, we have
    \begin{equation*}
        \hat{\lambda}_1 \ge \vw^{\top} \frac{\mG_k}{m+n} \vw = \frac{1}{m+n} \sum_{i,j=1}^{m+n} w_i w_j K(x_i,x_j) = \frac{1}{m+n} \sum_t \lambda_t \vw^{\top} \mM_t \vw   ,
    \end{equation*}
    where $\mM_t \in \R^{(m+n) \times (m+n)}$ is defined as $\mM_t[i,j] = \mu_t(x_i) \mu_t(x_j)$.
    Thus, we have
    \begin{equation*}
    \vw^{\top} \frac{\mG_{k^{2q-1}}}{m+n} \vw = \frac{1}{m+n} \sum_t \lambda_t^{2q-1} \vw^{\top} \mM_t \vw \le \lambda_1^{2q-2}\frac{1}{m+n} \sum_t \lambda_t \vw^{\top} \mM_t \vw  \le \lambda_1^{2q-2} \hat{\lambda}_1 ,
    \end{equation*}
    which implies that $\norm{\frac{\mG_{k^{2q-1}}}{m+n}}_2 \le \lambda_1^{2q-2} \hat{\lambda}_1$. Thus, $\| F_{p,q} \|_\hk^2 \le \lambda_1^{2q-2} \hat{\lambda}_1^{2p} \bt^2$.

Note that $\langle \vv_k, \vv_k \rangle_\px = \mG_{k^2}$. So for any $p, q \ge 1$ and any $i, j \in [m+n]$, there is:
\begin{align*}
    & \left | k^{p+q}(x_i, x_j) - \dotp{\hat{k}^p(x_i, \cdot), k^q (x_j,\cdot)}_\px \right | = \left | \ve_i^{\top} \mG_{k^{p+q}} \ve_j - \ve_i^{\top} \frac{\mG_k^{p-1}}{(m+n)^{p-1}} \mG_{k^{q+1}} \ve_j  \right |  \\ 
    \le \; & \sum_{t=1}^{p-1} \left | \ve_i^{\top} \frac{\mG_k^{p-t}}{(m+n)^{p-t}} \mG_{k^{q+t}} \ve_j - \ve_i^{\top} \frac{\mG_k^{p-t-1}}{(m+n)^{p-t-1}} \mG_{k^{q+t+1}} \ve_j  \right | \\ 
    = \; &    \sum_{t=1}^{p-1}  \left | \frac{1}{m+n} \sum_{l=1}^{m+n} \brac{\ve_i^{\top} \paren{\frac{\mG_k}{m+n}}^{p-t-1} \vv_k}(x_l) \brac{\ve_j^{\top} \vv_{k^{q+t}}}(x_l) \right . \\
    & \left . -  \dotp{ \ve_i^{\top} \paren{\frac{\mG_k}{m+n}}^{p-t-1} \vv_k , \ve_j^{\top} \vv_{k^{q+t}} }_\px      \right |  \\
    \le \; & \sum_{t=1}^{p-1} \lambda_1^{q+t-1} \hat{\lambda}_1^{p-t-1}  \frac{\bt^4  }{\sqrt{m+n}}  \paren{ 2 + \sqrt{2 \log \frac{1}{\delta}} } \le (p-1) \lambda_{\max}^{p+q-2} \frac{\bt^4  }{\sqrt{m+n}}  \paren{ 2 + \sqrt{2 \log \frac{1}{\delta}} } .
\end{align*}

Thus, we have
\begin{align*}
    \left |  \kssquare(x_i,x_j) - \langle \hatks(x_i, \cdot), \ks(x_j, \cdot) \rangle_\px \right | & = \sum_{p,q=1}^{\infty} \left | \pi_p \pi_q \paren{k^{p+q}(x_i, x_j) - \dotp{\hat{k}^p(x_i, \cdot), k^q (x_j,\cdot)}_\px } \right | \\ 
    & \le \sum_{p,q=1}^{\infty} \pi_p \pi_q (p-1) \lambda_{\max}^{p+q-2}  \frac{\bt^4  }{\sqrt{m+n}}  \paren{ 2 + \sqrt{2 \log \frac{1}{\delta}} }  .
\end{align*}

Similarly, we can show that:
\begin{align*}
    & \left | \dotp{\hat{k}^p(x_i, \cdot), \hat{k}^q (x_j,\cdot)}_\px - \dotp{\hat{k}^p(x_i, \cdot), k^q (x_j,\cdot)}_\px \right | \\
    = \; & \left | \ve_i^{\top} \frac{\mG_k^{p-1}}{(m+n)^{p-1}} \mG_{k^2} \frac{\mG_k^{q-1}}{(m+n)^{q-1}} \ve_j -  \ve_i^{\top} \frac{\mG_k^{p-1}}{(m+n)^{p-1}} \mG_{k^{q+1}} \ve_j   \right | \\ 
    \le \; & \sum_{t=1}^{q-1} \left | \ve_i^{\top} \frac{\mG_k^{p-1}}{(m+n)^{p-1}} \mG_{k^{t+1}} \frac{\mG_k^{q-t}}{(m+n)^{q-t}} \ve_j -  \ve_i^{\top}  \frac{\mG_k^{p-1}}{(m+n)^{p-1}} \mG_{k^{t+2}} \frac{\mG_k^{q-t-1}}{(m+n)^{q-t-1}} \ve_j \right |  \\ 
    = \; & \sum_{t=1}^{q-1} \left | \frac{1}{m+n} \sum_{l=1}^{m+n} \brac{\ve_i^{\top} \paren{\frac{\mG_k}{m+n}}^{p-1} \vv_{k^{t+1}}}(x_l) \brac{\ve_j^{\top} \paren{\frac{\mG_k}{m+n}}^{q-t-1} \vv_k}(x_l) \right . \\
    & \left . - \dotp{ \ve_i^{\top} \paren{\frac{\mG_k}{m+n}}^{p-1} \vv_{k^{t+1}} , \ve_j^{\top} \paren{\frac{\mG_k}{m+n}}^{q-t-1} \vv_k  }_\px    \right |    \\
    \le \; & \sum_{t=1}^{q-1} \lambda_1^{t} \hat{\lambda}_1^{p+q-t-2}  \frac{\bt^4  }{\sqrt{m+n}}  \paren{ 2 + \sqrt{2 \log \frac{1}{\delta}} } \le  (q-1) \lambda_{\max}^{p+q-2}  \frac{\bt^4  }{\sqrt{m+n}}  \paren{ 2 + \sqrt{2 \log \frac{1}{\delta}} } ,
\end{align*}
which implies that
\begin{align*}
& \left | \langle \hatks(x_i, \cdot), \hatks(x_j, \cdot) \rangle_\px - \langle \hatks(x_i, \cdot), \ks(x_j, \cdot) \rangle_\px \right | \\
= \; & \sum_{p,q=1}^{\infty} \left | \pi_p \pi_q \paren{\dotp{\hat{k}^p(x_i, \cdot), \hat{k}^q (x_j,\cdot)}_\px - \dotp{\hat{k}^p(x_i, \cdot), k^q (x_j,\cdot)}_\px } \right |  \\ 
\le \; & \sum_{p,q=1}^{\infty} \pi_p \pi_q (q-1) \lambda_{\max}^{p+q-2}  \frac{\bt^4  }{\sqrt{m+n}}  \paren{ 2 + \sqrt{2 \log \frac{1}{\delta}} }  .
\end{align*}
Combining the above inequalities, we obtain
\begin{align*}
& \left |  \kssquare(x_i,x_j) - \langle \hatks(x_i, \cdot), \ks(x_j, \cdot) \rangle_\px \right | + \left | \langle \hatks(x_i, \cdot), \hatks(x_j, \cdot) \rangle_\px - \langle \hatks(x_i, \cdot), \ks(x_j, \cdot) \rangle_\px \right | \\ 
\le \; & \sum_{p,q=1}^{\infty} \pi_p \pi_q (p+q-2) \lambda_{\max}^{p+q-2}  \frac{\bt^4  }{\sqrt{m+n}}  \paren{ 2 + \sqrt{2 \log \frac{1}{\delta}} } \\ 
= \; & \lambda_{\max} \left . \nabla_\lambda \paren{\frac{s(\lambda)^2}{\lambda^2}} \right |_{\lambda = \lambda_{\max}}  \frac{ \bt^4  }{\sqrt{m+n}}  \paren{ 2 + \sqrt{2 \log \frac{1}{\delta}} }  ,
\end{align*}
so we get the result by expanding the derivative.
\end{proof}

We now prove \Cref{thm:estimat-error-stkr}.
\begin{proof}
    Define $\vv_{\ks, n} (x) \in \R^n$ such that $\vv_{\ks, n}(x)[i] = \ks(x, x_i)$.
Define $\vv_{\hatks, n}(x)$ similarly.
Recall the formulas $\tilde{f} = \tilde{\valpha}^{\top} \vv_{\ks, n}$ and $\hat{f} = \hat{\valpha}^{\top} \vv_{\hatks, n}$.
Define $f^\dag := \hat{\valpha}^{\top} \vv_{\ks, n}$.
Since $\mG_{\hatks, n}$ is \psd{}, we can see that $\| \hat{\valpha} \|_2 \le \frac{\| \vy \|_2}{n \beta_n}$,
and $\| \hat{\valpha} \|_1 \le \sqrt{n} \| \hat{\valpha} \|_2$.
So by \Cref{cor:kp-gap-2}, we have
\begin{align*}
    \norm{\hat{f} - f^\dag}_\px^2 & = \hat{\valpha}^{\top} \dotp{\vv_{\hatks, n} - \vv_{\ks, n}, \vv_{\hatks, n} - \vv_{\ks, n} }_\px \hat{\valpha} \\ 
    & = \hat{\valpha}^{\top} \paren{   \langle \hatks(x_i, \cdot), \hatks(x_j, \cdot) \rangle_\px +  \kssquare(x_i,x_j) - 2 \langle \hatks(x_i, \cdot), \ks(x_j, \cdot) \rangle_\px } \hat{\valpha} \\ 
    & \le   2 s(\lambda_{\max})  \left . \nabla_\lambda \paren{\frac{s(\lambda)}{\lambda}} \right |_{\lambda = \lambda_{\max}}  \frac{ \beta_n^{-2} \bt^4  }{\sqrt{m+n}}  \paren{ 2 + \sqrt{2 \log \frac{1}{\delta}} }  \frac{\| \vy \|_2^2}{n}  .
\end{align*}
By the definitions of $\tilde{\valpha}$ and $\hat{\valpha}$, we can also see that:
\begin{equation}
\label{eqn:proof-thm-approx-1}
\paren{ \mG_{\ks, n} + n \beta_n \mI_n } \paren{\hat{\valpha} - \tilde{\valpha}} = \paren{ \mG_{\ks, n} - \mG_{\hatks, n} } \hat{\valpha}  .    
\end{equation}
Note that $\norm{\mG_{\ks, n} - \mG_{\hatks, n}}_2 \le n \norm{\mG_{\ks, n} - \mG_{\hatks, n}}_{\max}$.
Here $\norm{\mM}_{\max} = \max \abs{\mM[i,j]}$.
Thus, we have
\begin{align*}
\norm{\tilde{f} - f^\dag}_\hks^2 & = \paren{\hat{\valpha} - \tilde{\valpha}}^{\top} \mG_{\ks, n} \paren{\hat{\valpha} - \tilde{\valpha}}  \\
& = \paren{\hat{\valpha} - \tilde{\valpha}}^{\top} \paren{ \mG_{\ks, n} - \mG_{\hatks, n} } \hat{\valpha} - n \beta_n \paren{\hat{\valpha} - \tilde{\valpha}}^{\top} \paren{\hat{\valpha} - \tilde{\valpha}} \\ 
& \le \norm{\hat{\valpha} }_2 \norm{\mG_{\ks, n} - \mG_{\hatks, n}}_2 \norm{\hat{\valpha}}_2  + \norm{\tilde{\valpha} }_2 \norm{\mG_{\ks, n} - \mG_{\hatks, n}}_2 \norm{\hat{\valpha}}_2  - 0 \\
& \le  2   \left . \nabla_\lambda \paren{\frac{s(\lambda)}{\lambda}} \right |_{\lambda = \lambda_{\max}} \frac{\beta_n^{-2} 
 \bt^4  }{\sqrt{m+n}}   \paren{ 2 + \sqrt{2 \log \frac{1}{\delta}} } \frac{\| \vy \|_2^2}{n}       .
\end{align*}
And note that we have $\norm{\tilde{f} - f^\dag}_\px^2 \le s(\lambda_1) \norm{\tilde{f} - f^\dag}_\hks^2 \le s(\lambda_{\max}) \norm{\tilde{f} - f^\dag}_\hks^2$. Thus,
\begin{align*}
\norm{\hat{f} - \tilde{f}}_\px^2 &  \le 2 \paren{ \norm{\hat{f} - f^\dag}_\px^2 + \norm{\tilde{f} - f^\dag}_\px^2 } \\
& \le 8 s(\lambda_{\max})  \left . \nabla_\lambda \paren{\frac{s(\lambda)}{\lambda}} \right |_{\lambda = \lambda_{\max}}  \frac{ \beta_n^{-2} \bt^4  }{\sqrt{m+n}}  \paren{ 2 + \sqrt{2 \log \frac{1}{\delta}} }  \frac{\| \vy \|_2^2}{n}  , 
\end{align*}
as desired.
\end{proof}

\chapter{Proofs for Chapter 6}
\section{Proof of Theorem \ref{thm:grw-reg-linear}}
\label{app:proof-thm-grw-reg-linear}

We need the following classical result in convex optimization.
\begin{theorem}[\cite{tibshirani_gd}, p.~16]
\label{thm:smooth}
Let $f$ be a convex and $L$-smooth function (\Cref{def:smooth}) on $\gD \subseteq \R^{\dx}$.
Suppose it has a unique finite minimizer $x^*$.
If one minimizes $f$ with gradient descent $x_{t+1} = x_t - \eta \nabla f(x_t)$, staring from $x_0$ with a fixed learning rate $\eta \le \frac{1}{L}$, then we have
\begin{equation*}
    f(x_T) \le f(x^*) + \frac{1}{\eta T} \norm{x_0 - x^*}_2^2 \qquad \text{for all } T > 0 .
\end{equation*}
\end{theorem}

Now let us prove \Cref{thm:grw-reg-linear}.
\begin{proof}
    \textbf{Static GRW.} \quad We first prove the result for static GRW where $q_i^{(t)} = q_i > 0$ for all $t$. Let $q^* = \min_i q_i$. The minimization objective is $F(\theta) = \sum_{i=1}^n q_i (x_i^{\top} \theta - y_i)^2$, whose Hessian is $\nabla_\theta^2 F(\theta) = 2 \sum_{i=1}^n q_i x_i x_i^{\top}$. Let $A = \sum_{i=1}^n \norm{x_i}_2^2$.
    Since $q_i \in [0,1]$, for any unit vector $\vv \in \R^{\dx}$, we have
    \begin{equation*}
        \vv^{\top} \nabla_\theta^2 F(\theta) \vv = 2 \sum_{i=1}^n q_i (x_i^{\top} \vv)^2 \le 2 \sum_{i=1}^n q_i \norm{x_i}_2^2 \le 2A ,
    \end{equation*}
    which by \Cref{def:smooth} implies that $F$ is $2A$-smooth. Thus, we have
    \begin{equation}
    \label{eqn:grw-reg-linear-1}
        F(\theta_2) \le F(\theta_1) + \dotp{\nabla_\theta F(\theta_1), \theta_2 - \theta_1} + A \norm{\theta_2 - \theta_1}_2^2 \qquad \text{for all } \theta_1, \theta_2 \in \R^{\dx} .
    \end{equation}
    Denote $g(\theta^{(t)}) = \mX^{\top} \theta^{(t)} - \mY \in \R^n$.
    Let $\sqrt{\mQ} = \diag \paren{\sqrt{q_1},\cdots,\sqrt{q_n}}$.
    Then, we have $ F(\theta^{(t)}) = \norm{\sqrt{\mQ} g(\theta^{(t)})}_2^2$,
    which implies that $\nabla F(\theta^{(t)}) = 2 \mX \mQ g(\theta^{(t)})$.
    The update rule of static GRW with gradient descent is thus given by
    \begin{equation*}
        \theta^{(t+1)} = \theta^{(t)} - \eta \sum_{i=1}^n q_i x_i \paren{f^{(t)}(x_i) - y_i} = \theta^{(t)} - \eta \mX \mQ g(\theta^{(t)})  .
    \end{equation*}
    By \Cref{eqn:grw-reg-linear-1}, we have
    \begin{equation*}
        F(\theta^{(t+1)}) \le F(\theta^{(t)}) - 2 \eta g(\theta^{(t)})^{\top} \mQ^{\top} \mX^{\top} \mX \mQ g(\theta^{(t)}) + A \norm{\eta \mX \mQ g(\theta^{(t)})}_2^2  .
    \end{equation*}
    Since $x_1,\cdots,x_n$ are linearly independent, $\mX^{\top} \mX$ is a positive definite matrix. Let its smallest eigenvalue be $\lambda_{\min} > 0$.
    Note that $\norm{\mQ g(\theta^{(t)})}_2 \ge \sqrt{q^*} \norm{\sqrt{\mQ}g(\theta^{(t)})}_2 = \sqrt{q^* F(\theta^{(t)})}$. Thus, we have $g(\theta^{(t)})^{\top} \mQ^{\top} \mX^{\top} \mX \mQ g(\theta^{(t)}) \ge q^* \lambda_{\min} F(\theta^{(t)})$. So we have
    \begin{align*}
        F(\theta^{(t+1)}) & \le F(\theta^{(t)}) - 2 \eta q^* \lambda_{\min} F(\theta^{(t)}) + A \eta^2 \norm{\mX \sqrt{\mQ}}_2^2 \norm{\sqrt{\mQ} g(\theta^{(t)})}_2^2 \\ 
        & \le F(\theta^{(t)}) - 2 \eta q^* \lambda_{\min} F(\theta^{(t)}) + A \eta^2 \norm{\mX \sqrt{\mQ}}_F^2 F(\theta^{(t)}) \\ 
        & \le F(\theta^{(t)}) - 2 \eta q^* \lambda_{\min} F(\theta^{(t)}) + A \eta^2 \norm{\mX}_F^2 F(\theta^{(t)}) \\ 
        & = \paren{ 1 - 2 \eta q^* \lambda_{\min} + A^2 \eta^2 } F(\theta^{(t)})  .
    \end{align*}
    Let $\eta_0 = \frac{q^* \lambda_{\min}}{A^2}$.
    For any $\eta \le \eta_0$, we have $F(\theta^{(t+1)}) \le (1 - \eta q^* \lambda_{\min}) F(\theta^{(t)})$ for all $t$, which implies that $F(\theta^{(t)})$ must converge to zero.
    Since every $q_i > 0$, this implies that the ERM risk must converge to zero.

    \textbf{Dynamic GRW.} \quad By \Cref{ass:qstar}, for any $\epsilon > 0$, there exists $t_\epsilon$ such that for all $t \ge t_\epsilon$, we have $q_i^{(t)} \in (q_i - \epsilon, q_i + \epsilon)$ for all $i$. Let $\lambda_{\max}$ and $\lambda_{\min}$ be the largest and smallest eigenvalues of $\mX^{\top} \mX$, where $\lambda_{\min} > 0$.
    Fix $\epsilon = \min \oset{ \frac{q^*}{3}, \frac{(q^* \lambda_{\min})^2}{12 \lambda_{\max}^2} }$. Then, $t_\epsilon$ is also fixed.

    Denote $\mQ = \diag \paren{q_1,\cdots,q_n}$. When $t \ge t_{\epsilon}$, the update rule of dynamic GRW is $\theta^{(t+1)} = \theta^{(t)} - \eta \mX \mQ_\epsilon^{(t)} \paren{\mX^{\top} \theta^{(t)} - \mY}$. We use the subscript $\epsilon$ to indicate that $\norm{\mQ_{\epsilon}^{(t)} - \mQ}_2 < \epsilon$.
    Because $q_i + \epsilon \sqrt{(q_i + 3 \epsilon) q_i}$ and $q_i - \epsilon \ge \sqrt{(q_i - \epsilon) q_i}$ for all $\epsilon \le \frac{q_i}{3}$, we can rewrite $\mQ_\epsilon^{(t)}$ as $\mQ_\epsilon^{(t)} = \sqrt{\mQ_{3 \epsilon}^{(t)}} \sqrt{\mQ}$.
    So by \Cref{eqn:grw-reg-linear-1}, we have
    \begin{equation*}
        F(\theta^{(t+1)}) \le F(\theta^{(t)}) - 2 \eta g(\theta^{(t)})^{\top} \mQ^{\top} \mX^{\top} \mX \sqrt{\mQ_{3 \epsilon}^{(t)}} \sqrt{\mQ} g(\theta^{(t)}) + A \norm{\eta \mX \sqrt{\mQ_{3 \epsilon}^{(t)}} \sqrt{\mQ} g(\theta^{(t)})}_2^2 .
    \end{equation*}
    For all $\epsilon < \frac{q_i}{3}$, we have $\sqrt{q_i + 3 \epsilon} - \sqrt{q_i} \le \sqrt{3 \epsilon}$ and $\sqrt{q_i} - \sqrt{q_i - 3 \epsilon} \le \sqrt{3 \epsilon}$.
    Thus, we have
    \begin{align*}
    & \left | g(\theta^{(t)})^{\top} \mQ^{\top} \mX^{\top} \mX \left (\sqrt{\mQ_{3\epsilon}^{(t)}} - \sqrt{\mQ} \right )\sqrt{\mQ} g(\theta^{(t)}) \right | \\
    \le \; & \left\| \sqrt{\mQ}^{\top} \mX^{\top} \mX \left (\sqrt{\mQ_{3\epsilon}^{(t)}} - \sqrt{\mQ} \right ) \right\|_2 \left\|\sqrt{\mQ} g(\theta^{(t)}) \right\|_2^2 \\ 
    \le \; & \left\| \sqrt{\mQ} \right\|_2 \left\|  \mX^{\top} \mX \right\|_2 \left\| \sqrt{\mQ_{3\epsilon}^{(t)}} - \sqrt{\mQ} \right\|_2 \left\| \sqrt{\mQ}g(\theta^{(t)}) \right\|_2^2 \\ 
    \le \; & \lambda_{\max} \sqrt{3 \epsilon} F(\theta^{(t)})  .
    \end{align*}
    Since $g(\theta^{(t)})^{\top} \mQ^{\top} \mX^{\top} \mX \mQ g(\theta^{(t)}) \ge q^* \lambda_{\min} F(\theta^{(t)})$, and $\epsilon \le \frac{(q^* \lambda_{\min})^2}{12 \lambda_{\max}^2}$, we have
    \begin{equation*}
    g(\theta^{(t)})^{\top} \mQ^{\top} \mX^{\top} \mX \sqrt{\mQ_{3\epsilon}^{(t)}}\sqrt{\mQ} g(\theta^{(t)}) \ge \left ( q^* \lambda_{\min} - \lambda_{\max} \sqrt{3 \epsilon} \right ) F(\theta^{(t)}) \ge \frac{1}{2}q^* \lambda_{\min} F(\theta^{(t)})  .
\end{equation*}
Thus, for all $\epsilon \le \frac{1}{3}$, we have
\begin{align*}
    F(\theta^{(t+1)}) & \le F(\theta^{(t)}) - \eta q^* \lambda_{\min} F(\theta^{(t)}) + A \eta^2 \left\| \mX \sqrt{\mQ_{3\epsilon}^{(t)}} \right\|_2^2 \left\|\sqrt{\mQ} g(\theta^{(t)}) \right\|_2^2 \\
    & \le (1 - \eta q^* \lambda_{\min} + A^2 \eta^2 (1 + 3\epsilon)) F(\theta^{(t)}) \\ 
    & \le (1 - \eta q^* \lambda_{\min} + 2 A^2 \eta^2) F(\theta^{(t)})
\end{align*}
Let $\eta_0 = \frac{q^* \lambda_{\min}}{4 A^2}$. For any $\eta \le \eta_0$, we have $F(\theta^{(t+1)}) \le (1 - \eta q^* \lambda_{\min} /2) F(\theta^{(t)}) $ for all $t \ge t_\epsilon$, which implies that $\lim_{t \rightarrow \infty} F(\theta^{(t)}) = 0$. Thus, the ERM risk converges to 0.
\end{proof}

\section{Proof of Theorem \ref{thm:grw-reg-nn}}
\label{app:proof-thm-grw-reg-nn}

The proof of this theorem is largely based on the following result.

\begin{lemma}[Approximation Theorem]
\label{thm:approx}
For a wide NN $f^{(t)}$ trained by any GRW satisfying \Cref{ass:qstar} with the squared loss, let $f^{(t)}_{\lin}(x) = f^{(0)}(x) + \dotp{\theta^{(t)} - \theta^{(0)}, \nabla_{\theta} f^{(0)}(x)}$ be its linearized neural network trained by the same GRW (i.e. $q_i^{(t)}$ are the same for both networks for any $i$ and $t$).
Under the conditions of \Cref{thm:grw-reg-nn}, with a sufficiently small learning rate, for any $\delta > 0$, there exist constants $\tilde{D} > 0$ and $C > 0$ such that as long as $\dl \ge \tilde{D}$, with probability at least $(1 - \delta)$ over random initialization we have: for any test point $x \in \R^d$ such that $\norm{x}_2 \le 1$,
\begin{equation*}
    \sup_{t \ge 0} \abs{ f^{(t)}_{\lin}(x) - f^{(t)}(x)} \le C \dl^{-1/4}
\end{equation*}
\end{lemma}

\begin{proof}
We will use the following short-hand in the proof:
\begin{equation*}
\left \{
\begin{aligned}
        g(\theta^{(t)}) &= f^{(t)}(\mX) - \mY \\ 
    J(\theta^{(t)}) &= \nabla_\theta f(\mX; \theta^{(t)}) \in \R^{p \times n} \\ 
    \Theta^{(t)} &= J(\theta^{(t)})^{\top} J(\theta^{(t)})
    \end{aligned}
    \right .
\end{equation*}

For any $\epsilon > 0$, there exists $t_\epsilon$ such that for all $t \ge t_\epsilon$ and all $i$, $q_i^{(t)} \in (q_i - \epsilon, q_i + \epsilon)$. 
Let $\mQ = \diag(q_1,\cdots,q_n)$.
Similar to \Cref{app:proof-thm-grw-reg-linear}, we can rewrite $\mQ^{(t)} = \mQ^{(t)}_{\epsilon} = \sqrt{\mQ_{3\epsilon}^{(t)}} \sqrt{\mQ}$.
The update rule of GRW of wide NN is
\begin{equation*}
    \theta^{(t+1)} = \theta^{(t)} - \eta J(\theta^{(t)}) \mQ^{(t)} g(\theta^{(t)}) ,
\end{equation*}
and when $t \ge t_\epsilon$, this can be rewritten as
\begin{equation}
\label{eqn:update-rule-teps}
    \theta^{(t+1)} = \theta^{(t)} - \eta J(\theta^{(t)}) \sqrt{\mQ_{3\epsilon}^{(t)}} \left [ \sqrt{\mQ} g(\theta^{(t)}) \right ]  .
\end{equation}

Next, we need three lemmas.

\begin{lemma}[\cite{vershynin2010introduction}, Corollary~5.35]
\label{lem:vershynin}
    If $\mA \in \R^{p \times q}$ is a random matrix whose entries are independent standard normal random variables, then for every $t \ge 0$, with probability at least $1 - 2 \exp (-t^2 / 2)$,
    \begin{equation*}
        \sqrt{p} - \sqrt{q} - t \le \lambda_{\min}(\mA) \le \lambda_{\max}(\mA) \le \sqrt{p} + \sqrt{q} + t  .
    \end{equation*}
\end{lemma}

\begin{lemma}[Local Lipschitzness of the Jacobian]
\label{lem:local-lip}
    There exists a constant $M > 0$ such that for any $C_0 > 0$ and any $\delta \in (0,1)$, there exists a $\tilde{D}$ such that: If $\dl \ge \tilde{D}$, then with probability at least $(1 - \delta)$ over random initialization, for any $x$ such that $\norm{x}_2 \le 1$, we have
\begin{equation}
\label{eqn:lem-local-lip}
        \left \{
    \begin{aligned}  
     \left \| \nabla_\theta f(x; \theta) - \nabla_\theta f(x; \tilde{\theta}) \right\|_2 & \le \frac{M}{\sqrt[4]{\dl}} \left\| \theta - \tilde{\theta} \right \|_2\\
     \left\| \nabla_\theta f(x; \theta) \right\|_2 & \le M \\ 
     \left\| J(\theta) - J(\tilde{\theta}) \right \|_F & \le \frac{M}{\sqrt[4]{\dl}} \left\| \theta - \tilde{\theta} \right \|_2 \\
     \left\| J(\theta) \right \|_F & \le M
    \end{aligned}
    \right.  ,
    \qquad
    \forall \theta, \tilde{\theta} \in B(\theta^{(0)}, C_0)  ,
\end{equation}
where $B(\theta^{(0)}, R) = \oset{ \theta: \norm{ \theta - \theta^{(0)} }_2 < R }$.
\end{lemma}
\begin{proof}
    By \Cref{lem:vershynin}, for any $\delta \in (0,1)$, there exists $\tilde{D} > 0$ and $M_1 > 0$ such that if $\dl \ge \tilde{D}$, then with probability at least $1 - \delta$, the following holds for all $\theta$ such that $\norm{\theta - \theta^{(0)}}_{2} < C_0$:
    \begin{equation*}
        \forall 0 \le l \le L-1:  \norm{\mW^{l}}_2 \le 3 \sqrt{\dl}; \quad \norm{\mW^L}_2 \le C_0 \le 3 \dl^{1/4} ; \quad \forall 0 \le l \le L:  \norm{\beta \vb^l} \le M_1 \sqrt{\dl}  .
    \end{equation*}
With these inequalities, for any $x$ such that $\norm{x}_2 \le 1$, we have
\begin{align*}
    & \left\| \vh^1 \right\|_2 = \left\| \frac{1}{\sqrt{d_0}} \mW^0 x + \beta\vb^0 \right\|_2 \le \frac{1}{\sqrt{d_0}} \left\| \mW^0 \right\|_2 \left\| x \right \|_2 + \left\| \beta\vb^0 \right\|_2 \le (\frac{3}{\sqrt{d_0}} + M_1) \sqrt{\dl} ; \\
& \left\| \vh^{l+1} \right\|_2 = \left\| \frac{1}{\sqrt{\dl}} \mW^l \vx^l + \beta\vb^l \right\|_2 \le \frac{1}{\sqrt{\dl}} \left\| \mW^l \right\|_2 \left\| \vx^l \right \|_2 + \left\| \beta\vb^l \right\|_2  \qquad (\forall l \ge 1)   ; \\
& \left\| \vx^l \right\|_2 = \left\| \sigma(\vh^l) - \sigma(\vzero^l) + \sigma(\vzero^l) \right\|_2 \le L_0 \left\| \vh^l \right\|_2 + \sigma(0) \sqrt{\dl} \qquad
(\forall l \ge 1)  .
\end{align*}
Here, $L_0$ is the Lipschitz constant of $\sigma$ and $\sigma(\vzero^l) = (\sigma(0),\cdots,\sigma(0)) \in \R^{d_l}$.
Thus, we can prove by induction that there exists an $M_2 > 0$ such that $\left\| \vx^l \right\|_2 \le M_2 \sqrt{\dl}$ and $\left\| \vh^l \right\|_2 \le M_2 \sqrt{\dl}$ for all $l=1,\cdots,L$.

Denote $\valpha^l = \nabla_{\vh^l} f(x) = \nabla_{\vh^l} \vh^{L+1}$. For all $1 \le l \le L$, we have $\valpha^{l} = \diag(\dot{\sigma}(\vh^l)) \frac{\mW^{l \top}}{\sqrt{\dl}} \valpha^{l+1}$, $\valpha^{L+1} =1$ and $\left\| \valpha^{L} \right\|_2 = \left\| \diag(\dot{\sigma}(\vh^L)) \frac{\mW^{L \top}}{\sqrt{\dl}} \right\|_2 \le \frac{3}{\sqrt[4]{\dl}} L_0$.
Since $\sigma$ is $L_0$-Lipschitz, we have $\dot{\sigma}(x) \le L_0$ for all $x \in \R$.
Thus, we can prove by induction that there exists an $M_3 > 1$ such that $\left\| \valpha^l \right \|_2 \le M_3 / \sqrt[4]{\dl}$ for all $l=1,\cdots,L$ (note that this is not true for $L+1$ because $\valpha^{L+1} = 1$).

For $l = 0$, $\nabla_{\mW^0} f(x) = \frac{1}{\sqrt{d_0}} \vx^0 \valpha^{1 \top}$, so $\left\| \nabla_{\mW^l} f(x) \right\|_2 \le \frac{1}{\sqrt{d_0}} \left\| \vx^0 \right\|_2 \left\| \valpha^{1} \right\|_2 \le \frac{1}{\sqrt{d_0}} M_3 / \sqrt[4]{\dl}$.
And for any $l=1,\cdots,L$, $\nabla_{\mW^l} f(x) = \frac{1}{\sqrt{\dl}} \vx^l \valpha^{l+1}$, so $\left\| \nabla_{\mW^l} f(x) \right\|_2 \le \frac{1}{\sqrt{\dl}} \left\| \vx^l \right\|_2 \left\| \valpha^{l+1} \right\|_2 \le M_2 M_3$. (Note that if $M_3>1$, then $\left\| \valpha^{L+1} \right\|_2 \le M_3$; and since $\dl \ge 1$, there is $\left\| \valpha^{l} \right\|_2 \le M_3$ for $l \le L$.) Moreover, for $l = 0,\cdots,L$, $\nabla_{\vb^l} f(\vx) = \beta \valpha^{l+1} $, so $\left\| \nabla_{\vb^l} f(x) \right\|_2 \le \beta M_3$. Thus,  there exists an $M_4 > 0$, such that $\left\| \nabla_\theta f(x) \right\|_2 \le M_4 / \sqrt{n}$. And since $\left\| \vx_i \right\|_2 \le 1$ for all $i$, so $\left\| J(\theta) \right\|_F \le M_4$.

Next, we consider the difference in $\nabla_\theta f(x)$ between $\theta$ and $\tilde{\theta}$. Let $\tilde{f}, \tilde{\mW},\tilde{\vb},\tilde{\vx},\tilde{\vh},\tilde{\valpha}$ be the function and the values corresponding to $\tilde{\theta}$. We have
\begin{align*}
& \begin{aligned} \left\| \vh^1 - \tilde{\vh}^1 \right\|_2 &= \left\| \frac{1}{\sqrt{d_0}} (\mW^0-\tilde{\mW}^0) x + \beta(\vb^0-\tilde{\vb}^0) \right\|_2 \\ &\le \frac{1}{\sqrt{d_0}} \left\| \mW^0-\tilde{\mW}^0 \right\|_2 \left\| x \right\|_2 + \beta \left\| \vb^0-\tilde{\vb}^0 \right\|_2 \le \left (\frac{1}{\sqrt{d_0}} + \beta \right ) \left\| \theta - \tilde{\theta} \right\|_2 \end{aligned} \\ 
& \begin{aligned} \left\| \vh^{l+1} - \tilde{\vh}^{l+1} \right\|_2 &= \left\| \frac{1}{\sqrt{\dl}} \mW^l (\vx^l- \tilde{\vx}^l) + \frac{1}{\sqrt{\dl}} (\mW^l - \tilde{\mW}^l) \tilde{\vx}^l + \beta(\vb^l-\tilde{\vb}^l) \right\|_2 \\ &\le \frac{1}{\sqrt{\dl}} \left\| \mW^l \right\|_2 \left\| \vx^l - \tilde{\vx}^l \right\|_2 + \frac{1}{\sqrt{\dl}} \left\| \mW^l-\tilde{\mW}^l \right\|_2 \left\| \tilde{\vx}^l \right\|_2 + \beta \left\| \vb^l-\tilde{\vb}^l \right\|_2 \\ & \le 3\left\| \vx^l - \tilde{\vx}^l \right\|_2 + (M_2 + \beta) \left\| \theta - \tilde{\theta} \right\|_2 \qquad (\forall l \ge 1)  \end{aligned} \\ 
& \left\| \vx^l - \tilde{\vx}^l \right\|_2 = \left\| \sigma(\vh^l) - \sigma(\tilde{\vh}^l) \right\|_2 \le L_0 \left\| \vh^l - \tilde{\vh}^l \right\|_2 \qquad (\forall l \ge 1) .
\end{align*}
Thus, we can prove by induction that there exists an $M_5 >0$ such that $\left\| \vx^l - \tilde{\vx}^l \right\|_2 \le M_5 \left\| \theta - \tilde{\theta} \right\|_2$ for all $l$.

For $\valpha^l$, we have $\valpha^{L+1} = \tilde{\valpha}^{L+1} = 1$, and for all $l \ge 1$,
\begin{equation}
\label{eqn:s86}
\begin{aligned}
            \left\|\valpha^l - \tilde{\valpha}^l\right\|_2 &= \left\| 
    \diag(\dot{\sigma}(\vh^l)) \frac{\mW^{l \top}}{\sqrt{\dl}} \valpha^{l+1}   - \diag(\dot{\sigma}(\tilde{\vh}^l)) \frac{\tilde{\mW}^{l \top}}{\sqrt{\dl}} \tilde{\valpha}^{l+1} \right \|_2 \\ 
    & \begin{aligned} \le  \left\| \diag(\dot{\sigma}(\vh^l)) \frac{\mW^{l \top}}{\sqrt{\dl}}  (\valpha^{l+1} - \tilde{\valpha}^{l+1}) \right\|_2 &+ \left\| \diag(\dot{\sigma}(\vh^l))  \frac{(\mW^l - \tilde{\mW}^l)^{\top}}{\sqrt{\dl}}  \tilde{\valpha}^{l+1} \right\|_2 \\ &+ \left\| \diag((\dot{\sigma}(\vh^l) - \dot{\sigma}(\tilde{\vh}^l)))  \frac{\tilde{\mW}^{l \top}}{\sqrt{\dl}} \tilde{\valpha}^{l+1} \right\|_2 \end{aligned} \\ 
    & \le 3 L_0 \left\| \valpha^{l+1} - \tilde{\valpha}^{l+1} \right\|_2 + \left ( M_3 L_0 \dl^{-1/2} + 3M_3M_5 L_1\dl^{-1/4} \right ) \left\| \theta - \tilde{\theta} \right\|_2  ,
\end{aligned}
\end{equation}
where $L_1$ is the Lipschitz constant of $\dot{\sigma}$. In particular, for $l = L$, though $\tilde{\valpha}^{L+1} = 1$, since $\left \| \tilde{\mW}^L \right\|_2 \le 3 \dl^{1/4}$, \Cref{eqn:s86} is still true. Thus, we can prove by induction that there exists an $M_6 > 0$ such that $\left\|\valpha^l - \tilde{\valpha}^l\right\|_2 \le \frac{M_6}{\sqrt[4]{\dl}} \left\| \theta - \tilde{\theta} \right\|_2$ for all $l \ge 1$ (note that this is also true for $l=L+1$).

Thus, for all $\theta, \tilde{\theta} \in B(\theta^{(0)}, C_0 )$, and any $x$ such that $\norm{x}_2 \le 1$, we have
\begin{align*}
        \left\| \nabla_{\mW^0} f(x) - \nabla_{\tilde{\mW}^0} \tilde{f}(x) \right\|_2  &= \frac{1}{\sqrt{d_0}} \left\| x \valpha^{1 \top} - x \tilde{\valpha}^{1 \top} \right\|_2 \\ 
    & \le \frac{1}{\sqrt{d_0}} \left\| \valpha^1 - \tilde{\valpha}^1 \right\|_2 \\
    & \le \frac{1}{\sqrt{d_0}} \frac{M_6}{\sqrt[4]{\dl}} \left\| \theta - \tilde{\theta} \right\|_2   ;
\end{align*}
and for $l = 1,\cdots,L$, we have
\begin{align*}
    \left\| \nabla_{\mW^l} f(x) - \nabla_{\tilde{\mW}^l} \tilde{f}(x) \right\|_2  &= \frac{1}{\sqrt{\dl}} \left\| x^l \valpha^{l+1 \top} - \tilde{x}^l \tilde{\valpha}^{l+1 \top} \right\|_2 \\ 
    & \le \frac{1}{\sqrt{\dl}} \left ( \left\| \vx^l \right\|_2 \left\| \valpha^{l+1} - \tilde{\valpha}^{l+1}  \right\|_2 + \left\| \vx^l - \tilde{\vx}^l \right\|_2 \left\| \tilde{\valpha}^{l+1} \right\|_2 \right ) \\
    & \le \left ( \frac{M_2 M_6}{\sqrt[4]{\dl}} + \frac{M_5 M_3}{\sqrt{\dl}}  \right ) \left\| \theta - \tilde{\theta} \right\|_2  .
\end{align*}
Moreover, for any $l=0,\cdots,L$, we have
\begin{equation*}
    \left\| \nabla_{b^l} f(x) - \nabla_{\tilde{b}^l} \tilde{f}(x) \right\|_2  = \beta \left\| \valpha^{l+1} - \tilde{\valpha}^{l+1} \right\|_2 \le  \frac{\beta M_6}{\sqrt[4]{\dl}} \left\| \theta - \tilde{\theta} \right\|_2  .
\end{equation*}
Combining all the above, we can see that there exists a constant $M_7 > 0$ such that $\left\| \nabla_\theta f(x) - \nabla_{\tilde{\theta}}\tilde{f}(x) \right\|_2 \le \frac{M_7}{\sqrt{n} \cdot \sqrt[4]{\dl}} \left\| \theta - \tilde{\theta} \right\|_2$, so that $\left\| J(\theta) - J(\tilde{\theta}) \right\|_F \le \frac{M_7}{\sqrt[4]{\dl}}  \left\| \theta - \tilde{\theta} \right\|_2$.
\end{proof}

\begin{lemma}
\label{lem:tbound}
There exist constants $M>0$ and $\epsilon_0 > 0$ such that for all $\epsilon \in (0, \epsilon_0]$, $\eta \le \eta^*$ and any $\delta > 0$, there exist $R_0 > 0$, $\tilde{D} > 0$ and $B > 1$ such that for any $\dl \ge \tilde{D}$, the following (i) and (ii) hold with probability at least $(1-\delta)$ over random initialization when applying gradient descent with learning rate $\eta$:
\begin{enumerate}[label=(\roman*)]
    \item For all $t \le t_\epsilon$, there is
    \begin{align}
    \label{eqn:7-i1}
        & \left\| g(\theta^{(t)}) \right\|_2 \le B^t R_0 \\ 
        \label{eqn:7-i2}
          & \sum_{j=1}^t \left\| \theta^{(j)} - \theta^{(j-1)} \right\|_2 \le \eta M R_0 \sum_{j=1}^t B^{j-1} < \frac{M B^{t_\epsilon} R_0}{B - 1}
         \end{align}
    \item For all $t \ge t_\epsilon$, we have
    \begin{align}   
\label{eqn:7-1}
    & \left\| \sqrt{\mQ} g(\theta^{(t)}) \right\|_2 \le \left( 1 - \frac{\eta q^* \lambda_{\min} }{3} \right)^{t - t_\epsilon} B^{t_\epsilon} R_0 \\ 
    \label{eqn:7-2}
    & \begin{aligned} \sum_{j=t_\epsilon+1}^t \left\| \theta^{(j)} - \theta^{(j-1)} \right\|_2 &\le \eta \sqrt{1 + 3\epsilon}  M  B^{t_\epsilon} R_0 \sum_{j=t_\epsilon+1}^t \left( 1 - \frac{\eta q^* \lambda_{\min}}{3} \right)^{j-t_\epsilon} \\ & < \frac{3 \sqrt{1 + 3\epsilon} M B^{t_\epsilon}R_0}{ q^* \lambda_{\min}} \end{aligned}
\end{align}
\end{enumerate}
\end{lemma}
\begin{proof}
    Note that for any $x$, $f^{(0)}(x) = \beta \vb^L$ where $\vb^L$ is sampled from the standard Gaussian distribution. Thus, for any $\delta > 0$, there exists a constant $R_0$ such that with probability at least $(1 - \delta/3)$ over random initialization, we have $\norm{g(\theta^{(0)})}_2 < R_0$.
    And by \Cref{lem:ntk}, there exists $D_2 \ge 0$ such that for any $\dl \ge D_2$, with probability at least $(1 - \delta / 3)$, we have $\norm{\Theta - \Theta^{(0)}} \le \frac{q^* \lambda_{\min}}{3}$.

    Let $M$ be the constant in \Cref{lem:local-lip}. Let $\epsilon_0 = \frac{(q^* \lambda_{\min})^2}{108M^4}$, $B = 1 + \eta^* M^2$, and $C_0 = \frac{M B^{t_\epsilon} R_0}{B - 1} +  \frac{3 \sqrt{1 + 3\epsilon} M B^{t_\epsilon}R_0}{ q^* \lambda_{\min}}$.
    By \Cref{lem:local-lip}, there exists $D_1 > 0$ such that with probability at least $(1-\delta / 3)$, for any $\dl \ge D_1$, \Cref{eqn:lem-local-lip} is true for all $\theta, \tilde{\theta} \in B(\theta^{(0)}, C_0)$. 
    By union bound, with probability at least $1 - \delta$, all the three above inequalities holds.

    Let us prove \Cref{eqn:7-i1,eqn:7-i2} by induction. They are obviously true for $t = 0$.
    Suppose they are true for $t$.
    Then, for $t+1$, we have
    \begin{align*}
    \left\| \theta^{(t+1)} - \theta^{(t)} \right\|_2 & \le \eta \left\| J(\theta^{(t)}) \mQ^{(t)} \right\|_2 \left\| g(\theta^{(t)}) \right\|_2 \le \eta \left\| J(\theta^{(t)}) \mQ^{(t)} \right\|_F \left\| g(\theta^{(t)}) \right\|_2 \\
    & \le \eta \left\| J(\theta^{(t)}) \right\|_F \left\|  g(\theta^{(t)}) \right\|_2  \le M \eta B^t R_0  ,
    \end{align*}
    which means that \Cref{eqn:7-i2} is also true for $t + 1$. In terms of \Cref{eqn:7-i1}, we have
    \begin{align*}
\left\| g(\theta^{(t+1)}) \right\|_2 &= \left\| g(\theta^{(t+1)}) - g(\theta^{(t)}) + g(\theta^{(t)}) \right\|_2 \\ 
& = \left\| J(\tilde{\theta}^{(t)})^{\top}(\theta^{(t+1)} - \theta^{(t)}) +  g(\theta^{(t)}) \right\|_2 \\ 
&= \left\| -\eta  J(\tilde{\theta}^{(t)})^{\top} J(\theta^{(t)}) \mQ^{(t)} g(\theta^{(t)}) + g(\theta^{(t)}) \right\|_2 \\ 
& \le \left\| \mI -\eta  J(\tilde{\theta}^{(t)})^{\top} J(\theta^{(t)}) \mQ^{(t)} \right\|_2 \left\| g(\theta^{(t)}) \right\|_2 \\ 
& \le \left (1 + \left\| \eta  J(\tilde{\theta}^{(t)})^{\top} J(\theta^{(t)}) \mQ^{(t)} \right\|_2 \right ) \left\| g(\theta^{(t)}) \right\|_2 \\ 
& \le \left ( 1 + \eta \left\| J(\tilde{\theta}^{(t)}) \right\|_F \left\| J(\theta^{(t)}) \right\|_F \right ) \left\| g(\theta^{(t)}) \right\|_2 \\ 
& \le (1 + \eta^* M^2) \left\| g(\theta^{(t)}) \right\|_2 \le B^{t+1} R_0  .
    \end{align*}
Hence, \Cref{eqn:7-i1,eqn:7-i2} are true for all $t \le t_\epsilon$, which implies that $ \left\| \sqrt{\mQ} g(\theta^{(t_\epsilon)}) \right\|_2 \le \left\| g(\theta^{(t_\epsilon)}) \right\|_2 \le B^{t_\epsilon} R_0$.
Thus, \Cref{eqn:7-1} is true for $t = t_\epsilon$.
And \Cref{eqn:7-2} is obviously true for $t = t_\epsilon$.

Next, let us prove \Cref{eqn:7-1,eqn:7-2} by induction. Suppose they are true for $t$. By \Cref{eqn:update-rule-teps}, for $t+1$ we have
\begin{align*}
    \left\| \theta^{(t+1)} - \theta^{(t)} \right\|_2 & \le \eta \left\| J(\theta^{(t)}) \sqrt{\mQ_{3\epsilon}^{(t)}} \right\|_2 \left\| \sqrt{\mQ} g(\theta^{(t)}) \right\|_2 \\
    & \le \eta \left\| J(\theta^{(t)}) \sqrt{\mQ_{3\epsilon}^{(t)}} \right\|_F \left\| \sqrt{\mQ} g(\theta^{(t)}) \right\|_2 \\
    & \le \eta \sqrt{1 + 3\epsilon} \left\| J(\theta^{(t)}) \right\|_F \left\| \sqrt{\mQ} g(\theta^{(t)}) \right\|_2  \\
    & \le M \eta \sqrt{1 + 3\epsilon} \left ( 1-\frac{\eta q^* \lambda_{\min}}{3} \right )^{t-t_\epsilon} B^{t_\epsilon} R_0  ,
\end{align*}
which implies that \Cref{eqn:7-2} holds for $t+1$. In terms of \Cref{eqn:7-1}, we have
\begin{align*}
\left\| \sqrt{\mQ} g(\theta^{(t+1)}) \right\|_2 &= \left\| \sqrt{\mQ}g(\theta^{(t+1)}) - \sqrt{\mQ}g(\theta^{(t)}) +\sqrt{\mQ} g(\theta^{(t)}) \right\|_2 \\ 
& = \left\| \sqrt{\mQ} J(\tilde{\theta}^{(t)})^{\top}(\theta^{(t+1)} - \theta^{(t)}) + \sqrt{\mQ} g(\theta^{(t)}) \right\|_2 \\ 
&= \left\| -\eta  \sqrt{\mQ} J(\tilde{\theta}^{(t)})^{\top} J(\theta^{(t)}) \mQ^{(t)} g(\theta^{(t)}) + \sqrt{\mQ} g(\theta^{(t)}) \right\|_2 \\ 
& \le \left\| \mI -\eta \sqrt{\mQ} J(\tilde{\theta}^{(t)})^{\top} J(\theta^{(t)}) \sqrt{\mQ_{3\epsilon}^{(t)}} \right\|_2 \left\| \sqrt{\mQ} g(\theta^{(t)}) \right\|_2 \\ 
& \le \left\| \mI -\eta \sqrt{\mQ} J(\tilde{\theta}^{(t)})^{\top} J(\theta^{(t)}) \sqrt{\mQ_{3\epsilon}^{(t)}} \right\|_2 \left ( 1 - \frac{\eta q^* \lambda_{\min} }{3} \right)^t R_0  ,
\end{align*}
where $\tilde{\theta}^{(t)}$ is some linear interpolation between $\theta^{(t)}$ and $\theta^{(t+1)}$. Now we prove that
\begin{equation}
\label{eqn:thm12-1}
    \left\| \mI -\eta \sqrt{\mQ} J(\tilde{\theta}^{(t)})^{\top} J(\theta^{(t)}) \sqrt{\mQ_{3\epsilon}^{(t)}} \right\|_2 \le  1 - \frac{\eta q^* \lambda_{\min}}{3}  .
\end{equation}
For any unit vector $\vv \in \R^n$, we have
\begin{equation*}
     \vv^{\top} (\mI - \eta \sqrt{\mQ} \Theta \sqrt{\mQ}) \vv = 1 - \eta \vv^{\top} \sqrt{\mQ} \Theta \sqrt{\mQ}  \vv  ,
\end{equation*}
where $\left\| \sqrt{\mQ}  \vv \right\|_2 \in [\sqrt{q^*}, 1]$, so for any $\eta \le \eta^*$, $\vv^{\top} (\mI - \eta \sqrt{\mQ} \Theta \sqrt{\mQ}) \vv \in [0, 1 - \eta \lambda_{\min} q^*]$, which implies that $\left\| \mI - \eta \sqrt{\mQ} \Theta \sqrt{\mQ} \right\|_2 \le 1 - \eta \lambda_{\min} q^*$. Thus, we have
\begin{align*}
& \left\| \mI -\eta  \sqrt{\mQ} J(\tilde{\theta}^{(t)})^{\top} J(\theta^{(t)}) \sqrt{\mQ} \right\|_2 \\ 
\le & \left\| \mI - \eta \sqrt{\mQ} \Theta \sqrt{\mQ} \right\|_2 + \eta \left\|\sqrt{\mQ} (\Theta - \Theta^{(0)}) \sqrt{\mQ} \right\|_2 + \eta \left\| \sqrt{\mQ} (J(\theta^{(0)})^{\top}J(\theta^{(0)}) - J(\tilde{\theta}^{(t)})^{\top} J(\theta^{(t)}) ) \sqrt{\mQ} \right\|_2 \\ 
\le & 1 - \eta \lambda_{\min} q^* +  \eta \left\| \sqrt{\mQ} (\Theta - \Theta^{(0)}) \sqrt{\mQ} \right\|_F + \eta \left\| \sqrt{\mQ} (J(\theta^{(0)})^{\top}J(\theta^{(0)}) - J(\tilde{\theta}^{(t)})^{\top} J(\theta^{(t)}) ) \sqrt{\mQ} \right\|_F \\
\le & 1 - \eta \lambda_{\min} q^* +  \eta \left\| \Theta - \Theta^{(0)}  \right\|_F + \eta \left\| J(\theta^{(0)})^{\top}J(\theta^{(0)}) - J(\tilde{\theta}^{(t)})^{\top} J(\theta^{(t)})  \right\|_F \\ 
\le & 1 - \eta \lambda_{\min} q^* + \frac{\eta q^* \lambda_{\min}}{3} + \frac{\eta M^2}{\sqrt[4]{\dl}} \left (\left\| \theta^{(t)} - \theta^{(0)} \right\|_2 + \left\| \tilde{\theta}^{(t)} - \theta^{(0)} \right\|_2 \right ) \le 1 - \frac{\eta q^* \lambda_{\min}}{2}
\end{align*}
for all $\dl \ge \max \left\{D_1, D_2, \left ( \frac{12 M^2 C_0}{q^* \lambda_{\min}} \right) ^4 \right \}$. This implies that
\begin{align*}
    & \left\| \mI -\eta \sqrt{\mQ} J(\tilde{\theta}^{(t)})^{\top} J(\theta^{(t)}) \sqrt{\mQ_{3\epsilon}^{(t)}} \right\|_2 \\ 
    \le \; &  1 - \frac{\eta q^* \lambda_{\min}}{2} + \left\| \eta \sqrt{\mQ} J(\tilde{\theta}^{(t)})^{\top} J(\theta^{(t)}) \left ( \sqrt{\mQ_{3\epsilon}^{(t)}} - \sqrt{\mQ} \right ) \right\|_2 \\ 
    \le \; & 1 - \frac{\eta q^* \lambda_{\min}}{2} + \eta M^2 \sqrt{3 \epsilon} \le 1 - \frac{\eta q^* \lambda_{\min}}{3} 
\end{align*}
holds for all $\epsilon \le \epsilon_0$. Thus, \Cref{eqn:7-1} holds for $t+1$.
\end{proof}

Now let us return to the proof of \Cref{thm:approx}.
Choose and fix an $\epsilon$ such that $\epsilon < \min \{ \epsilon_0, \frac{1}{3} \left ( \frac{q^* \lambda_{\min}}{3 \lambda_{\max} + q^* \lambda_{\min}} \right ) ^2 \}$, where $\epsilon_0$ is given by Theorem \ref{lem:tbound}.
Then, $t_\epsilon$ is also fixed. There exists $\tilde{D} \ge 0$ such that for any $\dl \ge \tilde{D}$, with probability at least $(1 - \delta)$,
the inequalities in \Cref{lem:local-lip,lem:tbound} hold, and $\norm{\Theta - \Theta^{(0)}}_F \le \frac{q^* \lambda_{\min}}{3}$.
This implies that
\begin{equation*}
    \norm{\Theta^{(0)}}_2 \le \norm{\Theta}_2 + \norm{\Theta - \Theta^{(0)}}_F \le \lambda_{\max} + \frac{q^* \lambda_{\min}}{3}   .
\end{equation*}
We still denote $B = 1 + \eta^* M^2$ and $C_0 = \frac{M B^{t_\epsilon} R_0}{B - 1} +  \frac{3 \sqrt{1 + 3\epsilon} M B^{t_\epsilon}R_0}{ q^* \lambda_{\min}}$.
\Cref{lem:tbound} guarantees that for all $t$, we have $\theta^{(t)} \in B(\theta^{(0)},C_0)$.
Thus, we have
\begin{align*}
    \left\| \mI - \eta \sqrt{\mQ} \Theta^{(0)}\sqrt{\mQ} \right\|_2 & \le \left\| \mI - \eta \sqrt{\mQ} \Theta \sqrt{\mQ} \right\|_2 + \eta \left\| \sqrt{\mQ} (\Theta - \Theta^{(0)}) \sqrt{\mQ} \right\|_2 \\ 
    & \le 1 - \eta \lambda_{\min} q^* + \frac{\eta q^* \lambda_{\min}}{3} = 1 - \frac{2 \eta q^* \lambda_{\min}}{3}  .
\end{align*}
It follows that
\begin{align*}
    \left\| \mI - \eta \sqrt{\mQ} \Theta^{(0)} \sqrt{\mQ_{3\epsilon}^{(t)}} \right\|_2 & \le \left\| \mI - \eta \sqrt{\mQ} \Theta^{(0)}\sqrt{\mQ} \right\|_2 + \left\| \eta \sqrt{\mQ} \Theta^{(0)} \left ( \sqrt{\mQ_{3\epsilon}^{(t)}} - \sqrt{\mQ} \right ) \right\|_2 \\ 
    & \le 1 - \frac{2 \eta q^* \lambda_{\min}}{3} + \eta (\lambda_{\max} + \frac{ q^* \lambda_{\min}}{3}) \sqrt{3 \epsilon}  .
\end{align*}
Thus, for all $\epsilon < \frac{1}{3} \left ( \frac{q^* \lambda_{\min}}{3 \lambda_{\max} + q^* \lambda_{\min}} \right ) ^2$, we have
\begin{equation}
\label{eqn:lem-approx-thm}
    \left\| \mI - \eta \sqrt{\mQ} \Theta^{(0)} \sqrt{\mQ_{3\epsilon}^{(t)}} \right \|_2 \le 1 - \frac{\eta q^* \lambda_{\min}}{3}  .
\end{equation}

The update rule of the GRW for the linearized neural network is:
\begin{equation}
\label{eqn:update-lin}
    \theta^{(t+1)}_{\lin} = \theta^{(t)}_{\lin} - \eta J(\theta^{(0)}) \mQ^{(t)} g_{\lin}(\theta^{(t)})
\end{equation}
where we use the subscript ``lin'' to denote the linearized neural network, and with a slight abuse of notion denote $g_{\lin}(\theta^{(t)}) = g(\theta^{(t)}_{\lin})$.

First, let us consider the training data $\mX$. Denote $\Delta_t = g_{\lin}(\theta^{(t)}) -  g(\theta^{(t)})$. We have
\begin{equation*}
\left \{
\begin{aligned}
 g_{\lin}(\theta^{(t+1)}) -  g_{\lin}(\theta^{(t)}) &= - \eta  J(\theta^{(0)})^{\top} J(\theta^{(0)}) \mQ^{(t)} g_{\lin}(\theta^{(t)}) \\ 
 g(\theta^{(t+1)}) - g(\theta^{(t)}) &= - \eta  J(\tilde{\theta}^{(t)})^{\top} J(\theta^{(t)}) \mQ^{(t)} g(\theta^{(t)})
\end{aligned}
\right .
\end{equation*}
where $\tilde{\theta}^{(t)}$ is some linear interpolation between $\theta^{(t)}$ and $\theta^{(t+1)}$. Thus, we have
\begin{align*}
    \Delta_{t+1} - \Delta_t = &\eta \left [J(\tilde{\theta}^{(t)})^{\top} J(\theta^{(t)}) - J(\theta^{(0)})^{\top} J(\theta^{(0)}) \right ] \mQ^{(t)} g(\theta^{(t)}) \\ 
    & -  \eta J( \theta^{(0)})^{\top} J(\theta^{(0)}) \mQ^{(t)} \Delta_t  .
\end{align*}
By \Cref{lem:local-lip}, we have
\begin{align*}
    & \left\| J(\tilde{\theta}^{(t)})^{\top} J(\theta^{(t)}) - J(\theta^{(0)})^{\top} J(\theta^{(0)}) \right\|_F \\ 
    \le  \; & \left\| \left ( J(\tilde{\theta}^{(t)}) - J(\theta^{(0)}) \right )^{\top} J(\theta^{(t)}) \right\|_F + \left\| J(\theta^{(0)})^{\top} \left ( J(\theta^{(t)}) - J(\theta^{(0)}) \right ) \right\|_F \\
    \le \;  & 2M^2 C_0 \dl^{-1/4}  ,
\end{align*}
which implies that for all $t < t_\epsilon$, we have
\begin{align*}
    & \left\| \Delta_{t+1} \right\|_2 \\
    \le & \left\| \left [ \mI -  \eta J( \theta^{(0)})^{\top} J(\theta^{(0)}) \mQ^{(t)} \right ] \Delta_t \right\|_2 + \left\| \eta \left [J(\tilde{\theta}^{(t)})^{\top} J(\theta^{(t)}) - J(\theta^{(0)})^{\top} J(\theta^{(0)}) \right ] \mQ^{(t)} g(\theta^{(t)}) \right\|_2 \\ 
    \le & \left\| \mI -  \eta J( \theta^{(0)})^{\top} J(\theta^{(0)}) \mQ^{(t)} \right\|_F \left\| \Delta_t \right\|_2 + \eta \left\| J(\tilde{\theta}^{(t)})^{\top} J(\theta^{(t)}) - J(\theta^{(0)})^{\top} J(\theta^{(0)}) \right\|_F \left\| g(\theta^{(t)}) \right\|_2 \\ 
    \le & (1 + \eta M^2) \left\| \Delta_t \right\|_2 + 2 \eta M^2 C_0 B^t R_0  \dl^{-1/4} \\ 
    \le & B \left\| \Delta_t \right\|_2 + 2 \eta M^2 C_0  B^t R_0\dl^{-1/4}  .
\end{align*}
This implies that
\begin{equation*}
    B^{-(t+1)} \left\| \Delta_{t+1} \right\|_2 \le B^{-t} \left\| \Delta_t \right\|_2 + 2 \eta M^2 C_0 B^{-1} R_0 \dl^{-1/4} .
\end{equation*}
Since $\Delta_0 = 0$, it follows that for all $t \le t_\epsilon$, we have
\begin{equation*}
\left\| \Delta_{t} \right\|_2 \le 2 t \eta M^2 C_0 B^{t-1} R_0  \dl^{-1/4}  ,
\end{equation*}
and in particular, we have
\begin{equation*}
    \left\| \sqrt{\mQ} \Delta_{t_\epsilon} \right\|_2  \le  \left\| \Delta_{t_\epsilon} \right\|_2 \le 2 t_\epsilon \eta M^2 C_0 B^{t_\epsilon-1} R_0  \dl^{-1/4}  .
\end{equation*}
For $t \ge t_\epsilon$, by \Cref{eqn:update-rule-teps}, we have
\begin{align*}
    \sqrt{\mQ} \Delta_{t+1} - \sqrt{\mQ} \Delta_t = &\eta \sqrt{\mQ} \left [J(\tilde{\theta}^{(t)})^{\top} J(\theta^{(t)}) - J(\theta^{(0)})^{\top} J(\theta^{(0)}) \right ]  \sqrt{\mQ_{3\epsilon}^{(t)}} \left [ \sqrt{\mQ} g(\theta^{(t)}) \right ] \\ 
    & -  \eta \sqrt{\mQ} J( \theta^{(0)})^{\top} J(\theta^{(0)}) \sqrt{\mQ_{3\epsilon}^{(t)}} \left [ \sqrt{\mQ} \Delta_t \right ]  .
\end{align*}
Let $\mA = \mI - \eta \sqrt{\mQ} J( \theta^{(0)})^{\top} J(\theta^{(0)}) \sqrt{\mQ_{3\epsilon}^{(t)}} = \mI - \eta \sqrt{\mQ} \Theta^{(0)} \sqrt{\mQ_{3\epsilon}^{(t)}}$. Then, we have
\begin{equation*}
    \sqrt{\mQ}  \Delta_{t+1} = \mA  \sqrt{\mQ} \Delta_t + \eta \sqrt{\mQ} \left [J(\tilde{\theta}^{(t)})^{\top} J(\theta^{(t)}) - J(\theta^{(0)})^{\top} J(\theta^{(0)}) \right ] \sqrt{\mQ_{3\epsilon}^{(t)}} \left ( \sqrt{\mQ} g(\theta^{(t)}) \right )  .
\end{equation*}
Let $\gamma = 1 - \frac{\eta q^* \lambda_{\min} }{3} < 1$.
Combining the above with \Cref{lem:tbound} and \Cref{eqn:lem-approx-thm}, we have
\begin{align*}
& \left\| \sqrt{\mQ}  \Delta_{t+1} \right\|_2 \\
\le \; & \left\| \mA \right\|_2 \left\| \sqrt{\mQ}  \Delta_t \right\|_2 + \eta \left\|\sqrt{\mQ}  \left [J(\tilde{\theta}^{(t)})^{\top} J(\theta^{(t)}) - J(\theta^{(0)})^{\top} J(\theta^{(0)}) \right ] \sqrt{\mQ_{3\epsilon}^{(t)}} \right\|_2 \left\| \sqrt{\mQ} g(\theta^{(t)}) \right\|_2 \\ 
\le \; & \gamma \left\| \sqrt{\mQ}  \Delta_t \right\|_2 + \eta \left\| J(\tilde{\theta}^{(t)})^{\top} J(\theta^{(t)}) - J(\theta^{(0)})^{\top} J(\theta^{(0)}) \right\|_F \sqrt{1 + 3 \epsilon} \gamma^{t - t_\epsilon} B^{t_\epsilon} R_0 \\ 
\le \; & \gamma \left\| \sqrt{\mQ} \Delta_t \right\|_2 + 2 \eta M^2 C_0 \sqrt{1 + 3 \epsilon} \gamma^{t - t_\epsilon} B^{t_\epsilon} R_0 \dl^{-1/4}  ,
\end{align*}
which implies that
\begin{equation*}
    \gamma^{-(t+1)} \left\| \sqrt{\mQ} \Delta_{t+1} \right\|_2 \le \gamma^{-t} \left\| \sqrt{\mQ} \Delta_t \right\|_2 + 2 \eta M^2 C_0 \sqrt{1 + 3 \epsilon} \gamma^{-1 - t_\epsilon} B^{t_\epsilon} R_0 \dl^{-1/4} .
\end{equation*}

Next, we consider an arbitrary test point $x$ such that $\norm{x}_2 \le 1$.
Denote $\delta_t = f_{\lin}^{(t)}(x) - f^{(t)}(x)$.
Then, we have
\begin{equation*}
\left \{
\begin{aligned}
f_{\lin}^{(t+1)}(x) - f_{\lin}^{(t)}(x) &= - \eta \nabla_\theta f(x; \theta^{(0)})^{\top} J(\theta^{(0)}) \mQ^{(t)} g_{\lin}(\theta^{(t)})   ;\\ 
f^{(t+1)}(x) - f^{(t)}(x) &= - \eta \nabla_\theta f(x; \tilde{\theta}^{(t)})^{\top} J(\theta^{(t)}) \mQ^{(t)} g(\theta^{(t)})   .
\end{aligned}
\right .
\end{equation*}
Therefore, we have
\begin{align*}
    \delta_{t+1} - \delta_t = &\eta \left [\nabla_\theta f(x; \tilde{\theta}^{(t)})^{\top} J(\theta^{(t)}) - \nabla_\theta f(x; \theta^{(0)})^{\top} J(\theta^{(0)}) \right ] \mQ^{(t)} g(\theta^{(t)}) \\ 
    & -  \eta \nabla_\theta f(x; \theta^{(0)})^{\top} J(\theta^{(0)}) \mQ^{(t)} \Delta_t  .
\end{align*}
For $t \le t_\epsilon$, we have
\begin{align*}
\left\| \delta_{t} \right\|_2 \le  & \eta \sum_{s=0}^{t-1} \left\|  \left [\nabla_\theta f(x; \tilde{\theta}^{(s)})^{\top} J(\theta^{(s)}) - \nabla_\theta f(x; \theta^{(0)})^{\top} J(\theta^{(0)}) \right ] \mQ^{(s)} \right\|_2 \left\|  g(\theta^{(s)}) \right\|_2 
\\ & + \eta \sum_{s=0}^{t-1} \left\| \nabla_\theta f(x; \theta^{(0)})^{\top} J(\theta^{(0)}) \mQ^{(s)} \right\|_2 \left\| \Delta_s \right\|_2 \\ 
\le & \eta \sum_{s=0}^{t-1} \left\| \nabla_\theta f(x; \tilde{\theta}^{(s)})^{\top} J(\theta^{(s)}) - \nabla_\theta f(x; \theta^{(0)})^{\top} J(\theta^{(0)}) \right\|_F \left\|  g(\theta^{(s)}) \right\|_2 \\
& + \eta \sum_{s=0}^{t-1}  \left\| \nabla_\theta f(x; \theta^{(0)}) \right\|_2 \left\| J(\theta^{(0)}) \right\|_F \left\| \Delta_s \right\|_2 \\ 
\le & 2 \eta M^2 C_0 \dl^{-1/4}  \sum_{s=0}^{t-1} B^s R_0  + \eta M^2 \sum_{s=0}^{t-1}(2 s \eta M^2 C_0 B^{s-1} R_0  \dl^{-1/4})  .
\end{align*}
Hence, there exists a constant $C_1$ such that $\left \| \delta_{t_\epsilon} \right\|_2 \le C_1 \dl^{-1/4}$. Then, for $t > t_\epsilon$, we have
\begin{align*}
\left\| \delta_{t} \right\|_2 - \left\| \delta_{t_\epsilon} \right\|_2 \le  & \eta \sum_{s=t_\epsilon}^{t-1} \left\|  \left [\nabla_\theta f(x; \tilde{\theta}^{(s)})^{\top} J(\theta^{(s)}) - \nabla_\theta f(x; \theta^{(0)})^{\top} J(\theta^{(0)}) \right ] \sqrt{\mQ_{3\epsilon}^{(s)}} \right\|_2 \left\| \sqrt{\mQ}  g(\theta^{(s)}) \right\|_2 
\\ & + \eta \sum_{s=t_\epsilon}^{t-1} \left\| \nabla_\theta f(x; \theta^{(0)})^{\top} J(\theta^{(0)}) \sqrt{\mQ_{3\epsilon}^{(s)}} \right\|_2 \left\| \sqrt{\mQ} \Delta_s \right\|_2 \\ 
\le & 2 \eta M^2 C_0 \dl^{-1/4} \sqrt{1 + 3 \epsilon} \sum_{s=t_\epsilon}^{t-1} \gamma^{s - t_\epsilon} B^{t_\epsilon} R_0 \\ 
& + \eta M^2 \sqrt{1 + 3 \epsilon} \sum_{s=t_\epsilon}^{t-1} \left ( 2 \gamma^{s - t_\epsilon} \eta M^2 C_0 B^{t_\epsilon} R_0 \left [ t_\epsilon B^{-1} + \sqrt{1 + 3 \epsilon} \gamma^{-1} (s - t_\epsilon)  \right ] \dl^{-1/4} \right )   .
\end{align*}
Note that $\sum_{t=0}^{\infty} t \gamma^{t}$ is finite as long as $\gamma \in (0,1)$. Therefore, there is a constant $C$ such that for any $t$, $\left\| \delta_t \right\|_2 \le C \dl^{-1/4}$ with probability at least $(1 - \delta)$ for any $\dl \ge \tilde{D}$.
\end{proof}

Now, let us finish the proof of \Cref{thm:grw-reg-nn}.

\begin{proof}
    Consider the linearized neural network $f_{\lin}^{(t)}(x)$ defined in \Cref{eqn:def-linearized-nn}.
    If we view $\oset{\nabla_\theta f^{(0)}(x_i)}_{i=1}^n$ as inputs and $\oset{y_i - f^{(0)}(x_i) + \dotp{\theta^{(0)}, \nabla_\theta f^{(0)} (x_i) }}_{i=1}^n$ as the targets, then the linearized neural network is a linear model. Thus by \Cref{thm:grw-reg-linear}, we have the following corollary.
    \begin{corollary}
    \label{cor:linear-nn-reg}
    If $\nabla_\theta f^{(0)}(x_1),\cdots,\nabla_\theta f^{(0)}(x_n)$ are linearly independent, then there exists $\eta_0 > 0$ such that when $f_{\lin}^{(t)}(x)$ is trained with any GRW method that satisfies \Cref{ass:qstar}, $\theta^{(t)}$ converges to a unique $\theta^*$ that does not depend on the sample weights $q_i^{(t)}$.
    \end{corollary}
    
Let $\eta_1 = \min \{ \eta_0, \eta^* \}$, where $\eta_0$ is given by \Cref{cor:linear-nn-reg}.
Let $f^{(t)}_{\lin}(x)$ and $f^{(t)}_{\lin \erm}(x)$ be the linearized neural networks of $f^{(t)}(x)$ and $f^{(t)}_{\erm}(x)$ (which are two networks trained with GRW and ERM), respectively.
By \Cref{thm:approx}, for any $\delta \in (0,1)$, there exists $\tilde{D} > 0$ and a constant $C$ such that with probability at least $1 - \delta$, the following holds:
\begin{equation*}
\left \{ 
\begin{aligned}
&\sup_{t \ge 0} \left | f^{(t)}_{\lin}(x) - f^{(t)}(x) \right | \le C \dl^{-1/4}  ; \\ 
    &\sup_{t \ge 0} \left | f^{(t)}_{\lin \erm}(x) - f^{(t)}_{\erm}(x) \right | \le C \dl^{-1/4}  .
\end{aligned}
\right .
\end{equation*}
By \Cref{cor:linear-nn-reg}, we have
\begin{equation*}
    \lim_{t \rightarrow \infty} \left| f^{(t)}_{\lin}(x) - f^{(t)}_{\lin \erm}(x) \right| = 0 .
\end{equation*}
Summing the above yields
\begin{equation*}
    \limsup_{t \rightarrow \infty} \left| f^{(t)}(x) - f^{(t)}_{ \erm}(x) \right| \le 2C \dl^{-1/4}  ,
\end{equation*}
as desired.
\end{proof}

\section{Proof of Theorem \ref{thm:grw-reg-nn-regu}}
\label{app:proof-thm-reg-nn-regu}

Similar to \Cref{thm:grw-reg-nn}, proving this result needs a slightly different approximation theorem. We start with two necessary propositions.

\begin{proposition}[\cite{jacot2018ntk}, Proposition~1]
\label{prop:jacot-1}
If $\sigma$ is Lipschitz and $d_l \rightarrow \infty$ for $l = 1,\cdots,L$ sequentially, then for all $l = 1,\cdots,L$, the distribution of a single element of $\vh^l$ converges in probability to a zero-mean Gaussian process of covariance $\Sigma^{l}$ that is defined recursively by:
\begin{align*}
\Sigma^{1}(x, x') &= \frac{1}{d_0} x^{\top}x' + \beta^2  ; \\
\Sigma^{l}(x, x') &= \E_f[\sigma(f(x)) \sigma(f(x'))] + \beta^2  ,
\end{align*}
where $f$ is sampled from a zero-mean Gaussian process of covariance $\Sigma^{(l-1)}$.
\end{proposition}

\begin{proposition}
\label{prop:eigen}
For any positive definite symmetric matrix $\mH \in \R^{n \times n}$, denote its largest and smallest eigenvalues by $\lambda_{\max}$ and $\lambda_{\min}$. Then, for any positive semi-definite diagonal matrix $\mQ = \diag (q_1,\cdots,q_n)$, $\mH \mQ$ has $n$ eigenvalues that all lie in $[\min_i q_i \cdot \lambda_{\min}, \max_i q_i \cdot \lambda_{\max}]$.
\end{proposition}
\begin{proof}
$\mH$ is a positive definite symmetric matrix, so there exists $\mA \in \R^{n \times n}$ such that $\mH = \mA^{\top} \mA$, and $\mA$ is full-rank. First, any eigenvalue of $\mA \mQ \mA^{\top}$ is also an eigenvalue of $\mA^{\top} \mA \mQ$, because for any eigenvalue $\lambda$ of $\mA \mQ \mA^{\top}$ we have some $\vv \neq 0$ such that $\mA \mQ \mA^{\top} \vv = \lambda \vv$. Multiplying both sides by $\mA^{\top}$ on the left yields $\mA^{\top} \mA \mQ (\mA^{\top} \vv) = \lambda (\mA^{\top} \vv)$ which implies that $\lambda$ is also an eigenvalue of $\mA^{\top} \mA \mQ$ because $\mA^{\top} \vv \neq 0$ as $\lambda \vv \neq 0$.

Second, by condition we know that the eigenvalues of $\mA^{\top} \mA$ are all in $[\lambda_{\min}, \lambda_{\max}]$ where $\lambda_{\min} > 0$, which implies for any unit vector $\vv$, $\vv^{\top} \mA^{\top} \mA \vv \in [\lambda_{\min}, \lambda_{\max}]$, which is equivalent to $\left\| \mA \vv \right\|_2 \in [\sqrt{\lambda_{\min}}, \sqrt{\lambda_{\max}}]$. 
Thus, $\vv^{\top} \mA^{\top} \mQ \mA \vv \in [\lambda_{\min} \min_i q_i, \lambda_{\max} \max_i q_i]$, which implies that the eigenvalues of $\mA^{\top} \mQ \mA$ are all in $[\lambda_{\min} \min_i q_i, \lambda_{\max} \max_i q_i]$.

Thus, the eigenvalues of $\mH \mQ = \mA^{\top} \mA \mQ$ are all in $[\lambda_{\min} \min_i q_i, \lambda_{\max} \max_i q_i]$.
\end{proof}

\begin{lemma}[Approximation Theorem for Regularized GRW]
\label{thm:approx-regu}
For a wide fully-connected neural network $f$, denote $J(\theta) = \nabla_\theta f(\mX;\theta) \in \R^{p \times n}$ and $g(\theta) = \nabla_{\hat{y}} \ell(f(\mX; \theta), \mY) \in \R^n$. Given that the loss function $\ell$ satisfies: $\nabla_\theta g(\theta) = J(\theta) \mU(\theta)$ for any $\theta$, and $\mU(\theta)$ is a positive semi-definite diagonal matrix whose elements are uniformly bounded, we have: for any GRW that minimizes the regularized weighted empirical risk \Cref{eqn:grw-l2-pen} with a sufficiently small learning rate $\eta$, there is: for a sufficiently large $\dl$, with high probability over random initialization, on any test point $x$ such that $\| x \|_2 \le 1$, we have
\begin{equation}
    \sup_{t \ge 0} \left | f^{(t)}_{\lin \regu}(x) - f^{(t)}_{\regu}(x) \right | \le C \dl^{-1/4}  ,
\end{equation}
where both $f^{(t)}_{\lin \regu}$ and $f^{(t)}_{\regu}$ are trained by the same regularized GRW and start from the same initial point.
\end{lemma}
\begin{proof}
Without loss of generality, assume that all elements of $\mU(\theta)$ are in $[0,1]$ for all $\theta$, and set $\eta \le (\mu + \lambda_{\min} + \lambda_{\max})^{-1}$.
If the elements of $\mU(\theta)$ are bounded by $[0,C]$, then we can set $\eta \le (\mu + C \lambda_{\min} + C \lambda_{\max})^{-1}$ and prove the result in the same way.

With $L_2$ penalty, the update rule of the GRW for the neural network is:
\begin{equation}
\label{eqn:update-wide-reg}
    \theta^{(t+1)} = \theta^{(t)} - \eta J(\theta^{(t)}) \mQ^{(t)} g(\theta^{(t)}) - \eta \mu (\theta^{(t)} - \theta^{(0)})  .
\end{equation}
And the update rule for the linearized neural network is:
\begin{equation}
\label{eqn:update-lin-reg}
    \theta^{(t+1)}_{\lin} = \theta^{(t)}_{\lin} - \eta J(\theta^{(0)}) \mQ^{(t)} g(\theta_{\lin}^{(t)}) - \eta \mu (\theta_{\lin}^{(t)} - \theta^{(0)})  .
\end{equation}
By \Cref{prop:jacot-1}, $f(x;\theta)$ converges in probability to a zero-mean Gaussian process. Thus, for any $\delta > 0$, there exists a constant $R_0 > 0$ such that with probability at least $(1-\delta/3)$, $\left\| g(\theta^{(0)} ) \right\|_2 < R_0$. 
Let $M$ be given by \Cref{lem:local-lip}. Denote $A = \eta M R_0$, and let $C_0 = \frac{4A}{\eta \mu}$ be given by \Cref{lem:local-lip}. Note that \Cref{lem:local-lip} only depends on the network structure, but does not depend on the update rule, so we can still use the lemma here.
By \Cref{lem:local-lip}, there exists $D_1$ such that for all $\dl \ge D_1$, with probability at least $(1 - \delta/3)$, \Cref{eqn:lem-local-lip} holds.

Similar to the proof of \Cref{prop:eigen}, we can show that for any $\tilde{\theta}$, all the non-zero eigenvalues of $J(\theta^{(0)}) \mQ^{(t)} U(\tilde{\theta}) J(\theta^{(0)})^{\top}$ are also eigenvalues of $J(\theta^{(0)})^{\top}J(\theta^{(0)}) \mQ^{(t)} U(\tilde{\theta}) $.
This is because for any eigenvalue $\lambda \neq 0$ such that $J(\theta^{(0)}) \mQ^{(t)} U(\tilde{\theta}) J(\theta^{(0)})^{\top} \vv = \lambda \vv$ for some $\vv$, we must have $J(\theta^{(0)})^{\top}J(\theta^{(0)}) \mQ^{(t)} U(\tilde{\theta}) (J(\theta^{(0)})^{\top} \vv) = \lambda (J(\theta^{(0)})^{\top} \vv) $, and $J(\theta^{(0)})^{\top} \vv \neq 0$ since $\lambda \vv \neq 0$, so $\lambda$ is also an eigenvalue of $J(\theta^{(0)})^{\top}J(\theta^{(0)}) \mQ^{(t)} U(\tilde{\theta})$.
On the other hand, by \Cref{lem:ntk}, $J(\theta^{(0)})^{\top}J(\theta^{(0)}) \mQ^{(t)}U(\tilde{\theta}) $ converges in probability to $\Theta \mQ^{(t)}U(\tilde{\theta})$ whose eigenvalues are all in $[0, \lambda_{\max}]$ by \Cref{prop:eigen}.
Hence, there exists $D_2$ such that for all $\dl \ge D_2$, with probability at least $(1 - \delta/3)$, the eigenvalues of $J(\theta^{(0)}) \mQ^{(t)} U(\tilde{\theta}) J(\theta^{(0)})^{\top}$ are all in $[0, \lambda_{\max} + \lambda_{\min}]$ for all $t$. 

By union bound, with probability at least $1 - \delta$, all the above ``$(1-\delta/3)$'' statements are true. 
Now we prove that there exists $D_0$ such that for all $\dl \ge D_0$, $\sup_{t \ge 0} \left\| \theta^{(t)} - \theta^{(0)} \right\|_2$ is bounded with high probability. Denote $a_t = \theta^{(t)} - \theta^{(0)}$.
By \Cref{eqn:update-wide-reg}, we have
\begin{align*}
    a_{t+1} = &(1 - \eta \mu) a_t - \eta [J(\theta^{(t)}) - J(\theta^{(0)})] \mQ^{(t)} g(\theta^{(t)}) \\ 
    & - \eta J(\theta^{(0)}) \mQ^{(t)} [g(\theta^{(t)}) - g(\theta^{(0)})] - \eta J(\theta^{(0)})\mQ^{(t)} g(\theta^{(0)})  ,
\end{align*}
which implies that
\begin{equation}
\label{eqn:thm10-1}
\begin{aligned}
    \left\| a_{t+1} \right\|_2 \le & \left\| (1-\eta \mu) \mI - \eta  J(\theta^{(0)}) \mQ^{(t)} U(\tilde{\theta}^{(t)}) J(\tilde{\theta}^{(t)})^{\top} \right\|_2 \left\| a_t \right\|_2 \\ 
    & + \eta \left\| J(\theta^{(t)}) - J(\theta^{(0)}) \right\|_F \left\| g(\theta^{(t)}) \right\|_2 + \eta \left\| J(\theta^{(0)}) \right\|_F \left\| g(\theta^{(0)}) \right\|_2 ,
\end{aligned}
\end{equation}
where $\tilde{\theta}^{(t)}$ is some linear interpolation between $\theta^{(t)}$ and $\theta^{(0)}$. Our choice of $\eta$ ensures that $\eta \mu < 1$.  

Next, we prove by induction that $\norm{a_t}_2 < C_0$.
It is true for $t=0$.
Suppose $\norm{a_t}_2 < C_0$, and consider $a_{t+1}$.
Let us look at the three terms on the right-hand side of \Cref{eqn:thm10-1}.
For the first term, we have
\begin{align*}
    & \left\| (1-\eta \mu) \mI - \eta  J(\theta^{(0)}) \mQ^{(t)} U(\tilde{\theta}^{(t)})J(\tilde{\theta}^{(t)})^{\top} \right\|_2 \\ 
    \le \; &  (1 - \eta \mu) \left\| \mI - \frac{\eta}{1 - \eta \mu} J(\theta^{(0)}) \mQ^{(t)} U(\tilde{\theta}^{(t)}) J(\theta^{(0)})^{\top} \right\|_2\\ 
    & + \eta \left\| J(\theta^{(0)}) \right\|_F \left\| J(\tilde{\theta}^{(t)}) - J(\theta^{(0)}) \right\|_F  .
\end{align*}
Since $\eta / (1 - \eta \mu) \le (\lambda_{\min} + \lambda_{\max})^{-1}$ by our choice of $\eta$, we have  
\begin{equation*}
    \left\| \mI - \frac{\eta}{1 - \eta \mu} J(\theta^{(0)}) \mQ^{(t)}  U(\tilde{\theta}^{(t)})J(\theta^{(0)})^{\top} \right\|_2 \le 1 .
\end{equation*}
On the other hand, since $\left\| a_t \right\|_2 < C_0$, we have $\left\| J(\theta^{(0)}) \right\|_F \left\| J(\tilde{\theta}^{(t)}) - J(\theta^{(0)}) \right\|_F \le \frac{M^2}{\sqrt[4]{\dl}}C_0$ by \Cref{eqn:lem-local-lip}. Therefore, there exists $D_3$ such that for all $\dl \ge D_3$,
\begin{equation}
\label{eqn:thm10-3}
    \left\| (1-\eta \mu) \mI - \eta  J(\theta^{(0)}) \mQ^{(t)} U(\tilde{\theta}^{(t)}) J(\tilde{\theta}^{(t)})^{\top} \right\|_2 \le 1 - \frac{\eta \mu}{2}  .
\end{equation}

For the second term, we have
\begin{equation}
\label{eqn:thm10-2}
\begin{aligned}
    \left\| g(\theta^{(t)}) \right\|_2 & \le \left\| g(\theta^{(t)}) - g(\theta^{(0)}) \right\|_2 + \left\| g(\theta^{(0)}) \right\|_2 \\ 
    & \le \left\| J(\tilde{\theta}^{(t)}) \right\|_2 \left \| U(\tilde{\theta}^{(t)}) \right \|_2 \left\| \theta^{(t)} - \theta^{(0)} \right\|_2 + R_0 \le M C_0 + R_0  .
\end{aligned}
\end{equation}

For the third term, we have
\begin{equation*}
    \eta \left\| J(\theta^{(0)}) \right\|_F \left\| g(\theta^{(0)}) \right\|_2 \le \eta M R_0 = A .
\end{equation*}
Thus, we have
\begin{equation*}
    \left\| a_{t+1} \right\|_2 \le \left ( 1 - \frac{\eta \mu}{2} \right )\left\| a_t \right\|_2 +  \frac{\eta M(M C_0 + R_0)}{\sqrt[4]{\dl}} + A .
\end{equation*}
Thus, there exists $D_4$ such that for all $\dl \ge D_4$, $ \left\| a_{t+1} \right\|_2 \le \left ( 1 - \frac{\eta \mu}{2} \right )\left\| a_t \right\|_2 + 2A$. This shows that if $\left\| a_t \right\|_2 < C_0$ is true, then $\left\| a_{t+1} \right\|_2 < C_0$ will also be true.

In conclusion, for all $\dl \ge D_0 = \max \{ D_1,D_2,D_3,D_4 \}$, $\left\| \theta^{(t)} - \theta^{(0)} \right\|_2 < C_0$ is true for all $t$. This also implies that for $C_1 = M C_0 + R_0$, we have $\left\| g(\theta^{(t)}) \right\|_2 \le C_1$ for all $t$ by \Cref{eqn:thm10-2}. Similarly, we can prove that $\| \theta_{\lin}^{(t)} - \theta^{(0)} \|_2 < C_0$ for all $t$.

Second, let $\Delta_t = \theta_{\lin}^{(t)} - \theta^{(t)}$. Then, we have
\begin{equation*}
    \Delta_{t+1} - \Delta_t = \eta (J(\theta^{(t)}) \mQ^{(t)} g(\theta^{(t)}) - J(\theta^{(0)}) \mQ^{(t)} g(\theta_{\lin}^{(t)}) - \mu \Delta_t)  ,
\end{equation*}
which implies that
\begin{equation*}
    \Delta_{t+1} = \left [ (1 - \eta \mu) \mI - \eta J(\theta^{(0)}) \mQ^{(t)} U(\tilde{\theta}^{(t)})J(\tilde{\theta}^{(t)})^{\top}  \right ] \Delta_t + \eta (J(\theta^{(t)}) - J(\theta^{(0)})) \mQ^{(t)} g(\theta^{(t)})  ,
\end{equation*}
where $\tilde{\theta}^{(t)}$ is some linear interpolation between $\theta^{(t)}$ and $\theta_{\lin}^{(t)}$. By \Cref{eqn:thm10-3}, with probability at least $(1-\delta)$ for all $\dl \ge D_0$, we have
\begin{align*}
    \left\| \Delta_{t+1} \right\|_2 &\le \left\| (1 - \eta \mu) \mI - \eta J(\theta^{(0)}) \mQ^{(t)}U(\tilde{\theta}^{(t)}) J(\tilde{\theta}^{(t)})^{\top} \right\|_2 \left\| \Delta_t \right\|_2 + \eta \left\| J(\theta^{(t)}) - J(\theta^{(0)}) \right\|_F \left\| g(\theta^{(t)}) \right\|_2 \\ 
    & \le \left ( 1 - \frac{\eta \mu}{2} \right ) \left\| \Delta_t \right\|_2 + \eta \frac{M}{\sqrt[4]{\dl}} C_0C_1  .
\end{align*}

Again, as $\Delta_0 = 0$, we can prove by induction that for all $t$,
\begin{equation*}
    \left\| \Delta_t \right\|_2 < \frac{2 M C_0 C_1}{\mu} \dl^{-1/4}  .
\end{equation*}

For any test point $x$ such that $\norm{x}_2 \le 1$, we have
\begin{align*}
    \left | f_{\regu}^{(t)}(x) - f_{\lin \regu}^{(t)}(x) \right | &= \left| f(x; \theta^{(t)}) - f_{\lin}(x; \theta^{(t)}_{\lin}) \right|  \\
    & \le \left| f(x; \theta^{(t)}) - f_{\lin}(x; \theta^{(t)}) \right | + \left | f_{\lin}(x; \theta^{(t)}) - f_{\lin}(x; \theta^{(t)}_{\lin}) \right | \\ 
    & \le \left| f(x; \theta^{(t)}) - f_{\lin}(x; \theta^{(t)}) \right | + \left\| \nabla_\theta f(x; \theta^{(0)}) \right\|_2 \left\| \theta^{(t)} - \theta_{\lin}^{(t)} \right\|_2  \\ 
    & \le \left| f(x; \theta^{(t)}) - f_{\lin}(x; \theta^{(t)}) \right | + M \left\| \Delta_t \right\|_2  .
\end{align*}

For the first term, note that
\begin{equation*}
\left \{ 
\begin{aligned}
f(x; \theta^{(t)}) - f(x; \theta^{(0)}) &= \nabla_{\theta} f(x; \tilde{\theta}^{(t)})(\theta^{(t)} - \theta^{(0)})  ; \\ 
f_{\lin}(x; \theta^{(t)}) - f_{\lin}(x; \theta^{(0)}) &= \nabla_{\theta} f(x; \theta^{(0)})(\theta^{(t)} - \theta^{(0)})  ,
\end{aligned}
\right . 
\end{equation*}
where $\tilde{\theta}^{(t)}$ is some linear interpolation between $\theta^{(t)}$ and $\theta^{(0)}$. Since $f(x; \theta^{(0)}) = f_{\lin}(x; \theta^{(0)})$,
\begin{equation*}
    \left| f(x; \theta^{(t)}) - f_{\lin}(x; \theta^{(t)}) \right | \le \left\| \nabla_{\theta} f(x; \tilde{\theta}^{(t)}) -  \nabla_{\theta} f(x; \theta^{(0)}) \right\|_2 \left\| \theta^{(t)} - \theta^{(0)} \right\|_2 \le \frac{M}{\sqrt[4]{\dl}} C_0^2  .
\end{equation*}
Thus, for all $\dl \ge D_0$, with probability at least $(1 - \delta)$ for all $t$ and all $x$,
\begin{equation*}
    \left | f_{\regu}^{(t)}(x) - f_{\lin \regu}^{(t)}(x) \right | \le \left( MC_0^2 + \frac{2 M^2 C_0 C_1}{\mu} \right) \dl^{-1/4} = O(\dl^{-1/4})  ,
\end{equation*}
which proves the lemma.
\end{proof}

\begin{lemma}
\label{lem:lin-reg}
    Suppose there exists $M_0 > 0$ such that $\left\| \nabla_\theta f^{(0)}(x) \right\|_2 \le M_0$ for all $x$ within the unit ball.
    If the gradients $\nabla_{\theta} f^{(0)}(x_1) ,\cdots,\nabla_{\theta} f^{(0)}(x_n) $ are linearly independent, and the empirical training risk of $f_{\lin \regu}^{(t)}$ satisfies $\limsup_{t \rightarrow \infty} \hat{\gR}(f_{\lin \regu}^{(t)}) < \epsilon$ for some $\epsilon > 0$, then for any $x$ within the unit ball, we have
    \begin{equation*}
        \limsup_{t \rightarrow \infty} \left | f_{\lin \regu}^{(t)}(x) - f_{\lin \erm}^{(t)}(x) \right | = O(\sqrt{\epsilon})  .
    \end{equation*}
\end{lemma}
\begin{proof}
    First, for all $t$ we have $\theta^{(t)} - \theta^{(0)} \in \sspan \{ \nabla_\theta f^{(0)}(x_1),\cdots,\nabla_\theta f^{(0)}(x_n) \}$.
    Let $\theta^*$ be the interpolator in $\sspan(\nabla_\theta f^{(0)}(x_1),\cdots,\nabla_\theta f^{(0)}(x_n))$, then the empirical risk of $\theta$ is $\frac{1}{2n} \sum_{i=1}^n \langle \theta - \theta^*, \nabla_\theta f^{(0)}(x_i) \rangle^2 = \frac{1}{2n} \left\| \nabla_\theta f^{(0)}(\mX)^{\top} (\theta - \theta^*) \right\|_2^2$.
    Thus, there exists $T> 0$ such that
    \begin{equation*}
         \left\| \nabla_\theta f^{(0)}(\mX)^{\top}(\theta^{(t)} - \theta^*) \right\|_2^2 \le 2n \epsilon \qquad \text{for all } t \ge T  .
    \end{equation*}

    Let the smallest singular value of $\frac{1}{\sqrt{n}} \nabla_\theta f^{(0)}(\mX)$ be $s_{\min}$.
    Then, we have $s_{\min} > 0$.
    Note that the column space of $\nabla_\theta f^{(0)}(\mX)$ is exactly $\sspan(\nabla_\theta f^{(0)}(x_1),\cdots,\nabla_\theta f^{(0)}(x_n))$. Define $\mH \in \R^{p \times n}$ such that its columns form an orthonormal basis of this subspace, then there exists $\mG \in R^{n \times n}$ such that $\nabla_\theta f^{(0)}(\mX) = \mH \mG$, and the smallest singular value of $\frac{1}{\sqrt{n}} \mG$ is also $s_{\min}$. Since $\theta^{(t)} - \theta^{(0)}$ is also in this subspace, there exists $\vv \in \R^n$ such that $\theta^{(t)} - \theta^* = \mH \vv$. Then we have $ \sqrt{2n \epsilon} \ge \left\| \mG^{\top} \mH^{\top} \mH \vv \right\|_2 = \left\| \mG^{\top} \vv \right\|_2$. Thus, $\left\| \vv \right\|_2 \le \frac{\sqrt{2 \epsilon}}{s_{\min}}$, which implies that 
    \begin{equation*}
        \left\| \theta^{(t)} - \theta^* \right\|_2 \le \frac{\sqrt{2 \epsilon}}{s_{\min}}  .
    \end{equation*}
    
    We have already proved in the previous results that if we minimize the unregularized risk with ERM, then $\theta$ always converges to the interpolator $\theta^*$. So for any $t \ge T$ and any test point $x$ such that $\left\| x \right\|_2 \le 1$,  we have
    \begin{equation*}
        |f_{\lin \regu}^{(t)}(x) - f_{\lin \erm}^{(t)}(x)| = |\langle \theta^{(t)} - \theta^*, \nabla_\theta f^{(0)}(x) \rangle | \le \frac{M_0 \sqrt{2 \epsilon}}{s_{\min}}  ,
    \end{equation*}
    as desired.
\end{proof}

Now we prove \Cref{thm:grw-reg-nn-regu}.
\begin{proof}
    Given that $\hat{\gR}(f_{\lin \regu}^{(t)}) < \epsilon$ for sufficiently large $t$, \Cref{thm:approx-regu} implies that
    \begin{equation*}
        \left | \hat{\gR}(f_{\lin \regu}^{(t)}) - \hat{\gR}(f_{\regu}^{(t)}) \right | = O(\dl^{-1/4} \sqrt{\epsilon} + \dl^{-1/2}) .
    \end{equation*}
    Thus, for a fixed $\epsilon$, there exists $D>0$ such that for all $d \ge D$, for sufficiently large $t$,
    \begin{equation*}
        \hat{\gR}(f_{\regu}^{(t)}) < \epsilon \Rightarrow \hat{\gR}(f_{\lin \regu}^{(t)}) < 2 \epsilon  .
    \end{equation*}
    By \Cref{thm:approx} and \Cref{thm:approx-regu}, we have
    \begin{equation*}
    \left \{
    \begin{aligned}
        & \sup_{t \ge 0} \left| f_{\lin \erm}^{(t)}(x) - f_{\erm}^{(t)}(x) \right| = O(\dl^{-1/4})  ; \\
& \sup_{t \ge 0} \left| f_{\lin \regu}^{(t)}(x) - f_{\regu}^{(t)}(x) \right| = O(\dl^{-1/4})  .
    \end{aligned}
    \right .
    \end{equation*}
    Combining these with \Cref{lem:lin-reg} yields
    \begin{equation*}
         \limsup_{t \rightarrow \infty} \left | f_{\regu}^{(t)}(x) - f_{\erm}^{(t)}(x) \right | = O(\dl^{-1/4} + \sqrt{\epsilon})  .
    \end{equation*}
Letting $\dl \rightarrow \infty$ leads to the result we need.
\end{proof}

\begin{remark}
    One might wonder whether $\| \nabla_\theta f^{(0)}(x) \|_2$ will diverge as $\dl \rightarrow \infty$. 
    In fact, in \Cref{lem:local-lip}, we have proved that there exists a constant $M$ such that with high probability, for any $\dl$ there is $\| \nabla_\theta f^{(0)}(x) \|_2 \le M$ for any $x$ such that $\| x \|_2 \le 1$. Therefore, it is fine to suppose that there exists such an $M_0$.
\end{remark}

\section{Proof of Theorem \ref{thm:grw-cls-linear}}
\label{app:proof-thm-grw-cls-linear}
\begin{proof}
    First, we show that $\hat{\theta}_{\mm}$ is unique.
    Suppose both $\theta_1$ and $\theta_2$ maximize $\min_i y_i \cdot \dotp{\theta, \vx_i}$ and $\norm{\theta_1}_2 = \norm{\theta_2}_2 = 1$.
    If $\theta_1 \neq \theta_2$, then we define $\theta = (\theta_1 + \theta_2)/2$ and $\theta_0 = \theta / \norm{\theta}_2$. Obviously, $\norm{\theta}_2 < 1$, and for all $i \in [n]$ there is $y_i \dotp{\theta, \vx_i} = ( y_i \dotp{\theta_1, \vx_i} + y_i \dotp{\theta_2, \vx_i} )/2$.
    Thus, we have $y_i \dotp{\theta_0, \vx_i} > \min \oset{ y_i \dotp{\theta_1, \vx_i}, y_i \dotp{\theta_2, \vx_i} }$,
    which implies that $\min_i y_i \dotp{\theta_0, \vx_i} > \min \oset{ \min_i y_i \dotp{\theta_1, \vx_i}, \min_i y_i \dotp{\theta_2, \vx_i} }$, which contradicts the fact that $\theta_1,\theta_2$ are max-margin classifiers.

    Without loss of generality, let $(\vx_1,y_1),\cdots,(\vx_m,y_m)$ be the samples with the smallest margin for $\vu$, that is $\argmin_i y_i \dotp{\vu, \vx_i} = \oset{1,\cdots,m}$.
    Denote $\gamma = \min_i y_i \dotp{\vu, \vx_i}$; then, $\gamma > 0$ since the training error converges to zero. Note that for the logistic loss, if $y_i \dotp{\theta, \vx_i} < y_j \dotp{\theta, \vx_j}$, then for any $M > 0$, there exists an $R_M > 0$ such that for all $R \ge R_M$, there is $\frac{\nabla_\theta \ell(\langle R \theta, \vx_i \rangle, y_i)}{\nabla_\theta \ell( \langle R \theta, \vx_j \rangle, y_j)} > M$.
    Since the training error converges to zero, we have $\norm{\theta^{(t)}}_2 \rightarrow \infty$. 
    So when $t$ is sufficiently large, the impact of $(\vx_j, y_j)$ on $\theta^{(t)}$ for $j > m$ is an infinitesimal compared to $j \le m$ since $\liminf_{t \rightarrow \infty} q_j^{(t)} > 0$.
    Thus, we must have $\vu \in \sspan \oset{\vx_1,\cdots,\vx_m}$.

    Let $\vu = \alpha_1 y_1 \vx_1 + \cdots + \alpha_m y_m \vx_m$. Now we show that $\alpha_i \ge 0$ for all $i=1,\cdots,m$.
    For a sufficiently large $t$, there is $\theta^{(t+1)} - \theta^{(t)} \approx \eta \sum_{i=1}^m \frac{q_i^{(t)}\exp(y_i \cdot \langle \theta^{(t)}, \vx_i \rangle)}{1 + \exp(y_i \cdot \langle \theta^{(t)}, \vx_i \rangle)} y_i \vx_i$.
    Since $\norm{\theta^{(t)}} \rightarrow \infty$, for all $i \in [m]$ we have $\alpha_i \propto \lim_{T \rightarrow \infty} \sum_{t=T_0}^{T} \frac{q_i^{(t)}\exp(y_i \cdot \langle \theta^{(t)}, \vx_i \rangle)}{1 + \exp(y_i \cdot \langle \theta^{(t)}, \vx_i \rangle)} := \lim_{T \rightarrow \infty} \alpha_i(T)$, where $T_0$ is sufficiently large.
    Here the notion $\alpha_i \propto \lim_{T \rightarrow \infty} \alpha_i(T)$ means that $\lim_{T \rightarrow \infty} \frac{\alpha_i(T)}{\alpha_j(T)} = \frac{\alpha_i}{\alpha_j}$ for any pair of $i,j$ and $\alpha_j \neq 0$. Note that each term in the sum is non-negative. This implies that $\alpha_1,\cdots,\alpha_m$ have the same sign.
    Meanwhile, $\sum_{i=1}^m \alpha_i \gamma = \sum_{i=1}^m \alpha_i y_i \cdot \langle \vu, \vx_i \rangle = \langle \vu, \vu \rangle > 0$. Thus $\alpha_i \ge 0$ for all $i \in [m]$ and at least one is positive.
    Now suppose $\vu \neq \hat{\theta}_{\mm}$, which means that $\gamma$ is smaller than the margin of $\hat{\theta}_{\mm}$.
    Then, for all $i=1,\cdots,m$, there is $y_i \cdot \langle \vu, \vx_i \rangle < y_i \cdot \langle \hat{\theta}_{\mm}, \vx_i \rangle$. This implies that $\langle \vu, \vu \rangle = \sum_{i=1}^m \alpha_i y_i \cdot \langle \vu, \vx_i \rangle < \sum_{i=1}^m \alpha_i y_i \cdot \langle \hat{\theta}_{\mm}, \vx_i \rangle = \langle \hat{\theta}_{\mm}, \vu \rangle$,
    which is a contradiction. Thus, $\vu = \hat{\theta}_{\mm}$.
\end{proof}

\section{Proof of Theorem \ref{thm:grw-cls-converge}}
\label{app:proof-thm-grw-cls-converge}

\begin{proof}
Denote the largest and smallest eigenvalues of $\mX^{\top} \mX$ by $\lambda^{\max}$ and $\lambda^{\min}$, and by condition we have $\lambda^{\min} > 0$. Let $\epsilon =\min \{ \frac{q^*}{3}, \frac{(q^* \lambda^{\min})^2}{192 \lambda^{\max 2}} \}$.
Then, similar to the proof in \Cref{app:proof-thm-grw-reg-linear}, there exists $t_\epsilon$ such that for all $t \ge t_\epsilon$ and all $i$, $q_i^{(t)} \in (q_i - \epsilon, q_i + \epsilon)$. Denote $\mQ = \diag (q_1,\cdots,q_n)$, then for all $t \ge t_\epsilon$, $\mQ^{(t)} := \mQ_\epsilon^{(t)} = \sqrt{\mQ} \sqrt{\mQ_{3 \epsilon}^{(t)}}$, where we use the subscript $\epsilon$ to indicate that $\left \| \mQ_\epsilon^{(t)} - \mQ \right \|_2 < \epsilon$.

First, we prove that $F(\theta)$ is $L$-smooth if $\| x_i \|_2 \le 1$ for all $i$. The gradient of $F$ is
\begin{equation*}
    \nabla F(\theta) = \sum_{i=1}^n q_i \nabla_{\hat{y}}\ell(\langle \theta,\vx_i \rangle, y_i) \vx_i  .
\end{equation*}

Since $\ell(\hat{y},y)$ is $L$-smooth in $\hat{y}$, for any $\theta_1,\theta_2$ and any $i$, we have
\begin{align*}
    & \ell(\langle \theta_2, \vx_i \rangle, y_i ) - \ell(\langle \theta_1, \vx_i \rangle, y_i ) \\ 
    \le \;  & \nabla_{\hat{y}}\ell(\langle \theta_1, \vx_i \rangle, y_i) \cdot (\langle \theta_2, \vx_i \rangle - \langle \theta_1, \vx_i \rangle) + \frac{L}{2} (\langle \theta_2, \vx_i \rangle - \langle \theta_1, \vx_i \rangle)^2 \\ 
    = \; & \langle \nabla_{\hat{y}}\ell(\langle \theta_1, \vx_i \rangle, y_i) \cdot \vx_i , \theta_2 - \theta_1 \rangle + \frac{L}{2} (\langle \theta_2 - \theta_1, \vx_i \rangle )^2 \\
    \le  \; & \langle \nabla_{\hat{y}}\ell(\langle \theta_1, \vx_i \rangle, y_i) \cdot \vx_i , \theta_2 - \theta_1 \rangle + \frac{L}{2} \left\| \theta_2 - \theta_1 \right\|_2^2  .
\end{align*}
Thus, we have
\begin{align*}
    F(\theta_2) - F(\theta_1) = & \sum_{i=1}^n q_i \left[ \ell(\langle \theta_2, \vx_i \rangle, y_i ) - \ell(\langle \theta_1, \vx_i \rangle, y_i ) \right] \\ 
    \le & \sum_{i=1}^n q_i \langle \nabla_{\hat{y}}\ell(\langle \theta_1, \vx_i \rangle, y_i) \cdot \vx_i , \theta_2 - \theta_1 \rangle + \frac{L}{2} \sum_{i=1}^n q_i \left\| \theta_2 - \theta_1 \right\|_2^2 \\ 
    = & \langle \nabla F(\theta_1), \theta_2 - \theta_1 \rangle + \frac{L}{2} \left\| \theta_2 - \theta_1 \right\|_2^2  ,
\end{align*}
which implies that $F(\theta)$ is $L$-smooth.

Denote $\tilde{g}(\theta) = \nabla_{\hat{y}} \ell(f(\mX;\theta), \mY) \in \R^n$, then $\nabla F(\theta^{(t)}) = \mX \mQ \tilde{g}(\theta^{(t)})$, and the update rule is
\begin{equation}
\label{eqn:update-134}
    \theta^{(t+1)} = \theta^{(t)} - \eta \mX \mQ^{(t)} \tilde{g}(\theta^{(t)})  .
\end{equation}
So by \Cref{def:smooth}, we have
\begin{equation}
\label{eqn:proof-135}
    F(\theta^{(t+1)}) \le F(\theta^{(t)}) - \eta \langle \mX \mQ \tilde{g}(\theta^{(t)}), \mX \mQ^{(t)} \tilde{g}(\theta^{(t)}) \rangle + \frac{\eta^2 L}{2} \left\|  \mX \mQ^{(t)} \tilde{g}(\theta^{(t)}) \right\|_2^2  .
\end{equation}

Let $\eta_1 = \frac{q^* \lambda^{\min}}{2 L (1+3\epsilon) \lambda^{\max}}$. Similar to the proof in \Cref{app:proof-thm-grw-reg-linear}, we can prove that for all $\eta \le \eta_1$, and for all $t \ge t_\epsilon$, we have
\begin{align*}
    F(\theta^{(t+1)}) & \le F(\theta^{(t)}) - \frac{\eta q^* \lambda^{\min}}{2} \left\| \sqrt{\mQ} \tilde{g}(\theta^{(t)}) \right\|_2^2 + \frac{\eta^2 L}{2} \left\| \mX \sqrt{\mQ_{3\epsilon}^{(t)}} \right\|_2^2 \left\| \sqrt{\mQ} \tilde{g}(\theta^{(t)}) \right\|_2^2 \\
    & \le F(\theta^{(t)}) - \frac{\eta q^* \lambda^{\min}}{2} \left\| \sqrt{\mQ} \tilde{g}(\theta^{(t)}) \right\|_2^2 + \frac{\eta^2 L}{2} \left \| \mX \right \|_2^2 (1+3\epsilon) \left\| \sqrt{\mQ} \tilde{g}(\theta^{(t)}) \right\|_2^2 \\ 
    & \le F(\theta^{(t)})- \frac{\eta q^* \lambda^{\min}}{4} \left\| \sqrt{\mQ} \tilde{g}(\theta^{(t)}) \right\|_2^2 \\ 
    & \le F(\theta^{(t)}) - \frac{\eta q^{*2} \lambda^{\min}}{4} \left\| \tilde{g}(\theta^{(t)}) \right\|_2^2  .
\end{align*}
This implies that $F(\theta^{(t)})$ is monotonically non-increasing. Since $F(\theta) \ge 0$, $F(\theta^{(t)})$ must converge as $t \rightarrow \infty$, and we need to prove that it converges to 0. Suppose that $F(\theta^{(t)})$ does not converge to 0, then there exists a constant $C > 0$ such that $F(\theta^{(t)}) \ge 2C$ for all $t$. On the other hand, it is easy to see that there exists $\theta^*$ such that $\ell(\langle \theta^*, \vx_i \rangle, y_i) < C$ for all $i$. 
The above inequality also implies that $\left\| \tilde{g}(\theta^{(t)}) \right\|_2 \rightarrow 0$ as $t \rightarrow \infty$ because we must have $F(\theta^{(t)}) - F(\theta^{(t+1)}) \rightarrow 0$.

Note that by \Cref{eqn:update-134}, we have
\begin{equation*}
    \left \| \theta^{(t+1)} - \theta^* \right \|_2^2 = \left \| \theta^{(t)} - \theta^* \right \|_2^2 + 2 \eta \langle \mX \mQ^{(t)} \tilde{g}(\theta^{(t)}), \theta^* - \theta^{(t)} \rangle + \eta^2 \left \| \mX \mQ^{(t)} \tilde{g}(\theta^{(t)}) \right\|_2^2  .
\end{equation*}
Define $F_t(\theta) = \sum_{i=1}^n q_i^{(t)} \ell(\langle \theta, \vx_i \rangle, y_i)$.
$F_t$ is convex because $\ell$ is convex and $q_i^{(t)}$ are non-negative, and $\nabla F_t(\theta^{(t)}) = \mX \mQ^{(t)} \tilde{g}(\theta^{(t)})$. Convexity guarantees that $F_t(\vy) \ge F_t(\vx) + \langle \nabla F_t(\vx), \vy - \vx \rangle$, so for all $t$ we have
\begin{equation*}
    \langle \mX \mQ^{(t)} \tilde{g}(\theta^{(t)}), \theta^* - \theta^{(t)} \rangle \le F_t(\theta^*) - F_t(\theta^{(t)}) \le F_t(\theta^*) - \frac{2}{3} F(\theta^{(t)}) \le C - \frac{4C}{3} = -\frac{C}{3} ,
\end{equation*}
because $q_i^{(t)} \ge q_i - \epsilon \ge \frac{2}{3} q_i$ and $\sum_{i=1}^n q_i^{(t)} = 1$. Since $\left\| \tilde{g}(\theta^{(t)}) \right\|_2 \rightarrow 0$, there exists $T > 0$ such that for all $t \ge T$ and all $\eta \le \eta_0$,
\begin{equation*}
    \left \| \theta^{(t+1)} - \theta^* \right \|_2^2 \le \left \| \theta^{(t)} - \theta^* \right \|_2^2  - \frac{\eta C}{3}  ,
\end{equation*}
which means that $ \left \| \theta^{(t)} - \theta^* \right \|_2^2 \rightarrow - \infty$ because $\frac{\eta C}{3}$ is a positive constant. This is a contradiction! Thus, $F(\theta^{(t)})$ must converge to 0, which is result (i).

(i) immediately implies (ii) because $\ell$ is strictly decreasing to 0 by the condition.

Now let us prove (iii).
First of all, the uniqueness of $\theta_R$ can be easily proved from the convexity of $F(\theta)$. The condition implies that $y_i \langle \theta_R, \vx_i \rangle > 0$, i.e. $\theta_R$ must classify all training samples correctly. If there are two different minimizers $\theta_R$ and $\theta_R'$ in whose norm is at most $R$, then consider $\theta_R'' = \frac{1}{2}(\theta_R + \theta_R')$. By the convexity of $F$, we know that $\theta_R''$ must also be a minimizer, and $\| \theta_R'' \|_2 < R$. Thus, $F(\frac{R}{\| \theta_R'' \|_2} \theta_R'') < F(\theta_R'')$ and $\| \frac{R}{\| \theta_R'' \|_2} \theta_R'' \|_2=R$, which contradicts with the fact that $\theta_R''$ is a minimizer.

To prove the rest of (iii), look at \Cref{eqn:proof-135}. On one hand, for all $t \ge t_\epsilon$, we have
\begin{equation*}
    \left | \langle \mX \mQ^{(t)} \tilde{g}(\theta^{(t)}), \mX (\mQ^{(t)} - \mQ) \tilde{g}(\theta^{(t)}) \rangle \right | \le \lambda^{\max} \sqrt{3 \epsilon} \left \| \sqrt{\mQ^{(t)}} \tilde{g}(\theta^{(t)}) \right \|_2^2  .
\end{equation*}
Since we chose $\epsilon = \min \{\frac{q^*}{3}, \frac{(q^* \lambda^{\min})^2}{192 \lambda^{\max 2}} \}$, this inequality implies that
\begin{align*}
    \left \| \nabla F_t(\theta^{(t)}) \right \|_2^2 &= \left \| \mX \mQ^{(t)} \tilde{g}(\theta^{(t)}) \right \|_2^2 \ge \lambda^{\min}  \left \| \mQ^{(t)} \tilde{g}(\theta^{(t)}) \right \|_2^2 \ge \lambda^{\min} (q^*-\epsilon) \left \| \sqrt{\mQ^{(t)}} \tilde{g}(\theta^{(t)}) \right\|_2^2 \\ 
    & \ge \frac{\lambda^{\min} q^*}{2} \left \| \sqrt{\mQ^{(t)}} \tilde{g}(\theta^{(t)}) \right\|_2^2 \ge 4\left | \langle \mX \mQ^{(t)} \tilde{g}(\theta^{(t)}), \mX (\mQ^{(t)} - \mQ) \tilde{g}(\theta^{(t)}) \rangle \right | .
\end{align*}
On the other hand, if $\eta \le \eta_2 = \frac{1}{2L}$, then we have
\begin{equation*}
     \frac{\eta^2 L}{2} \left\|  \mX \mQ^{(t)} \tilde{g}(\theta^{(t)}) \right\|_2^2 \le \frac{\eta}{4} \left \| \nabla F_t(\theta^{(t)}) \right \|_2^2  .
\end{equation*}
Combining the above with \Cref{eqn:proof-135}, we get
\begin{equation*}
F(\theta^{(t+1)}) - F(\theta^{(t)}) \le -\frac{\eta}{2} \left \| \nabla F_t(\theta^{(t)}) \right \|_2^2  .
\end{equation*}
Denote $\vu = \lim_{R \rightarrow \infty} \frac{\theta_R}{\| \theta_R \|_2}$. Similar to Lemma 9 in \cite{pmlr-v125-ji20a}, we can prove that: for any $\alpha > 0$, there exists a constant $\rho(\alpha) > 0$ such that for any $\theta$ subject to $\| \theta \|_2 \ge \rho(\alpha)$, the following holds for all $t$:
\begin{equation*}
    F_t((1 + \alpha) \| \theta \|_2 \vu) \le F_t(\theta)  .
\end{equation*}
Let $t_\alpha \ge t_\epsilon$ satisfy that for all $t \ge t_\alpha$, $\| \theta^{(t)} \|_2 \ge \max \{\rho(\alpha), 1 \}$. By the convexity of $F_t$, for all $t \ge  t_\alpha$, we have
\begin{equation}
\label{eqn:proof-146}
    \langle \nabla F_t(\theta^{(t)}), \theta^{(t)} - (1 + \alpha) \| \theta^{(t)} \|_2 \vu \rangle \ge F_t(\theta^{(t)}) - F_t((1 + \alpha) \| \theta^{(t)} \|_2 \vu) \ge 0  .
\end{equation}
Thus, we have
\begin{equation}
\label{eqn:proof-147}
\begin{aligned}
    \langle \theta^{(t+1)} - \theta^{(t)}, \vu \rangle &= \langle -\eta \nabla F_t(\theta^{(t)}), \vu \rangle \\ 
    & \ge \langle -\eta \nabla F_t(\theta^{(t)}), \theta^{(t)} \rangle \frac{1}{(1+\alpha) \| \theta^{(t)} \|_2} \\ 
    &= \langle \theta^{(t+1)} - \theta^{(t)}, \theta^{(t)} \rangle \frac{1}{(1+\alpha) \| \theta^{(t)} \|_2} \\ 
    &= \left ( \frac{1}{2} \left \| \theta^{(t+1)} \right\|_2^2 - \frac{1}{2} \left \| \theta^{(t)} \right\|_2^2 - \frac{1}{2} \left \| \theta^{(t+1)} - \theta^{(t)} \right\|_2^2 \right ) \frac{1}{(1+\alpha) \| \theta^{(t)} \|_2}  .    
\end{aligned}
\end{equation}
By $\frac{1}{2} (\| \theta^{(t+1)} \|_2 - \| \theta^{(t)} \|_2)^2 \ge 0$, we have $(\frac{1}{2} \left \| \theta^{(t+1)} \right\|_2^2 - \frac{1}{2} \left \| \theta^{(t)} \right\|_2^2) / \| \theta^{(t)} \|_2 \ge \left \| \theta^{(t+1)} \right\|_2 - \left \| \theta^{(t)} \right\|_2 $. Moreover, by \Cref{eqn:proof-146}, we have
\begin{equation*}
    \frac{\left \| \theta^{(t+1)} - \theta^{(t)} \right\|_2^2}{2 (1+\alpha) \| \theta^{(t)} \|_2} \le \frac{\left \| \theta^{(t+1)} - \theta^{(t)} \right\|_2^2}{2} = \frac{\eta^2 \left \| \nabla F_t(\theta^{(t)}) \right \|_2^2}{2} \le \eta \left ( F(\theta^{(t)}) - F(\theta^{(t+1)}) \right )  .
\end{equation*}
Summing up \Cref{eqn:proof-147} from $t = t_\alpha$ to $t-1$ yields
\begin{equation*}
    \langle \theta^{(t)} - \theta^{(t_\alpha)}, \vu \rangle \ge \frac{\left \| \theta^{(t)} \right \|_2 - \left \| \theta^{(t_\alpha)} \right \|_2}{1+\alpha} + \eta \left ( F(\theta^{(t)}) - F(\theta^{(t_\alpha)}) \right ) \ge \frac{\left \| \theta^{(t)} \right \|_2 - \left \| \theta^{(t_\alpha)} \right \|_2}{1+\alpha} - \eta F(\theta^{(t_\alpha)})  ,
\end{equation*}
which implies that
\begin{equation*}
    \left \langle \frac{\theta^{(t)}}{\left \| \theta^{(t)} \right \|_2}, \vu \right \rangle \ge \frac{1}{1 + \alpha} + \frac{1}{\left \| \theta^{(t)} \right \|_2} \left ( \langle \theta^{(t_\alpha)}, \vu \rangle - \frac{\| \theta^{(t_\alpha)} \|_2}{1+\alpha} - \eta F(\theta^{(t_\alpha)}) \right )  .
\end{equation*}
Since $\lim_{t \rightarrow \infty} \| \theta^{(t)} \|_2 = \infty$, we have
\begin{equation*}
    \liminf_{t \rightarrow \infty} \left \langle \frac{\theta^{(t)}}{\left \| \theta^{(t)} \right \|_2}, \vu \right \rangle \ge \frac{1}{1 + \alpha}  .
\end{equation*}
Since $\alpha$ is arbitrary, we must have $\lim_{t \rightarrow \infty} \frac{\theta^{(t)}}{\left \| \theta^{(t)} \right \|_2} = \vu$ as long as $\eta \le \min \{ \eta_1, \eta_2 \}$. 
\end{proof}

\section{Analysis of the Logistic Loss}
\label{app:proof-logistic-loss}

Here, we show that the logistic loss satisfies all the conditions in \Cref{thm:grw-cls-converge}, and $\lim_{R \rightarrow \infty} \frac{\theta_R}{R} = \hat{\theta}_{\mm}$.

First, for the logistic loss we have $\nabla^2_{\hat{y}} \ell(\hat{y},y) = \frac{y^2}{e^{y \hat{y}} + e^{-y \hat{y}} + 2} \le \max_i \frac{y_i^2}{4}$, so $\ell$ is smooth.

Second, let us analyze $\lim_{R \rightarrow \infty} \frac{\theta_R}{R}$.
For the logistic loss, it is easy to show that for any $\hat{\theta}' \neq \hat{\theta}_{\mm}$, there exists an $R(\hat{\theta}') > 0$ and an $\delta(\hat{\theta}') > 0$ such that $F(R \cdot \theta) > F(R \cdot \hat{\theta}_{\mm})$ for all $R \ge R(\hat{\theta}')$ and $\theta \in  B(\hat{\theta}', \delta(\hat{\theta}'))$. 
Let $S = \{ \theta: \| \theta \|_2 = 1 \}$. For any $\epsilon > 0$, $S - B(\hat{\theta}_{\mm}, \epsilon)$ is a compact set. And for any $\theta \in S - B(\hat{\theta}_{\mm}, \epsilon)$, there exist $R(\theta)$ and $\delta(\theta)$ as defined above. Thus, there must exist $\theta_1,\cdots,\theta_m \in S - B(\hat{\theta}_{\mm}, \epsilon)$ such that $S - B(\hat{\theta}_{\mm}, \epsilon) \subseteq \cup_{i=1}^m B(\theta_i, \delta(\theta_i))$. Let $R(\epsilon) = \max \{ R(\theta_1),\cdots,R(\theta_m) \}$, then for all $R \ge R(\epsilon)$ and all $\theta \in S - B(\hat{\theta}_{\mm}, \epsilon)$, $F(R \cdot \theta) > F(R \cdot \hat{\theta}_{\mm})$, which means that $\frac{\theta_R}{R} \in B(\hat{\theta}_{\mm}, \epsilon)$ for all $R \ge R(\epsilon)$. Therefore, $\lim_{R \rightarrow \infty} \frac{\theta_R}{R}$ exists and is equal to $\hat{\theta}_{\mm}$.

\section{Proof of Theorem \ref{thm:grw-cls-nn}}
\label{app:proof-thm-grw-cls-nn}
\begin{proof}
    Let $M_0$ be the bound of $\norm{\nabla_\theta f^{(0)}(x)}_2$.
    We first consider the regularized linearized neural network $f_{\lin \regu}^{(t)}$.
    By \Cref{prop:jacot-1}, $f^{(0)}(\vx)$ is sampled from a zero-mean Gaussian process, so there exists a constant $M > 0$ such that $|f^{(0)}(\vx_i)| < M$ for all $i$ with high probability. 
    Define
    \begin{equation*}
        F(\theta) = \sum_{i=1}^n q_i \ell(\langle \theta, \nabla_\theta f^{(0)}(\vx_i) \rangle + f^{(0)}(\vx_i), y_i)  .
    \end{equation*}
    Denote $\tilde{\theta}_R = \argmin_\theta \{ F(R \cdot \theta): \| \theta \|_2 \le 1\}$. 
when the linearized neural network is trained by GRW satisfying \Cref{ass:qstar} with regularization, since this is convex optimization and the objective function is smooth, we can prove that with a sufficiently small learning rate, as $t \rightarrow \infty$, $\theta^{(t)} \rightarrow R \cdot \tilde{\theta_R} + \theta^{(0)}$ where $R = \lim_{t \rightarrow \infty}  \| \theta^{(t)} - \theta^{(0)} \|_2$ (which is the minimizer). And define
\begin{equation*}
    \gamma = \min_{i=1,\cdots,n} y_i \cdot \langle \hat{\theta}_{\mm}, \nabla_\theta f^{(0)}(x_i) \rangle   .
\end{equation*}

First, we derive the lower bound of $R$. By \Cref{thm:approx-regu}, with a sufficiently large $\tilde{d}$, with high probability $\hat{\gR}(f_{\regu}^{(t)}) < \epsilon$ implies $\hat{\gR}(f_{\lin \regu}^{(t)}) < 2 \epsilon$. By the convexity of $\ell$, we have
\begin{align*}
    2 \epsilon & > \frac{1}{n} \sum_{i=1}^n \ell(\langle R \tilde{\theta}_R, x_i \rangle + f^{(0)}(x_i), y_i) \ge \log \left(1 + \exp \left ( -\frac{1}{n} \sum_{i=1}^n (\langle R \tilde{\theta}_R, \vx_i \rangle + f^{(0)}(x_i)) y_i \right ) \right ) \\ 
    & \ge \log \left(1 + \exp \left ( -\frac{1}{n} \sum_{i=1}^n R \langle \tilde{\theta}_R, x_i \rangle y_i - M \right ) \right )  ,
\end{align*}
which implies that $R = \Omega(- \log 2\epsilon)$ for all $\epsilon \in (0, \frac{1}{4})$.

Denote $\delta = \| \hat{\theta}_{\mm} - \tilde{\theta}_R \|_2$. Let $\theta' = \frac{\hat{\theta}_{\mm} + \tilde{\theta}_R}{2}$, then we can see that $\| \theta' \|_2 = \sqrt{1 - \frac{\delta^2}{4}}$. Let $\tilde{\theta}' = \frac{\theta'}{\| \theta' \|_2}$. By the definition of $\hat{\theta}_{\mm}$, there exists $j$ such that $y_j \cdot \langle \tilde{\theta}', \nabla_\theta f^{(0)}(\vx_j) \rangle \le \gamma$, which implies that
\begin{equation*}
    y_j \cdot \left \langle \frac{\hat{\theta}_{\mm} + \tilde{\theta}_R}{2} \frac{1}{\sqrt{1 - \frac{\delta^2}{4}}}, \nabla_\theta f^{(0)}(\vx_j) \right \rangle \le \gamma  .
\end{equation*}
Thus, we have
\begin{align*}
    y_j \cdot \langle \tilde{\theta}_R, \nabla_\theta f^{(0)}(x_j) \rangle & \le 2 \sqrt{1 - \frac{\delta^2}{4}} \gamma - y_j \cdot \langle \hat{\theta}_{\mm}, \nabla_\theta f^{(0)}(\vx_j) \rangle \\ 
    & \le \left ( 2 \sqrt{1 - \frac{\delta^2}{4}} - 1\right ) \gamma \\ 
    & \le \left ( 2 (1 - \frac{\delta^2}{8}) - 1\right ) \gamma \qquad (\text{since } \sqrt{1 - x} \le 1 - \frac{x}{2}) \\ 
    & = (1 - \frac{\delta^2}{4}) \gamma .
\end{align*}
On the other hand, we have
\begin{align*}
& q_j \log (1 + \exp(- y_j \cdot \langle R \cdot \tilde{\theta}_R, \nabla_\theta f^{(0)}(\vx_j) \rangle - M )) \le F(R \cdot \tilde{\theta}_R) \\ 
\le \; & F(R \cdot \hat{\theta}_{\mm}) \le \log (1 + \exp (-R \gamma + M)) ,
\end{align*}
which implies that
\begin{equation*}
     q^* \log \left ( 1 + \exp \left (  -(1 - \frac{\delta^2}{4}) R \gamma - M \right ) \right ) \le  \log (1 + \exp (-R \gamma + M))  .
\end{equation*}
Thus, we have
\begin{equation*}
    1 + \exp(-R \gamma+M) \ge \left ( 1 + \exp \left (  -(1 - \frac{\delta^2}{4}) R \gamma -M \right ) \right )^{q^*} \ge 1 + q^* \exp \left ( - (1 - \frac{\delta^2}{4}) R \gamma - M \right )  ,
\end{equation*}
which is equivalent to
\begin{equation*}
    -R \gamma + M \ge  -  (1 - \frac{\delta^2}{4}) R \gamma - M + \log(q^*)  .
\end{equation*}
From this, we conclude that
\begin{equation*}
    \delta = O(R^{-1/2}) = O((-\log 2\epsilon)^{-1/2}).
\end{equation*}
So for any test point $x$ such that $\norm{\nabla_\theta f^{(0)}(x)} \le M_0$, we have
\begin{equation*}
    \abs{ \dotp{ \hat{\theta}_{\mm} - \tilde{\theta}_R, \nabla_\theta f^{(0)}(x) }} \le \delta M_0 = O((-\log 2\epsilon)^{-1/2}) .
\end{equation*}
Combining this with \Cref{thm:approx-regu}, with high probability, we have
\begin{equation*}
    \limsup_{t \rightarrow \infty} | R \cdot f_{\mm}(x) - f_{\regu}^{(t)}(x) | = O(R \cdot (-\log 2\epsilon)^{-1/2} + \tilde{d}^{-1/4}).
\end{equation*}

Hence, there exists a constant $C > 0$ such that: As $\dl \rightarrow \infty$, with high probability, for all $\epsilon \in (0, \frac{1}{4})$, if $|f_{\mm}(x)| > C \cdot (-\log 2\epsilon)^{-1/2}$, then $f_{\regu}^{(t)}(x)$ will have the same sign as $f_{\mm}(x)$ for a sufficiently large $t$. Note that this $C$ only depends on $n$, $q^*$, $\gamma$, $M$ and $M_0$, so it is a constant independent of $\epsilon$.
\end{proof}

\section{Proof of Proposition \ref{prop:doro-dual}}
\label{app:proof-prop-doro-dual}
\begin{proof}
    We have
\begin{equation}
\label{eqn:proof-26}
\begin{aligned}
\gR_{D_\beta,\rho,\epsilon}(\theta;\ptrain) &= \inf_{P'}\left \lbrace \gR_{D_\beta,\rho}(\theta;P') : \exists \tilde{P'} \text{ s.t. } \ptrain = (1-\epsilon) P' + \epsilon \tilde{P'} \right \rbrace \\ 
&= \inf_{P',\eta} \left \lbrace c_\beta(\rho) \E_{P'} [(\ell(\theta;Z) - \eta)_+^{\beta_*}]^{\frac{1}{\beta_*}} + \eta \right \rbrace \\ 
&= \inf_{\eta} \left \lbrace c_\beta(\rho) \inf_{P'} \{ [ \int_{\R_+} P'( (\ell(\theta;Z)-\eta)_+^{\beta_*} >u)du]^{\frac{1}{\beta_*}} \} + \eta \right \rbrace  .
\end{aligned}
\end{equation}
Since $\ptrain = (1-\epsilon) P' + \epsilon \tilde{P'}$, for all $\ell_0 \in \R$, we have
\begin{equation*}
    P'(\ell(\theta;Z) \leq \ell_0) \leq \min \left \lbrace 1, \frac{1}{1-\epsilon}\ptrain (\ell(\theta;Z) \leq \ell_0) \right \rbrace  ,
\end{equation*}
and we can show that there exists a $P^* = P'$ that attains the equality for all $\ell_0$.
This is because $\ptrain(\ell(\theta;z))$ is a continuous function of $z$ for any fixed $\theta$ since both $\ell$ and $\ptrain$ are continuous, so there exists an $\ell^*$ such that $\ptrain(\ell(\theta;Z) > \ell^*) = \epsilon$.
Hence, we can define
\begin{equation*}
P^*(z) = \left \{
\begin{array}{cc}
\frac{1}{1-\epsilon} \ptrain(z) & ,\ell(\theta;z) \leq \ell^*  ;\\ 
0 &, \ell(\theta;z) > \ell^*  .
\end{array}
\right .
\end{equation*}
For this $P^*$, we have $\int_{\gX \times \gY}P^*(z)dz = \frac{1}{1-\epsilon}\int_{\ell(\theta;z) < \ell^*} \ptrain(z)dz = \frac{1}{1-\epsilon}\ptrain(\ell(\theta;Z) < \ell^*) = 1$ because $\ptrain(\ell(\theta;Z)=\ell^*)=0$, which means that $P^*$ is a proper probability density function.

Let $v=u^{\frac{1}{\beta_*}}$. Plugging $P^*(\ell(\theta;Z) \leq \ell_0) = \min \left \lbrace 1, \frac{1}{1-\epsilon}\ptrain (\ell(\theta;Z) \leq \ell_0) \right \rbrace$ into \Cref{eqn:proof-26} produces
\begin{align*}
\gR_{D_\beta,\rho,\epsilon}(\theta;\ptrain) &= \inf_{\eta} \left \lbrace c_\beta(\rho) \left [ \int_{\R_+} [1-P^*( (\ell(\theta;Z)-\eta)_+^{\beta_*} \leq v^{\beta_*})]dv^{\beta_*} \right ]^{\frac{1}{\beta_*}}  + \eta \right \rbrace \\
&= \inf_\eta \left \lbrace c_\beta(\rho) \left [ \int_{\R_+} [1-\frac{1}{1-\epsilon}\ptrain (\ell(\theta;Z) \leq \eta+v)]_+ dv^{\beta_*} \right]^{\frac{1}{\beta_*}} + \eta \right \rbrace \\ 
&= \inf_\eta \left \lbrace  c_\beta(\rho)  \left [\int_{0}^{(\ell^*-\eta)_+} \frac{1}{1-\epsilon}[(1-\epsilon) - \ptrain (\ell(\theta;Z) \leq \eta+v)]_+ dv^{\beta_*} \right ]^{\frac{1}{\beta_*}} +\eta \right \rbrace  .
\end{align*}
On the other hand, we have
\begin{align*}
&\E_{\ptrain}[(\ell - \eta)_+^{\beta_*} \mid P_{Z' \sim \ptrain}(\ell(\theta;Z') > \ell(\theta;Z)) \geq \epsilon] \\
=& \frac{1}{1-\epsilon}\int_0^{\ell^*} (u - \eta)_+^{\beta_*} d(\ptrain(\ell \leq u)) \\ 
=& \frac{1}{1-\epsilon} \left \lbrace  \left [ (u-\eta)_+^{\beta_*} \ptrain(\ell \leq u) \right ]_{0}^{\ell^*} - \int_0^{\ell^*} \ptrain(\ell \leq u) d((u-\eta)_+^{\beta_*}) \right \rbrace \\
=& \frac{1}{1-\epsilon} \left \lbrace  (\ell^*-\eta)_+^{\beta_*}(1-\epsilon) - \int_0^{\ell^*} \ptrain(\ell \leq u) d((u-\eta)_+^{\beta_*}) \right \rbrace \\ 
=& \frac{1}{1-\epsilon} \left \lbrace \int_{0}^{(\ell^*-\eta)_+} (1-\epsilon) dv^{\beta_*} - \int_{0}^{(\ell^*-\eta)_+} \ptrain (\ell \leq \eta + w) dw^{\beta_*} \right \rbrace  ,
\end{align*}
where $w = (u-\eta)_+$. This completes the proof.
\end{proof}

\begin{remark}
We can prove a similar dual formula even if $\ptrain$ is not continuous.
For any $\ptrain$, there exists an $\ell^*$ such that $\ptrain(\ell(\theta;Z) > \ell^*) \le \epsilon$ and $\ptrain(\ell(\theta;Z) < \ell^*) \le 1 - \epsilon$.
If $\ptrain(\ell(\theta;Z) = \ell^*) = 0$, then we still define $P^*$ the same as in the above proof.
If $\ptrain(\ell(\theta;Z) = \ell^*) > 0$, then we define
\begin{equation*}
P^*(z) = \left \{
\begin{array}{cc}
\frac{1}{1-\epsilon} \ptrain(z) & ,\ell(\theta;z) < \ell^*  ; \\ 
\left[1 - \frac{1}{1-\epsilon} \ptrain(\ell(\theta;Z) < \ell^*) \right] / \ptrain(\ell(\theta;Z) = \ell^*) &, \ell(\theta;z) = \ell^*  ; \\ 
0 &, \ell(\theta;z) > \ell^*  ,
\end{array}
\right .
\end{equation*}
with which the dual formula becomes
\begin{align*}
\gR_{D_\beta, \rho, \epsilon}(\theta;\ptrain)  \\
= \; & \inf_\eta \{ c_\beta(\rho) ( \frac{\ptrain(\ell < \ell^*)}{1-\epsilon} \E_{Z}[(\ell(\theta;Z) -\eta)_+^{\beta_*} \mid  P_{Z'} (\ell(\theta;Z')>\ell(\theta;Z)) > \epsilon] \\ 
& +\frac{1-\ptrain(\ell < \ell^*)}{1-\epsilon} (\ell^* - \eta)_+^{\beta_*} )^{\frac{1}{\beta_*} } + \eta \}    .
\end{align*}
\end{remark}

\section{Proof of Theorem \ref{thm:doro-guarantee}}
\label{app:proof-thm-doro-guarantee}

The proof relies on the following key technical lemma.
\begin{lemma}
\label{lem:key-technical}
For any distributions $P, P'$, non-negative loss function $l(\cdot, Z)$ and $ 1\le \beta_* <2k$, such that $\E_P[l(\theta, Z)^{2k}] <\infty$, we have
\begin{align*}
	& \E_{P}[(\ell - \eta)^{\beta_*}_+]^{\frac{1}{\beta_*}} \leq \\
    & \E_{P'}[(\ell - \eta)^{\beta_*}_+]^{\frac{1}{\beta_*}} + \E_P[(l(\theta, Z)-\eta)_+^{2k}]^{\frac{1}{2k}} \tv(P, P')^{\left(\frac{1}{\beta_*} - \frac{1}{2k}\right)} \beta_*^{-\frac{1}{2k}} \cdot \left(\frac{2k}{2k-\beta_*}\right)^{\frac{1}{\beta_*}}
\end{align*}
\end{lemma}
\begin{proof}
    By the definition of the total variation distance, we have $P(\ell(\theta;Z) > u) - P'(\ell(\theta;Z') > u) \leq  \tv(P,P')$ for all $u \ge 0$.
    Let $s_{2k} := \E[(\ell - \eta)_+^{2k}]^{\frac{1}{2k}}$.
    By Markov's inequality and the non-negativity of $\ell$, we have
    \begin{equation}
    \label{eqn:markov}
        P(\ell - \eta > u) \le  \frac{\E[(\ell - \eta)_+^{2k}]}{u^{2k}} := (\frac{s_{2k}}{u})^{2k}  \qquad \text{for all } \eta \ge 0 .
    \end{equation}
    Using integration by parts, we have
    \begin{equation*}
        \E_{P}[(\ell - \eta)^{\beta_*}_+] = \int_\eta^\infty \beta_* (t - \eta)^{(\beta_*-1)} P(\ell \ge t) dt = \int_0^\infty \beta_* u^{(\beta_*-1)} P(\ell -\eta \ge u) du  .
    \end{equation*}
    This implies that
    \begin{align*}
        \E_{P}[(\ell - \eta)^{\beta_*}_+] - \E_{P'}[(\ell - \eta)^{\beta_*}_+] &=  \int_0^\infty \beta_* u^{(\beta_*-1)} \left( P(\ell -\eta \ge u) - P'(\ell -\eta \ge u)\right) du \\
	&=  \left(\int_0^M +\int_M^\infty\right)  \left(\beta_* u^{(\beta_*-1)} \left( P(\ell -\eta \ge u) - P'(\ell -\eta \ge u)\right) du \right)  .
    \end{align*}
    Here, $M$ is a positive parameter whose value will be determined later. For the first integral, we have
    \begin{align*}
        \int_0^M \beta_* u^{(\beta_*-1)} \left( P(\ell -\eta \ge u) - P'(\ell -\eta \ge u)\right) du  &\leq \int_0^M \beta_* u^{(\beta_*-1)} \tv(P, P')du \\
	&= M^{\beta_*} \tv(P, P')  .
    \end{align*}
    For the second integral, by \Cref{eqn:markov}, we have
    \begin{align*}
        \int_M^{\infty} \beta_* u^{(\beta_*-1)} \left( P(\ell -\eta \ge u) - P'(\ell -\eta \ge u)\right) du  &\leq \int_M^{\infty} \beta_* u^{(\beta_*-1)} P(\ell -\eta \ge u)du \\
	&\leq  \int_M^{\infty} \beta_* u^{(\beta_*-1)} \left( \frac{s_{2k}}{u}\right)^{2k} \\
	&= \frac{s_{2k}^{2k}}{2k - \beta_*} \cdot \frac{1}{M^{2k-\beta_*}}  .
    \end{align*}
    Therefore, by setting $M =s_{2k} (\tv(P, P') \beta_*)^{-1/2k}$ which minimizes the sum of two terms, we have
    \begin{align*}
        & \E_{P}[(\ell - \eta)^{\beta_*}_+] - \E_{P'}[(\ell - \eta)^{\beta_*}_+] \leq \\
        & \inf_{M>0} \left(M^{\beta_*} \tv(P, P') + \frac{s_{2k}^{2k}}{2k - \beta_*} \cdot \frac{1}{M^{2k-\beta_*}}\right) = s_{2k}^{\beta_*} \tv(P, P')^{1 - \frac{\beta_*}{2k}} \beta_*^{-\frac{\beta_*}{2k}} \cdot \frac{2k}{2k-\beta_*}  .
    \end{align*}
    Using the inequality $(A+B)^{\frac{1}{\beta_*}} \leq A^{\frac{1}{\beta_*}} + B^{\frac{1}{\beta_*}}$ when $\beta_* \geq 1$, we have
    \begin{equation*}
        \E_{P}[(\ell - \eta)^{\beta_*}_+]^{\frac{1}{\beta_*}} \leq 	\E_{P'}[(\ell - \eta)^{\beta_*}_+]^{\frac{1}{\beta_*}} + s_{2k} \tv(P, P')^{\left(\frac{1}{\beta_*} - \frac{1}{2k}\right)} \beta_*^{-\frac{1}{2k}} \cdot \left(\frac{2k}{2k-\beta_*}\right)^{\frac{1}{\beta_*}} ,
    \end{equation*}
    as desired.
\end{proof}

Now we prove \Cref{thm:doro-guarantee}.
\begin{proof}
    By Lemma \ref{lem:key-technical}, for any $P'$ such that $\tv (P,P') \leq \frac{\epsilon}{1-\epsilon}$,  we have
    \begin{equation*}
        \cvar_{\alpha}(\theta;P) - \cvar_{\alpha}(\theta;P')	\leq 2\alpha^{-1}\sigma \sqrt{\frac{\epsilon}{1-\epsilon}}  .
    \end{equation*}
    By \Cref{cor:surrogate-dro}, if $\gR_{\max}(\theta;P) > 3\alpha^{-1}\sigma\sqrt{\frac{\epsilon}{1-\epsilon}}$, then $\cvar_{\alpha}(\theta;P) > 3\alpha^{-1}\sigma\sqrt{\frac{\epsilon}{1-\epsilon}}$, which implies that
    \begin{equation*}
        \frac{\cvar_\alpha(\theta;P')}{\gR_{\max}(\theta;P)} \geq \frac{\cvar_\alpha(\theta;P')}{\cvar_\alpha(\theta;P)} = 1 - \frac{\delta}{\cvar_\alpha(\theta;P)} \geq 1-\frac{2\alpha^{-1}\sigma\sqrt{\frac{\epsilon}{1-\epsilon}}}{3\alpha^{-1}\sigma\sqrt{\frac{\epsilon}{1-\epsilon}}} = \frac{1}{3}
    \end{equation*}
    holds for any $P'$ such that $\tv (P,P') \leq \frac{\epsilon}{1-\epsilon}$. By \Cref{lem:robust-dro}, taking the infimum over $P'$ yields the first inequality of \Cref{eqn:thm-doro-guarantee}. 
    And by \Cref{cor:surrogate-dro}, we have $D_{\chi^2,\rho}(\theta;P') \geq \cvar_\alpha(\theta;P')$ for all $\theta$ and $P'$.
    This combined with the above inequality yields the second inequality of \Cref{eqn:thm-doro-guarantee}.
\end{proof}

\backmatter


\renewcommand{\bibsection}{\chapter{\bibname}}
\bibliographystyle{plain}

\begin{thebibliography}{100}

  \bibitem{Achiam2023GPT4TR}
  Josh Achiam, Steven Adler, Sandhini Agarwal, et~al.
  \newblock Gpt-4 technical report.
  \newblock {\em OpenAI Blog}, 2023.
  
  \bibitem{achille2018emergence}
  Alessandro Achille and Stefano Soatto.
  \newblock Emergence of invariance and disentanglement in deep representations.
  \newblock {\em Journal of Machine Learning Research}, 19(50):1--34, 2018.
  
  \bibitem{NEURIPS2022_70596d70}
  Kumar~K Agrawal, Arnab~Kumar Mondal, Arna Ghosh, and Blake Richards.
  \newblock \textbackslash alpha-req : Assessing representation quality in self-supervised learning by measuring eigenspectrum decay.
  \newblock In S.~Koyejo, S.~Mohamed, A.~Agarwal, D.~Belgrave, K.~Cho, and A.~Oh, editors, {\em Advances in Neural Information Processing Systems}, volume~35, pages 17626--17638. Curran Associates, Inc., 2022.
  
  \bibitem{bahri2022scarf}
  Dara Bahri, Heinrich Jiang, Yi~Tay, and Donald Metzler.
  \newblock Scarf: Self-supervised contrastive learning using random feature corruption.
  \newblock In {\em International Conference on Learning Representations}, 2022.
  
  \bibitem{balestriero2022contrastive}
  Randall Balestriero and Yann LeCun.
  \newblock Contrastive and non-contrastive self-supervised learning recover global and local spectral embedding methods.
  \newblock {\em Advances in Neural Information Processing Systems}, 35:26671--26685, 2022.
  
  \bibitem{bardes2021vicreg}
  Adrien Bardes, Jean Ponce, and Yann LeCun.
  \newblock {VICR}eg: Variance-invariance-covariance regularization for self-supervised learning.
  \newblock In {\em International Conference on Learning Representations}, 2022.
  
  \bibitem{baudat2000generalized}
  Gaston Baudat and Fatiha Anouar.
  \newblock Generalized discriminant analysis using a kernel approach.
  \newblock {\em Neural computation}, 12(10):2385--2404, 2000.
  
  \bibitem{belkin2003laplacian}
  Mikhail Belkin and Partha Niyogi.
  \newblock Laplacian eigenmaps for dimensionality reduction and data representation.
  \newblock {\em Neural computation}, 15(6):1373--1396, 2003.
  
  \bibitem{bengio2013representation}
  Yoshua Bengio, Aaron Courville, and Pascal Vincent.
  \newblock Representation learning: A review and new perspectives.
  \newblock {\em IEEE transactions on pattern analysis and machine intelligence}, 35(8):1798--1828, 2013.
  
  \bibitem{bengio2004learning}
  Yoshua Bengio, Olivier Delalleau, Nicolas~Le Roux, Jean-Fran{\c{c}}ois Paiement, Pascal Vincent, and Marie Ouimet.
  \newblock Learning eigenfunctions links spectral embedding and kernel pca.
  \newblock {\em Neural computation}, 16(10):2197--2219, 2004.
  
  \bibitem{berthelot2019mixmatch}
  David Berthelot, Nicholas Carlini, Ian Goodfellow, Nicolas Papernot, Avital Oliver, and Colin~A Raffel.
  \newblock Mixmatch: A holistic approach to semi-supervised learning.
  \newblock {\em Advances in neural information processing systems}, 32, 2019.
  
  \bibitem{bickel2007discriminative}
  Steffen Bickel, Michael Br{\"u}ckner, and Tobias Scheffer.
  \newblock Discriminative learning for differing training and test distributions.
  \newblock In {\em Proceedings of the 24th international conference on Machine learning}, pages 81--88, 2007.
  
  \bibitem{blanchard2007statistical}
  Gilles Blanchard, Olivier Bousquet, and Laurent Zwald.
  \newblock Statistical properties of kernel principal component analysis.
  \newblock {\em Machine Learning}, 66:259--294, 2007.
  
  \bibitem{blodgett2016demographic}
  Su~Lin Blodgett, Lisa Green, and Brendan O{'}Connor.
  \newblock Demographic dialectal variation in social media: A case study of {A}frican-{A}merican {E}nglish.
  \newblock In {\em Proceedings of the 2016 Conference on Empirical Methods in Natural Language Processing}, pages 1119--1130, Austin, Texas, November 2016. Association for Computational Linguistics.
  
  \bibitem{bommasani2021opportunities}
  Rishi Bommasani, Drew~A Hudson, Ehsan Adeli, Russ Altman, Simran Arora, Sydney von Arx, Michael~S Bernstein, Jeannette Bohg, Antoine Bosselut, Emma Brunskill, et~al.
  \newblock On the opportunities and risks of foundation models.
  \newblock {\em arXiv preprint arXiv:2108.07258}, 2021.
  
  \bibitem{borkan2019nuanced}
  Daniel Borkan, Lucas Dixon, Jeffrey Sorensen, Nithum Thain, and Lucy Vasserman.
  \newblock Nuanced metrics for measuring unintended bias with real data for text classification.
  \newblock In {\em Companion Proceedings of The 2019 World Wide Web Conference}, pages 491--500, 2019.
  
  \bibitem{brent1971algorithm}
  Richard~P. Brent.
  \newblock An algorithm with guaranteed convergence for finding a zero of a function.
  \newblock {\em The Computer Journal}, 14(4):422--425, 1971.
  
  \bibitem{brezis2011functional}
  Haim Brezis.
  \newblock {\em Functional analysis, Sobolev spaces and partial differential equations}.
  \newblock Springer, 2011.
  
  \bibitem{byrd2019effect}
  Jonathon Byrd and Zachary Lipton.
  \newblock What is the effect of importance weighting in deep learning?
  \newblock In Kamalika Chaudhuri and Ruslan Salakhutdinov, editors, {\em Proceedings of the 36th International Conference on Machine Learning}, volume~97 of {\em Proceedings of Machine Learning Research}, pages 872--881. PMLR, 09--15 Jun 2019.
  
  \bibitem{cabannes2023ssl}
  Vivien Cabannes, Bobak Kiani, Randall Balestriero, Yann Lecun, and Alberto Bietti.
  \newblock The {SSL} interplay: Augmentations, inductive bias, and generalization.
  \newblock In Andreas Krause, Emma Brunskill, Kyunghyun Cho, Barbara Engelhardt, Sivan Sabato, and Jonathan Scarlett, editors, {\em Proceedings of the 40th International Conference on Machine Learning}, volume 202 of {\em Proceedings of Machine Learning Research}, pages 3252--3298. PMLR, 23--29 Jul 2023.
  
  \bibitem{cai2024medusa}
  Tianle Cai, Yuhong Li, Zhengyang Geng, Hongwu Peng, Jason~D. Lee, Deming Chen, and Tri Dao.
  \newblock Medusa: Simple llm inference acceleration framework with multiple decoding heads.
  \newblock {\em arXiv preprint arXiv: 2401.10774}, 2024.
  
  \bibitem{chen2016xgboost}
  Tianqi Chen and Carlos Guestrin.
  \newblock Xgboost: A scalable tree boosting system.
  \newblock In {\em Proceedings of the 22nd acm sigkdd international conference on knowledge discovery and data mining}, pages 785--794, 2016.
  
  \bibitem{chen2020simple}
  Ting Chen, Simon Kornblith, Mohammad Norouzi, and Geoffrey Hinton.
  \newblock A simple framework for contrastive learning of visual representations.
  \newblock In {\em International conference on machine learning}, pages 1597--1607. PMLR, 2020.
  
  \bibitem{Chen_2021_CVPR}
  Xinlei Chen and Kaiming He.
  \newblock Exploring simple siamese representation learning.
  \newblock In {\em Proceedings of the IEEE/CVF Conference on Computer Vision and Pattern Recognition (CVPR)}, pages 15750--15758, June 2021.
  
  \bibitem{chen2025deconstructing}
  Xinlei Chen, Zhuang Liu, Saining Xie, and Kaiming He.
  \newblock Deconstructing denoising diffusion models for self-supervised learning.
  \newblock In {\em The Thirteenth International Conference on Learning Representations}, 2025.
  
  \bibitem{chung1997spectral}
  Fan~RK Chung.
  \newblock {\em Spectral graph theory}, volume~92.
  \newblock American Mathematical Soc., 1997.
  
  \bibitem{Clark2020ELECTRA}
  Kevin Clark, Minh-Thang Luong, Quoc~V. Le, and Christopher~D. Manning.
  \newblock Electra: Pre-training text encoders as discriminators rather than generators.
  \newblock In {\em International Conference on Learning Representations}, 2020.
  
  \bibitem{cohen2021gradient}
  Jeremy Cohen, Simran Kaur, Yuanzhi Li, J~Zico Kolter, and Ameet Talwalkar.
  \newblock Gradient descent on neural networks typically occurs at the edge of stability.
  \newblock In {\em International Conference on Learning Representations}, 2021.
  
  \bibitem{coifman2006diffusion}
  Ronald~R Coifman and St{\'e}phane Lafon.
  \newblock Diffusion maps.
  \newblock {\em Applied and computational harmonic analysis}, 21(1):5--30, 2006.
  
  \bibitem{cressie1984multinomial}
  Noel Cressie and Timothy~RC Read.
  \newblock Multinomial goodness-of-fit tests.
  \newblock {\em Journal of the Royal Statistical Society Series B: Statistical Methodology}, 46(3):440--464, 1984.
  
  \bibitem{dao2019kernel}
  Tri Dao, Albert Gu, Alexander Ratner, Virginia Smith, Chris De~Sa, and Christopher R{\'e}.
  \newblock A kernel theory of modern data augmentation.
  \newblock In {\em International conference on machine learning}, pages 1528--1537. PMLR, 2019.
  
  \bibitem{demmel2007fast}
  James Demmel, Ioana Dumitriu, and Olga Holtz.
  \newblock Fast linear algebra is stable.
  \newblock {\em Numerische Mathematik}, 108(1):59--91, 2007.
  
  \bibitem{pmlr-v162-deng22b}
  Zhijie Deng, Jiaxin Shi, and Jun Zhu.
  \newblock Neuralef: Deconstructing kernels by deep neural networks.
  \newblock In Kamalika Chaudhuri, Stefanie Jegelka, Le~Song, Csaba Szepesvari, Gang Niu, and Sivan Sabato, editors, {\em Proceedings of the 39th International Conference on Machine Learning}, volume 162 of {\em Proceedings of Machine Learning Research}, pages 4976--4992. PMLR, 17--23 Jul 2022.
  
  \bibitem{devlin2018bert}
  Jacob Devlin, Ming-Wei Chang, Kenton Lee, and Kristina Toutanova.
  \newblock {BERT}: Pre-training of deep bidirectional transformers for language understanding.
  \newblock In {\em Proceedings of the 2019 Conference of the North {A}merican Chapter of the Association for Computational Linguistics: Human Language Technologies, Volume 1 (Long and Short Papers)}, pages 4171--4186, Minneapolis, Minnesota, June 2019. Association for Computational Linguistics.
  
  \bibitem{devries2017improved}
  Terrance DeVries and Graham~W Taylor.
  \newblock Improved regularization of convolutional neural networks with cutout.
  \newblock {\em arXiv preprint arXiv:1708.04552}, 2017.
  
  \bibitem{dosovitskiy2021an}
  Alexey Dosovitskiy, Lucas Beyer, Alexander Kolesnikov, Dirk Weissenborn, Xiaohua Zhai, Thomas Unterthiner, Mostafa Dehghani, Matthias Minderer, Georg Heigold, Sylvain Gelly, Jakob Uszkoreit, and Neil Houlsby.
  \newblock An image is worth 16x16 words: Transformers for image recognition at scale.
  \newblock In {\em International Conference on Learning Representations}, 2021.
  
  \bibitem{Dubey2024TheL3}
  Abhimanyu Dubey, Abhinav Jauhri, Abhinav Pandey, et~al.
  \newblock The llama 3 herd of models.
  \newblock {\em ArXiv}, abs/2407.21783, 2024.
  
  \bibitem{duchi2018learning}
  John Duchi and Hongseok Namkoong.
  \newblock Learning models with uniform performance via distributionally robust optimization.
  \newblock {\em arXiv preprint arXiv:1810.08750}, 2018.
  
  \bibitem{feng2023towards}
  Guhao Feng, Bohang Zhang, Yuntian Gu, Haotian Ye, Di~He, and Liwei Wang.
  \newblock Towards revealing the mystery behind chain of thought: A theoretical perspective.
  \newblock In {\em Thirty-seventh Conference on Neural Information Processing Systems}, 2023.
  
  \bibitem{Fey/Lenssen/2019}
  Matthias Fey and Jan~E. Lenssen.
  \newblock Fast graph representation learning with {PyTorch Geometric}.
  \newblock In {\em ICLR Workshop on Representation Learning on Graphs and Manifolds}, 2019.
  
  \bibitem{fischer2020sobolev}
  Simon Fischer and Ingo Steinwart.
  \newblock Sobolev norm learning rates for regularized least-squares algorithms.
  \newblock {\em The Journal of Machine Learning Research}, 21(1):8464--8501, 2020.
  
  \bibitem{freund1997decision}
  Yoav Freund and Robert~E Schapire.
  \newblock A decision-theoretic generalization of on-line learning and an application to boosting.
  \newblock {\em Journal of computer and system sciences}, 55(1):119--139, 1997.
  
  \bibitem{fumero2024latent}
  Marco Fumero, Marco Pegoraro, Valentino Maiorca, Francesco Locatello, and Emanuele Rodol\`{a}.
  \newblock Latent functional maps: a spectral framework for representation alignment.
  \newblock In A.~Globerson, L.~Mackey, D.~Belgrave, A.~Fan, U.~Paquet, J.~Tomczak, and C.~Zhang, editors, {\em Advances in Neural Information Processing Systems}, volume~37, pages 66178--66203. Curran Associates, Inc., 2024.
  
  \bibitem{gao2023scaling}
  Leo Gao, John Schulman, and Jacob Hilton.
  \newblock Scaling laws for reward model overoptimization.
  \newblock In {\em International Conference on Machine Learning}, pages 10835--10866. PMLR, 2023.
  
  \bibitem{2023arXiv231211805G}
  {Gemini Team}.
  \newblock {Gemini: A Family of Highly Capable Multimodal Models}.
  \newblock {\em arXiv e-prints}, page arXiv:2312.11805, December 2023.
  
  \bibitem{gidaris2018unsupervised}
  Spyros Gidaris, Praveer Singh, and Nikos Komodakis.
  \newblock Unsupervised representation learning by predicting image rotations.
  \newblock In {\em International Conference on Learning Representations}, 2018.
  
  \bibitem{JMLR:v12:gonen11a}
  Mehmet G{{\"o}}nen and Ethem Alpaydin.
  \newblock Multiple kernel learning algorithms.
  \newblock {\em Journal of Machine Learning Research}, 12(64):2211--2268, 2011.
  
  \bibitem{goodfellow2014generative}
  Ian Goodfellow, Jean Pouget-Abadie, Mehdi Mirza, Bing Xu, David Warde-Farley, Sherjil Ozair, Aaron Courville, and Yoshua Bengio.
  \newblock Generative adversarial nets.
  \newblock {\em Advances in neural information processing systems}, 27, 2014.
  
  \bibitem{gorishniy2021revisiting}
  Yury Gorishniy, Ivan Rubachev, Valentin Khrulkov, and Artem Babenko.
  \newblock Revisiting deep learning models for tabular data.
  \newblock In A.~Beygelzimer, Y.~Dauphin, P.~Liang, and J.~Wortman Vaughan, editors, {\em Advances in Neural Information Processing Systems}, 2021.
  
  \bibitem{NEURIPS2020_f3ada80d}
  Jean-Bastien Grill, Florian Strub, Florent Altch\'{e}, Corentin Tallec, Pierre Richemond, Elena Buchatskaya, Carl Doersch, Bernardo Avila~Pires, Zhaohan Guo, Mohammad Gheshlaghi~Azar, Bilal Piot, koray kavukcuoglu, Remi Munos, and Michal Valko.
  \newblock Bootstrap your own latent - a new approach to self-supervised learning.
  \newblock In H.~Larochelle, M.~Ranzato, R.~Hadsell, M.F. Balcan, and H.~Lin, editors, {\em Advances in Neural Information Processing Systems}, volume~33, pages 21271--21284. Curran Associates, Inc., 2020.
  
  \bibitem{gulrajani2021in}
  Ishaan Gulrajani and David Lopez-Paz.
  \newblock In search of lost domain generalization.
  \newblock In {\em International Conference on Learning Representations}, 2021.
  
  \bibitem{pmlr-v80-gunasekar18a}
  Suriya Gunasekar, Jason Lee, Daniel Soudry, and Nathan Srebro.
  \newblock Characterizing implicit bias in terms of optimization geometry.
  \newblock In Jennifer Dy and Andreas Krause, editors, {\em Proceedings of the 35th International Conference on Machine Learning}, volume~80 of {\em Proceedings of Machine Learning Research}, pages 1832--1841. PMLR, 10--15 Jul 2018.
  
  \bibitem{guo2025deepseek}
  Daya Guo, Dejian Yang, Haowei Zhang, Junxiao Song, Ruoyu Zhang, Runxin Xu, Qihao Zhu, Shirong Ma, Peiyi Wang, Xiao Bi, et~al.
  \newblock Deepseek-r1: Incentivizing reasoning capability in llms via reinforcement learning.
  \newblock {\em arXiv preprint arXiv:2501.12948}, 2025.
  
  \bibitem{hahn2011statistical}
  Gerald~J Hahn and William~Q Meeker.
  \newblock {\em Statistical intervals: a guide for practitioners}, volume~92.
  \newblock John Wiley \& Sons, 2011.
  
  \bibitem{hamilton2020graph}
  William~L Hamilton.
  \newblock {\em Graph representation learning}.
  \newblock Morgan \& Claypool Publishers, 2020.
  
  \bibitem{haochen2022theoretical}
  Jeff~Z HaoChen and Tengyu Ma.
  \newblock A theoretical study of inductive biases in contrastive learning.
  \newblock {\em arXiv preprint arXiv:2211.14699}, 2022.
  
  \bibitem{haochen2021provable}
  Jeff~Z HaoChen, Colin Wei, Adrien Gaidon, and Tengyu Ma.
  \newblock Provable guarantees for self-supervised deep learning with spectral contrastive loss.
  \newblock {\em Advances in Neural Information Processing Systems}, 34:5000--5011, 2021.
  
  \bibitem{haochen2022beyond}
  Jeff~Z. HaoChen, Colin Wei, Ananya Kumar, and Tengyu Ma.
  \newblock Beyond separability: Analyzing the linear transferability of contrastive representations to related subpopulations.
  \newblock In Alice~H. Oh, Alekh Agarwal, Danielle Belgrave, and Kyunghyun Cho, editors, {\em Advances in Neural Information Processing Systems}, 2022.
  
  \bibitem{pmlr-v80-hashimoto18a}
  Tatsunori Hashimoto, Megha Srivastava, Hongseok Namkoong, and Percy Liang.
  \newblock Fairness without demographics in repeated loss minimization.
  \newblock In Jennifer Dy and Andreas Krause, editors, {\em International Conference on Machine Learning}, volume~80 of {\em Proceedings of Machine Learning Research}, pages 1929--1938, Stockholmsmässan, Stockholm Sweden, 10--15 Jul 2018. PMLR.
  
  \bibitem{he2022masked}
  Kaiming He, Xinlei Chen, Saining Xie, Yanghao Li, Piotr Doll{\'a}r, and Ross Girshick.
  \newblock Masked autoencoders are scalable vision learners.
  \newblock In {\em Proceedings of the IEEE/CVF Conference on Computer Vision and Pattern Recognition}, pages 16000--16009, 2022.
  
  \bibitem{He_2016_CVPR}
  Kaiming He, Xiangyu Zhang, Shaoqing Ren, and Jian Sun.
  \newblock Deep residual learning for image recognition.
  \newblock In {\em Proceedings of the IEEE Conference on Computer Vision and Pattern Recognition (CVPR)}, June 2016.
  
  \bibitem{SVM}
  M.A. Hearst, S.T. Dumais, E.~Osuna, J.~Platt, and B.~Scholkopf.
  \newblock Support vector machines.
  \newblock {\em IEEE Intelligent Systems and their Applications}, 13(4):18--28, 1998.
  
  \bibitem{hinton2015distilling}
  Geoffrey~E. Hinton, Oriol Vinyals, and Jeffrey Dean.
  \newblock Distilling the knowledge in a neural network.
  \newblock {\em ArXiv}, abs/1503.02531, 2015.
  
  \bibitem{ho2020denoising}
  Jonathan Ho, Ajay Jain, and Pieter Abbeel.
  \newblock Denoising diffusion probabilistic models.
  \newblock {\em Advances in neural information processing systems}, 33:6840--6851, 2020.
  
  \bibitem{hollmann2023tabpfn}
  Noah Hollmann, Samuel M{\"u}ller, Katharina Eggensperger, and Frank Hutter.
  \newblock Tab{PFN}: A transformer that solves small tabular classification problems in a second.
  \newblock In {\em The Eleventh International Conference on Learning Representations}, 2023.
  
  \bibitem{hollmann2025accurate}
  Noah Hollmann, Samuel M{\"u}ller, Lennart Purucker, Arjun Krishnakumar, Max K{\"o}rfer, Shi~Bin Hoo, Robin~Tibor Schirrmeister, and Frank Hutter.
  \newblock Accurate predictions on small data with a tabular foundation model.
  \newblock {\em Nature}, 637(8045):319--326, 2025.
  
  \bibitem{hovy2015tagging}
  Dirk Hovy and Anders S{\o}gaard.
  \newblock Tagging performance correlates with author age.
  \newblock In {\em Proceedings of the 53rd annual meeting of the Association for Computational Linguistics and the 7th international joint conference on natural language processing (volume 2: Short papers)}, pages 483--488, 2015.
  
  \bibitem{huang2006correcting}
  Jiayuan Huang, Arthur Gretton, Karsten Borgwardt, Bernhard Sch{\"o}lkopf, and Alex Smola.
  \newblock Correcting sample selection bias by unlabeled data.
  \newblock {\em Advances in neural information processing systems}, 19:601--608, 2006.
  
  \bibitem{huang2025survey}
  Lei Huang, Weijiang Yu, Weitao Ma, Weihong Zhong, Zhangyin Feng, Haotian Wang, Qianglong Chen, Weihua Peng, Xiaocheng Feng, Bing Qin, et~al.
  \newblock A survey on hallucination in large language models: Principles, taxonomy, challenges, and open questions.
  \newblock {\em ACM Transactions on Information Systems}, 43(2):1--55, 2025.
  
  \bibitem{pmlr-v119-huang20f}
  Xiao~Shi Huang, Felipe Perez, Jimmy Ba, and Maksims Volkovs.
  \newblock Improving transformer optimization through better initialization.
  \newblock In Hal~Daumé III and Aarti Singh, editors, {\em Proceedings of the 37th International Conference on Machine Learning}, volume 119 of {\em Proceedings of Machine Learning Research}, pages 4475--4483. PMLR, 13--18 Jul 2020.
  
  \bibitem{huber1992robust}
  Peter~J Huber.
  \newblock Robust estimation of a location parameter.
  \newblock In {\em Breakthroughs in statistics}, pages 492--518. Springer, 1992.
  
  \bibitem{huh2016makes}
  Minyoung Huh, Pulkit Agrawal, and Alexei~A Efros.
  \newblock What makes imagenet good for transfer learning?
  \newblock {\em arXiv:1608.08614}, 2016.
  
  \bibitem{pmlr-v235-huh24a}
  Minyoung Huh, Brian Cheung, Tongzhou Wang, and Phillip Isola.
  \newblock Position: The platonic representation hypothesis.
  \newblock In {\em Proc. International Conference on Machine Learning}, Vienna, Austria, July 2024.
  
  \bibitem{JMLR:v6:hyvarinen05a}
  Aapo Hyv{{\"a}}rinen.
  \newblock Estimation of non-normalized statistical models by score matching.
  \newblock {\em Journal of Machine Learning Research}, 6(24):695--709, 2005.
  
  \bibitem{insulla2025towards}
  Francesco Insulla, Shuo Huang, and Lorenzo Rosasco.
  \newblock Towards a learning theory of representation alignment.
  \newblock In {\em The Thirteenth International Conference on Learning Representations}, 2025.
  
  \bibitem{jacobs1991adaptive}
  Robert~A Jacobs, Michael~I Jordan, Steven~J Nowlan, and Geoffrey~E Hinton.
  \newblock Adaptive mixtures of local experts.
  \newblock {\em Neural computation}, 3(1):79--87, 1991.
  
  \bibitem{jacot2018ntk}
  Arthur Jacot, Franck Gabriel, and Clement Hongler.
  \newblock Neural tangent kernel: Convergence and generalization in neural networks.
  \newblock In S.~Bengio, H.~Wallach, H.~Larochelle, K.~Grauman, N.~Cesa-Bianchi, and R.~Garnett, editors, {\em Advances in Neural Information Processing Systems}, volume~31. Curran Associates, Inc., 2018.
  
  \bibitem{jaech2024openai}
  Aaron Jaech, Adam Kalai, Adam Lerer, Adam Richardson, Ahmed El-Kishky, Aiden Low, Alec Helyar, Aleksander Madry, Alex Beutel, Alex Carney, et~al.
  \newblock Openai o1 system card.
  \newblock {\em arXiv preprint arXiv:2412.16720}, 2024.
  
  \bibitem{pmlr-v125-ji20a}
  Ziwei Ji, Miroslav Dud{\'i}k, Robert~E. Schapire, and Matus Telgarsky.
  \newblock Gradient descent follows the regularization path for general losses.
  \newblock In Jacob Abernethy and Shivani Agarwal, editors, {\em Proceedings of Thirty Third Conference on Learning Theory}, volume 125 of {\em Proceedings of Machine Learning Research}, pages 2109--2136. PMLR, 09--12 Jul 2020.
  
  \bibitem{jing2022understanding}
  Li~Jing, Pascal Vincent, Yann LeCun, and Yuandong Tian.
  \newblock Understanding dimensional collapse in contrastive self-supervised learning.
  \newblock In {\em International Conference on Learning Representations}, 2022.
  
  \bibitem{johnson2022contrastive}
  Daniel~D. Johnson, Ayoub~El Hanchi, and Chris~J. Maddison.
  \newblock Contrastive learning can find an optimal basis for approximately view-invariant functions.
  \newblock In {\em The Eleventh International Conference on Learning Representations}, 2023.
  
  \bibitem{jumper2021highly}
  John Jumper, Richard Evans, Alexander Pritzel, Tim Green, Michael Figurnov, Olaf Ronneberger, Kathryn Tunyasuvunakool, Russ Bates, Augustin {\v{Z}}{\'\i}dek, Anna Potapenko, et~al.
  \newblock Highly accurate protein structure prediction with alphafold.
  \newblock {\em nature}, 596(7873):583--589, 2021.
  
  \bibitem{kahneman2011thinking}
  Daniel Kahneman.
  \newblock {\em Thinking, fast and slow}.
  \newblock macmillan, 2011.
  
  \bibitem{kaplan2020scaling}
  Jared Kaplan, Sam McCandlish, Tom Henighan, Tom~B Brown, Benjamin Chess, Rewon Child, Scott Gray, Alec Radford, Jeffrey Wu, and Dario Amodei.
  \newblock Scaling laws for neural language models.
  \newblock {\em arXiv preprint arXiv:2001.08361}, 2020.
  
  \bibitem{khosla2020supervised}
  Prannay Khosla, Piotr Teterwak, Chen Wang, Aaron Sarna, Yonglong Tian, Phillip Isola, Aaron Maschinot, Ce~Liu, and Dilip Krishnan.
  \newblock Supervised contrastive learning.
  \newblock {\em Advances in neural information processing systems}, 33:18661--18673, 2020.
  
  \bibitem{DBLP:journals/corr/KingmaB14}
  Diederik~P. Kingma and Jimmy Ba.
  \newblock Adam: A method for stochastic optimization.
  \newblock In {\em International Conference on Learning Representations}, 2015.
  
  \bibitem{VAE}
  Diederik~P. Kingma and Max Welling.
  \newblock {Auto-Encoding Variational Bayes}.
  \newblock In {\em International Conference on Learning Representations}, 2014.
  
  \bibitem{pmlr-v139-koh21a}
  Pang~Wei Koh, Shiori Sagawa, Henrik Marklund, Sang~Michael Xie, Marvin Zhang, Akshay Balsubramani, Weihua Hu, Michihiro Yasunaga, Richard~Lanas Phillips, Irena Gao, Tony Lee, Etienne David, Ian Stavness, Wei Guo, Berton Earnshaw, Imran Haque, Sara~M Beery, Jure Leskovec, Anshul Kundaje, Emma Pierson, Sergey Levine, Chelsea Finn, and Percy Liang.
  \newblock Wilds: A benchmark of in-the-wild distribution shifts.
  \newblock In Marina Meila and Tong Zhang, editors, {\em Proceedings of the 38th International Conference on Machine Learning}, volume 139 of {\em Proceedings of Machine Learning Research}, pages 5637--5664. PMLR, 18--24 Jul 2021.
  
  \bibitem{pmlr-v97-kornblith19a}
  Simon Kornblith, Mohammad Norouzi, Honglak Lee, and Geoffrey Hinton.
  \newblock Similarity of neural network representations revisited.
  \newblock In Kamalika Chaudhuri and Ruslan Salakhutdinov, editors, {\em Proceedings of the 36th International Conference on Machine Learning}, volume~97 of {\em Proceedings of Machine Learning Research}, pages 3519--3529. PMLR, 09--15 Jun 2019.
  
  \bibitem{kriegeskorte2008representational}
  Nikolaus Kriegeskorte, Marieke Mur, and Peter~A Bandettini.
  \newblock Representational similarity analysis-connecting the branches of systems neuroscience.
  \newblock {\em Frontiers in systems neuroscience}, 2:249, 2008.
  
  \bibitem{krizhevsky2009learning}
  Alex Krizhevsky, Geoffrey Hinton, et~al.
  \newblock Learning multiple layers of features from tiny images, 2009.
  
  \bibitem{krizhevsky2012imagenet}
  Alex Krizhevsky, Ilya Sutskever, and Geoffrey~E Hinton.
  \newblock Imagenet classification with deep convolutional neural networks.
  \newblock {\em Advances in neural information processing systems}, 25, 2012.
  
  \bibitem{laine2017temporal}
  Samuli Laine and Timo Aila.
  \newblock Temporal ensembling for semi-supervised learning.
  \newblock In {\em International Conference on Learning Representations}, 2017.
  
  \bibitem{larson2016we}
  Jeff Larson, Surya Mattu, Lauren Kirchner, and Julia Angwin.
  \newblock How we analyzed the compas recidivism algorithm.
  \newblock {\em ProPublica (5 2016)}, 9(1):3--3, 2016.
  
  \bibitem{lecun1998gradient}
  Yann LeCun, L{\'e}on Bottou, Yoshua Bengio, and Patrick Haffner.
  \newblock Gradient-based learning applied to document recognition.
  \newblock {\em Proceedings of the IEEE}, 86(11):2278--2324, 1998.
  
  \bibitem{lee2021predicting}
  Jason~D Lee, Qi~Lei, Nikunj Saunshi, and Jiacheng Zhuo.
  \newblock Predicting what you already know helps: Provable self-supervised learning.
  \newblock {\em Advances in Neural Information Processing Systems}, 34:309--323, 2021.
  
  \bibitem{li2024chain}
  Zhiyuan Li, Hong Liu, Denny Zhou, and Tengyu Ma.
  \newblock Chain of thought empowers transformers to solve inherently serial problems.
  \newblock In {\em The Twelfth International Conference on Learning Representations}, 2024.
  
  \bibitem{liang2024foundations}
  Paul~Pu Liang.
  \newblock Foundations of multisensory artificial intelligence.
  \newblock {\em arXiv preprint arXiv:2404.18976}, 2024.
  
  \bibitem{liu2022masked}
  Bingbin Liu, Daniel Hsu, Pradeep Ravikumar, and Andrej Risteski.
  \newblock Masked prediction tasks: a parameter identifiability view.
  \newblock {\em Advances in Neural Information Processing Systems}, 2022.
  
  \bibitem{liu2021selfsupervised}
  Hong Liu, Jeff~Z. HaoChen, Adrien Gaidon, and Tengyu Ma.
  \newblock Self-supervised learning is more robust to dataset imbalance.
  \newblock In {\em NeurIPS 2021 Workshop on Distribution Shifts: Connecting Methods and Applications}, 2021.
  
  \bibitem{liu2004improving}
  Qingshan Liu, Hanqing Lu, and Songde Ma.
  \newblock Improving kernel fisher discriminant analysis for face recognition.
  \newblock {\em IEEE transactions on circuits and systems for video technology}, 14(1):42--49, 2004.
  
  \bibitem{Liu2019RoBERTaAR}
  Yinhan Liu, Myle Ott, Naman Goyal, Jingfei Du, Mandar Joshi, Danqi Chen, Omer Levy, Mike Lewis, Luke Zettlemoyer, and Veselin Stoyanov.
  \newblock Roberta: A robustly optimized bert pretraining approach.
  \newblock {\em ArXiv}, abs/1907.11692, 2019.
  
  \bibitem{liu2015deep}
  Ziwei Liu, Ping Luo, Xiaogang Wang, and Xiaoou Tang.
  \newblock Deep learning face attributes in the wild.
  \newblock In {\em Proceedings of the IEEE international conference on computer vision}, pages 3730--3738, 2015.
  
  \bibitem{loshchilov2018decoupled}
  Ilya Loshchilov and Frank Hutter.
  \newblock Decoupled weight decay regularization.
  \newblock In {\em International Conference on Learning Representations}, 2019.
  
  \bibitem{mcelfresh2023neural}
  Duncan McElfresh, Sujay Khandagale, Jonathan Valverde, Vishak Prasad~C, Ganesh Ramakrishnan, Micah Goldblum, and Colin White.
  \newblock When do neural nets outperform boosted trees on tabular data?
  \newblock {\em Advances in Neural Information Processing Systems}, 36, 2023.
  
  \bibitem{mika1999fisher}
  Sebastian Mika, Gunnar Ratsch, Jason Weston, Bernhard Scholkopf, and Klaus-Robert Mullers.
  \newblock Fisher discriminant analysis with kernels.
  \newblock In {\em Neural networks for signal processing IX: Proceedings of the 1999 IEEE signal processing society workshop (cat. no. 98th8468)}, pages 41--48. Ieee, 1999.
  
  \bibitem{miyato2018virtual}
  Takeru Miyato, Shin-ichi Maeda, Masanori Koyama, and Shin Ishii.
  \newblock Virtual adversarial training: a regularization method for supervised and semi-supervised learning.
  \newblock {\em IEEE transactions on pattern analysis and machine intelligence}, 41(8):1979--1993, 2018.
  
  \bibitem{mroueh2015learning}
  Youssef Mroueh, Stephen Voinea, and Tomaso~A Poggio.
  \newblock Learning with group invariant features: A kernel perspective.
  \newblock {\em Advances in neural information processing systems}, 28, 2015.
  
  \bibitem{oquab2023dinov2}
  Maxime Oquab, Timoth{\'e}e Darcet, Th{\'e}o Moutakanni, Huy Vo, Marc Szafraniec, Vasil Khalidov, Pierre Fernandez, Daniel Haziza, Francisco Massa, Alaaeldin El-Nouby, et~al.
  \newblock Dinov2: Learning robust visual features without supervision.
  \newblock {\em arXiv preprint arXiv:2304.07193}, 2023.
  
  \bibitem{ouyang2022training}
  Long Ouyang, Jeffrey Wu, Xu~Jiang, Diogo Almeida, Carroll Wainwright, Pamela Mishkin, Chong Zhang, Sandhini Agarwal, Katarina Slama, Alex Ray, et~al.
  \newblock Training language models to follow instructions with human feedback.
  \newblock In {\em Proc. Advances in Neural Information Processing Systems}, New Orleans, LA, December 2022.
  
  \bibitem{pan2010transfer}
  Sinno~Jialin Pan and Qiang Yang.
  \newblock A survey on transfer learning.
  \newblock {\em IEEE Transactions on Knowledge and Data Engineering}, 22(10):1345--1359, 2010.
  
  \bibitem{papyan2020prevalence}
  Vardan Papyan, XY~Han, and David~L Donoho.
  \newblock Prevalence of neural collapse during the terminal phase of deep learning training.
  \newblock {\em Proceedings of the National Academy of Sciences}, 117(40):24652--24663, 2020.
  
  \bibitem{pfau2018spectral}
  David Pfau, Stig Petersen, Ashish Agarwal, David G.~T. Barrett, and Kimberly~L. Stachenfeld.
  \newblock Spectral inference networks: Unifying deep and spectral learning.
  \newblock In {\em International Conference on Learning Representations}, 2019.
  
  \bibitem{PokleTLR22}
  Ashwini Pokle, Jinjin Tian, Yuchen Li, and Andrej Risteski.
  \newblock Contrasting the landscape of contrastive and non-contrastive learning.
  \newblock In Gustau Camps{-}Valls, Francisco J.~R. Ruiz, and Isabel Valera, editors, {\em International Conference on Artificial Intelligence and Statistics, {AISTATS} 2022, 28-30 March 2022, Virtual Event}, volume 151 of {\em Proceedings of Machine Learning Research}, pages 8592--8618. {PMLR}, 2022.
  
  \bibitem{quionero2009dataset}
  Joaquin Quionero-Candela, Masashi Sugiyama, Anton Schwaighofer, and Neil~D Lawrence.
  \newblock {\em Dataset shift in machine learning}.
  \newblock The MIT Press, 2009.
  
  \bibitem{CLIP}
  Alec Radford, Jong~Wook Kim, Chris Hallacy, Aditya Ramesh, Gabriel Goh, Sandhini Agarwal, Girish Sastry, Amanda Askell, Pamela Mishkin, Jack Clark, et~al.
  \newblock Learning transferable visual models from natural language supervision.
  \newblock In {\em International Conference on Machine Learning}, pages 8748--8763. PMLR, 2021.
  
  \bibitem{radford2019language}
  Alec Radford, Jeffrey Wu, Rewon Child, David Luan, Dario Amodei, and Ilya Sutskever.
  \newblock Language models are unsupervised multitask learners.
  \newblock {\em OpenAI blog}, 2019.
  
  \bibitem{raj2017local}
  Anant Raj, Abhishek Kumar, Youssef Mroueh, Tom Fletcher, and Bernhard Sch{\"o}lkopf.
  \newblock Local group invariant representations via orbit embeddings.
  \newblock In {\em Artificial Intelligence and Statistics}, pages 1225--1235. PMLR, 2017.
  
  \bibitem{richardson1911ix}
  Lewis~Fry Richardson.
  \newblock Ix. the approximate arithmetical solution by finite differences of physical problems involving differential equations, with an application to the stresses in a masonry dam.
  \newblock {\em Philosophical Transactions of the Royal Society of London. Series A, containing papers of a mathematical or physical character}, 210(459-470):307--357, 1911.
  
  \bibitem{ILSVRC15}
  Olga Russakovsky, Jia Deng, Hao Su, Jonathan Krause, Sanjeev Satheesh, Sean Ma, Zhiheng Huang, Andrej Karpathy, Aditya Khosla, Michael Bernstein, Alexander~C. Berg, and Li~Fei-Fei.
  \newblock {ImageNet Large Scale Visual Recognition Challenge}.
  \newblock {\em International Journal of Computer Vision (IJCV)}, 115(3):211--252, 2015.
  
  \bibitem{sagawa2019distributionally}
  Shiori Sagawa, Pang~Wei Koh, Tatsunori~B. Hashimoto, and Percy Liang.
  \newblock Distributionally robust neural networks for group shifts: On the importance of regularization for worst-case generalization.
  \newblock In {\em International Conference on Learning Representations}, 2020.
  
  \bibitem{saunshi2022understanding}
  Nikunj Saunshi, Jordan Ash, Surbhi Goel, Dipendra Misra, Cyril Zhang, Sanjeev Arora, Sham Kakade, and Akshay Krishnamurthy.
  \newblock Understanding contrastive learning requires incorporating inductive biases.
  \newblock In {\em International Conference on Machine Learning}, pages 19250--19286. PMLR, 2022.
  
  \bibitem{saunshi2019theoretical}
  Nikunj Saunshi, Orestis Plevrakis, Sanjeev Arora, Mikhail Khodak, and Hrishikesh Khandeparkar.
  \newblock A theoretical analysis of contrastive unsupervised representation learning.
  \newblock In {\em International Conference on Machine Learning}, pages 5628--5637. PMLR, 2019.
  
  \bibitem{scholkopf2002learning}
  Bernhard Sch{\"o}lkopf and Alexander~J Smola.
  \newblock {\em Learning with kernels: support vector machines, regularization, optimization, and beyond}.
  \newblock MIT press, 2002.
  
  \bibitem{shawe2005eigenspectrum}
  John Shawe-Taylor, Christopher~KI Williams, Nello Cristianini, and Jaz Kandola.
  \newblock On the eigenspectrum of the gram matrix and the generalization error of kernel-pca.
  \newblock {\em IEEE Transactions on Information Theory}, 51(7):2510--2522, 2005.
  
  \bibitem{shen2022connect}
  Kendrick Shen, Robbie~M Jones, Ananya Kumar, Sang~Michael Xie, Jeff~Z HaoChen, Tengyu Ma, and Percy Liang.
  \newblock Connect, not collapse: Explaining contrastive learning for unsupervised domain adaptation.
  \newblock In {\em International Conference on Machine Learning}, pages 19847--19878. PMLR, 2022.
  
  \bibitem{shen2019learning}
  Yanyao Shen and Sujay Sanghavi.
  \newblock Learning with bad training data via iterative trimmed loss minimization.
  \newblock In {\em International Conference on Machine Learning}, pages 5739--5748. PMLR, 2019.
  
  \bibitem{shimodaira2000improving}
  Hidetoshi Shimodaira.
  \newblock Improving predictive inference under covariate shift by weighting the log-likelihood function.
  \newblock {\em Journal of statistical planning and inference}, 90(2):227--244, 2000.
  
  \bibitem{shwartz2023information}
  Ravid Shwartz-Ziv, Randall Balestriero, Kenji Kawaguchi, Tim~GJ Rudner, and Yann LeCun.
  \newblock An information theory perspective on variance-invariance-covariance regularization.
  \newblock {\em Advances in Neural Information Processing Systems}, 36:33965--33998, 2023.
  
  \bibitem{SST2}
  Richard Socher, Alex Perelygin, Jean Wu, Jason Chuang, Christopher~D. Manning, Andrew Ng, and Christopher Potts.
  \newblock Recursive deep models for semantic compositionality over a sentiment treebank.
  \newblock In {\em Proceedings of the 2013 Conference on Empirical Methods in Natural Language Processing}, pages 1631--1642, Seattle, Washington, USA, October 2013. Association for Computational Linguistics.
  
  \bibitem{sohn2020fixmatch}
  Kihyuk Sohn, David Berthelot, Nicholas Carlini, Zizhao Zhang, Han Zhang, Colin~A Raffel, Ekin~Dogus Cubuk, Alexey Kurakin, and Chun-Liang Li.
  \newblock Fixmatch: Simplifying semi-supervised learning with consistency and confidence.
  \newblock {\em Advances in neural information processing systems}, 33:596--608, 2020.
  
  \bibitem{song2020denoising}
  Jiaming Song, Chenlin Meng, and Stefano Ermon.
  \newblock Denoising diffusion implicit models.
  \newblock {\em arXiv preprint arXiv:2010.02502}, 2020.
  
  \bibitem{soudry2018implicit}
  Daniel Soudry, Elad Hoffer, Mor~Shpigel Nacson, Suriya Gunasekar, and Nathan Srebro.
  \newblock The implicit bias of gradient descent on separable data.
  \newblock {\em The Journal of Machine Learning Research}, 19(1):2822--2878, 2018.
  
  \bibitem{ilya2024talk}
  Ilya Sutskever.
  \newblock Test of time award talk: Sequence to sequence learning with neural networks.
  \newblock Advances in Neural Information Processing Systems, 2024.
  
  \bibitem{648ae489-f5d9-3200-b1e9-ab0b264416e5}
  Gábor~J. Székely, Maria~L. Rizzo, and Nail~K. Bakirov.
  \newblock Measuring and testing dependence by correlation of distances.
  \newblock {\em The Annals of Statistics}, 35(6):2769--2794, 2007.
  
  \bibitem{tarvainen2017mean}
  Antti Tarvainen and Harri Valpola.
  \newblock Mean teachers are better role models: Weight-averaged consistency targets improve semi-supervised deep learning results.
  \newblock {\em Advances in neural information processing systems}, 30, 2017.
  
  \bibitem{tatman2017gender}
  Rachael Tatman.
  \newblock Gender and dialect bias in youtube’s automatic captions.
  \newblock In {\em Proceedings of the First ACL Workshop on Ethics in Natural Language Processing}, pages 53--59, 2017.
  
  \bibitem{thrampoulidis2022imbalance}
  Christos Thrampoulidis, Ganesh~Ramachandra Kini, Vala Vakilian, and Tina Behnia.
  \newblock Imbalance trouble: Revisiting neural-collapse geometry.
  \newblock {\em Advances in Neural Information Processing Systems}, 35:27225--27238, 2022.
  
  \bibitem{tian2022deep}
  Yuandong Tian.
  \newblock Deep contrastive learning is provably (almost) principal component analysis.
  \newblock {\em Advances in Neural Information Processing Systems}, 2022.
  
  \bibitem{TianCG21}
  Yuandong Tian, Xinlei Chen, and Surya Ganguli.
  \newblock Understanding self-supervised learning dynamics without contrastive pairs.
  \newblock In Marina Meila and Tong Zhang, editors, {\em Proceedings of the 38th International Conference on Machine Learning, {ICML} 2021, 18-24 July 2021, Virtual Event}, volume 139 of {\em Proceedings of Machine Learning Research}, pages 10268--10278. {PMLR}, 2021.
  
  \bibitem{tibshirani_gd}
  Ryan Tibshirani.
  \newblock Gradient descent.
  \newblock Lecture Notes, 2019.
  \newblock Available at: \url{https://www.stat.cmu.edu/~ryantibs/convexopt/lectures/grad-descent.pdf}.
  
  \bibitem{tosh2021contrastive}
  Christopher Tosh, Akshay Krishnamurthy, and Daniel Hsu.
  \newblock Contrastive estimation reveals topic posterior information to linear models.
  \newblock {\em J. Mach. Learn. Res.}, 22:281--1, 2021.
  
  \bibitem{tosh2021multiview}
  Christopher Tosh, Akshay Krishnamurthy, and Daniel Hsu.
  \newblock Contrastive learning, multi-view redundancy, and linear models.
  \newblock In {\em Algorithmic Learning Theory}, pages 1179--1206. PMLR, 2021.
  
  \bibitem{pmlr-v235-van-breugel24a}
  Boris Van~Breugel and Mihaela Van Der~Schaar.
  \newblock Position: Why tabular foundation models should be a research priority.
  \newblock In Ruslan Salakhutdinov, Zico Kolter, Katherine Heller, Adrian Weller, Nuria Oliver, Jonathan Scarlett, and Felix Berkenkamp, editors, {\em Proceedings of the 41st International Conference on Machine Learning}, volume 235 of {\em Proceedings of Machine Learning Research}, pages 48976--48993. PMLR, 21--27 Jul 2024.
  
  \bibitem{van2020survey}
  Jesper~E Van~Engelen and Holger~H Hoos.
  \newblock A survey on semi-supervised learning.
  \newblock {\em Machine learning}, 109(2):373--440, 2020.
  
  \bibitem{OpenML2013}
  Joaquin Vanschoren, Jan~N. van Rijn, Bernd Bischl, and Luis Torgo.
  \newblock Openml: Networked science in machine learning.
  \newblock {\em SIGKDD Explorations}, 15(2):49--60, 2013.
  
  \bibitem{vershynin2010introduction}
  Roman Vershynin.
  \newblock Introduction to the non-asymptotic analysis of random matrices.
  \newblock {\em arXiv preprint arXiv:1011.3027}, 2010.
  
  \bibitem{Wainwright_2019}
  Martin~J. Wainwright.
  \newblock {\em High-Dimensional Statistics: A Non-Asymptotic Viewpoint}.
  \newblock Cambridge Series in Statistical and Probabilistic Mathematics. Cambridge University Press, 2019.
  
  \bibitem{wang2018glue}
  Alex Wang, Amanpreet Singh, Julian Michael, Felix Hill, Omer Levy, and Samuel~R. Bowman.
  \newblock Glue: A multi-task benchmark and analysis platform for natural language understanding.
  \newblock {\em BLACKBOXNLP@EMNLP}, 2018.
  
  \bibitem{wang2022is}
  Ke~Alexander Wang, Niladri~Shekhar Chatterji, Saminul Haque, and Tatsunori Hashimoto.
  \newblock Is importance weighting incompatible with interpolating classifiers?
  \newblock In {\em International Conference on Learning Representations}, 2022.
  
  \bibitem{wang2018deep}
  Mei Wang and Weihong Deng.
  \newblock Deep visual domain adaptation: A survey.
  \newblock {\em Neurocomputing}, 312:135--153, 2018.
  
  \bibitem{wang2023a}
  Yifei Wang, Qi~Zhang, Tianqi Du, Jiansheng Yang, Zhouchen Lin, and Yisen Wang.
  \newblock A message passing perspective on learning dynamics of contrastive learning.
  \newblock In {\em The Eleventh International Conference on Learning Representations}, 2023.
  
  \bibitem{wei2021pretrained}
  Colin Wei, Sang~Michael Xie, and Tengyu Ma.
  \newblock Why do pretrained language models help in downstream tasks? an analysis of head and prompt tuning.
  \newblock {\em Advances in Neural Information Processing Systems}, 34:16158--16170, 2021.
  
  \bibitem{wei2022emergent}
  Jason Wei, Yi~Tay, Rishi Bommasani, Colin Raffel, Barret Zoph, Sebastian Borgeaud, Dani Yogatama, Maarten Bosma, Denny Zhou, Donald Metzler, Ed~H. Chi, Tatsunori Hashimoto, Oriol Vinyals, Percy Liang, Jeff Dean, and William Fedus.
  \newblock Emergent abilities of large language models.
  \newblock {\em Transactions on Machine Learning Research}, 2022.
  \newblock Survey Certification.
  
  \bibitem{wei2022chain}
  Jason Wei, Xuezhi Wang, Dale Schuurmans, Maarten Bosma, Fei Xia, Ed~Chi, Quoc~V Le, Denny Zhou, et~al.
  \newblock Chain-of-thought prompting elicits reasoning in large language models.
  \newblock {\em Advances in neural information processing systems}, 35:24824--24837, 2022.
  
  \bibitem{WenLi21}
  Zixin Wen and Yuanzhi Li.
  \newblock Toward understanding the feature learning process of self-supervised contrastive learning.
  \newblock In Marina Meila and Tong Zhang, editors, {\em Proceedings of the 38th International Conference on Machine Learning, {ICML} 2021, 18-24 July 2021, Virtual Event}, volume 139 of {\em Proceedings of Machine Learning Research}, pages 11112--11122. {PMLR}, 2021.
  
  \bibitem{wen2022mechanism}
  Zixin Wen and Yuanzhi Li.
  \newblock The mechanism of prediction head in non-contrastive self-supervised learning.
  \newblock {\em Advances in Neural Information Processing Systems}, 2022.
  
  \bibitem{wettig2022should}
  Alexander Wettig, Tianyu Gao, Zexuan Zhong, and Danqi Chen.
  \newblock Should you mask 15\% in masked language modeling?
  \newblock In {\em Proceedings of the 17th Conference of the European Chapter of the Association for Computational Linguistics: Main Volume}, 2023.
  
  \bibitem{xiang2023denoising}
  Weilai Xiang, Hongyu Yang, Di~Huang, and Yunhong Wang.
  \newblock Denoising diffusion autoencoders are unified self-supervised learners.
  \newblock In {\em Proceedings of the IEEE/CVF International Conference on Computer Vision}, pages 15802--15812, 2023.
  
  \bibitem{xie2020self}
  Qizhe Xie, Minh-Thang Luong, Eduard Hovy, and Quoc~V Le.
  \newblock Self-training with noisy student improves imagenet classification.
  \newblock In {\em Proceedings of the IEEE/CVF conference on computer vision and pattern recognition}, pages 10687--10698, 2020.
  
  \bibitem{xiong2020layer}
  Ruibin Xiong, Yunchang Yang, Di~He, Kai Zheng, Shuxin Zheng, Chen Xing, Huishuai Zhang, Yanyan Lan, Liwei Wang, and Tieyan Liu.
  \newblock On layer normalization in the transformer architecture.
  \newblock In {\em International Conference on Machine Learning}, pages 10524--10533. PMLR, 2020.
  
  \bibitem{xu2021understanding}
  Da~Xu, Yuting Ye, and Chuanwei Ruan.
  \newblock Understanding the role of importance weighting for deep learning.
  \newblock In {\em International Conference on Learning Representations}, 2021.
  
  \bibitem{zbontar2021barlow}
  Jure Zbontar, Li~Jing, Ishan Misra, Yann LeCun, and St{\'e}phane Deny.
  \newblock Barlow twins: Self-supervised learning via redundancy reduction.
  \newblock In {\em International Conference on Machine Learning}, pages 12310--12320. PMLR, 2021.
  
  \bibitem{zhai2023grw}
  Runtian Zhai, Chen Dan, J~Zico Kolter, and Pradeep~Kumar Ravikumar.
  \newblock Understanding why generalized reweighting does not improve over {ERM}.
  \newblock In {\em The Eleventh International Conference on Learning Representations}, 2023.
  
  \bibitem{pmlr-v139-zhai21a}
  Runtian Zhai, Chen Dan, Zico Kolter, and Pradeep Ravikumar.
  \newblock Doro: Distributional and outlier robust optimization.
  \newblock In Marina Meila and Tong Zhang, editors, {\em Proceedings of the 38th International Conference on Machine Learning}, volume 139 of {\em Proceedings of Machine Learning Research}, pages 12345--12355. PMLR, 18--24 Jul 2021.
  
  \bibitem{zhai2021}
  Runtian Zhai, Chen Dan, Arun Suggala, J~Zico Kolter, and Pradeep~Kumar Ravikumar.
  \newblock Boosted {CV}ar classification.
  \newblock In {\em Thirty-Fifth Conference on Neural Information Processing Systems}, 2021.
  
  \bibitem{zhai2023understanding}
  Runtian Zhai, Bingbin Liu, Andrej Risteski, Zico Kolter, and Pradeep Ravikumar.
  \newblock Understanding augmentation-based self-supervised representation learning via rkhs approximation and regression.
  \newblock In {\em International Conference on Learning Representations}, 2024.
  
  \bibitem{zhai2024spectrally}
  Runtian Zhai, Rattana Pukdee, Roger Jin, Maria~Florina Balcan, and Pradeep~Kumar Ravikumar.
  \newblock Spectrally transformed kernel regression.
  \newblock In {\em The Twelfth International Conference on Learning Representations}, 2024.
  
  \bibitem{zhang2018mixup}
  Hongyi Zhang, Moustapha Cisse, Yann~N. Dauphin, and David Lopez-Paz.
  \newblock mixup: Beyond empirical risk minimization.
  \newblock In {\em International Conference on Learning Representations}, 2018.
  
  \bibitem{zhou2003learning}
  Dengyong Zhou, Olivier Bousquet, Thomas Lal, Jason Weston, and Bernhard Sch{\"o}lkopf.
  \newblock Learning with local and global consistency.
  \newblock {\em Advances in neural information processing systems}, 16, 2003.
  
  \bibitem{Zhu2002LearningFL}
  Xiaojin Zhu and Zoubin Ghahramani.
  \newblock Learning from labeled and unlabeled data with label propagation.
  \newblock In {\em CMU CALD tech report CMU-CALD-02-107}, 2002.
  
  \end{thebibliography}

}

\end{document}